\documentclass[Afour,sageh,times]{sagej}

\usepackage{moreverb,url}

\usepackage[colorlinks,bookmarksopen,bookmarksnumbered,citecolor=red,urlcolor=red]{hyperref}

\newcommand\BibTeX{{\rmfamily B\kern-.05em \textsc{i\kern-.025em b}\kern-.08em
T\kern-.1667em\lower.7ex\hbox{E}\kern-.125emX}}

\usepackage{multicol}
\usepackage{verbatim}
\usepackage{amsmath}
\usepackage{amssymb}
\usepackage{amsthm}
\usepackage{bm}
\usepackage{graphicx}
\usepackage{float}

\usepackage[mathscr]{euscript}
\usepackage{float}
\usepackage[version=4]{mhchem}
\usepackage{siunitx}
\usepackage{longtable,tabularx}
\usepackage{caption}
\usepackage{subcaption}
\usepackage[utf8]{inputenc}
\usepackage{tikz}

\usepackage{booktabs}
\usepackage{multirow}

\usepackage{setspace}
\usepackage[normalem]{ulem}
\usepackage{array}

\usepackage{thmtools}
\usepackage{thm-restate}

\usepackage{algorithm}
\usepackage[noend]{algpseudocode}
\makeatletter
\def\BState{\State\hskip-\ALG@thistlm}
\makeatother
\algnewcommand{\Or}{\textbf{or}\,}
\algnewcommand\algorithmicswitch{\textbf{switch}}
\algnewcommand\algorithmiccase{\textbf{case}}
\algdef{SE}[SWITCH]{Switch}{EndSwitch}[1]{\algorithmicswitch\ #1\ \algorithmicdo}{\algorithmicend\ \algorithmicswitch}%
\algdef{SE}[CASE]{Case}{EndCase}[1]{\algorithmiccase\ #1}{\algorithmicend\ \algorithmiccase}%
\algdef{SE}[ELSE]{Else}{EndElse}[1]{\algorithmicelse\ #1}{\algorithmicend\ \algorithmicelse}%
\algtext*{EndSwitch}%
\algtext*{EndCase}%
\algtext*{EndElse}%

\def\myMSFigureScale{0.22}

\def\myLineScale{1}

\def\myMSFigureScalend{0.22}

\newcommand{\tikzcircle}[2][red,fill=red]{\tikz[baseline=-0.5ex]\draw[#1,radius=#2] (0,0) circle ;}
\newcommand{\tikzsquare}[2][red,fill=red]{\tikz\draw[#1] (0,0) rectangle (#2, #2) ;}
\newcommand{\tikzline}[2][red,thick]{\tikz\draw[#1,#2] (0,0) -- (0.5em, 0.5em) ;}

\newcommand{\N}{\mathbb{N}}
\newcommand{\R}{\mathbb{R}}

\newcommand{\set}[1]{\{#1\}}

\newcommand{\eps}{\varepsilon}

\newcommand{\argmin}{\mathop{\mathrm{argmin}}}

\makeatletter
\newcommand*{\defeq}{\mathrel{\rlap{\raisebox{0.3ex}{$\m@th\cdot$}}\raisebox{-0.3ex}{$\m@th\cdot$}}=}
\makeatother

\begin{document}

\runninghead{Lim et al.}

\title{Lazy Incremental Search for Efficient Replanning with Bounded Suboptimality Guarantees}

\author{Jaein Lim\affilnum{1}, Mahdi Ghanei\affilnum{1}, R. Connor Lawson\affilnum{1}, Siddhartha Srinivasa\affilnum{2} and Panagiotis Tsiotras\affilnum{1}}

\affiliation{\affilnum{1}Georgia Institute of Technology, Atlanta, Georgia\\
\affilnum{2}University of Washington, Seattle, WA}

\corrauth{Jaein Lim, Georgia Institute of Technology, Atlanta, Georgia 30332.}

\email{jaeinlim126@gatech.edu}

\begin{abstract}
We present a lazy incremental search algorithm, Lifelong-GLS (L-GLS), along with its bounded suboptimal version, Bounded L-GLS (B-LGLS) that combine the search efficiency of incremental search algorithms with the evaluation efficiency of lazy search algorithms for fast replanning in  problem domains where edge-evaluations are more expensive than vertex-expansions.
The proposed algorithms generalize Lifelong Planning A* (LPA*) and its bounded suboptimal version, Truncated LPA* (TLPA*), within the Generalized Lazy Search (GLS) framework, so as to restrict expensive edge evaluations only to the current shortest subpath when the cost-to-come inconsistencies are propagated during repair.
We also present dynamic versions of the L-GLS and B-LGLS algorithms, called Generalized D* (GD*) and Bounded Generalized D* (B-GD*), respectively, for efficient replanning with non-stationary queries, designed specifically for navigation of mobile robots.  
We prove that the proposed algorithms are complete and correct in finding a solution that is guaranteed not to exceed the optimal solution cost by a user-chosen factor.
Our numerical and experimental results support the claim that the proposed integration of the incremental and lazy search frameworks can help find solutions faster compared to the regular incremental or regular lazy search algorithms when the underlying graph representation changes often. 
\end{abstract}


\maketitle

\section{Introduction}
Replanning is essential to every decision-making agent operating in a complex or dynamic environment.
Agents in such environments often operate with limited computational resources and with partial information, making it prohibitive to construct an accurate model of a complex environment once and for all. 
Even if constructing an accurate representation of the environment were feasible, it may not be prudent to do so, as the model itself becomes out-of-date in a dynamic environment, and hence any initial plan quickly becomes obsolete or irrelevant.
The ability to replan fast in order to adapt to environment changes is crucial for robust and responsive autonomy.

In this paper, we (re)consider path-planning problems on graph representations of a complex and dynamic environment suitable for efficient replanning.
To better understand the issues involved, consider the case of
%
%
 a mobile robot navigating through a partially known environment with limited sensing. 
To produce a motion plan all the way to a distant location, the robot must make assumptions about the feasibility or cost of out-of-sight, or yet unexplored, edges on the underlying search graph, some of which will likely be incorrect.
As the robot executes its initial plan and perceives new information about the environment, it may need to update the plan to avoid unknown or dynamic obstacles and leverage discovered shortcuts~\cite{Stentz1995, Koenig2005}.

Replanning also appears in settings where the world may be fully known but is too complex to be modeled accurately.
In this case, an initial solution on a coarse model may be improved via replanning on refined models, as the on-board resources allow.
Sampling-based motion planning algorithms offer a case in point.
These algorithms generate an increasingly dense series of graphs by sampling the free configuration space in an anytime fashion. 
Sampling-based planning becomes especially useful in high-dimensional spaces where a global graph representation is neither sufficient nor tractable in order to find a good quality path~\cite{Lavalle1998, Karaman2011}.
The rate of convergence toward the optimal solution is determined by the choice of the replanning strategy used to efficiently update the old plan upon refinement by sampling, and by employing an efficient replanning strategy.
These have been shown to indeed improve the convergence rate in a variety of problems~\cite{Arslan2013, Gammell2015, Strub2020a, Strub2020b}. 
The key idea to efficient replanning is to restrict the replanning routine to only the relevant paths that could possibly improve the current solution upon subsequent graph densifications. 
As a result, the exploitation (replanning) phase can be shortened, leaving more resources available for the exploration (sampling) phase. 

The necessity of replanning is not limited to improving solutions over time; replanning can also help find a solution of a complex problem rather quickly by solving the problem sequentially in relaxed settings. 
For instance, the multi-agent path finding problem (MAPF)~\cite{Stern2019} deals with finding collision-free paths of multiple agents on a graph.
This problem has intractable complexity~\cite{Yu2013}, as the size of the joint planning space grows exponentially with the number of agents.
There exist complete and optimal MAPF solvers \cite{Sharon2015, Boyarski2015, Felner2018, Li2020} which relax the search complexity of the joint planning space by decomposing the problem into multiple single-agent planning problems with incremental inter-agent collision avoiding constraints.
%
%
This decomposition reduces the exponential complexity of the search at the expense of many replannings of the individual agents, as they repair their old paths to satisfy additionally discovered inter-agent constraints online.
Efficient replanning of individual planning problems, thus, helps find the overall solution faster~\cite{Boyarski2020}.
%
Fortunately, in these settings the change in the environment is often relatively small, such that a large part of the previous plan can be reused.
Incremental search algorithms that reuse existing search results can facilitate the update of the current plan, resulting in reduced computations compared to searching from scratch~\cite{Ramalingam1996, Koenig2004, Koenig2005}.
These incremental methods minimally propagate the cost inconsistencies induced by changes in the graph to repair the search tree to become consistent again.
The efficiency of this minimal inconsistency propagation has been widely manifested in many classical robotics applications~\cite{Arslan2013, Gammell2015, Boyarski2020}.

Unfortunately, existing incremental search algorithms~\cite{Koenig2004, Koenig2005, Aine2016} are designed specifically to reduce the number of vertex expansions, and they are agnostic to the number of edge evaluations. 
These methods often evaluate edges excessively to find the new optimal (or bounded suboptimal) solution, incurring a significant computational overhead in problem domains where edge evaluations are expensive. Indeed, in many replanning problems (e.g., navigating in a partially known environment, nonholomic dynamics, sampling-based planning, multi-agent path finding problem), evaluating an edge can be more expensive than expanding a vertex. 
An edge evaluation typically involves multiple collision checks in the configuration space~\cite{Kavraki1996,Lavalle1998}, solving a two-point boundary value problem~\cite{Karaman2010,Webb2013}, propagating the system dynamics with a closed-loop controller~\cite{Kuwata2009}, or checking all-to-all inter-agent collisions~\cite{Shome2020}. 
In this work, we seek to remedy such excessive edge evaluations of prior incremental search algorithms by borrowing ideas from the lazy search framework of~\cite{Bohlin2000, Cohen2014, Hauser2015, Dellin2016, Mandalika2018, Mandalika2019} in problem domains where edge evaluations are more expensive than vertex expansions.

\subsection*{Contributions}
In this paper we present a class of incremental search algorithms that improve upon existing incremental search algorithms such as Lifelong Planning A* (LPA*)~\cite{Koenig2004} and D*-Lite~\cite{Koenig2005} and their bounded suboptimal variants, namely, Truncated LPA* (TLPA*) and Truncated D*-Lite (TD*)~\cite{Aine2016}. 
We address their common drawback, often overlooked in classical planning, namely, excessive edge evaluations, and we propose to generalize the current incremental search algorithms using a lazy search framework to mitigate unnecessary edge evaluations.

Our proposed generalization saves a significant amount of computation that is wasted when evaluating irrelevant edges. 
As a result, the proposed algorithms find the solution much faster compared to classical incremental search algorithms, especially in problem domains where edge evaluations are expensive.

We provide theoretical results that guarantee the completeness and correctness of the proposed algorithms.
Specifically, we show that the returned solution is bounded above by a user-chosen multiplicative factor of the optimal solution, given the graph.
Our numerical experiments support our claim that generalizing incremental search algorithms within the lazy search framework indeed result in much faster replanning. 

This paper extends our previous work~\cite{Lim2021}, where we proposed L-GLS, an algorithm
that combines the vertex efficiency of LPA*~\cite{Koenig2004} with the edge efficiency of GLS~\cite{Bohlin2000, Mandalika2018, Mandalika2019} for more efficient replanning in problem domains where edge-evaluations are expensive, by adding several implementation details, and 
by relaxing the optimality constraint so as to find a bounded suboptimal solution more quickly based on the two relaxation techniques, namely, truncation of inconsistency propagation during repair~\cite{Aine2016} and by adding an inflation heuristic edge estimate during search~\cite{Pohl1970}. 
We introduce three additional algorithms (B-LGLS, GD*, B-GD*) and show that 
by relaxing the optimality constraint of previous lazy incremental search algorithms we can further save computational resources, especially when the changes are frequent, yet rarely significant. 
Extensive experimental results are also included  to demonstrate the efficacy of our proposed 
methods for replanning applications in dynamic environments. 
%

%
%

The paper is organized as follows: in the next section we provide a taxonomy and a comprehensive literature review of the relevant algorithms, namely, those that utilize either lazy search, bounded suboptimal search, or incremental search, highlighting the connections with the proposed methods. 
Then, we formally introduce the problem formulation and the notation used throughout the paper and
we present an optimal lazy incremental search algorithm, namely, Lifelong-GLS (L-GLS).
%
Afterwards we present its bounded suboptimal version, namely, Bounded L-GLS (B-LGLS), which ensures that a replanned solution does not exceed the optimal solution by a given factor.  
The next two sections
present  dynamic versions of the lazy incremental search algorithms L-GLS and B-LGLS, called Generalized D* (GD*) and Bounded Generalized D* (B-GD*), respectively, for non-stationary planning queries in a partially known environment.
We  provide extensive experimental results comparing the proposed algorithms 
for solving replanning problems in dynamic environments using sampling-based algorithms and navigation problems in dynamic graphs.
Finally, we conclude the paper
with a summary of our results and their impact on real-world replanning applications.

\section{Related Work} \label{sec:related_work}

Many prior works have studied incremental graph search and lazy motion planning, but few integrate these approaches into a single algorithm.
In this section we offer a comprehensive literature review of these works and situate our own work within the existing taxonomy of lazy and incremental search algorithms.

We classify planners according to their utilization in terms of three important properties:

\begin{enumerate}
    \item \textit{Lazy search}: The planner uses an admissible edge-weight heuristic to reduce the number of relatively expensive edge evaluations. 
    Lazy planners vary primarily in the manner they select edges for full evaluation.
    
    \item \textit{Bounded suboptimality}: The planner returns a solution with cost within a specified multiplicative bound of the optimal cost. 
    Given a bounded suboptimal planner, an anytime asymptotically optimal planner may be constructed by repeatedly solving the same query, while decreasing the suboptimality bound.
    
    \item \textit{Lifelong planning}: The planner reuses prior search results to accelerate subsequent planning queries. 
    For sequences of similar queries, lifelong planning significantly reduces the total cost of solving all queries compared to replanning from scratch. 
    We restrict the definition of lifelong planning to include only planners that can handle arbitrary changes to the underlying graph.
\end{enumerate}

\begin{figure*}
    \centering
    \includegraphics[width=\textwidth]{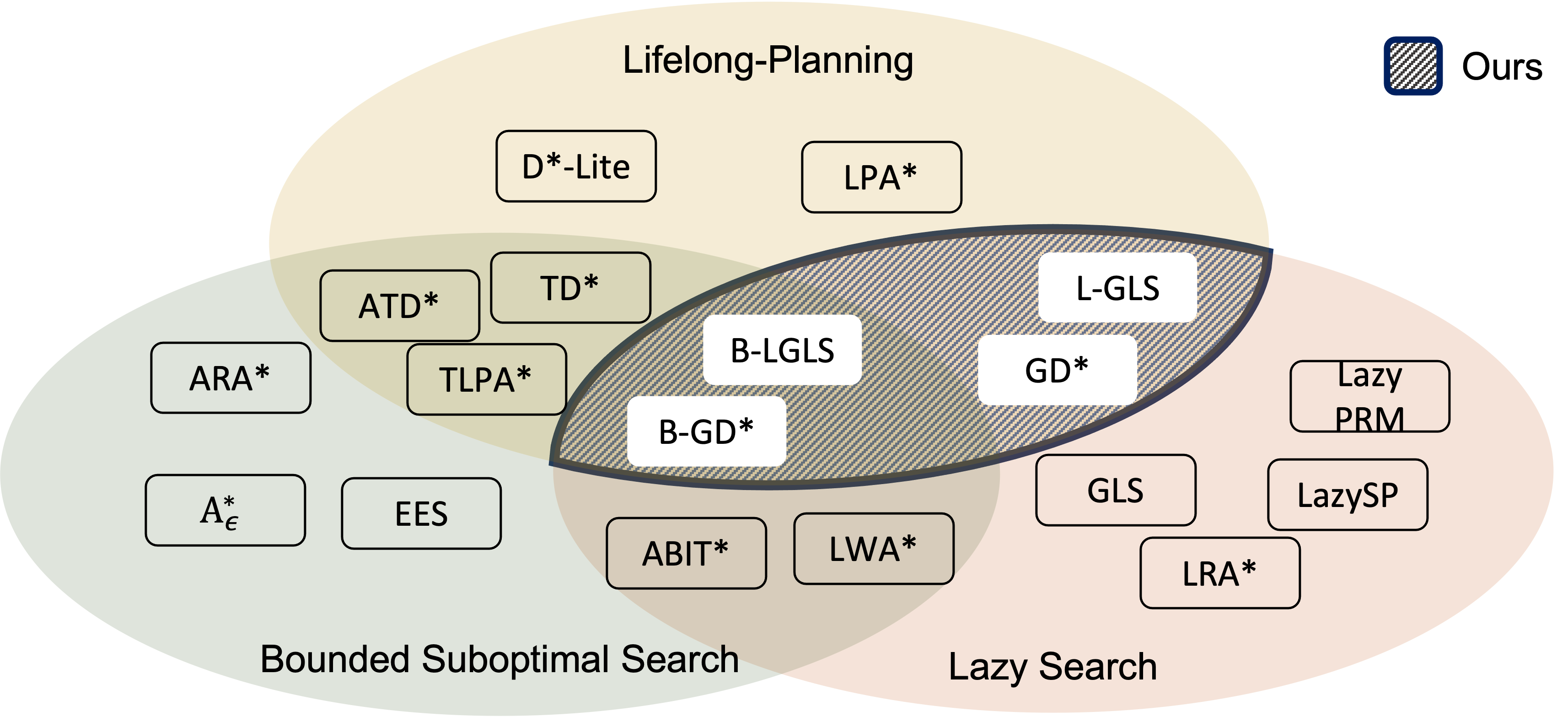}
    \caption{Taxonomy of relevant motion planning algorithms.}
    \label{fig:taxonomy}
\end{figure*}

Figure~\ref{fig:taxonomy} illustrates some of the most popular current planners according to the above classification.
The following subsections give a brief summary each class of planners, describing the salient features of each algorithm.

\subsection{Lazy Search}

The earliest and most widely studied category of planners are based lazy search. 
Lazy edge evaluation as a form of heuristic for faster motion planning stems from the observation that in many practical settings, the time spent on edge evaluation dominates the rest of the computation time~\cite{Hauser2015}.
Thus, reducing the number of edge evaluations presents a natural opportunity to speed up planning.
When inexpensive admissible heuristics that underestimate the cost of an edge are available, planning may proceed for a while using only these heuristics to guide the search, deferring evaluation until it is necessary to ensure correctness.

Lazy Probabilistic Roadmap (LazyPRM) \cite{Bohlin2000} was the first algorithm to introduce laziness, initializing all edges by a heuristic instead of performing actual edge evaluations.
Upon completing the graph search, the edges on the candidate solution path are evaluated, after which the graph is updated with their true edge costs, and the search begins anew.
The first solution containing only already-evaluated edges is the optimal one.

The Lazy Shortest Path (LazySP) \cite{Haghtalab2018} framework generalizes LazyPRM by explicitly introducing an \textit{edge selector} function.
Upon finding a new candidate solution, the user-specified edge selector function can choose any combination of edges to evaluate, including those not belonging to the candidate path.
The authors in \cite{Haghtalab2018} demonstrate this approach by incorporating prior information about the environment into the edge selector.
Evaluating edges on a candidate path in order, from start to goal, as in LazyPRM, is called \textit{forward selection}.

Both LazyPRM and LazySP algorithms search for a complete candidate path to the goal before any edge evaluation is performed, a behavior of \textit{infinite lookahead}.
Infinite lookahead is proven to minimize edge evaluations before an optimal path is found, but it requires a high search effort to generate many complete candidate paths.
Lazy Receding-Horizon A* (LRA*)~\cite{Mandalika2018} introduces the more general notion of \textit{$n$-step lookahead}, in which the search is limited to $n$ edges beyond the current evaluated frontier, after which a cost-to-go heuristic is used to estimate the remainder of the path cost.
The parameter $n$ allows the user to tune the computational effort spent on lazy search.

Generalized Lazy Search (GLS)~\cite{Mandalika2019} further generalizes LazySP by introducing user-specified \textit{evaluation events}.
In GLS the search continues until the conditions of a specified event are met, at which point the search immediately stops and edge evaluation begins.
Infinite and $n$-step lookahead are examples of evaluation events.
The user may also specify an arbitrary edge selector, as in LazySP.
The theoretical properties of GLS do not explicitly depend on the underlying search algorithm.
The new algorithms presented in this work are instances of the GLS framework, specifying particular underlying search algorithms and evaluation events that preserve correctness, while incorporating additional desirable properties.

Advanced Informed Trees (AIT$^*$)~\cite{Strub2020a} is another algorithm that uses lazy edge evaluations to build informed heuristics via a backward search rooted at the goal vertex.
%
%
While AIT$^*$ is comparable in spirit and performance to the other lazy algorithms presented here, it notably lacks in flexibility.
A heuristic built via lazy backward search must necessarily use infinite lookahead, since it must connect to the forward search tree to be usable at all.
The GLS framework offers more flexibility balancing computational effort between search and evaluation, as well as offering the potential to incorporate prior knowledge into both search and heuristic estimation.
Therefore, in this work we limit our attention to GLS-based algorithms.

\subsection{Bounded Suboptimality}

Strictly optimal motion planning is computationally difficult and may be unnecessary in many problem instances.
Algorithms with bounded suboptimality exploit this fact to achieve dramatic performance improvements by returning suboptimal solutions with cost within a constant factor of the optimal one.
Note that, unlike lazy search, the heuristic in these algorithms estimates the cost-to-go rather than the cost of traversing an edge.
%

Weighted A$^*$ ($w$A$^*$)~\cite{Pohl1970} inflates the given heuristic by a constant factor to bias the search.
Increasing the inflation factor prioritizes expansion of vertices close to the goal, since the heuristic comprises a smaller proportion of their estimated solution cost.
This modification lends the search a greedy character, leading to faster discovery of an initial solution.
The returned solution suboptimality factor equals the heuristic inflation factor.
Anytime Repairing A$^*$ (ARA$^*$)~\cite{Likhachev2003} tracks and propagates only the inconsistencies introduced by changing the suboptimality bound in Weighted A$^*$, resulting in a more efficient anytime implementation.
Note that ARA* is an anytime algorithm but not a lifelong-planning algorithm as it does not handle general graph changes.

\newcommand{\Aeps}{A$^*_\eps\,$}
Focal search (\Aeps)~\cite{Pearl1982} incorporates directly the use of an inadmissible heuristic.
Without modification, A$^*$ always uses an admissible heuristic to decide vertex expansion, increasing the lower bound on possible solutions at every iteration until an optimal solution is found.
\Aeps, on the other hand, applies a suboptimality factor to the current A$^*$ lower bound, producing a range of solution costs which it considers \textit{in focus} and a corresponding \textit{focal set} of in-focus vertices.
The inadmissible heuristic can be used to order the expansion of the focal set without violating the suboptimality bound.
When the focal set is empty, \Aeps expands the next vertex by an admissible heuristic, raising the lower bound and replenishing the focal set.
Since inadmissible heuristics can be much better estimators of path cost compared to admissible heuristics, \Aeps can find a solution with far fewer expansions than A$^*$.

Several extensions based on focal search have been explored. \citet{Cohen2018} introduce the Anytime Focal Search (AFS) framework that addresses considerations for transforming focal search into an efficient anytime algorithm, concluding that the use of bounded-cost rather than bounded-suboptimal subsearches yields a more efficient algorithm overall.
%
Explicit Estimation Search (EES)~\cite{Thayer2011} builds on focal search, using an additional heuristic to take the potential cost increment into consideration to remove the negative correlation between OPEN and FOCAL. 
Multi-Heuristic A$^*$ (MHA$^*$)~\cite{Aine2016a} provides a framework to combine multiple arbitrarily inadmissible heuristics with a single admissible heuristic for bounded suboptimal search and a finite number of expansions of each vertex.
It should be noted that
these algorithms solve a single-query problem, and they do not handle general graph changes.

\subsection{Lifelong Planning}

Incremental search algorithms solve a sequence of similar problems efficiently by reusing the previous search results to facilitate finding a new plan quickly~\cite{Ramalingam1996, Koenig2004, Aine2016}. Instead of building a new search tree from scratch, these methods identify the portion of the previous search tree that is inconsistent with the graph changes. Expanding only the inconsistent vertices produces a new optimal solution and a tree consistent with the changed graph. These methods can often expand significantly fewer vertices compared to searching from scratch after each graph change.

The LPA* algorithm~\cite{Koenig2004} utilizes a consistent heuristic to restrict the repair of the search tree to only the relevant part for the current problem, making the search tree consistent with the relevant graph changes. 
LPA* chooses the inconsistent vertices from a priority queue similar to the one used in A* search~\cite{Hart1968}, such that only the optimal path candidates are chosen to be repaired, 
in a best-first search fashion. 
LPA* stores, for each vertex, two distinct cost-to-come values to identify cost inconsistencies, and finds the new optimal solution by propagating the cost-to-come inconsistency upon graph changes. 
LPA* is provably optimal and efficient, in the sense that no vertex is expanded more than twice; also, it does not make any limiting assumptions about the structure of the underlying graph. 
It can work with any type of graphs as well as any type of changes (e.g., edge addition/removal, weight increase/decrease). 
These theoretical properties have made the LPA* algorithm 
(and, in particular, its dynamic version, D*-Lite)
the backbone to numerous applications where efficient replanning is imperative~\cite{Koenig2005, Arslan2013, Gammell2015, Koenig2002}. 

Recently, the efficiency of LPA* has been further improved at the expense of optimality, by truncating the inconsistency propagation as early as the current solution is guaranteed to be bounded suboptimal~\cite{Aine2016}. 
Finding the optimal solution exactly is usually a computationally expensive process, especially when there exists many good optimal path candidates~\cite{Pearl1982}. 
If finding a good enough solution instead of the exact solution is satisfactory, then a significant amount of computations spent to find the best among a set of good solutions can be eliminated.
Truncated Lifelong Planning A* (TLPA*)~\cite{Aine2016} achieves this by stopping the repair procedure of LPA* as soon as the current solution is guaranteed not to exceed the best possible solution in the current graph more than a given factor. 
Compared to LPA*, which uses a binary notion of change to propagate all changes to find the optimal solution regardless of their impact, TLPA* only expands the vertices with significant changes. 
Hence, TLPA* can expand much fewer vertices compared to LPA*, as it restricts repair to both the relevant and the significantly changed part of the tree.

Unfortunately, both of the LPA* and TLPA* algorithms are vertex-optimal, but they are indifferent to the number of edge evaluations. They often incur a significant amount of unnecessary edge evaluations, slowing down the replanning process. 
In this paper we address this issue by generalizing these algorithms within the lazy search framework.
Before delving into this topic further, let us first discuss the two properties of LPA*/TLPA* that cause unnecessary edge evaluations.

\begin{figure*}[ht]
	\centering
	\begin{subfigure}{\myMSFigureScale\textwidth}
		\includegraphics[width=\myLineScale\linewidth]{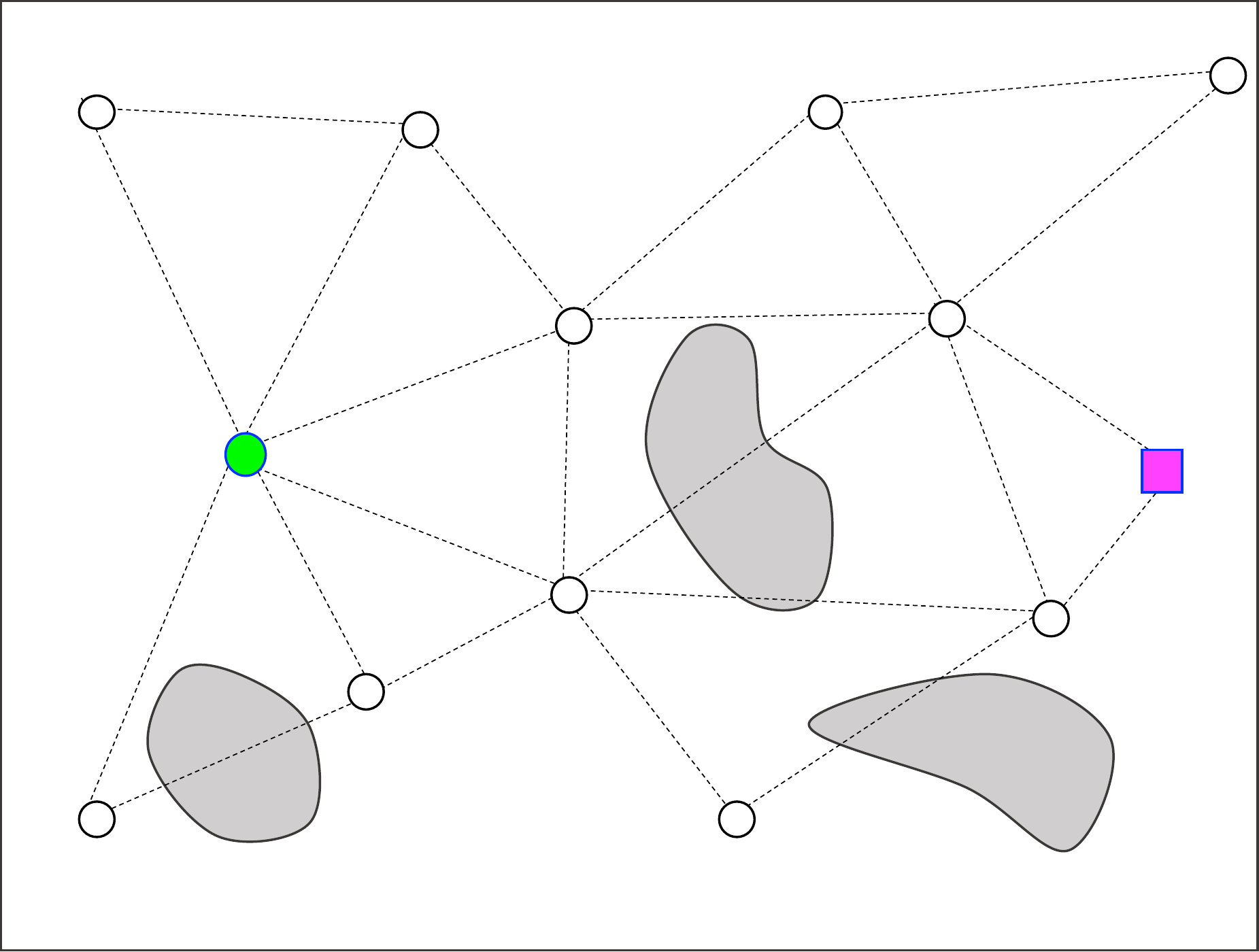}
		\caption{}
	\end{subfigure}
	\begin{subfigure}{\myMSFigureScale\textwidth}
		\includegraphics[width=\myLineScale\linewidth]{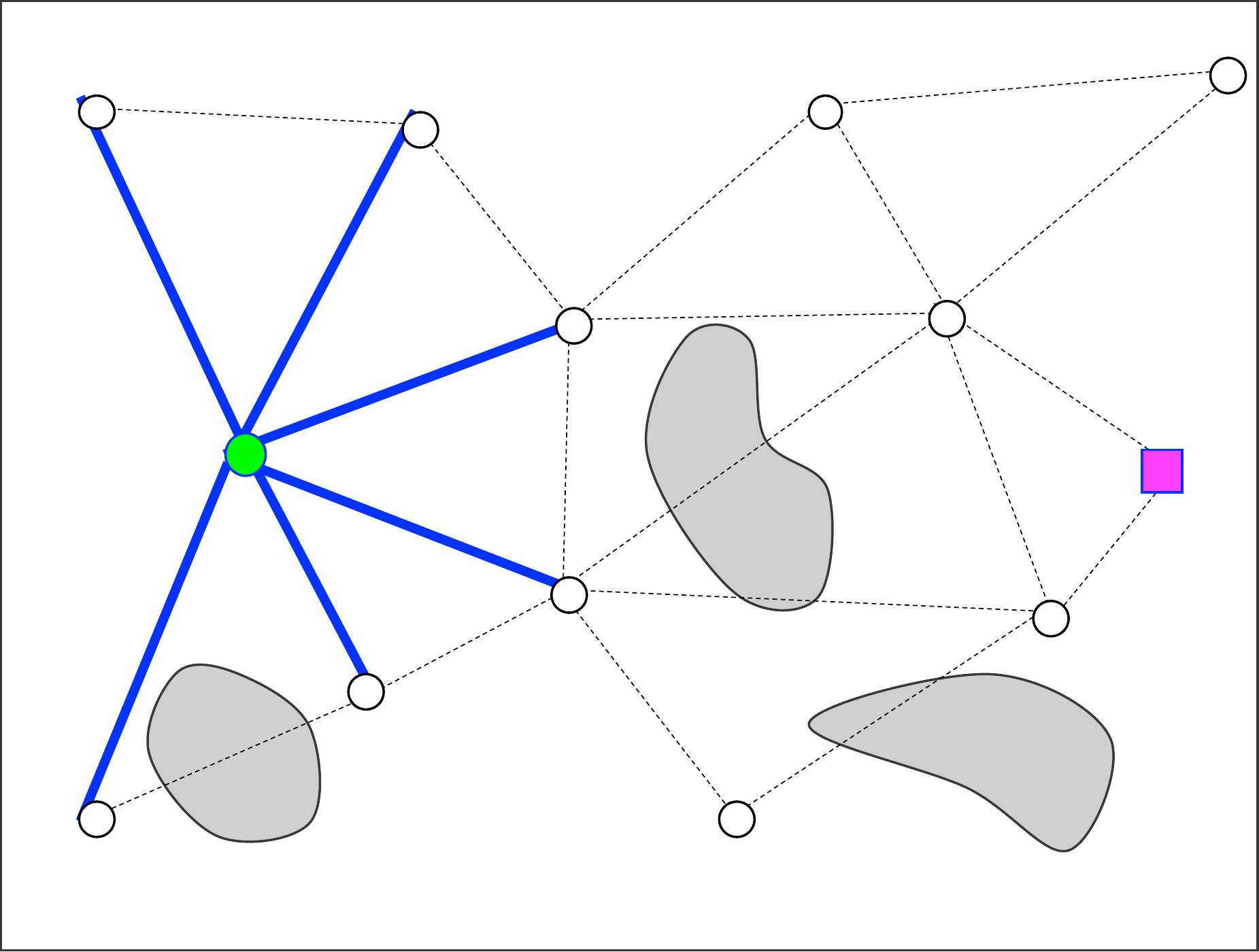}
		\caption{}
	\end{subfigure}
		\begin{subfigure}{\myMSFigureScale\textwidth}
		\includegraphics[width=\myLineScale\linewidth]{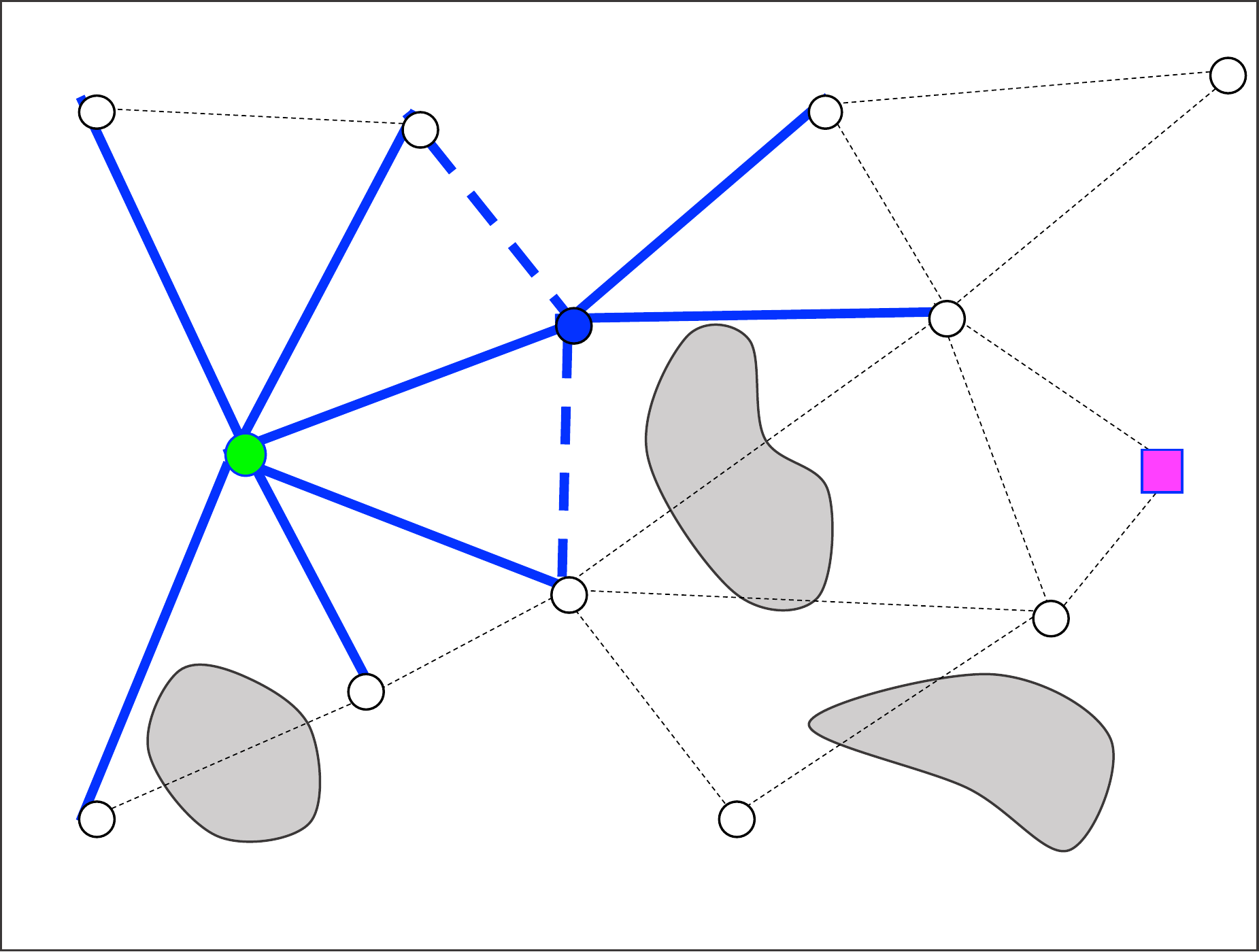}
		\caption{}
	\end{subfigure}
	\begin{subfigure}{\myMSFigureScale\textwidth}
		\includegraphics[width=\myLineScale\linewidth]{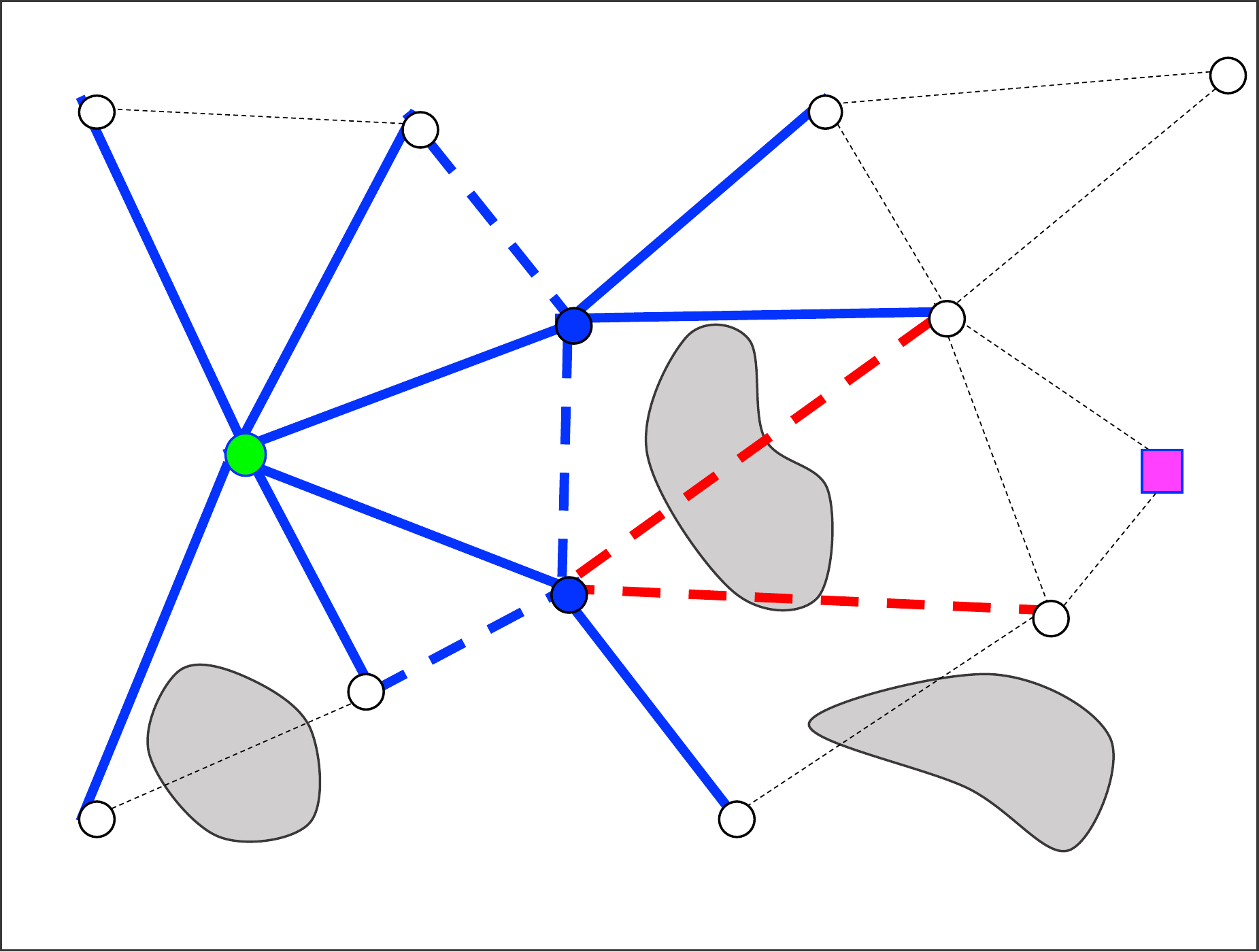}
		\caption{}
	\end{subfigure}
		\begin{subfigure}{\myMSFigureScale\textwidth}
		\includegraphics[width=\myLineScale\linewidth]{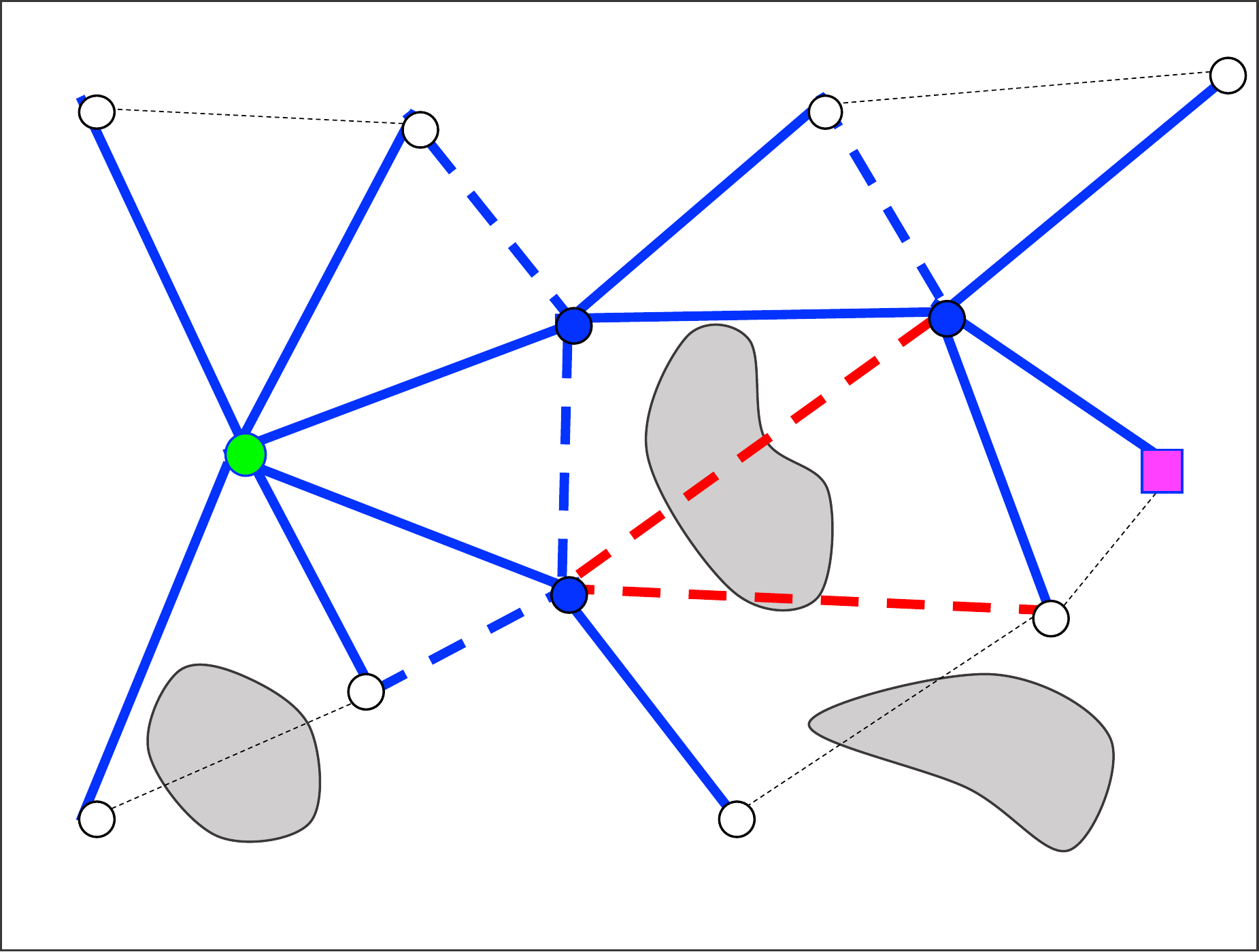}
		\caption{}
	\end{subfigure}
	\begin{subfigure}{\myMSFigureScale\textwidth}
		\includegraphics[width=\myLineScale\linewidth]{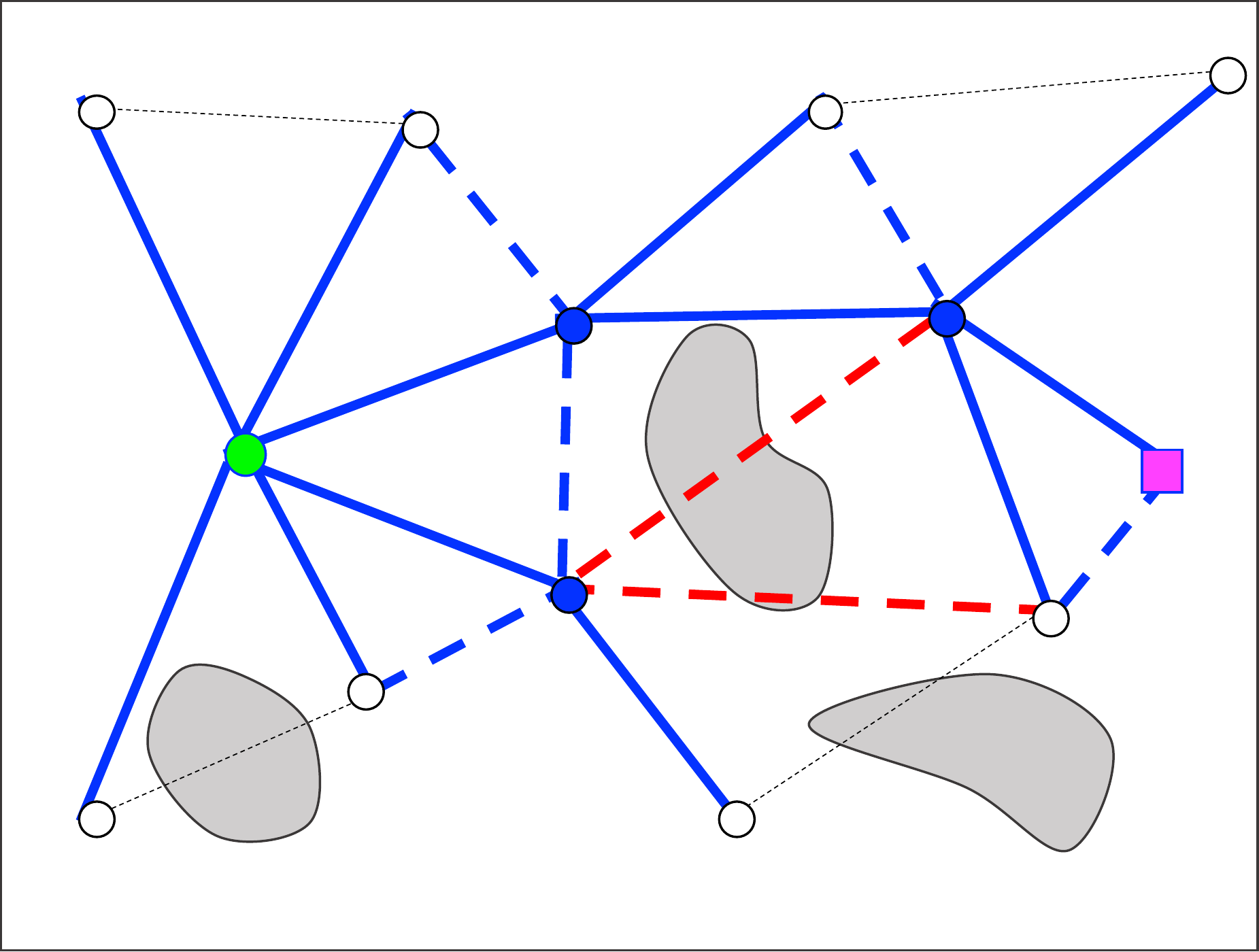}
		\caption{}
	\end{subfigure}
		\begin{subfigure}{\myMSFigureScale\textwidth}
		\includegraphics[width=\myLineScale\linewidth]{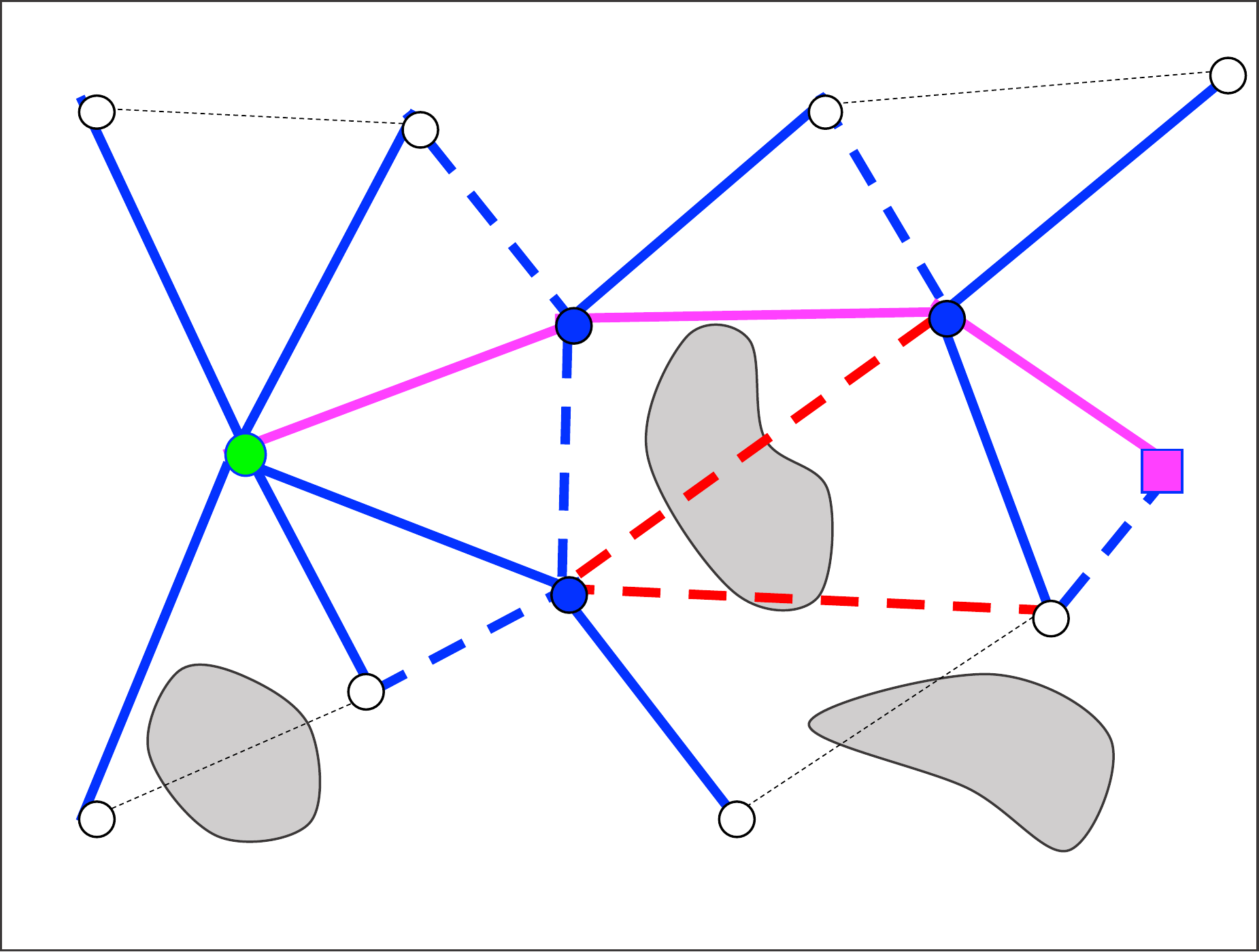}
		\caption{}
	\end{subfigure}
	\begin{subfigure}{\myMSFigureScale\textwidth}
		\includegraphics[width=\myLineScale\linewidth]{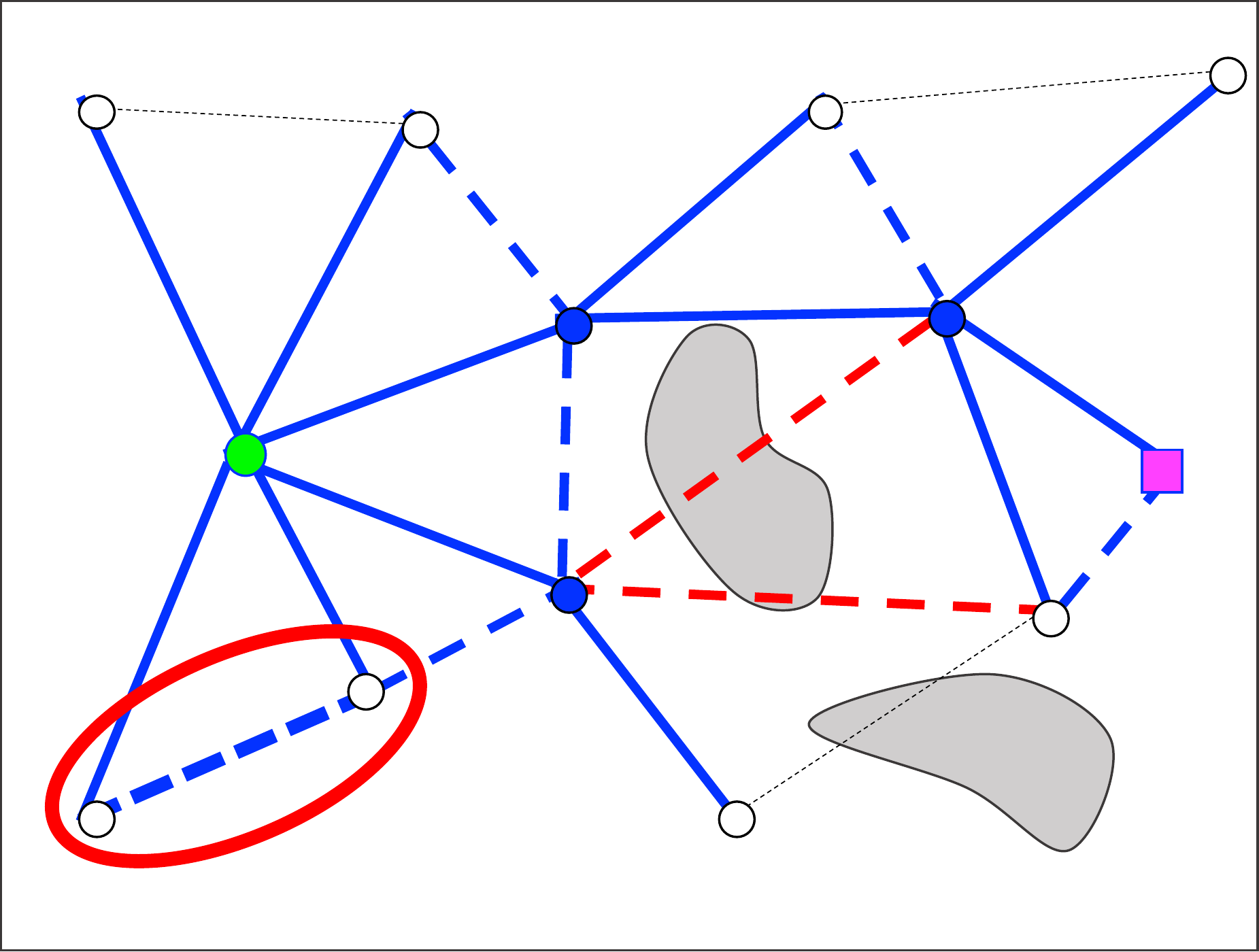}
		\caption{}
	\end{subfigure}
	\caption{The propagation of the LPA* search from (a) to (h) while it searches to find the shortest path from the start vertex~(\tikzcircle[blue, fill=green]{2.5pt}) to the goal vertex~(\tikzsquare[blue, fill=magenta]{4.5pt}) given 
	the graph in a dynamic environment. 
	The colored lines are the evaluated edges, where the bold edges are part of the current search tree and the dashed edges are not. The expanded vertices are shown with blue dots. 
	The search begins by expanding the start vertex as shown in (b), 
	and stops after expanding the goal vertex, as shown in (f).
	The optimal solution is found in the search tree in (g). 
	All the incident edges of the expanded vertices are evaluated regardless of their relevancy to the current problem. 
	After finding the optimal solution, the environment changes in the bottom left corner in (h), where the previous solution does not change.
    Evaluating this edge does not produce any cost-to-come inconsistencies which are relevant to the new problem instance.
	}
	\label{lgls:f:lpa_drawback}
\end{figure*}

LPA* needs to evaluate all the incident edges when expanding an inconsistent vertex to find a new optimal parent. 
Similar to A*, which expands the frontier vertices with the lowest cost estimates (the so-called $f$-value) as a best-first search to grow the optimal search tree, LPA*/TLPA* expands only the inconsistent vertices in a best-first manner to repair the optimal search tree. 
Hence, when a vertex is expanded to find the new optimal parent, all the values of the incident edges must have been known for the correctness of the algorithm. 
This A*-like propagation is often referred as a “zero-step lookahead” in the literature~\cite{Mandalika2018, Mandalika2019}, where a no heuristic estimate of the edge value is used to prioritize the next best vertex\footnote{This terminology is different from the one-step lookahead used in LPA*~\cite{Koenig2004}, which refers to the rhs-value of an inconsistent vertex being one-step better informed than its g-value with respect to the current graph.}.
Regardless of their potential to be a part of the optimal path, in LPA* all the incident edges are evaluated upon expanding a vertex. 
Among those edges, some edges are not relevant to the current problem, i.e., they are not helpful in constructing a new optimal search tree, and hence evaluating them is not necessary.
Figure~\ref{lgls:f:lpa_drawback}(a)-(f) illustrates this point, where all the incident edges of expanded vertices are evaluated regardless of their relevancy to the current problem.

In order to identify the inconsistent part of the tree in the current graph, LPA* also needs to evaluate all the changed edges before repair propagation begins. 
When a new edge is evaluated, the vertex corresponding to the child (end) vertex of the edge is updated by assigning a new optimal parent and a new cost-to-come value. 
This vertex is then inserted in the priority queue if the cost-to-come is altered since last search (and is therefore inconsistent), so that the inconsistency information can propagate to its successor vertices. 
When the updated vertex in the search tree cannot possibly be on the optimal path, which can be determined using a consistent heuristic cost-to-go, then the vertex is not expanded, as the change cannot help find the optimal path in the current search, and the propagation stops. 
When the change is not relevant, LPA* never uses this new edge information to find a new optimal solution, wasting computational resources by evaluating the edge. 
Regardless of whether the propagation continues or stops, LPA* needs to evaluate all changed edges, and when there are many changed edges, then evaluating those edges can result in significant computational overhead. 
Figure~\ref{lgls:f:lpa_drawback}(h) illustrates this case, where the edge in the bottom left is evaluated as the obstacle disappears, although such an edge is not relevant to the new problem instance.

Our proposed algorithms eliminate these evaluations by restricting edge evaluations only to the relevant and significant part of the graph.

\section{Problem Formulation} \label{sec:problem_formulation}

We first introduce the common variables and notation that will be used throughout the rest of the paper.
Additional notation and definitions will appear in each section of the proposed algorithms, as they become relevant to describe each specific algorithm. 

\subsection{Lazy Edge Weight Function}
%
%
Let $G =(V,E)$ be a directed graph with vertex set $V$ and edge set $E.$
For a vertex $v\in V$, we denote the set of predecessor vertices of $v$ with $pred(v)$ and the set of its successor vertices with $succ(v)$.
%
%
For each edge $e\in E$, a weight function $w: E\to (0,\infty]$ assigns a positive real number, including infinity, to this edge, e.g., the distance to traverse this edge, and infinity if traversing the edge is infeasible.
%
%
We denote an admissible heuristic weight function with $\widehat{w}: E\to (0,\infty),$
which assigns to an edge a non-overestimating positive real number such that $\widehat{w}(e)\leq w(e)$ for all $e\in E.$ 
We assume that evaluating the true weight $w$ is computationally expensive, but the heuristic edge $\widehat{w}$-value is relatively easy to compute.
%
%
Let $E_\mathrm{eval} \subseteq E$ be the subset of the edge set $E$, whose $w$-values have been computed in the current graph.
%
%
We introduce a lazy weight function $\overline{w} : E \to (0,\infty]$ which assigns to an edge its admissible heuristic weight $\widehat{w}$ before evaluation and its true weight $w$ after evaluation, that is,
\begin{equation}
	\overline{w}(e) \defeq \begin{cases}
		w(e), & \mathrm{if}\; e\in E_\mathrm{eval}, \\
		\widehat{w}(e), & \mathrm{otherwise.}
	\end{cases}
\end{equation}

\subsection{Path Planning Problem}
%
%
Define a path $\pi = (v_1, v_2 ,\ldots, v_m)$ on the graph $G=(V,E)$ as an ordered set of distinct vertices $v_i \in V$, $i = 1,\ldots,m$ such that, $v_1 = v_\mathrm{s}$ the start vertex and $v_m = v_\mathrm{g}$ the goal vertex, and such that
for any two consecutive vertices $v_i, v_{i+1}$, there exists an edge $e = (v_i, v_{i+1}) \in E.$ 
Throughout this paper, we will interchangeably denote a path as the sequence of such edges. 
%
%
With some abuse of notation, we denote the cost of a path as $w(\pi) \defeq \sum_{e\in \pi} w(e)$. 
Likewise, we denote $\overline{w}(\pi) \defeq \sum_{e\in \pi} \overline{w}(e)$ for the lazy cost estimate of the path $\pi$.
%
%
Let $v_\mathrm{s}, v_\mathrm{g} \in V$ be the start and goal vertices, respectively, and define a path-planning problem as the problem of finding a path given a tuple $P= (G,w, v_\mathrm{s},v_\mathrm{g}).$ 
Let $\Pi$ be the set of all finite-cost paths connecting $v_\mathrm{s}$ to $v_\mathrm{g}$ in $G$.
Then, the shortest path-planning problem seeks to solve 
\begin{equation}
	\pi^* \defeq \argmin_{\pi \in \Pi}w(\pi).
\end{equation}
%
%
Given a tuple $P^\varepsilon=(G,w, v_\mathrm{s},v_\mathrm{g},\varepsilon)$ with $\varepsilon>1$, a bounded-suboptimal path-planning problem is the problem of finding a path $\pi^\varepsilon$ such that $w(\pi^\varepsilon) \leq \varepsilon w(\pi^*).$
In this case, we say that the path $\pi^\varepsilon$ is $\varepsilon$-bounded. 

\subsection{Replanning Problem}
%
%
The optimal replanning problem is the problem of finding the shortest paths $(\pi^*_i)_{i\in \N}$, given a sequence of path-planning problem instances $(P_i)_{i\in\N}$, where each $P_i=(G_i, w_i, v_\mathrm{s_i},v_\mathrm{g_i})$ is a path-planning problem, possibly with different graph, weight, start and goal vertices.
%
%
The stationary optimal replanning problem finds the shortest paths given a sequence of path-planning problems $(P_i)_{i\in\N}$ where the start and goal vertices remain the same, 
i.e., $v_\mathrm{s_i}=v_\mathrm{s}$ and $v_\mathrm{g_i}=v_\mathrm{g}$ for all $i\in\N.$ 
%
%
The stationary bounded-suboptimal replanning problem solves bounded-suboptimal path-planning problems given a sequence of of path-planning problems $(P^\varepsilon_i)_{i\in\N}$, where each $P^\varepsilon_i=(G_i,w_i,v_\mathrm{s},v_\mathrm{g},\varepsilon_i)$ has the same start and goal vertices.
%
%
The non-stationary optimal replanning problem is to find the shortest paths given a sequence of path-planning problems with different start vertices but with the same goal vertex. 
Similarly, the non-stationary bounded-suboptimal replanning problem is to find a bounded suboptimal path given a sequence of path-planning problems with different start vertices but with the same goal vertex. 

%
%

\section{Lifelong Lazy Search} \label{sec:lifelong_gls}

In this section we present our first proposed lazy incremental search algorithm, namely, Lifelong-GLS (L-GLS)~\cite{Lim2021} that combines the vertex efficiency of LPA* with the edge efficiency of GLS to solve a stationary optimal replanning problem.
L-GLS returns the shortest path given the current graph, regardless of changes in either the graph topology or the edge values.
L-GLS maintains a lazy LPA* search tree to update the inconsistencies that arise both from graph changes and edge value discrepancies between the heuristic weight and the actual weight.
Unlike LPA* however, L-GLS restricts the edge evaluations only to the optimal path candidates, and unlike GLS, L-GLS uses previous search results to find a new optimal path. 

The lazy LPA* search tree is identical to the standard LPA* search tree~\cite{Koenig2004}, except that lazy LPA* uses the lazy weight function $\overline{w}$ instead of the actual weight function $w$. 

\subsection{Lazy LPA* Search Tree}
%
%
Each vertex of the lazy search tree corresponds to a unique vertex in $G$, storing the two cost-to-come values to reach that vertex from the start vertex, namely, the $g$-value and the $rhs$-value. 
Similarly to LPA*, these two cost-to-come values are used to identify the inconsistent vertices of the search tree.
The $g$-value is the accumulated cost-to-come by traversing the previous search tree, whereas the $rhs$-value is the cost-to-come based on the $g$-value of the predecessor and the $\overline{w}$-value of the current edge.
Hence, the $rhs$-value is potentially better informed than the $g$-value, and it is defined as follows:
\begin{equation}
	rhs(v) \defeq 
	\begin{cases}
		0, & \mathrm{if}\; v=v_\mathrm{s}, \\
		\min_{u\in pred(v)} (g(u)+\overline{w}(u,v)), &
		\mathrm{otherwise.}
	\end{cases}
\end{equation}
A vertex $v$ with $g(v)=rhs(v)$ is called consistent, otherwise it is called inconsistent. 
An inconsistent vertex is locally overconsistent if $g(v)>rhs(v)$ and locally underconsistent if $g(v)<rhs(v).$
%
%
Additionally, the $rhs$-value minimizing the predecessor of $v$ is stored as a backpointer, denoted with 
\begin{equation}
	bp(v)\defeq \argmin_{u\in pred(v)}(g(u)+\overline{w}(u,v)).
\end{equation}
Hence, the subpath from $v_\mathrm{s}$ to $v$ is readily retrieved by following the backpointers from $v$ to $v_\mathrm{s}$. 

%
%
A queue $Q$ prioritizes the inconsistent vertices using the key 
\begin{equation}
	k(v)=[\min (g(v),rhs(v))+h(v) \; ; \min (g(v),rhs(v))], 
\end{equation}
with lexicographic ordering, where $h(v)$ is a consistent heuristic cost-to-go from $v$ to $v_\mathrm{g}$. 

\subsection{Lifelong-GLS}
%
%
The proposed algorithm, Lifelong-GLS (L-GLS), consists of two loops: the inner loop and the outer loop. 
The inner loop is the main search loop which guarantees to return the shortest path in the current graph upon termination. 
The outer loop updates the current graph heuristically to reflect any external graph changes.
The edge evaluations in the inner loop may induce internal changes to the graph. Both external and internal changes are efficiently repaired by a lazy LPA* search tree. 

In the inner loop, the lazy LPA* search tree updates the new shortest path from $v_\mathrm{s}$ toward $v_\mathrm{g}$ in the current graph $G$ based on the previous search results. 
The lazily evaluated LPA* search tree uses the lazy estimates of the edge values when it propagates the inconsistencies to find the shortest subpath to the goal in the current graph. 
The first unevaluated edge on the shortest subpath returned by the lazy LPA* is then evaluated. 
If the evaluation results in an inconsistency, then the lazy LPA* search tree is updated and returns the next best subpath for evaluation. 
If all the edges on the current shortest path to the goal returned by the lazy LPA* have already been evaluated, then L-GLS has found the optimal solution and exits the inner loop.

In the outer loop, L-GLS waits for graph changes. When the edges of the graph $G$ change, L-GLS assigns admissible heuristic values to the changed edges instead of assigning the true edge values by evaluation. 
Then, the inner loop begins again to search for the new optimal path. 
As long as the heuristic edge values do not overestimate the true edge values, L-GLS finds the new optimal path, a result we prove in the Section ``Analysis of the L-GLS Algorithm'' below.
%
By assigning admissible heuristic values to the changed edges, only a subset of the changed edges that could be on the shortest path in the current graph are actually evaluated. 

\subsection{Details of the Algorithm and Main Procedures}

\begin{algorithm}
	\caption{\textsf{Lifelong-GLS}($G, v_\mathrm{s}, v_\mathrm{g}$)}\label{lgls:a:lgls}
	\begin{algorithmic}[1]
		\Procedure{CalculateKey}{$v$} \Return
		\State {[$\min (g(v),rhs(v))+h(v) \; ; \min (g(v),rhs(v))$];}
		\EndProcedure
		\Procedure{UpdateVertex}{$v$}
		\If {$v\neq v_\mathrm{s}$}
		\State $bp(v) \gets \argmin_{u\in pred(v)} (g(u)+\overline{w}(u,v))$;
		\State $rhs(v) \gets g(bp(v)) + \overline{w}(bp(v),v)$;
		\EndIf
		\If {$v\in Q$} 
		$Q.\textsc{Remove}(v)$;
		\EndIf
		\If {$g(v)\neq rhs(v)$} 
		\State $Q.\textsc{Insert}((v,\textsc{CalculateKey}(v)))$; 
		\EndIf
		\EndProcedure
		\Procedure{ComputeShortestPath}{$\textsc{Event}$}
		\While{
				$Q.\textsc{TopKey} \prec \textsc{CalculateKey}(v_\mathrm{g})$ \Or \\
			 $g(v_\mathrm{g})\neq rhs(v_\mathrm{g})$
		}
		\State $u\gets Q.\textsc{Pop}()$;\label{lgls:algo:line:vertexexpansion}
		\If {$g(u) > rhs(u)$}
		\State $g(u)\gets rhs(u)$; 
		\If {$\textsc{Event}(u)$ is triggered} \label{lgls:algo:line:expansionevent}
		\State \Return path from $v_\mathrm{s}$ to u;
		\EndIf
		\For {\textbf{all }$v\in succ(u)$}
		\textsc{UpdateVertex}($v$);
		\EndFor
		\Else
		\State $g(u)\gets \infty$;
		\For {\textbf{all }$v\in succ(u) \cup \set{u}$}
		\State \textsc{UpdateVertex}($v$);
		\EndFor
		\EndIf 
		\EndWhile
		\EndProcedure
		\Procedure{EvaluateEdges}{$\overline{\pi}$}
		\For{\textbf{each} $ e \in \overline{\pi}$}
		\If{$e \notin E_\mathrm{eval}$} 
		\State $\overline{w}(e) \gets w(e)$; 
		\State $E_\mathrm{eval}\gets E_\mathrm{eval} \cup \set{e}$;
		\If {$\overline{w}(e) \neq \widehat{w}(e) $} \Return $e$; 
		\EndIf
		\EndIf
		\EndFor 
		\EndProcedure
		\Procedure{Main}{$ $}
		\For {$\textbf{all } e \in E$} $\overline{w}(e) \gets \widehat{w}(e)$; 
		\EndFor
		\State $E_\mathrm{eval}\gets \varnothing$
		\State $rhs(v_\mathrm{s})\gets 0$;
		\State $\textsc{UpdateVertex}(v_\mathrm{s})$;
		\While{true}
		\Repeat \label{lgls:algo:line:mainsearchbegin}
		\State $\overline{\pi} \gets \textsc{ComputeShortestPath}(\textsc{Event})$;
		\State $(u,v)\gets \textsc{EvaluateEdges}(\overline{\pi})$; \label{lgls:algo:line:evaluateedge}
		\State $\textsc{UpdateVertex}(v)$;
		\Until{$v_\mathrm{g} \in \overline{\pi}$ \textbf{and} $\overline{\pi} \subseteq E_\mathrm{eval}$} 
		\label{lgls:algo:line:mainsearchend}
		\State Wait for changes in $E$;
		\State $L \gets $the set of edges that changed;
		\For {$\textbf{all } e=(u,v) \in L$}
		\State $\overline{w}(e) \gets \widehat{w}(e)$; 
		\State $E_\mathrm{eval} \gets E_\mathrm{eval}\backslash \set{e}$;
		\State $\textsc{UpdateVertex}(v)$;
		\EndFor
		\EndWhile
		\EndProcedure
	\end{algorithmic}
\end{algorithm}

\begin{algorithm}
	\caption{Candidate $\textsc{Event}$ Definitions~\cite{Mandalika2019}}\label{lgls:a:events}
	\begin{algorithmic}[1]
		\Procedure{ShortestPath}{$v$}
		\If{$v=v_\mathrm{g}$} \Return true; \EndIf
		\EndProcedure
		\Procedure{ConstantDepth}{$v$, depth $\alpha$}
		\State $\overline{\pi} \gets$ path from $v_\mathrm{s}$ to $v$;
		\State $\alpha_v \gets $ number of unevaluated edges in $\overline{\pi}$;
		\If{$\alpha_v = \alpha $ \textbf{or} $v=v_\mathrm{g}$} \Return true; \EndIf
		\EndProcedure
	\end{algorithmic}
\end{algorithm}

In this section we describe step-by-step the procedures used in L-GLS in greater detail. 
Before the first search begins, all $g$-values of the vertices are initialized with $\infty$, similar to the regular LPA*, and all lazy estimates of edge values are assigned with admissible heuristic values.
The first search begins by setting $rhs(v_\mathrm{s})= 0$ and by inserting  $v_\mathrm{s}$ in the priority queue $Q$.
In the main search loop (Line~\ref{lgls:algo:line:mainsearchbegin}-\ref{lgls:algo:line:mainsearchend} of Algorithm~\ref{lgls:a:lgls}) the lazy LPA* search tree is grown with \textsc{ComputeShortestPath(Event)} until an \textsc{Event} is triggered by the expansion of a leaf vertex which just became consistent upon this expansion (Line~\ref{lgls:algo:line:expansionevent} of Algorithm~\ref{lgls:a:lgls}).
\textsc{Event} is a binary function that returns \textsf{true} by checking whether the current expansion of a vertex satisfies a user-defined condition, for example, the path to the goal is found or the path of a certain depth is found (see Algorithm~\ref{lgls:a:events}).
When \textsc{Event} returns \textsf{true}, the subpath to this leaf vertex which triggered the \textsc{Event} is returned for evaluation (Line~\ref{lgls:algo:line:evaluateedge} of Algorithm~\ref{lgls:a:lgls}). 
Then, \textsc{EvaluateEdges} evaluates the unevaluated edges along this subpath and updates the lazy estimates with their true weights. 
If the evaluation of an edge results in a different value than the previous lazy estimate, then \textsc{EvaluateEdges} returns the edge for the lazy LPA* to update this change accordingly by \textsc{UpdateVertex}. 
The inconsistency is propagated by the lazy LPA* again until the next time the \textsc{Event} is triggered. 
If the path to the goal is found, and all the edges along this path have been evaluated, then the path is the optimal path in the current graph. 
This procedure repeats again when the graph changes. 

The procedure \textsc{UpdateVertex} is identical to that of the regular LPA*. The only difference is that when $\textsc{UpdateVertex}(v)$ is called, the $rhs$-value of the vertex $v$ is updated based on the lazy estimate of the incident edge values. 
This is done to avoid edge evaluations of irrelevant incident edges of $v$. 
When a minimizing predecessor is found lazily, then the vertex assigns its backpointer to this predecessor. Finally, the key of this vertex is updated with \textsc{CalculateKey} to be prioritized in the queue $Q.$ 

The choice of the $\textsc{Event}$ function determines the balance between vertex expansions (Line~\ref{lgls:algo:line:vertexexpansion} of Algorithm~\ref{lgls:a:lgls}) and edge evaluations (Line~\ref{lgls:algo:line:evaluateedge} of Algorithm~\ref{lgls:a:lgls}), as in the GLS framework. 
For example, if one chooses the \textsc{ShortestPath} as the \textsc{Event}, then the algorithm becomes a version of Lifelong-LazySP~\cite{Dellin2016}. 
That is, the lazy LPA* repairs the inconsistent part of the tree all the way up to the goal vertex, and then returns the shortest path to the goal for evaluation. This minimizes the number of edge evaluations of the inner loop.
On the other hand, if one chooses the \textsc{ConstantDepth} of GLS as the \textsc{Event}, then the algorithm becomes a version of Lifelong-LRA*~\cite{Mandalika2018}. 
The number of tree repair steps (vertex expansions) of the lazy LPA* is reduced, since the inconsistency propagation is restricted not to exceed a certain depth before evaluating the edges.
This comes at the expense of possibly more edge evaluations. 
Some candidate \textsc{Event} definitions of GLS~\cite{Mandalika2019} are reproduced in Algorithm~\ref{lgls:a:events}.

Note that the lazy LPA* algorithm maintained under the L-GLS algorithm is almost identical to the regular LPA* algorithm, except for three points. 
First, the procedure \textsc{UpdateVertex} of L-GLS updates an inconsistent vertex with respect to the lazy estimates of the incident edges instead of their actual values. 
Second, \textsc{ComputeShortestPath} is identical to that of LPA* in the way it expands the inconsistent vertices of the lowest key first, that is, when it expands an inconsistent vertex, it makes an underconsistent vertex overconsistent and an overconsistent vertex consistent. 
The difference is that when the overconsistent vertices are expanded, the \textsc{Event} checks whether to continue or stop propagating the inconsistency information to the successor vertices. 
Finally, when the graph changes, L-GLS updates the changed edge values with admissible heuristic values lazily instead of evaluating them to find the exact values. 
Hence, the lazy LPA* inherits all the theoretical properties of the regular LPA*, albeit with respect to a different weight, namely $\overline{w}$ rather than $w$. 
This becomes useful when proving the correctness of the algorithm. 

The L-GLS algorithm is different from the GLS algorithm in the following points. 
First, L-GLS stores the previous search results to propagate any inconsistencies efficiently in dynamic graphs, whereas GLS is explicitly designed for a shortest path planning problem in a static graph. 
%
%
L-GLS can possibly evaluate much fewer edges compared to the GLS from scratch, since the search tree of L-GLS is better informed than that of GLS. 
Second, the exact values for all feasible edges are known a priori in the GLS framework, that is, the heuristic estimates of all feasible edges are accurate. 
The edge evaluation only reveals a binary trait of the edge, that is, whether the edge is feasible or not, rather than its exact cost. 
This is relaxed in L-GLS, such that the edge costs can vary upon evaluation. 
This relaxation is important in problem domains where obtaining an accurate heuristic edge cost may be difficult. 
As long as the heuristic edge cost does not overestimate the actual edge cost, L-GLS finds the optimal solution in the current graph.

\subsection{Illustrative Example}

We visualize the search and evaluation process of L-GLS with infinite-lookahead and compare the results to those of LPA* in a dynamic 2D environment, as depicted in Figure~\ref{lgls:f:2d_lgls}. 
An implicit graph is first constructed whose vertices are arranged in a grid and edges are defined for two vertices within a certain distance. 
The same graph topology is used throughout this example, and only some of the edge weights are changed because of environment changes. 
Each algorithm finds the shortest path within the graph using the previous search tree. 

\begin{figure*}[ht]
	\centering
	\begin{subfigure}{\myMSFigureScale\textwidth}
		\includegraphics[width=\myLineScale\linewidth]{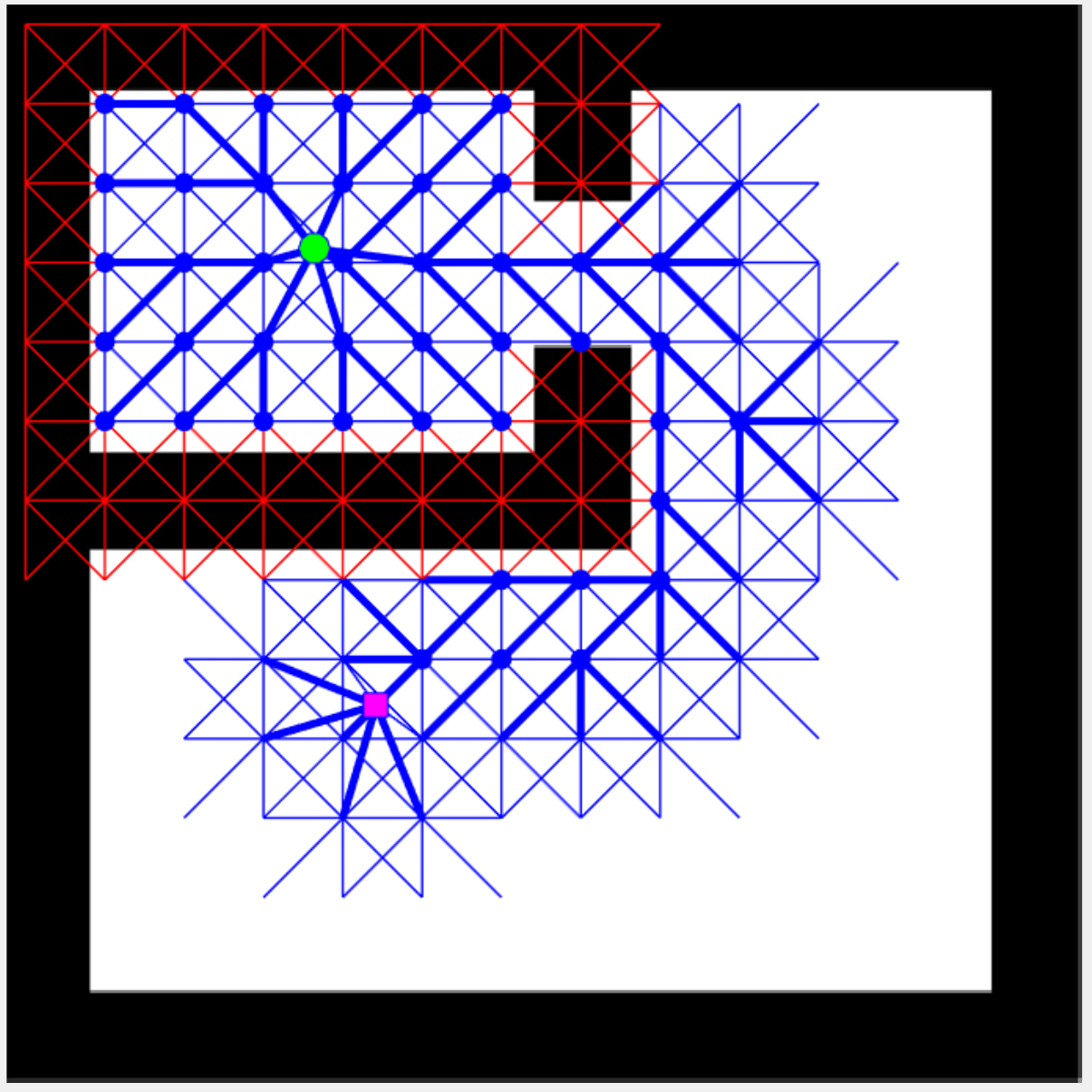}
	\end{subfigure}
	\begin{subfigure}{\myMSFigureScale\textwidth}
		\includegraphics[width=\myLineScale\linewidth]{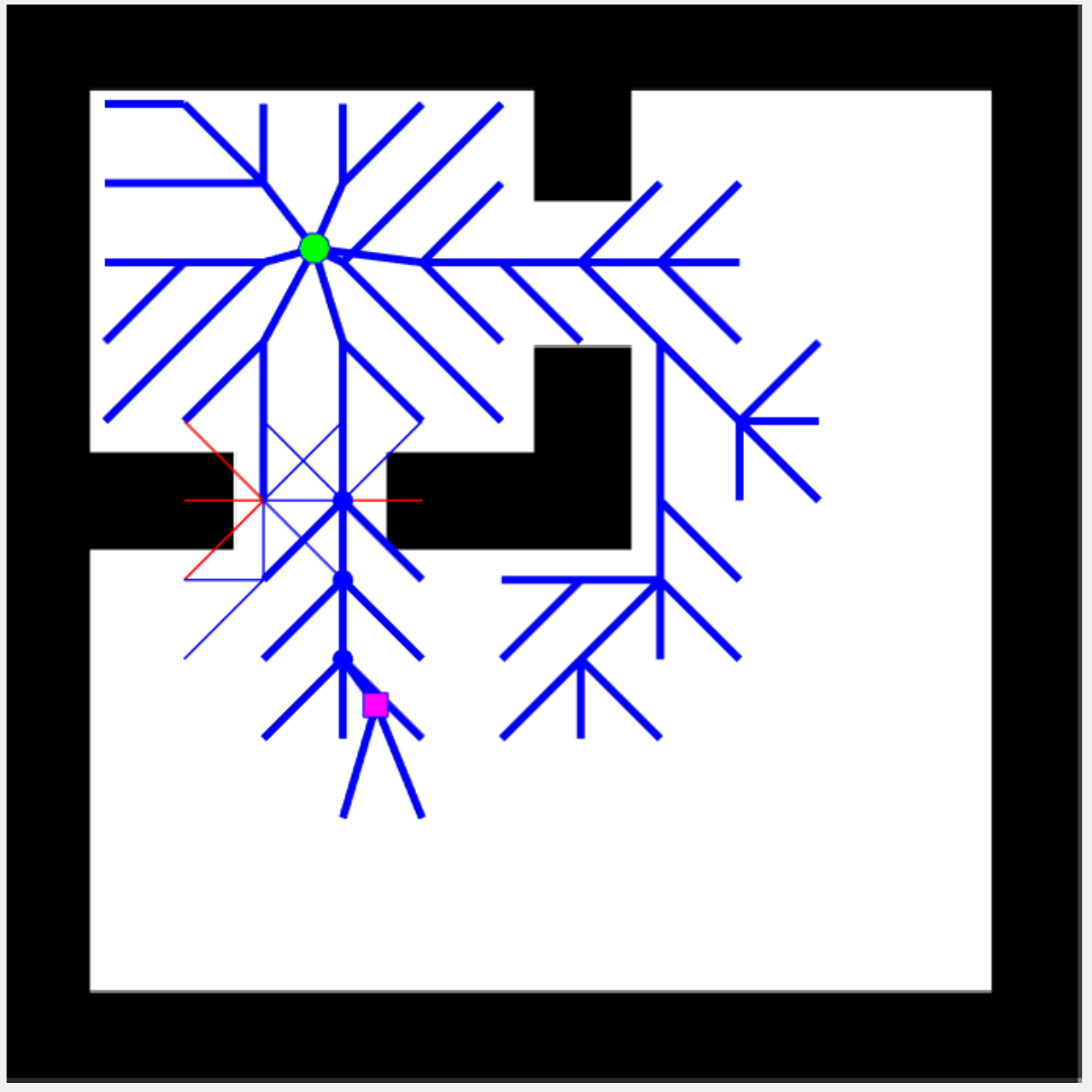}
	\end{subfigure}
	\begin{subfigure}{\myMSFigureScale\textwidth}
		\includegraphics[width=\myLineScale\linewidth]{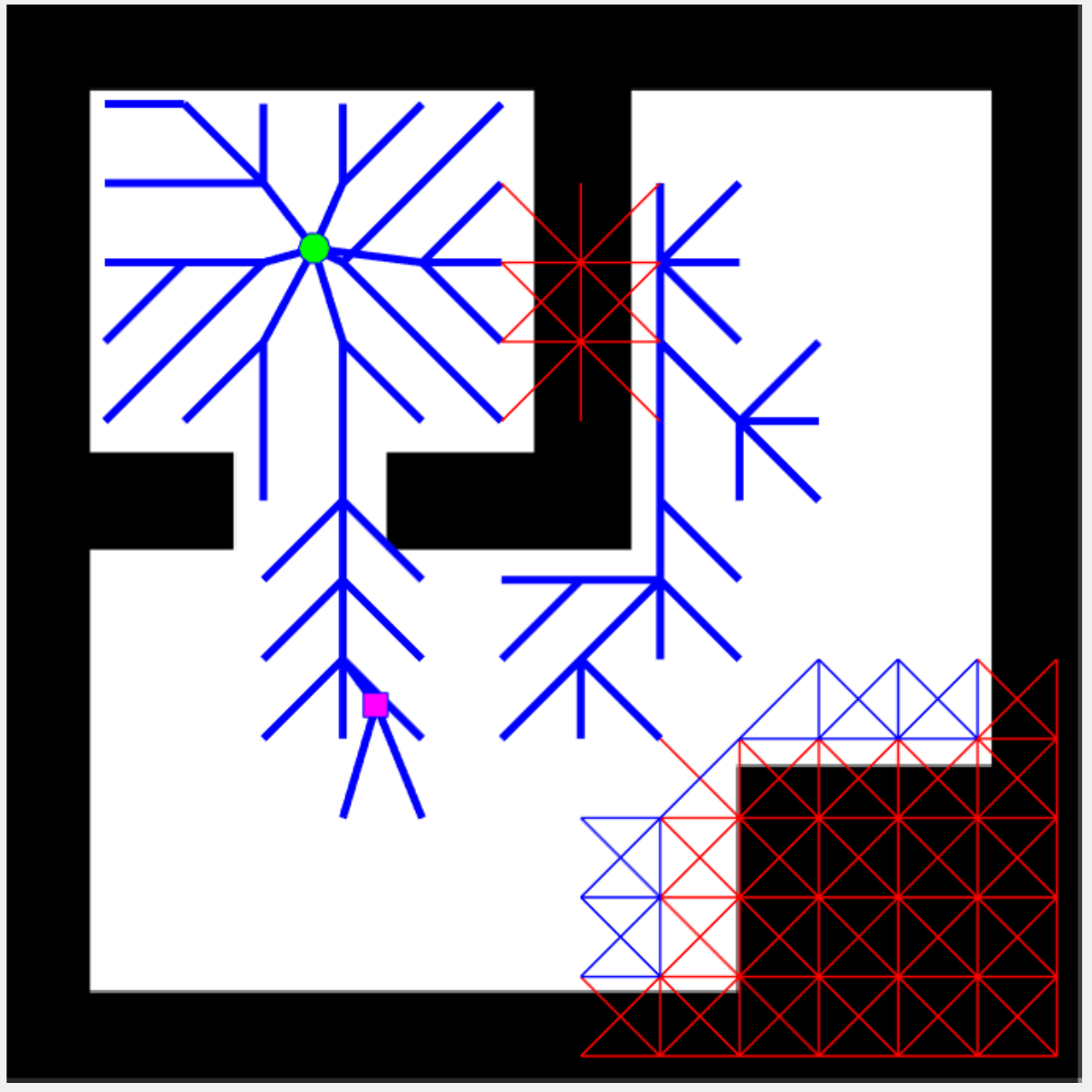}
	\end{subfigure}
	\begin{subfigure}{\myMSFigureScale\textwidth}
		\includegraphics[width=\myLineScale\linewidth]{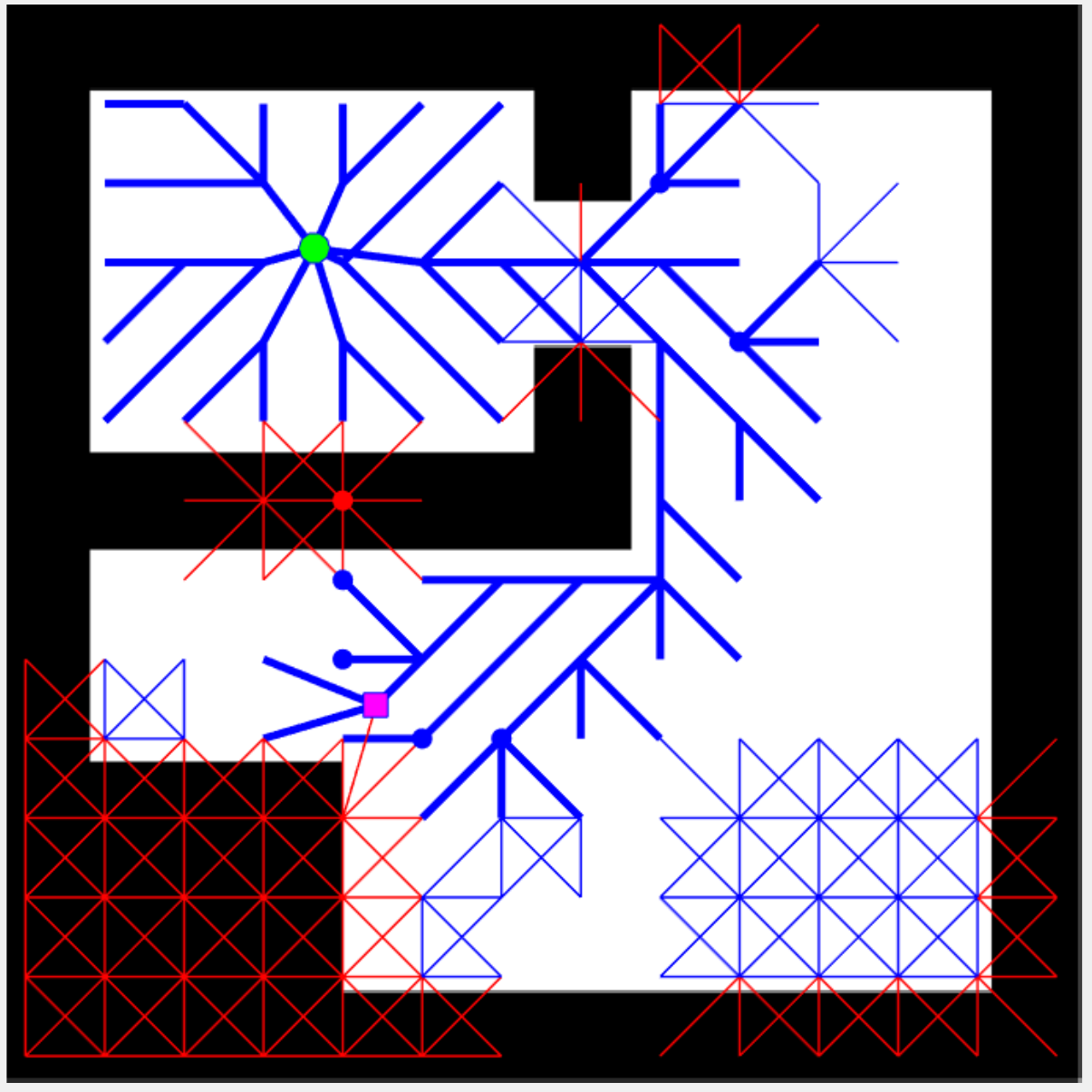}
	\end{subfigure}
	\begin{subfigure}{\myMSFigureScale\textwidth}
		\includegraphics[width=\myLineScale\linewidth]{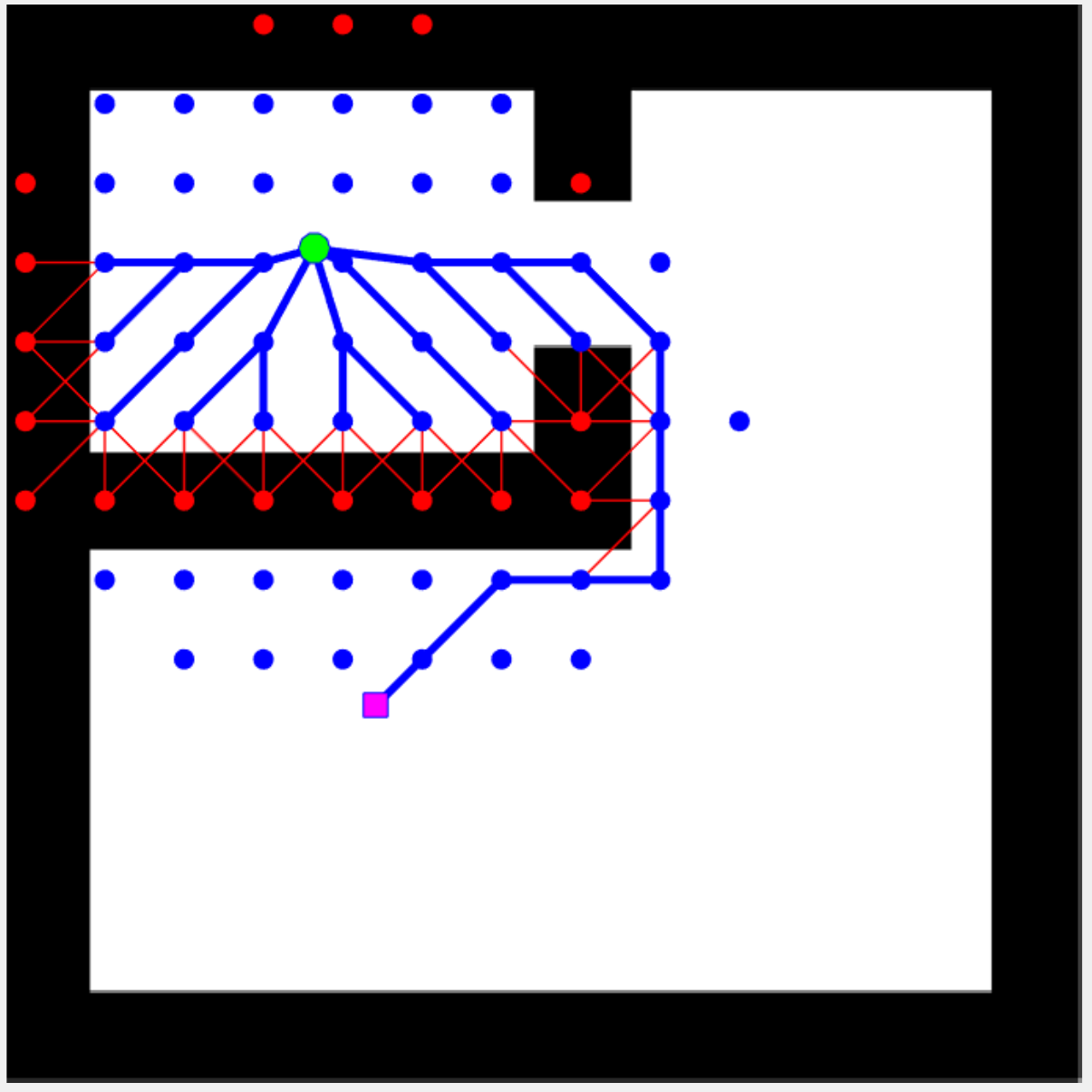}
		\caption{}
	\end{subfigure}
	\begin{subfigure}{\myMSFigureScale\textwidth}
		\includegraphics[width=\myLineScale\linewidth]{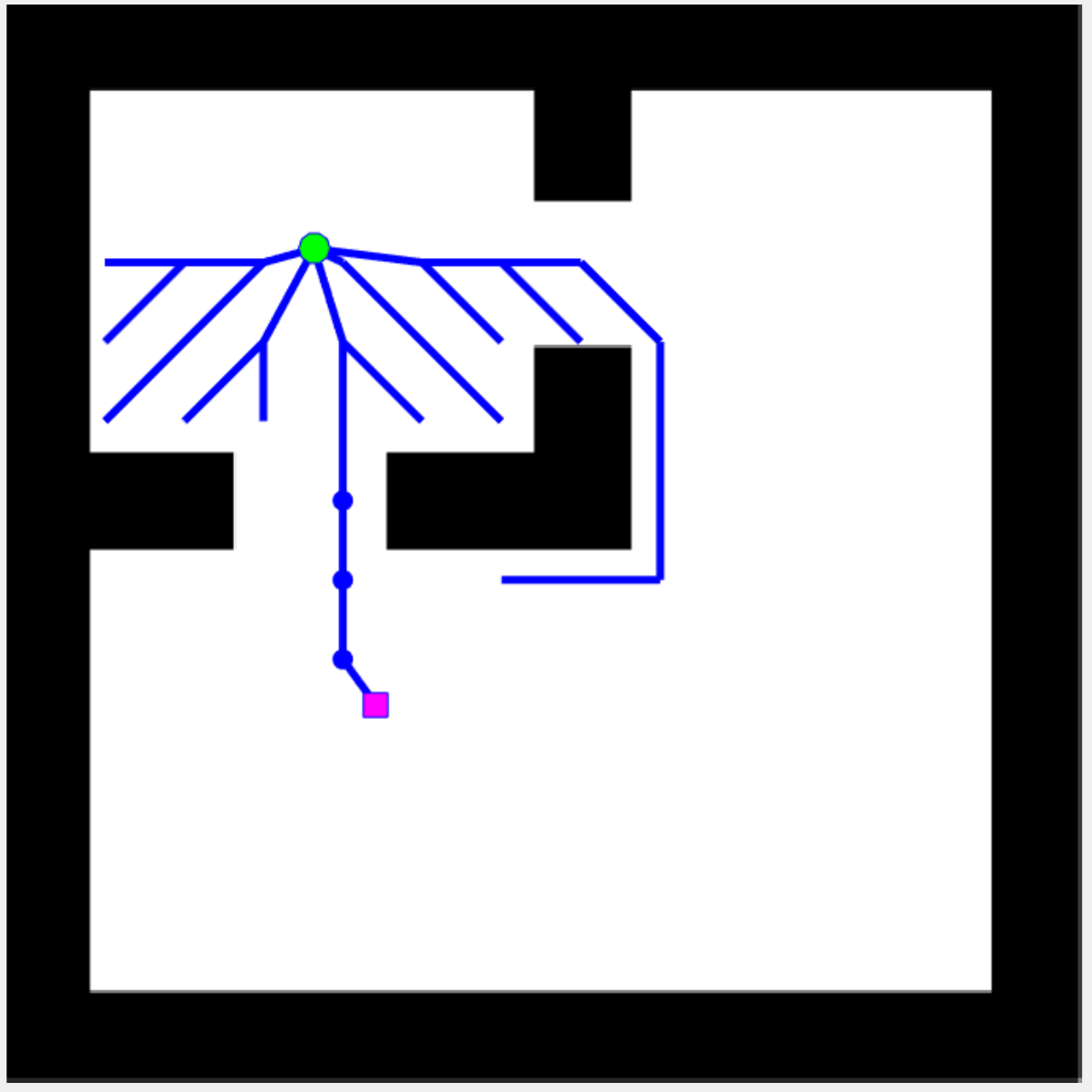}
		\caption{}
	\end{subfigure}
	\begin{subfigure}{\myMSFigureScale\textwidth}
		\includegraphics[width=\myLineScale\linewidth]{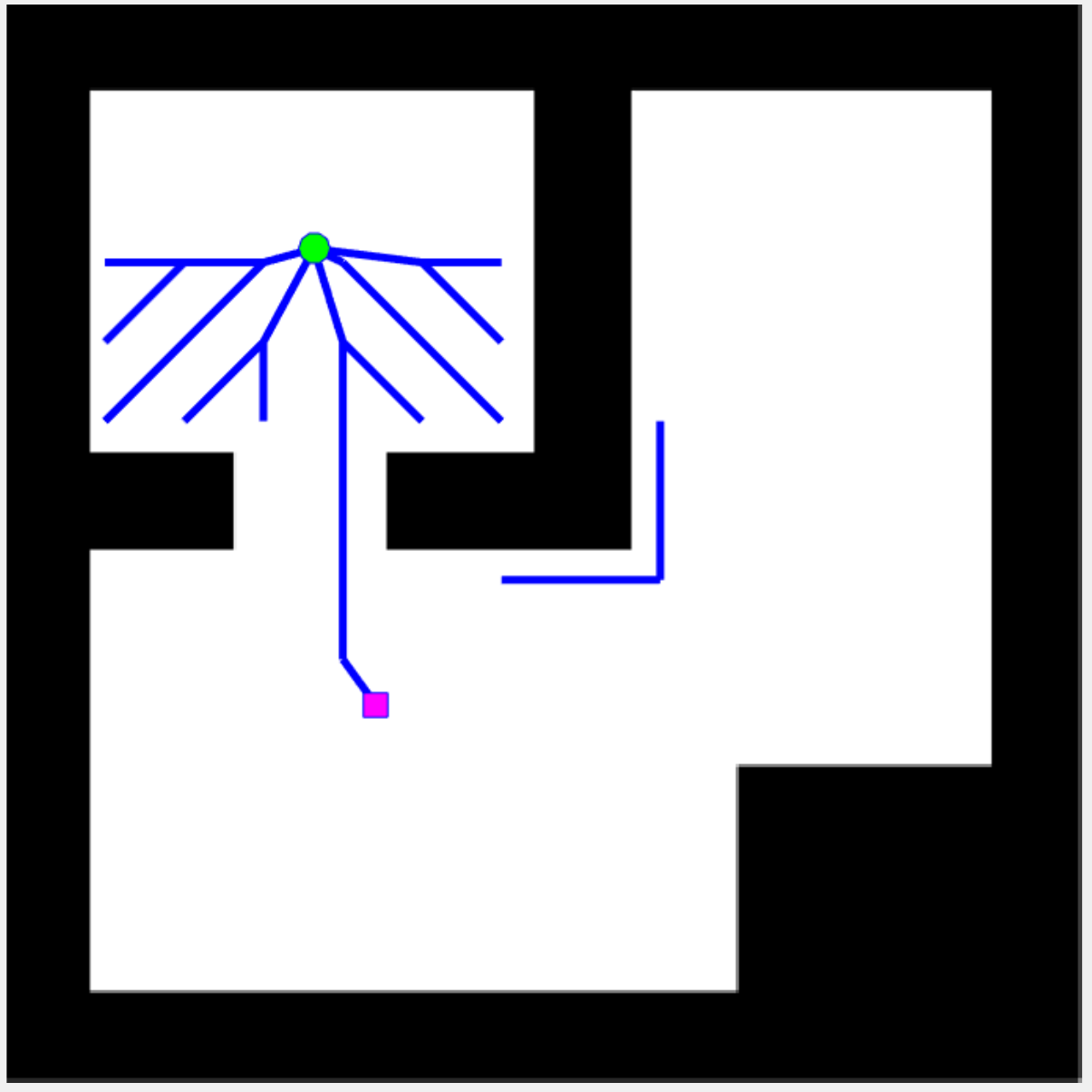}
		\caption{}
	\end{subfigure}
	\begin{subfigure}{\myMSFigureScale\textwidth}
		\includegraphics[width=\myLineScale\linewidth]{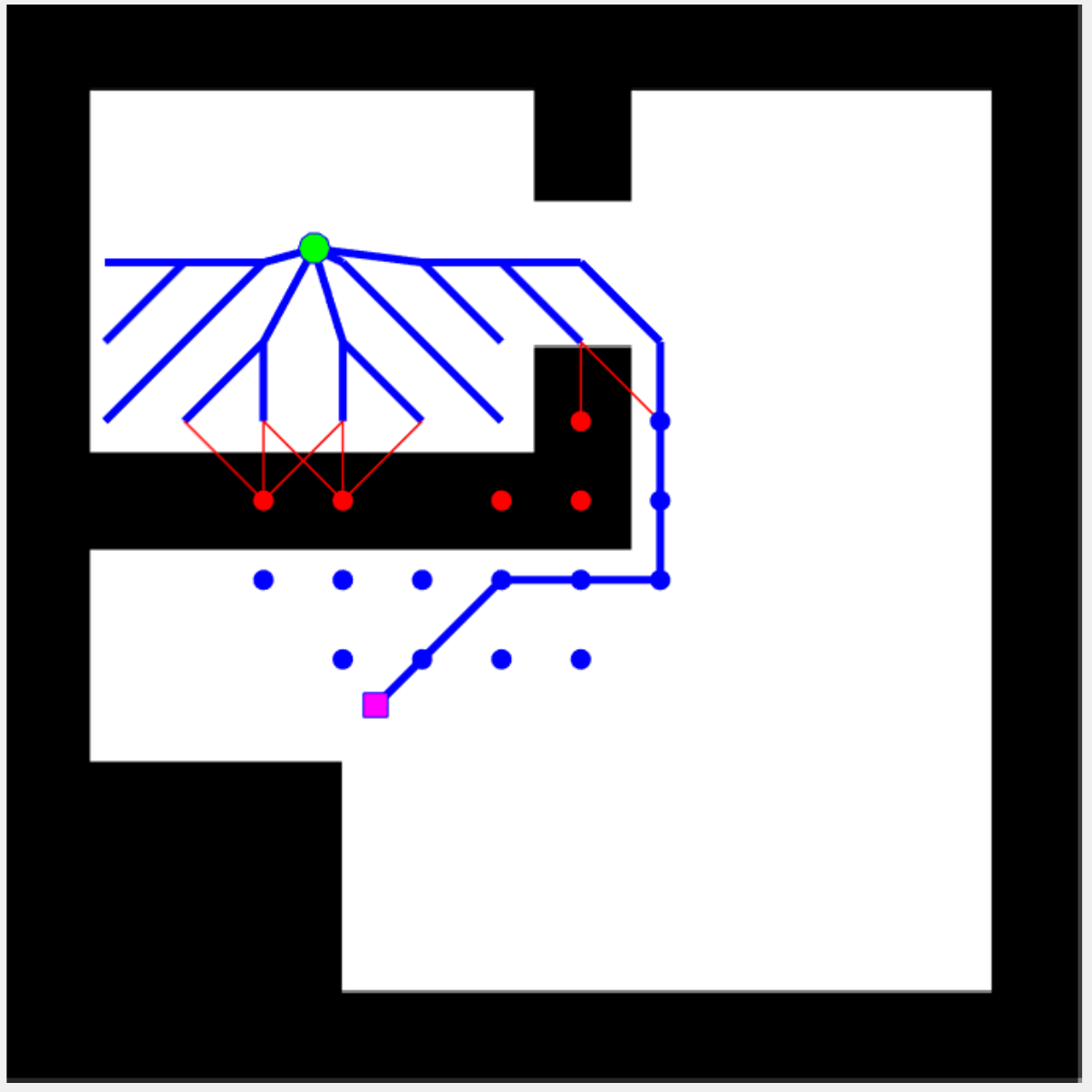}
		\caption{}
	\end{subfigure}
	\caption{LPA*(top row) and L-GLS with $\infty$-lookahead (bottom row) search results to find the shortest path from start vertex(\tikzcircle[blue, fill=green]{2.5pt}) to goal vertex(\tikzsquare[blue, fill=magenta]{4.5pt}) per environment change, from left to right: (a) first search, (b) second search, (c) third search, and (d) final search. 
		Lines(\tikzline[blue,semithick]{}\tikzline[red,semithick]{}) are the evaluated edges, and dots (\tikzcircle[blue, fill=blue]{2.0pt}\,\tikzcircle[red, fill=red]{2.0pt}) are the expanded vertices during the current search. Bold lines(\tikzline[blue,very thick]{}) are the edges belonging to the current search tree. Blue and red colors represent free space and obstacles, respectively.}
	\label{lgls:f:2d_lgls}
\end{figure*}

In the first search, LPA* is equivalent to A*, and L-GLS is equivalent to GLS (See Figure~\ref{lgls:f:2d_lgls}.a). LPA* evaluated 390 edges and expanded 45 vertices, whereas L-GLS and GLS both evaluated 61 edges and expanded 314 vertices.

After the first search, only a small part of the environment changes (see Figure~\ref{lgls:f:2d_lgls}.b), opening a shorter passage to the goal. 
LPA* evaluated 18 edges corresponding to the change, then expanded 4 inconsistent vertices to find the shortest path in the current graph. L-GLS evaluated 4 edges that belong to the new shortest path to the goal, and expanded 4 inconsistent vertices. 
The GLS evaluated 7 edges and expanded 6 inconsistent vertices. 

When the environment changed in the irrelevant region (see Figure~\ref{lgls:f:2d_lgls}.c), LPA* evaluated 153 edges corresponding to this environment change, but did not expand any vertices, as they were irrelevant to the current search. 
L-GLS did not do any additional operations to find the shortest path, since the path was already optimal. GLS was identical to the previous search with 7 edge evaluations and 6 vertex expansions.

Finally, the environment changed back to that of the first search episode with the addition to a new obstacle in the irrelevant region. The GLS search was identical to the first search episode with 61 edge evaluations and 314 vertex expansions. 
On the other hand, L-GLS evaluated only 11 edges and expanded 83 vertices. 
This is because the majority of the relevant edges were already evaluated during the previous searches, and the majority of the relevant vertices were already consistent. Similarly, LPA* expanded a fewer number of vertices and evaluated a fewer number of edges compared to the first search episode with 273 edge evaluations and 9 vertex expansions, since it utilized the previous search results.
The results are summarized in Table~\ref{lgls:t:lgls2d}.

\begin{table}[ht]
	\centering
 	\caption{Number of edge evaluations and number of vertex expansions for different planners over four consecutive search queries in a dynamic environment of Figure~\ref{lgls:f:2d_lgls}. }
	\begin{small}
	\begin{tabular}{lccc}
		\toprule
		\textbf{}      & LPA* & GLS & L-GLS\\ 
		\midrule
		\textbf{First Query}    &          &   & \\
		\# Edge Evaluation  & 390 & 61 & 61 \\ 
		\# Vertex Expansion & 45 & 314 & 314 \\
		&          &      &  \\
		\textbf{Second Query}    &          &  & \\
		\# Edge Evaluation  & 18 & 7  & 4 \\ 
		\# Vertex Expansion & 4 & 6  & 4 \\	         
		&          &      & \\
		\textbf{Third Query}     &          &  & \\
		\# Edge Evaluation  & 153 & 7  & 0\\ 
		\# Vertex Expansion & 0 & 6 &  0 \\	         
		&          &     &  \\
  		\textbf{Fourth Query}   &          &  &  \\
		\# Edge Evaluation  & 273 & 61 & 11 \\ 
		\# Vertex Expansion & 9 & 314 & 83 \\	         
		&          &    &   \\
  		\textbf{Total}     &          &  & \\
		\# Edge Evaluation  & 834 & 136 & 76 \\ 
		\# Vertex Expansion & 81 & 640 & 401 \\	         
		\bottomrule
	\end{tabular}
	\end{small}
	\label{lgls:t:lgls2d}
\end{table}

\subsection{Analysis of the L-GLS Algorithm} \label{lgls:sec:analysis}

We now present some of the properties of L-GLS to provide further insights how the algorithm works. 
We also prove the completeness and correctness of the algorithm, based on the inherited properties from both the LPA* and the GLS algorithms. First, let us state two facts that are invariant during the main search loop.

\begin{restatable}{invariant}{LGLSLazyestimateInvariant}
	\label{lgls:invariant:LGLSLazyestimateInvariant}
	The lazy estimate of an edge never overestimates the true edge value, that is, $\overline{w} \leq w$. 
\end{restatable}
\begin{proof}
	Since $\overline{w}(e) = w(e)$ for all $e\in E_\mathrm{eval},$ and $\overline{w}(e) = \widehat{w}(e) \leq w(e)$ for all $e\notin E_\mathrm{eval},$ it follows that $\overline{w}(e)\leq w(e)$ for all $e\in E.$
\end{proof}

The next invariant property shows that when the lazy LPA* returns the shortest subpath to the goal, then this subpath is optimal. 
This follows from the theoretical properties of LPA*, which are similar to A*. 
\begin{restatable}{invariant}{LGLSLPAInvariant}
	\label{lgls:invariant:LGLSLPAInvariant}
	The returned subpath $\overline{\pi}$ from $v_\mathrm{s}$ to $v$ of \textsc{ComputeShortestPath}(\textsc{Event}) is optimal with respect to $\overline{w}$, that is,
	$\overline{\pi} =\argmin_{\pi \in \Pi_v} \overline{w}(\pi)$, where $\Pi_v$ is the set of paths from $v_\mathrm{s}$ to $v.$
\end{restatable}
\begin{proof}
	\textsc{ComputeShortestPath} with an $\textsc{Event}$ returns the path $\overline{\pi}$ from $v_\mathrm{s}$ to $v$, when the triggering vertex $v$ is expanded. Right before the expansion, $v$ was locally overconsistent.
	Theorem~6 of LPA* \cite{Koenig2004} states that whenever \textsc{ComputeShortestPath} selects a locally overconsistent vertex for expansion, then the g-value of $v$ is optimal with respect to $\overline{w}$.
\end{proof}

Next, we show the completeness and correctness of the inner loop of L-GLS. The first theorem is due to the completeness of GLS~\cite{Mandalika2019}, which we restate here. 
\begin{restatable}{theorem}{LGLScomplete}
	\label{lgls:invariant:LGLScomplete}
	Let \textsc{Event} be a function that upon halting ensures that there is at least one unevaluated edge on the current shortest path or that the goal is reached. 
	Then, the inner loop (Line~\ref{lgls:algo:line:mainsearchbegin}-\ref{lgls:algo:line:mainsearchend}) of L-GLS implemented with \textsc{Event} on a finite graph terminates.
\end{restatable}
\begin{proof}
	Suppose the path to the goal has not been evaluated, such that \textsc{ComputeShortestPath(Event)} returns at least one unevaluated edge to evaluate. Since there is a finite number of edges, the inner loop will eventually terminate.
\end{proof}

\begin{restatable}{theorem}{LGLScorrect}
	\label{lgls:theorem:LGLScorrect}
	L-GLS finds the shortest path with respect to the current graph when the inner loop (Line~\ref{lgls:algo:line:mainsearchbegin}-\ref{lgls:algo:line:mainsearchend}) terminates.
\end{restatable}
\begin{proof}
	Let $\pi^*$ be the optimal path with respect to $w$ in the current graph, that is, $w(\pi^*)=\min_{\pi\in\Pi} w(\pi)$, where $\Pi$ is the set of all paths from $v_\mathrm{s}$ to $v_\mathrm{g}$. L-GLS terminates its inner-loop when $v_g \in \overline{\pi}$ and $\overline{\pi}\subseteq E_\mathrm{eval},$ where $\overline{\pi}$ is the output subpath of \textsc{ComputeShortestPath(Event)}.
	Then, we have 
	\begin{equation}
		\overline{w}(\overline{\pi}) = \sum_{e\in\overline{\pi}} \overline{w}(e)
		\leq \sum_{e'\in\pi^*} \overline{w}(e') 
		\leq \sum_{e'\in\pi^*} w(e') 
		= w(\pi^*),
	\end{equation}
	where the first inequality holds by Invariant~\ref{lgls:invariant:LGLSLPAInvariant}, and the second inequality follows by Invariant~\ref{lgls:invariant:LGLSLazyestimateInvariant}. 
	Hence, $\overline{w}(\overline{\pi}) \leq w(\pi^*),$ and since $\overline{\pi} \subseteq E_\mathrm{eval},$ we have $w(\overline{\pi}) = \overline{w}(\overline{\pi}) \leq w(\pi^*).$
	But $w(\pi^*) \leq w(\overline{\pi})$, since $\pi^*$ is the optimal path.
	Therefore, $\overline{\pi}$ must be the optimal path with respect to $w.$
\end{proof}


\section{Bounded Suboptimal Lazy Search} \label{sec:bounded_lgls}
%
%
The lazy search framework relies on using an admissible heuristic edge evaluation, and its efficiency depends on the accuracy of the heuristic edge values. In other words, the lazy search works best when the error between the actual edge value and the heuristic edge value is minimal. As in any heuristic search, the smaller the gap between the heuristic and the actual value, the faster the search. 
Similarly to A*, as a perfect heuristic cost-to-go leads to the minimal search (i.e., the minimum number of vertex expansions)~\cite{Hart1968}, a perfect heuristic edge value leads L-GLS to the minimal search and evaluation.  
However, obtaining an accurate and admissible heuristic can be as difficult as solving the original problem itself. 

Finding an inadmissible heuristic is relatively easy, which then often provides a more informed guidance in many problem domains~\cite{Pohl1970, Likhachev2008, Barer2014}. One way is to inflate the heuristic value by multiplying with some constant factor greater than 1. 
The inflation of the heuristic cost-to-go allows a goal-oriented greedy search, and it proves to be faster in many problem domains, as a bounded solution can be found rather quickly at the expense of optimality \cite{Pohl1970, Pearl1982, Aine2016a, Likhachev2003}.

Unfortunately, the L-GLS algorithm cannot be directly used with inadmissible edge heuristics, as it will harm the efficiency of the algorithm.
If the heuristic estimate is not admissible and there is a discrepancy between the actual value and the heuristic value, then the entire subtree must be repaired upon evaluation to find another admissible path, making the lazy search inefficient. 
This excessive search is necessary to find the optimal path.
Truncation~\cite{Aine2016} can relax this excessive search, as it restricts replanning to only the part of the tree needed to maintain a bounded suboptimal path. 
Incorporating truncation in the lazy search framework is the main topic of this section.

To this end, we present a bounded version of L-GLS, called B-LGLS, 
that uses a lazy version of Truncated LPA* (TLPA*)~\cite{Aine2016} in place of lazy LPA* to handle more general cases of heuristic edge weight functions. 
In fact, this extension is beneficial in two different aspects.
First, when an accurate heuristic is difficult to compute, or a more informed heuristic is available but not necessarily admissible, B-LGLS can handle inadmissible heuristics with a guaranteed solution quality. 
Second, when the change in the environment is not significant, such that complete and thorough replanning may not produce a significantly better quality solution, B-LGLS can reduce the search effort.

In essence, B-LGLS uses inflation of the heuristic edge weight and truncation of the inconsistency propagation. 
Truncation stops the propagation of inconsistencies of the lazy search tree, and inflation restricts the evaluation of non-promising edges. 
Both inflation and truncation are two techniques to make the search itself lazy, in the sense that they restrict the repair of the tree to only the part where it is necessary to guarantee that the current solution is bounded-suboptimal. 

Next, we first describe the core elements of B-LGLS, namely the lazy-TLPA* search, and then we provide a complete description of B-LGLS which builds on the lazy-TLPA*. 

\subsection{Lazy-TLPA* Search Tree} \label{lgls:sec:tlpa}

The lazy-TLPA* search tree is identical to TLPA*~\cite{Aine2016}, except that lazy-TLPA* uses a lazy weight function $\widetilde{w}$ instead of the actual weight function $w$. 
The lazy weight function $\widetilde{w} : E \to (0,\infty]$ 
assigns to an edge its inflated heuristic weight $\varepsilon_1 \widehat{w}$ before evaluation and its true weight $w$ after evaluation, that is,
\begin{equation}
	\widetilde{w}(e) \defeq \begin{cases}
		w(e), & \mathrm{if}\; e\in E_\mathrm{eval}, \\
		\varepsilon_1 \widehat{w}(e), & \mathrm{otherwise,}
	\end{cases}
\end{equation}
for some constant factor $\varepsilon_1 \geq 1,$ where $\widehat{w}\leq w$ is an admissible heuristic weight, and
$E_\mathrm{eval}$, as before, is the subset of the edge set $E$, whose $w$-values have been computed in the current graph.
We call $\varepsilon_1$ the inflation factor. 
Hence, the theoretical properties of TLPA* hold exactly the same for the lazy-TLPA*, except that now they hold 
with respect to $\widetilde{w}$ instead of the actual $w$. 

Similarly to lazy-LPA*, the $rhs$-value of a vertex $v$ is defined as follows:
\begin{equation}
	rhs(v) \defeq 
	\begin{cases}
		0, & \mathrm{if}\; v=v_\mathrm{s}, \\
		\min_{u\in pred(v)} (g(u)+\widetilde{w}(u,v)), &
		\mathrm{otherwise.}
	\end{cases}
\end{equation}
%
%
Also, the $rhs$-value minimizing the predecessor of $v$ is stored as a backpointer, denoted with 
\begin{equation}
	bp(v)\defeq \argmin_{u\in pred(v)}(g(u)+\widetilde{w}(u,v)).
\end{equation}
Next, we briefly describe the fundamental properties of TLPA*.
A full description can be found in~\cite{Aine2016}.

TLPA* is identical to LPA*, except that it uses two truncation rules to stop the best-first inconsistency propagation of the LPA* search. 
Recall that LPA* uses the two cost-to-come values, the $g$-value and $rhs$-value to indentify inconsistent vertices.
In addition to these, TLPA* uses an additional cost-to-come value, denoted with $g^\pi$-value, which assigns to a vertex the cost-to-come along the current path following its backpointer to $v_\mathrm{s}$. 
The $g^\pi$-value of a vertex reflects the cost-to-come on the current tree with or without repair, whereas the $rhs$-value of a vertex reflects the cost-to-come after the repair.
Hence, the $g^\pi$-value may be different than the $rhs$-value. 
Also, the $g^\pi$-value may be different than the $g$-value, the previous cost-to-come value without any repair. 
The $g^\pi$-value reflects the in-process cost-to-come value during repair propagation. 
TLPA* uses the $g^\pi$-value to decide whether an inconsistency should be further propagated or not when an inconsistent vertex from the priority queue is chosen for expansion. 
The decision is based on the two truncation rules.    

The first truncation rule applies to all inconsistent vertices selected for expansion.
When an inconsistent vertex $v$ with the lowest key value from the priority queue is selected for expansion, TLPA* first checks whether the current solution cost is already within the bound even without any further repair. 
That is, given some truncation factor $\varepsilon_2 \geq 1$, TLPA* checks whether the inequality 
\begin{equation} \label{blgls:eq:truncation1}
g^\pi(v_\mathrm{g}) \leq \varepsilon_2 (\min\{ g(v), rhs(v)\} +h(v)),
\end{equation}
holds for the vertex $v$ selected for expansion. 
If so, then TLPA* stops expanding any remaining inconsistent vertices and returns an existing bounded suboptimal solution. 
Since the lowest key value of the priority queue is a lower bound on the optimal solution cost of the current graph, i.e., $\min[g(v), rhs(v)] +h(v) \leq g^*(v_\mathrm{g})$, once the inequality (\ref{blgls:eq:truncation1}) is satisfied, then the current solution cost without further repair does not exceed the optimal solution by more than $\varepsilon_2,$ that is, $g^\pi(v_\mathrm{g}) \leq \varepsilon_2 g^*(v_\mathrm{g})$ holds. 
Thus, truncating the propagation still guarantees that the current solution is within the desired suboptimality bound. 

The second truncation rule applies to an underconsistent vertex (i.e., a vertex $v$ with $g(v)<rhs(v)$) selected for expansion by checking
\begin{equation}\label{blgls:eq:truncation2}
    g^\pi(v)+h(v)\leq \varepsilon_2 (g(v)+h(v)).
\end{equation}
If this inequality holds, then the underconsistent vertex is put in the $\textsf{TRUNCATED}$ list instead of the regular priority queue, so that the propagation of inconsistency constrained to this vertex stops.
Since $g(v)$ holds the previous shortest path cost and $h(v)$ is a consistent heuristic cost-to-go, $g(v)+h(v)$ is a lower bound on the solution cost constrained to pass through $v$. 
Thus, if $g^\pi(v)+h(v)\leq \varepsilon_2 (g(v)+h(v))$, then any vertex $v'$ that uses $g(v)$ to compute its $rhs(v')$ will not underestimate the actual solution cost by more than a $\varepsilon_2$ factor. 
Truncating this inconsistency propagation guarantees that the solution is within the desired bound as the old path has not deteriorated beyond this bound. 

The rest of the TLPA* algorithm is similar to LPA*. It uses the same \textsc{CalculateKey} procedure to compute the key values of inconsistent vertices as in LPA*, and it prioritizes inconsistent vertices according to the lexicographical ordering of the key values. TLPA* uses the same \textsc{UpdateVertex} procedure to update the $rhs$-value, except that TLPA* does not insert an inconsistent vertex into the priority queue if the vertex is already truncated, i.e., is in the \textsf{TRUNCATED} list.
The main propagation loop, the procedure \textsc{ComputePath} is equivalent to the procedure \textsc{ComputeShortestPath} of LPA*, except that \textsc{ComputePath} computes the $g^\pi$-value of a goal and an underconsistent vertex and applies the two truncation rules mentioned above. 
The auxiliary routines of TLPA*, namely, computing the $g^\pi$-value and obtaining the path, are provided in Algorithm~\ref{lgls:a:tlpastaraux}.
For now, line \ref{blgls:algo:line:expansionevent} in Algorithm~\ref{lgls:a:tlpastaraux} can be ignored. 
It can be shown that TLPA* guarantees the following properties: 
\begin{itemize}
\item \textbf{Bounded Suboptimality}: When the \textsc{ComputePath} function returns the path constructed using \textsc{ObtainPath}($v_\mathrm{g}$), the path has cost less than or equal to $\varepsilon_2 g^*(v_\mathrm{g})$ for a chosen $\varepsilon_2 \geq 1.$
\item \textbf{Efficiency}: In \textsc{ComputePath}, no vertex is expanded more than twice.
\end{itemize}


\subsection{The Bounded L-GLS Algorithm}

Bounded L-GLS (B-LGLS) is identical to L-GLS, except that the heuristic search tree is repaired with the lazy-TLPA* instead of the lazy-LPA*, and the heuristic edge value is inflated.
Like L-GLS, B-LGLS contains two loops: an inner loop and an outer loop. 
In the inner loop, the inconsistency between the heuristic search tree and the actual path is repaired using the lazy-TLPA* to produce a bounded suboptimal path candidate. 
In the outer loop, any changes are updated lazily with inflated heuristic edge values. 
The B-LGLS algorithm is presented in Algorithm~\ref{lgls:a:blgls}, where the differences between B-LGLS and L-GLS are colored in blue.

\begin{algorithm}
	\caption{\textsf{Bounded L-GLS}($G, v_\mathrm{s}, v_\mathrm{g}$)}\label{lgls:a:blgls}
	\begin{small}
	\begin{algorithmic}[1]
		\Procedure{CalculateKey}{$v$} \Return
			\State {[$\min (g(v),rhs(v))+h(v) \; ; \min (g(v),rhs(v))$];}
		\EndProcedure
		\Procedure{UpdateVertex}{$v$}
			\If {$v\neq v_\mathrm{s}$}
				\State $bp(v) \gets \argmin_{u\in pred(v)} (g(u)+\widetilde{w}(u,v))$;
				\State $rhs(v) \gets g(bp(v)) + \widetilde{w}(bp(v),v)$;
			\EndIf
			\If {$v\in Q$} 
				$Q.\textsc{Remove}(v)$;
			\EndIf
			\If {$g(v)\neq rhs(v)$ and \textcolor{blue}{$v\notin \textsf{TRUNCATED}$}} 
				\State $Q.\textsc{Insert}((v,\textsc{CalculateKey}(v)))$; 
			\EndIf
		\EndProcedure
		\Procedure{ComputePath}{$\textsc{Event}$}
			\While{$Q.\textsc{TopKey} \prec \textsc{CalculateKey}(v_\mathrm{g})$ \Or \\
				$g(v_\mathrm{g})\neq rhs(v_\mathrm{g})$}
				\State $u\gets Q.\textsc{Pop}()$;\label{blgls:algo:line:vertexexpansion}
				\State \textcolor{blue}{$g^\pi(v_\mathrm{g})\gets$  \textsc{ComputeG\small{PI}}($v_\mathrm{g}$);}
				\If{\textcolor{blue}{$g^\pi(v_\mathrm{g}) \leq \varepsilon_2 (\min \{g(u),rhs(u)\} + h(u))$}}
					\State \Return \textcolor{blue}{\textsc{ObtainPath}($v_\mathrm{g}$);}
				\EndIf
				\If {$g(u) > rhs(u)$}
					\State $g(u)\gets rhs(u)$; 
					\If {$\textsc{Event}(u)$ is triggered} \label{blgls:algo:line:expansionevent}
						\State \Return \textsc{ObtainPath}($u$);
					\EndIf
					\For {\textbf{all }$v\in succ(u)$}
						\textsc{UpdateVertex}($v$);
					\EndFor
				\Else
					\State \textcolor{blue}{$g^\pi(u)\gets \textsc{ComputeG\small{PI}}(u)$;}
					\If{\textcolor{blue}{$g^\pi(u)+h(u)\leq \varepsilon_2(g(u)+h(u))$}}
						\State \textcolor{blue}{\textsf{TRUNCATED}.\textsc{Insert}($u$);}
					\Else 
						\State $g(u)\gets \infty$;
						\For {\textbf{all }$v\in succ(u) \cup \set{u}$}
							\State \textsc{UpdateVertex}($v$);
						\EndFor
					\EndIf 
				\EndIf 
			\EndWhile
		\EndProcedure
		\Procedure{EvaluateEdges}{$\overline{\pi}$}
			\For{\textbf{each} $ e \in \overline{\pi}$}
			\If{$e \notin E_\mathrm{eval}$} 
			\State \textcolor{blue}{$\widetilde{w}(e) \gets w(e)$}; 
			 $E_\mathrm{eval}\gets E_\mathrm{eval} \cup \set{e}$;
			\If {\textcolor{blue}{$\widetilde{w}(e) \neq \widehat{w}(e) $}} \Return $e$; 
			\EndIf
			\EndIf
			\EndFor 
			\EndProcedure
		\Procedure{Main}{$ $}
			\For {$\textbf{all } e \in E$} $\textcolor{blue}{\widetilde{w}(e) \gets \varepsilon_1\widehat{w}(e)}$; 
			\EndFor
			\State $E_\mathrm{eval}\gets \varnothing$;
			$rhs(v_\mathrm{s})\gets 0$; \textcolor{blue}{$g^\pi(v_\mathrm{s})\gets 0$;}
			\State $\textsc{UpdateVertex}(v_\mathrm{s})$;
			\While{true}
			\Repeat \label{blgls:algo:line:mainsearchbegin}
				\State $\overline{\pi} \gets \textsc{ComputePath}(\textsc{Event})$;
				\State $(u,v)\gets \textsc{EvaluateEdges}(\overline{\pi})$; \label{blgls:algo:line:evaluateedge}
				\State $\textsc{UpdateVertex}(v)$;
				\For {\textcolor{blue}{all $s\in$ \textsf{TRUNCATED}}}
					\State \textcolor{blue}{Remove $s$ from \textsf{TRUNCATED};}
					\State \textcolor{blue}{$g^\pi(s)\gets \infty;$}
				\EndFor
			\Until{$v_\mathrm{g} \in \overline{\pi}$ \textbf{and} $\overline{\pi} \subseteq E_\mathrm{eval}$} 
				\label{blgls:algo:line:mainsearchend}
				\State Wait for changes in $E$;
				\State $L \gets $the set of edges that changed;
				\For {$\textbf{all } e=(u,v) \in L$}
				\State \textcolor{blue}{$\widetilde{w}(e) \gets \varepsilon_1 \widehat{w}(e)$}; 
				\State $E_\mathrm{eval} \gets E_\mathrm{eval}\backslash \set{e}$;
				\State $\textsc{UpdateVertex}(v)$;
				\EndFor
			\EndWhile
		\EndProcedure
	\end{algorithmic}
	\end{small}
\end{algorithm}

\begin{algorithm}
	\caption{Lazy Truncated LPA*: auxiliary routines: \cite{Aine2016}}\label{lgls:a:tlpastaraux}
	\begin{small}
	\begin{algorithmic}[1]
		\Procedure{ComputeG\small{PI}}{$v$}
			\State cost$\gets$0; $v'\gets v$; 
			\State visited $\gets \varnothing$;
			\While{$v' \neq v_\mathrm{s}$}
				\If{$v'\in$ visited \Or $bp(v')=$ \textbf{null}}
					\State $g^\pi(v) \gets \infty$; \Return 
				\Else
					\State insert $v'$ in visited;
					\State cost $\gets$ cost + \textcolor{blue}{$\widetilde{w}(bp(v'),v')$}; $v'\gets bp(v')$; 
				\EndIf
			\EndWhile
			\State \Return $g^\pi(v) \gets$ cost; 
		\EndProcedure
		\Procedure{ObtainPath}{$v$}
			\State $\pi \gets v$
			\While{$v\neq v_\mathrm{s}$}
			    \If{$bp(v)\in \textsf{TRUNCATED}$}
			        \State $\pi \gets \pi \cup \pi(bp(v))$;
			        \Return $\pi$; 
		        \EndIf
				\State $\pi \gets \pi \cup bp(v)$; $v \gets bp(v)$
			\EndWhile
			\State \Return $\pi$;
		\EndProcedure
	\end{algorithmic}
	\end{small}
\end{algorithm}

\subsubsection{Details of the Algorithm and Main Procedures.}

In this section we describe step-by-step the procedures of B-LGLS in greater detail. 

Before the first search begins, all the cost-to-come values are initialized with $\infty$, and all the edges are assigned with an inflated heuristic value. 
The search begins by assigning the $rhs$-value and the $g^\pi$-value of the start vertex $v_\mathrm{s}$ to 0, and inserting $v_\mathrm{s}$ in the priority queue via $\textsc{UpdateVertex}.$
In the main search loop (Lines~\ref{blgls:algo:line:mainsearchbegin}-\ref{blgls:algo:line:mainsearchend} of Algorithm~\ref{lgls:a:blgls}), the lazy TLPA* search tree is grown with \textsc{ComputePath(Event)} until the first truncation rule is satisfied or an \textsc{Event} is triggered. 
Then, \textsc{EvaluateEdges} evaluates the unevaluated edges along the subpath and updates their lazy estimate. 
The inconsistency induced by the edge evaluation is propagated by the lazy TLPA* until next time the first truncation rule applies, or an \textsc{Event} is triggered when expanding an overconsistent vertex.
If the path to the goal is found, and all the edges along this path are evaluated, then the path is a bounded suboptimal solution. This procedure repeats again when the graph changes.

When the graph changes, the weights of the changed edges are updated with inflated heuristic values, and the end vertices of the updated edges are put in the priority queue via \textsc{UpdateVertex}.
The procedure \textsc{UpdateVertex} is identical to that of LPA*, except that it uses the lazy edge estimate $\widetilde{w}$ to update the $rhs$-value and its $bp$,
and also it does not put a vertex in the priority queue if that vertex has been already truncated.

\subsection{Illustrative Example}

We visualize the differences between LPA*, TLPA*, L-GLS, and B-LGLS for two consecutive planning problems in a 2D environment shown in Figure~\ref{lgls:f:2d_blgls}.
The top row shows the first problem instance and the bottom row shows the second problem instance with a new opening in the middle of the map. 

\begin{figure*}[ht]
	\centering
	\begin{subfigure}{\myMSFigureScale\textwidth}
		\includegraphics[width=\myLineScale\linewidth]{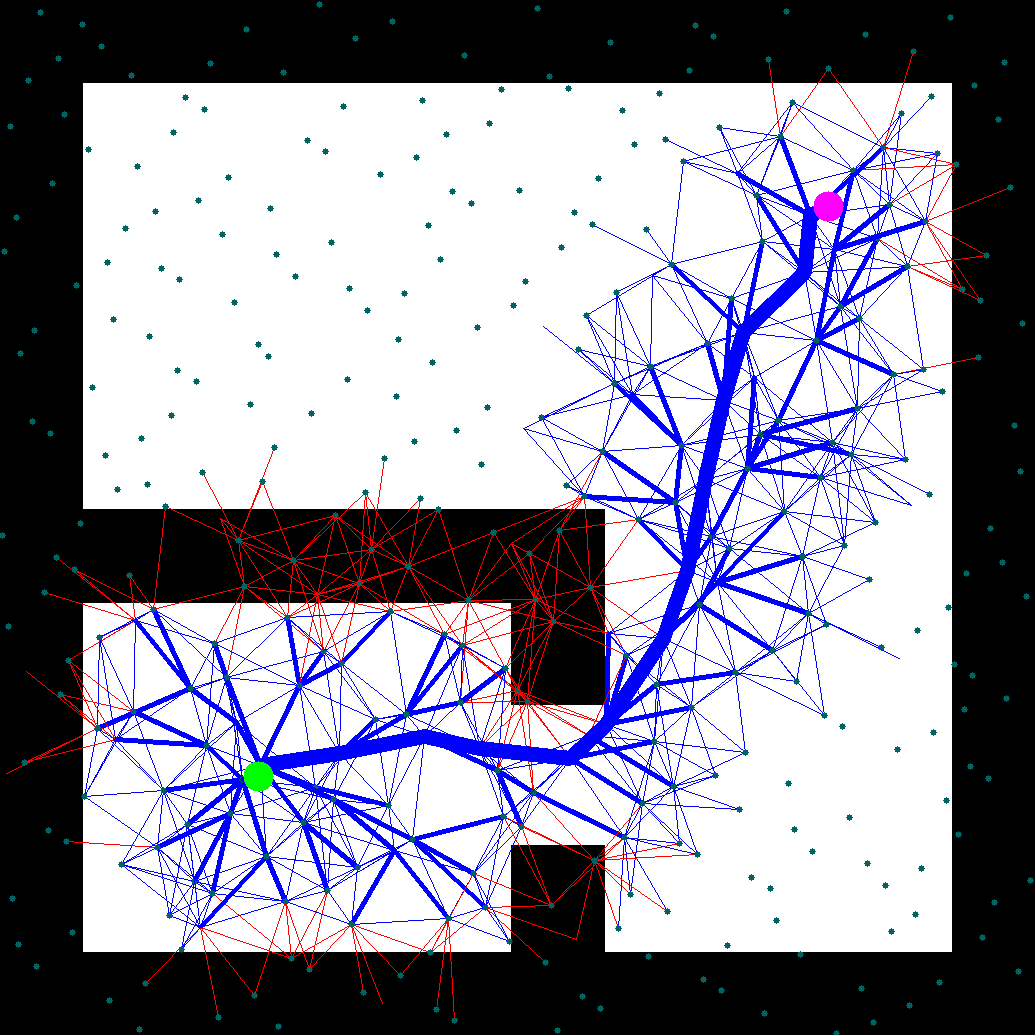}
	\end{subfigure}
	\begin{subfigure}{\myMSFigureScale\textwidth}
		\includegraphics[width=\myLineScale\linewidth]{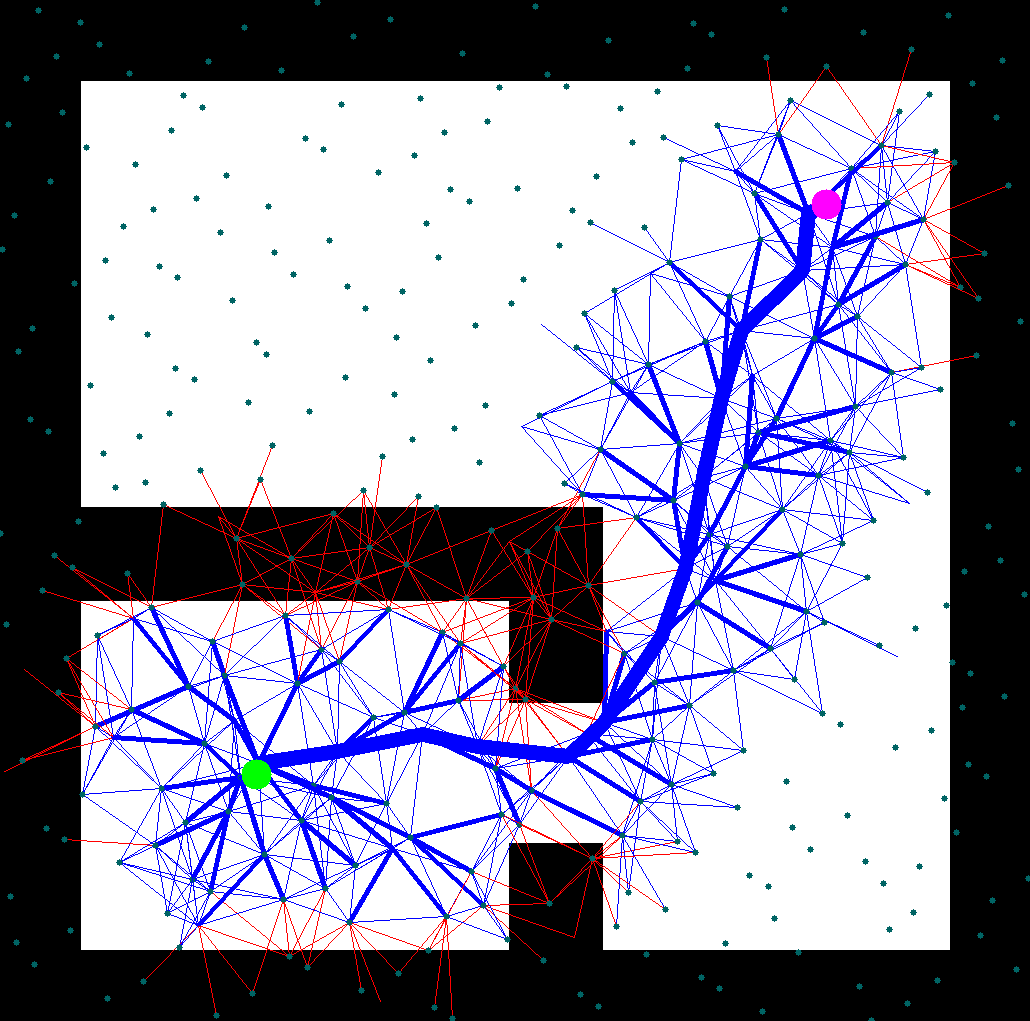}
	\end{subfigure}
	\begin{subfigure}{\myMSFigureScale\textwidth}
		\includegraphics[width=\myLineScale\linewidth]{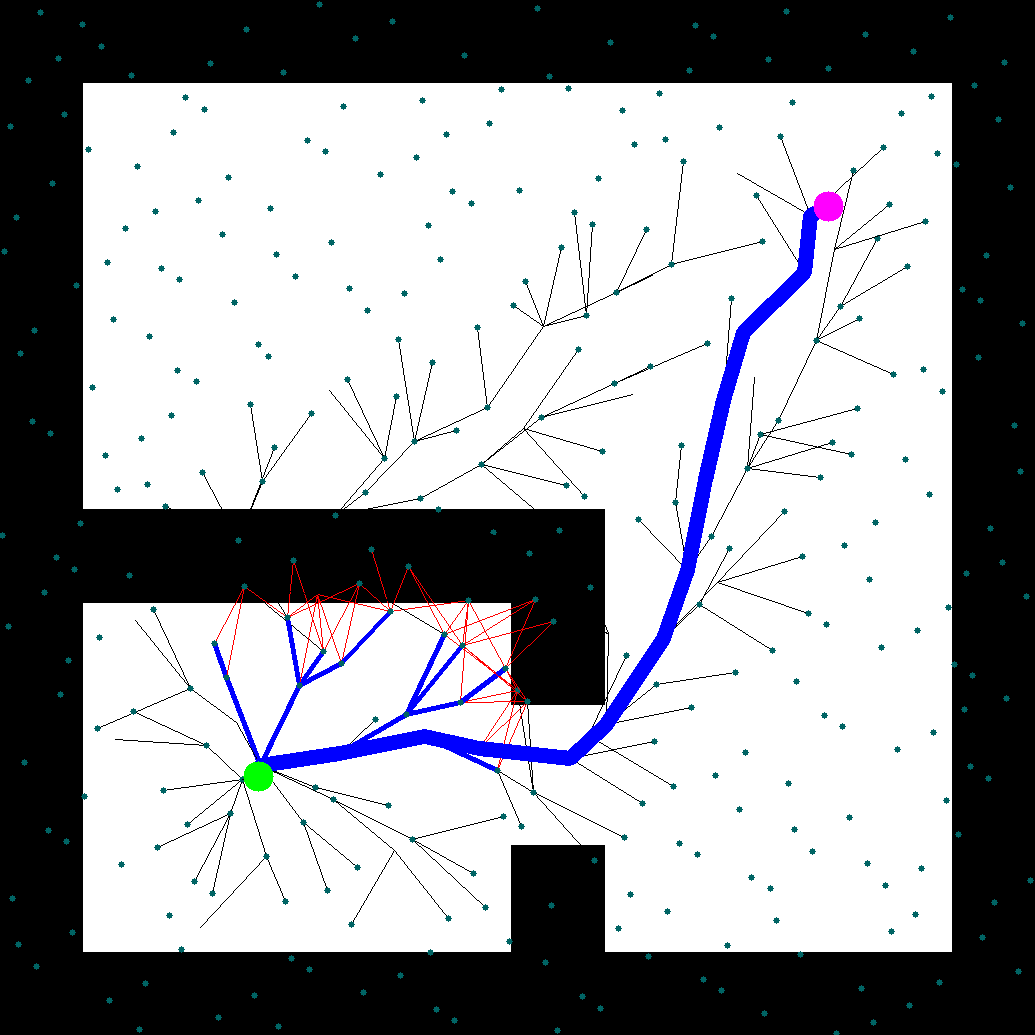}
	\end{subfigure}
	\begin{subfigure}{\myMSFigureScale\textwidth}
		\includegraphics[width=\myLineScale\linewidth]{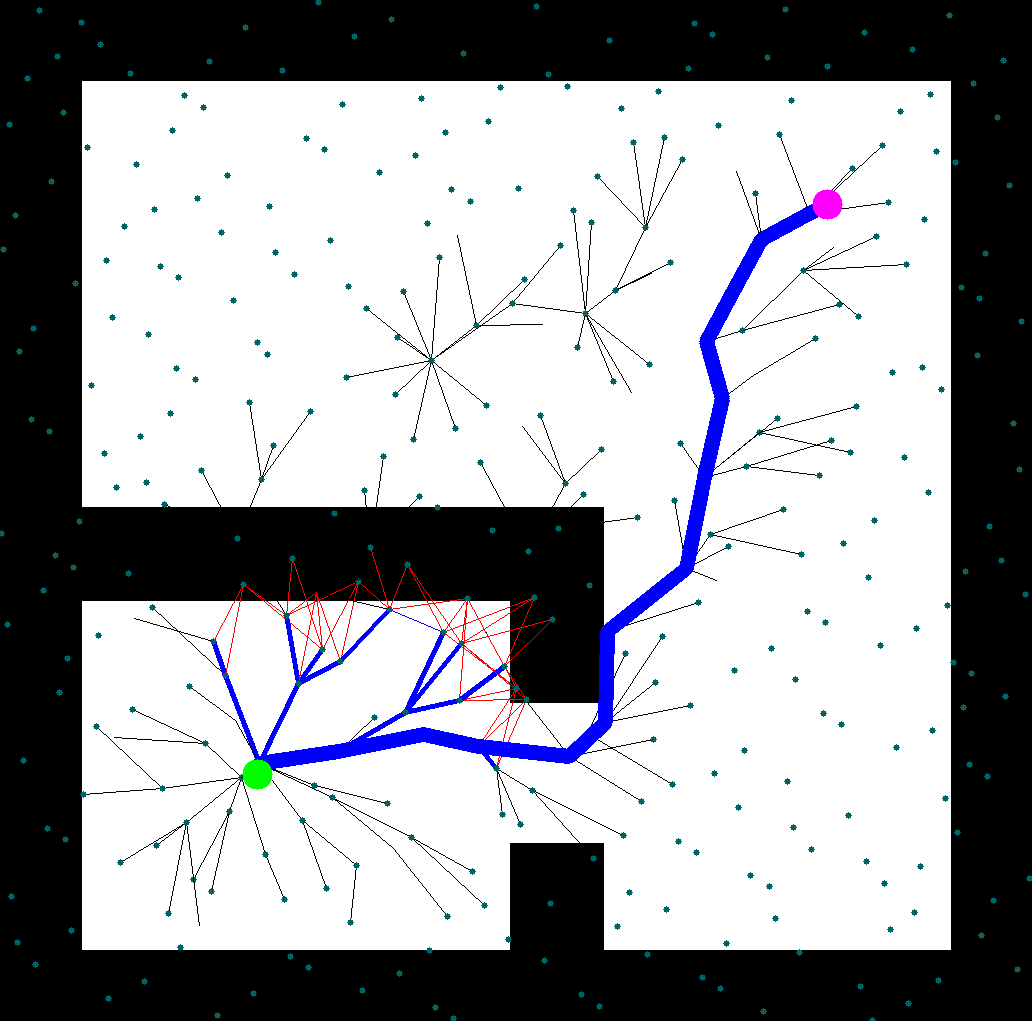}
	\end{subfigure}
	\begin{subfigure}{\myMSFigureScale\textwidth}
		\includegraphics[width=\myLineScale\linewidth]{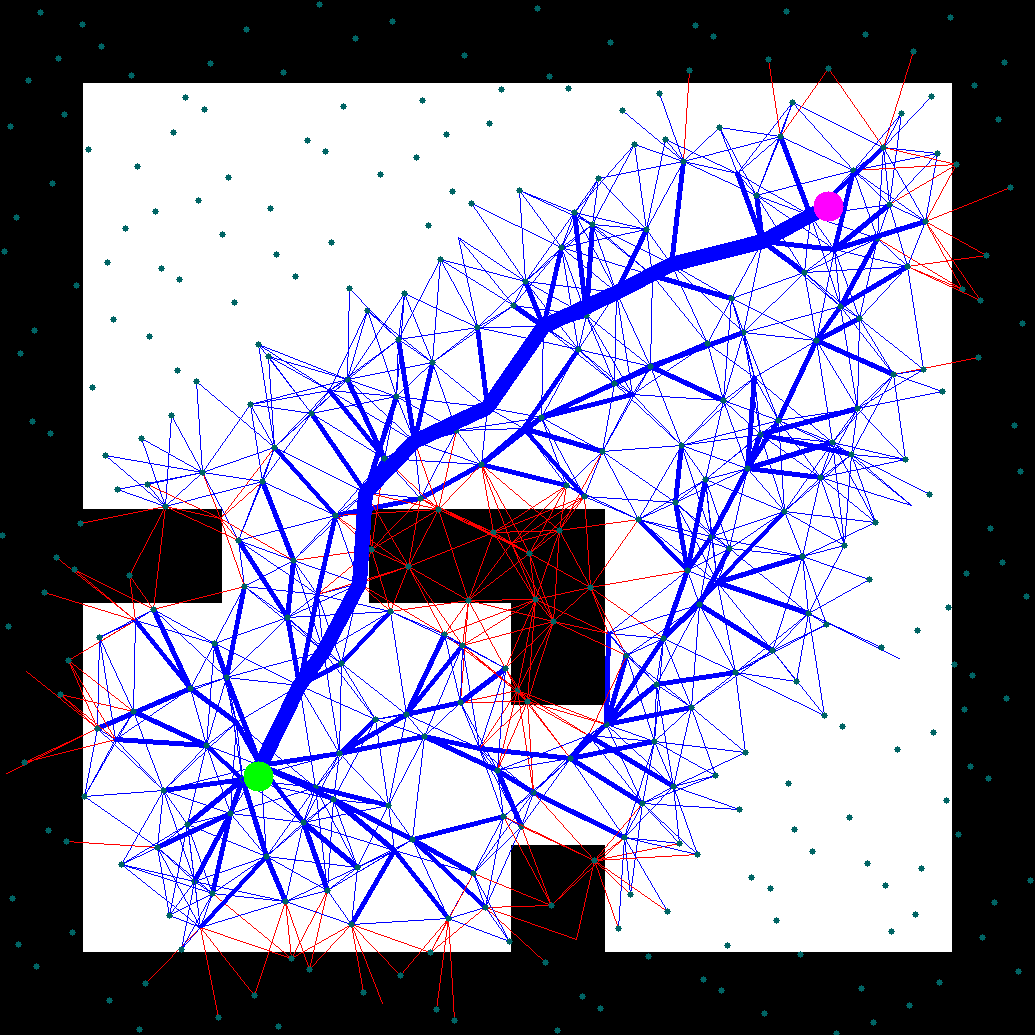}
		\caption{LPA*}
	\end{subfigure}
	\begin{subfigure}{\myMSFigureScale\textwidth}
		\includegraphics[width=\myLineScale\linewidth]{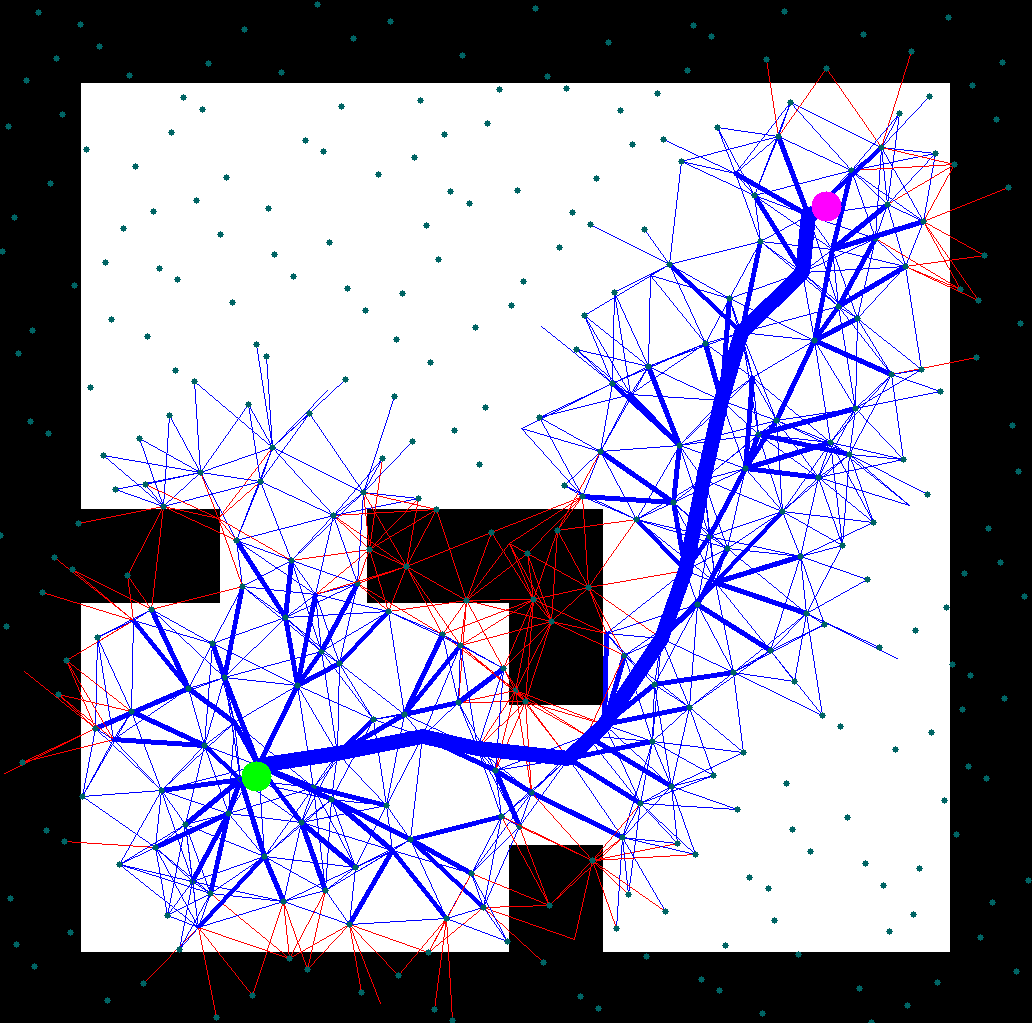}
		\caption{TLPA*(1.44)}
	\end{subfigure}
	\begin{subfigure}{\myMSFigureScale\textwidth}
		\includegraphics[width=\myLineScale\linewidth]{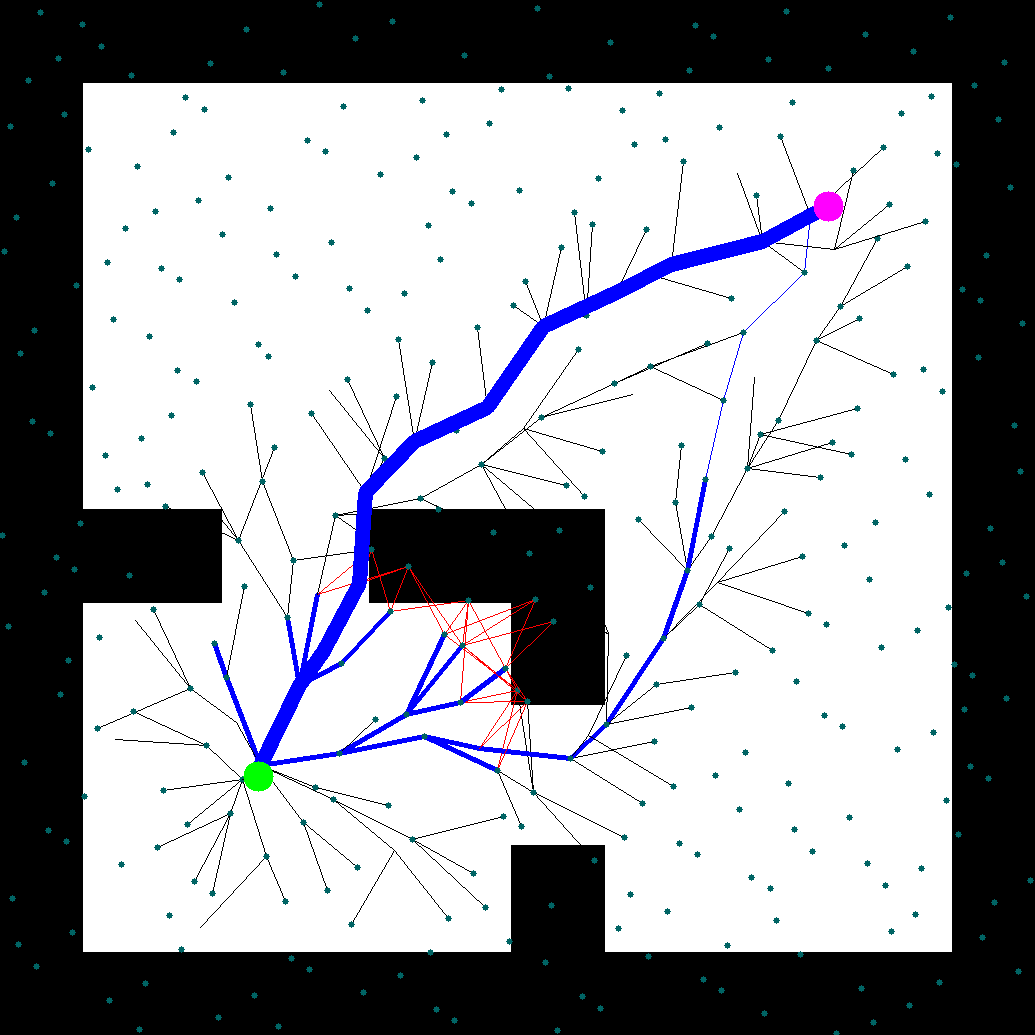}
		\caption{L-GLS}
	\end{subfigure}
	\begin{subfigure}{\myMSFigureScale\textwidth}
		\includegraphics[width=\myLineScale\linewidth]{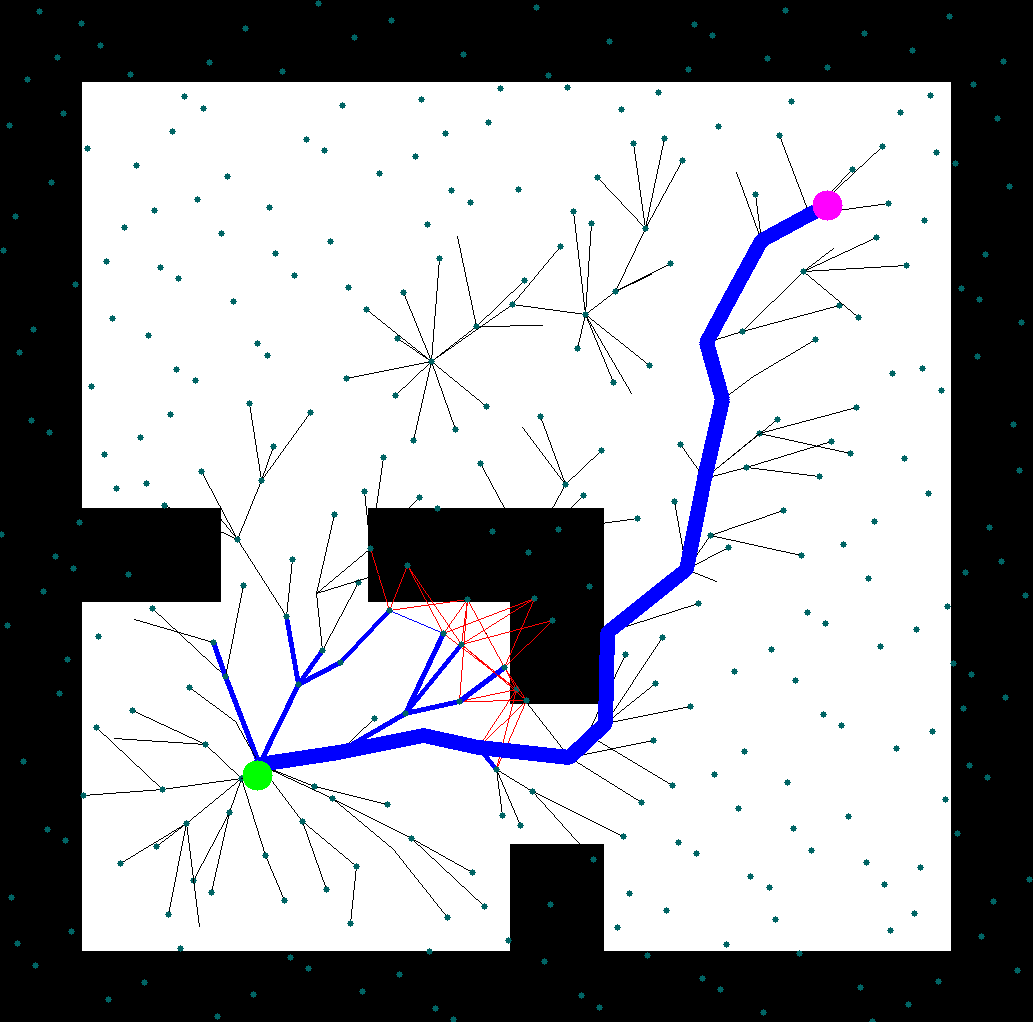}
		\caption{B-LGLS}
	\end{subfigure}
	\caption{First search (top row) and second search (bottom row) to find a bounded suboptimal path from start vertex(\tikzcircle[blue, fill=green]{2.5pt}) to goal vertex(\tikzsquare[blue, fill=magenta]{4.5pt}) per environment change with $\varepsilon_1=1.2$ and $\varepsilon_2=1.2$, from left to right: (a) LPA*, (b) TLPA*, (c) L-GLS, and (d) B-LGLS. 
		Lines(\tikzline[blue,semithick]{}\tikzline[red,semithick]{}) are the evaluated edges during the current search.  Blue and red colors represent free space and obstacles, respectively. Bold lines(\tikzline[blue,very thick]{}) are the edges belonging to the current search tree and the path is drawn with the boldest lines (\tikzline[blue,ultra thick]{}).}
	\label{lgls:f:2d_blgls}
\end{figure*}

\begin{table*}[ht]
	\centering
 	\caption{Number of edge evaluations, number of vertex expansions, and solution length for different planners over two consecutive search queries in a dynamic environment of Figure~\ref{lgls:f:2d_blgls}. }
	\begin{small}
	\begin{tabular}{lccccc}
		\toprule
		\textbf{}      & LPA* & TLPA*(1.44) & GLS& L-GLS & B-LGLS(1.2,1.2) \\ 
		\midrule
		\textbf{First Query} in scene 1    &          &       &          &        &       \\
		\# Edge Evaluation  & 826 & 826 & 66 & 66 & 68 \\ 
		\# Vertex Expansion & 58 & 58 & 2000 & 2000 & 604 \\
		Solution Length & 0.924 & 0.924 & 0.924 & 0.924 & 0.956  \\
		&          &       &          &        &       \\
		\textbf{Second Query} in scene 2    &          &       &          &        &       \\
		\# Edge Evaluation  & 267 & 93 & 35 & 14 & 0 \\ 
		\# Vertex Expansion & 23 & 1 & 871 & 315 & 1 \\
		Solution Length & 0.84 & 0.924 & 0.84 & 0.84 & 0.956 \\		         
		&          &       &          &        &       \\
		\textbf{Total}    &          &       &          &        &       \\
		\# Edge Evaluation  & 1093 & 919 & 101 & 80 & 68 \\ 
		\# Vertex Expansion & 81 & 59 & 2871 & 2315 & 605 \\
		\bottomrule
	\end{tabular}
	\end{small}
	\label{lgls:t:blgls2d}
\end{table*}

Each problem instance was solved using a graph with a fixed topology whose vertices are sampled with a Halton sequence~\cite{Halton1964} and an edge is defined for two vertices within a certain distance. 
The vertices are shown with black dots, and the evaluated edges are shown in either blue\,(feasible) or red\,(infeasible) lines. 
Unevaluated edges which belong to the lazy search trees are drawn in black thin lines in Figure~\ref{lgls:f:2d_blgls}(c) and \ref{lgls:f:2d_blgls}(d). 
Unused edges are not drawn.  
The final search trees grown from the start vertex (green dot) in the lower left toward the goal vertex (magenta dot) in the upper right are drawn with bold blue lines, and the resulting path is marked with thick blue lines, which is best viewed in color. 
From left to right, the search trees of LPA*, TLPA*, L-GLS, and B-LGLS of each instance are shown.

LPA* and L-GLS find the optimal paths for both of the problem instances, whereas TLPA* and B-LGLS do not waste computational resources to replan, as the current solution is guaranteed to be within a user-given parameter of $\varepsilon_1 \varepsilon_2= 1.44$.
Compared to TLPA*, B-LGLS finds the same solution with a significant fewer number of edge evaluations. 
Compared to L-GLS, B-LGLS saves both edge evaluations and vertex expansions.
These results are summarized in Table~\ref{lgls:t:blgls2d}. 

\subsection{Analysis of the Bounded L-GLS Algorithm} \label{b-lgls:sec:analysis}

The completeness of B-LGLS can be proven using the same argument as in Theorem~\ref{lgls:invariant:LGLScomplete}. 
Next, we prove that B-LGLS returns a bounded suboptimal solution. 
First, let us consider two useful invariants. 

\begin{restatable}{invariant}{BLGLSInvariantInflation}
	\label{lgls:invariant:BLGLSInvariantInflation}
	The lazy estimate of an edge never overestimates the true edge value by more than a factor of $\varepsilon_1$, that is, $\widetilde{w}\leq \varepsilon_1 w.$
\end{restatable}
\begin{proof}
    Since $\widetilde{w}(e)=w(e)$ for all $e\in E_\mathrm{eval}$, and $\widetilde{w}(e)=\varepsilon_1 \widehat{w}(e) \leq \varepsilon_1 w(e)$ for all $e\notin E_\mathrm{eval}$, it follows that $\widetilde{w}(e)\leq \varepsilon_1 w(e)$ for all $e\in E.$
\end{proof}

\begin{restatable}{invariant}{BLGLSInvariantTruncation}
	\label{lgls:invariant:BLGLSInvariantTruncation}
	When $\textsc{ComputePath}$ of TLPA* selects a locally overconsistent vertex $v$ for expansion, the path constrained to pass through $v$ does not overestimate the optimal solution cost by more than an $\varepsilon_2$ factor, that is, $g^\pi(v)+h(v) \leq \varepsilon_2 g^*(v_\mathrm{g}).$
\end{restatable}
\begin{proof}
    Suppose a locally overconsistent vertex $v$ is selected for expansion, that is, $g(v)\geq rhs(v).$ 
    From Lemma 6 of \cite{Aine2016} we have that $rhs(v)+h(v)\leq g^*(v_\mathrm{g})$ and from Lemma 7 of \cite{Aine2016} we have that $g^\pi(v)+h(v)\leq \varepsilon_2 (rhs(v)+h(v))$.
    Hence, the result $g^\pi(v)+h(v) \leq \varepsilon_2 g^*(v_\mathrm{g})$ follows. 
\end{proof}

Now we are ready to prove that B-LGLS is correct, that is, it returns a bounded suboptimal path when the inner loop terminates. This is stated in the next theorem. 
\begin{restatable}{theorem}{BLGLScorrect}
	\label{lgls:theorem:BLGLScorrect}
	B-LGLS finds a path that is not longer than the shortest path by more than a factor of $\varepsilon_1 \varepsilon_2$ with respect to the current graph when the inner loop (Lines~\ref{blgls:algo:line:mainsearchbegin}-\ref{blgls:algo:line:mainsearchend}) terminates.
\end{restatable}

\begin{proof}
	Let $\pi^*$ be the optimal path with respect to $w$ in the current graph, that is, $w(\pi^*)=\min_{\pi\in\Pi} w(\pi)$, where $\Pi$ is the set of all paths from $v_\mathrm{s}$ to $v_\mathrm{g}$. B-LGLS terminates its inner-loop when $v_g \in \overline{\pi}$ and $\overline{\pi}\subseteq E_\mathrm{eval},$ where $\overline{\pi}$ is the output subpath of \textsc{ComputePath(Event)}.
	Then, we have 
	\begin{equation}
		\begin{aligned}
		\widetilde{w}(\overline{\pi}) & = \sum_{e\in\overline{\pi}} \widetilde{w}(e)
		\leq \varepsilon_2 \sum_{e'\in\pi^*} \widetilde{w}(e') 
		\\ & \leq \varepsilon_2 \sum_{e'\in\pi^*} \varepsilon_1 w(e') 
		= \varepsilon_1 \varepsilon_2 w(\pi^*),
		\end{aligned}
	\end{equation}
	where the first inequality holds by Invariant~\ref{lgls:invariant:BLGLSInvariantTruncation}, and the second inequality follows by Invariant~\ref{lgls:invariant:BLGLSInvariantInflation}. 
	Hence, $\widetilde{w}(\overline{\pi}) \leq \varepsilon_1 \varepsilon_2 w(\pi^*)$ and since $\overline{\pi} \subseteq E_\mathrm{eval},$ we have $w(\overline{\pi}) = \widetilde{w}(\overline{\pi}) \leq \varepsilon_1 \varepsilon_2 w(\pi^*).$ Therefore, the path cost of $\overline{\pi}$ is not greater than the optimal path cost by a factor $\varepsilon_1 \varepsilon_2.$
\end{proof}



\section{Generalized Dynamic Search (GD*)} \label{sec:generalized_dstar}

Mobile robots usually operate in environments that are only partially known or continuously changing. Therefore, they need to quickly replan, as the environment changes or when previously unknown parts of it become known. 
D*-Lite is an algorithm that is based on similar ideas as LPA* and uses previous search information to quickly replan while the robot is moving. 
Hence, D*-Lite  has been used in many practical applications. 
It uses the same underlying tree as LPA* to repeatedly compute the shortest path to the goal. 
Compared to LPA*, the D*-Lite algorithm is primarily aimed to handle environment changes while the robot moves along the previously calculated optimal solution.  

In this section, we detail the new algorithm, called Generalized D* (GD*), 
which generalizes D*-Lite~\cite{Koenig2005} in the lazy search framework.

The Generalized D* (GD*) algorithm consists of two loops: an inner loop that conducts the actual search given the current graph, and the outer loop that is mainly responsible for perceiving the environment changes and for updating the graph.  
Note that, similarly to D*-Lite, the search direction in the GD* algorithm is reversed, and the search is started from the goal vertex, progressing towards the start vertex. Thus, the $rhs(v)$ and $g(v)$ values are estimated distances calculated from the vertex $v$ to the goal vertex. Similarly, the heuristic values are computed from the start vertex instead of the goal.  

In the inner loop, the candidate shortest path is computed from the goal vertex to the start vertex based on the current heuristic weight values of the GD* tree.  Once a candidate solution is found, the first unevaluated edge along the shortest subpath is then evaluated. 
If the evaluation results in an inconsistency (e.g., collision with an obstacle), then GD* gets updated and returns a new candidate shortest path. 
Finally, if all the edges along the candidate path are evaluated and they produce no inconsistency, the candidate shortest path is the actual shortest path with respect to the current graph with the true edge values.

Once the shortest path is found, the robot executes the optimal plan and moves along a small segment of the shortest path in the outer loop.  
The algorithm then scans for graph changes.  In case a change is perceived, instead of evaluating the changed edges, the algorithm assigns lazy estimates using the admissible heuristic function for those edges.  
This is to ensure that the lazy estimate of the path cost does not overestimate the optimal path cost.  
Then the inner loop finds the new optimal path with respect to the updated graph.  
Therefore, only a subset of the changed edges that could be on the shortest path are evaluated. Algorithm~\ref{alg:GD} lays out the main procedures for GD*.

\begin{algorithm}
\begin{small}
\caption{GD* ($G, v_s, v_g$)} \label{alg:GD}
\begin{algorithmic}[1]
\Procedure{CalculateKey}{$v$} \textbf{return} 
    \State $\left[ \min(g(v), rhs(v)) + h(v, v_s) + km,  \min(g(v), rhs(v)) \right]$ \label{alg:gd:calckey}
\EndProcedure

\Procedure{UpdateVertex}{$v$}
    \If {$v \neq v_g$}
        \State  $bp(v) \gets \arg\min_{u \in succ(v)} g(u) + \overline{w}(u,v);$
        \State $rhs(v) \gets g(bp(v)) + \overline{w}(bp(v),v);$
    \EndIf
    \If {$v \in Q$}   Q.Remove($v$);
    \EndIf
    
    \If {$g(v) \neq rhs(v)$} \State  Q.Insert(($v$, \textproc{CalculateKey}($v$)));
    \EndIf
\EndProcedure

\Procedure{ComputeShortestPath}{$event$}
    \While {\textproc{Q.TopKey}$<$\textproc{CalculateKey}($v_g$) \textproc{ or } \\ $g(v_g) \neq rhs(v_g)$}
        \State $k_{old}$ $\gets$ \textproc{Q.TopKey};
        \State u $\gets$ \textproc{Q.Pop()};
        \If{$k_{old}<$ \textproc{CalculateKey($u$)}} \label{alg:gd:keycheck}
            \State \textproc{Q.Insert($u$, CalculateKey}($u$));
        \ElsIf{$g(u)>rhs(u)$}
            \State $g(u) \gets rhs(u)$;
            \If{\textproc{Event}$(u)$ is triggered}  \label{alg:gd:event_trigger}
                \State \Return path from $v_g$ to $u$
            \EndIf
            \ForAll{$v \in Pred(u)$}
                \textproc{UpdateVertex($v$);} \EndFor
        \Else{}
            \State $g(u) \gets \infty$;
            \ForAll{$v \in Pred(u) \cup u$}
                \State \textproc{UpdateVertex($v$);}  \EndFor
        \EndIf
    \EndWhile
\EndProcedure

\Procedure{EvaluateEdges}{$\overline{\pi}$}
    \ForAll{$e \in \overline{\pi}$}
        \If{$e \not\in E_{eval}$}
            \State $\overline{w}(e) \leftarrow w(e)$;
            \State $E_{eval} \leftarrow E_{eval} \cup \{e\}$;
            \If{$\overline{w}(e) \neq \widehat{w}(e)$} \Return $e$
            \EndIf
        \EndIf
    \EndFor
\EndProcedure

\Procedure{Main}{$ $}
    \ForAll{$e \in E$} $\overline{w}(e) \leftarrow \widehat{w}(e)$ \EndFor
    
    \State $km \gets 0$;
    \State $v_\mathrm{last} = v_\mathrm{current}$;
    \State $E_{eval} \leftarrow \emptyset$;
    \State $rhs(v_g) = 0$;
    \State \textproc{UpdateVertex($v_g$)};
    
    \While{$v_\mathrm{current} \neq v_g$}
        \Repeat
            \State $\overline{\pi} \gets $ \textproc{ComputeShortestPath($event$)}; \label{alg:gd:computeshortest}
            \State $(u,v) \gets$   \textproc{EvaluateEdges($\overline{\pi}$)}; \label{alg:gd:evaluateedges}
            \State \textproc{UpdateVertex($v$)}; \label{alg:gd:edgeweight_diff}
        \Until{$v_s \in \overline{\pi}$ and $\overline{\pi} \subseteq E_{eval}$}
        \State $v_\mathrm{current} \gets bp(v_\mathrm{current})$;
        \State Move to $v_\mathrm{current}$; \label{alg:gd:move_robot}
        \State $km \gets km + h(v_\mathrm{current}, v_\mathrm{last})$; \label{alg:gd:kmincr}
        \State $v_\mathrm{last} \gets v_\mathrm{current}$;
        
        \State Perceive the edge changes in $E$;
        \ForAll{$e=(u,v)$ with changed edge cost} 
            \State $\overline{w}(e) \leftarrow \widehat{w}(e)$; \label{alg:gd:perceive_changes}
            \State $E_{eval} \leftarrow E_{eval} \backslash \{e\}$;
            \State \textproc{UpdateVertex($v$);}
        \EndFor
    \EndWhile
\EndProcedure

\end{algorithmic}
\end{small}
\end{algorithm}

\subsection{Details of the GD* Algorithm}

The procedures for the GD* are very similar to the ones of LPA* described in Section
``Lifelong Lazy Search'' and are summarized briefly below.
Given a graph, the $g$-values of all the vertices are set to $\infty$ and the lazy edge values are set via an admissible heuristic.
As the robot moves along the path, the priorities in the queue must be recalculated to maintain the correct order for expansion. 
D*-Lite avoids this repeated reordering via adjusting the key values using an additional $km$ variable (Line~\ref{alg:gd:calckey}). 
The $km$ variable alters the priorities of the new vertices when they are added to the queue. 
Secondly, when a vertex is selected for expansion, its key gets recomputed (Line~\ref{alg:gd:keycheck}) and is only expanded if its key value is indeed a lower bound of the queue. 
Otherwise, the vertex gets reinserted into the queue with the recomputed key value. 
Note that the $km$ value gets incremented each time the robot moves (Line~\ref{alg:gd:kmincr}).
%

The search is initialized by setting the $rhs$-value of the goal vertex to zero and inserting it in the priority queue. Furthermore, the $km$ value is set to zero, as the robot has not moved yet.  
The main search loop of the algorithm (Line~\ref{alg:gd:computeshortest}) expands the search tree via the procedure \textproc{ComputeShortestPath}. \textproc{ComputeShortestPath} continues expanding the search tree via expanding frontier vertices until an event is triggered after that particular frontier vertex has just became consistent. 
Similarly to L-GLS (Algorithm \ref{lgls:a:lgls}), the edges of the subpath from the goal to this frontier vertex are returned for evaluation (Line \ref{alg:gd:evaluateedges}).

Similarly to L-GLS (Algorithm~\ref{lgls:a:lgls}), all the evaluated edges then get evaluated, and all lazy estimates get updated with their true weight values in \textproc{EvaluateEdges}. 
If the evaluation results in a different weight value compared to the previous lazy estimate for a particular edge, then this change is reflected via \textproc{UpdateVertex} (Line~\ref{alg:gd:edgeweight_diff}).
The vertex that caused the inconsistency is then updated and the changes from the inconsistency are handled via propagating them through the lazy D*-Lite tree \textproc{ComputeShortestPath} until interrupted via a triggered \textproc{event} (Line~\ref{alg:gd:event_trigger}). 
Once a path is found, the goal and all the edges get evaluated, and the path found will be the optimal path to the goal, given the current graph. 
Similarly to D*-Lite, the robot then progresses towards the goal by traversing the optimal path until it reaches the successor (Line~\ref{alg:gd:move_robot}) of the current vertex ($v_\mathrm{current}$).  
The $km$-value also gets updated via the heuristic to reflect the new start position for the algorithm. 
The algorithm then scans the graph for changes and repeats the same procedures in case of a change.

The procedure \textproc{UpdateVertex} is similar to that of the regular D*-Lite~\cite{Koenig2005}, the main difference being that the lazy estimate of the edge value ($\overline{w}$) is used for updating the $rhs$-value of the current vertex. 
Notice that this is done to avoid evaluation of the irrelevant edges. 
The backpointer of the current vertex is set to the minimizing successor using $\overline{w}$. 
As the last step, and if the vertex is inconsistent, the priority of the current vertex is updated based on the newly calculated key value using \textproc{CalculateKey}.

Note that similarly to the L-GLS  algorithm (Algorithm \ref{lgls:a:lgls}) the choice of the \textproc{Event} function determines the trade-off between edge evaluations and vertex expansions, as explained in 
Section ``Lookahead Variation'' later on.
Also, the procedures of the GD* algorithm are identical to the ones of the D*-Lite algorithm except for three places. 
First, the GD* algorithm uses the lazy weight function instead of the actual weight function during the \textproc{UpdateVertex} procedure. 
Second, the GD* algorithm assigns admissible heuristic weights to the edges instead of the actual weights when their values are changed (Line~\ref{alg:gd:perceive_changes}). 
Lastly, when a vertex is made overconsistent in \textproc{ComputeShortestPath}, the algorithm checks whether \textproc{Event} is triggered to continue or stop propagating the inconsistency information to the predecessors (Line~\ref{alg:gd:event_trigger}).

\subsection{Illustrative Example}

We visualize the progression of the D*-Lite and the GD* algorithms with the infinite-step lookahead, as shown in Figure \ref{ds:f:gd-2d}.  
Note that the same implicit graph is used for both algorithms with the only change being the edge weights due to environment changes. 
Similarly to LPA*, both algorithms find the shortest path during replanning using the previous search tree information.

\begin{figure*}[ht]
	\centering
	\begin{subfigure}{\myMSFigureScalend\textwidth}
		\includegraphics[width=\myLineScale\linewidth]{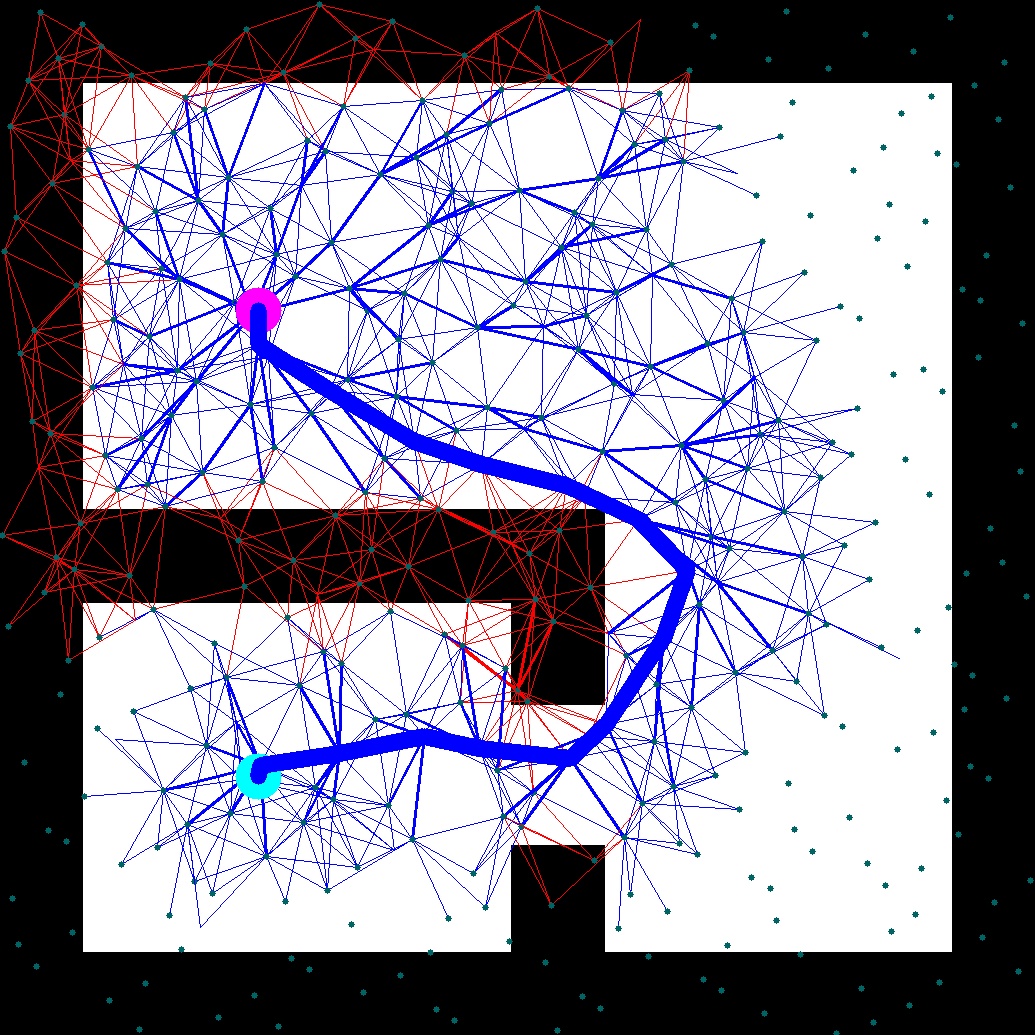}
	\end{subfigure}
	\begin{subfigure}{\myMSFigureScalend\textwidth}
		\includegraphics[width=\myLineScale\linewidth]{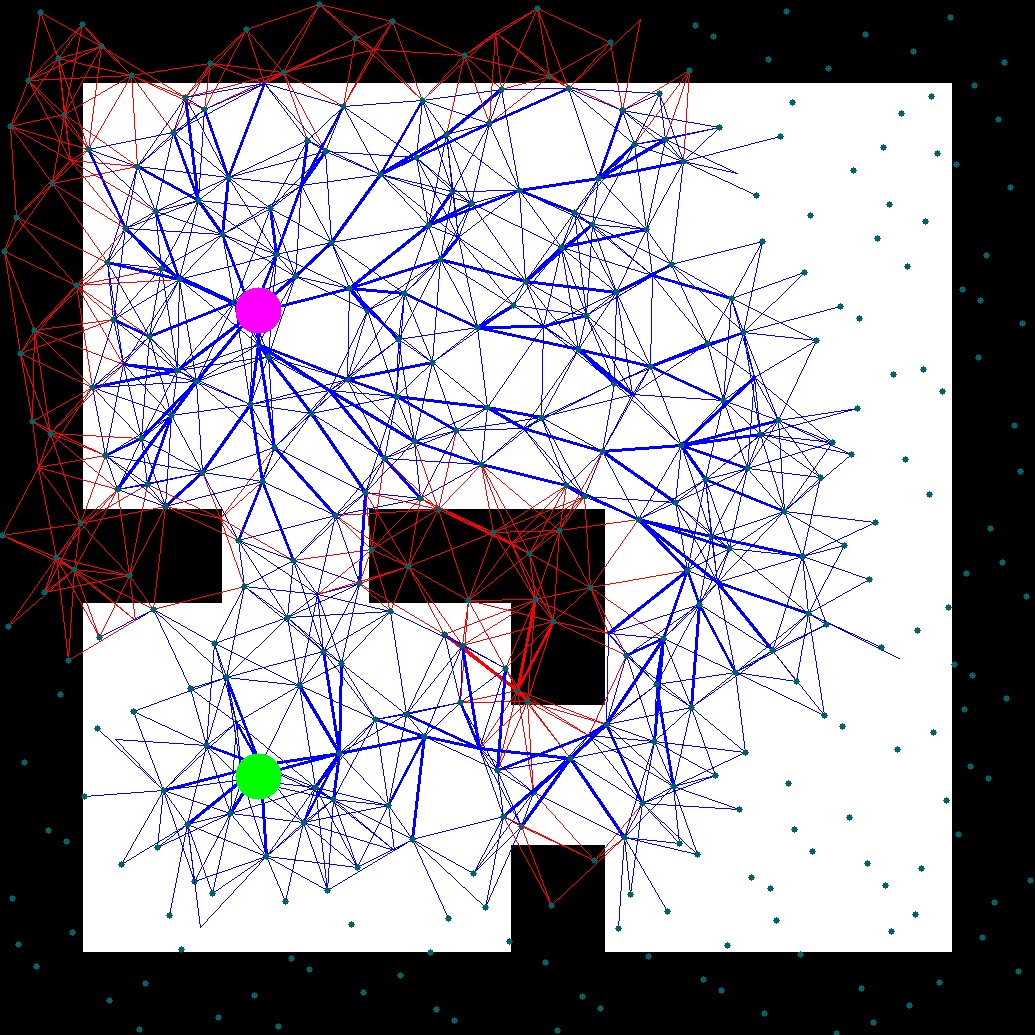}
	\end{subfigure}
	\begin{subfigure}{\myMSFigureScalend\textwidth}
		\includegraphics[width=\myLineScale\linewidth]{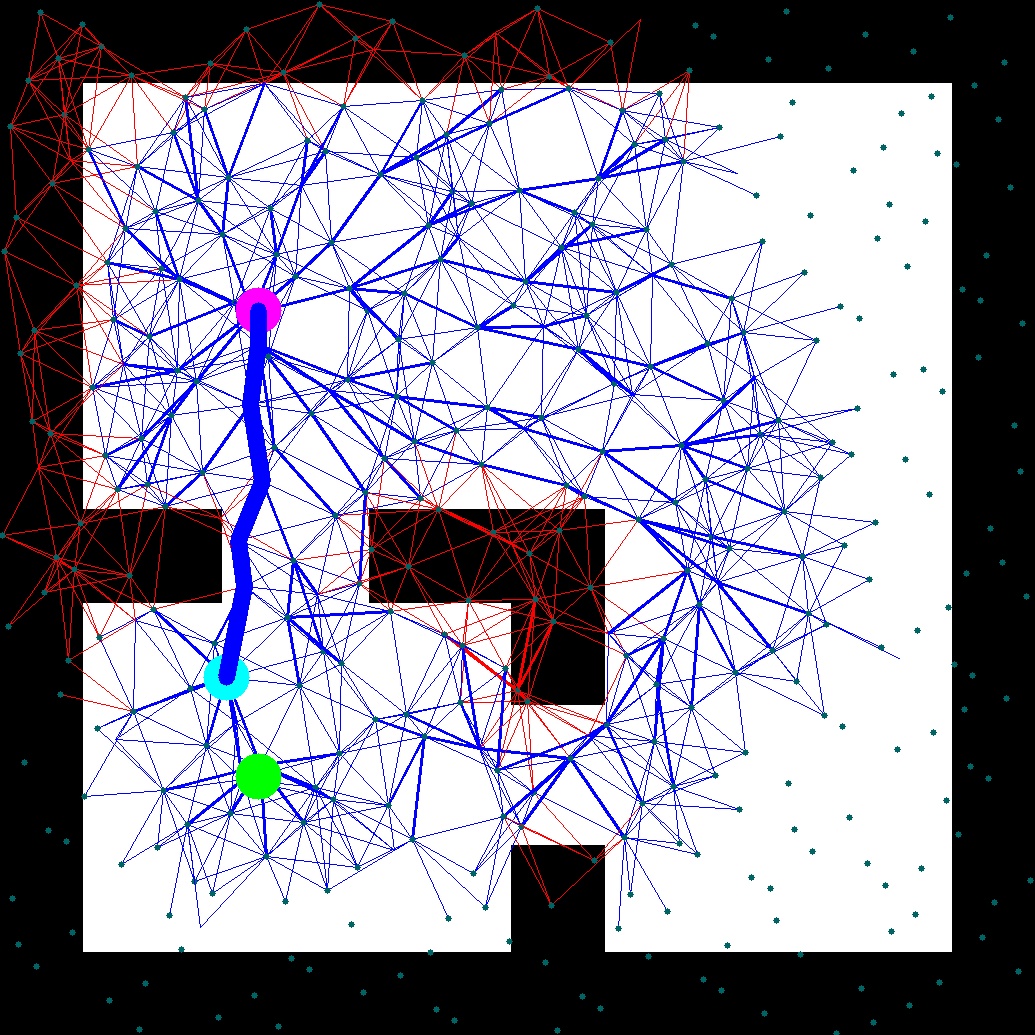}
	\end{subfigure}
	
	\medskip
	\begin{subfigure}{\myMSFigureScalend\textwidth}
		\includegraphics[width=\myLineScale\linewidth]{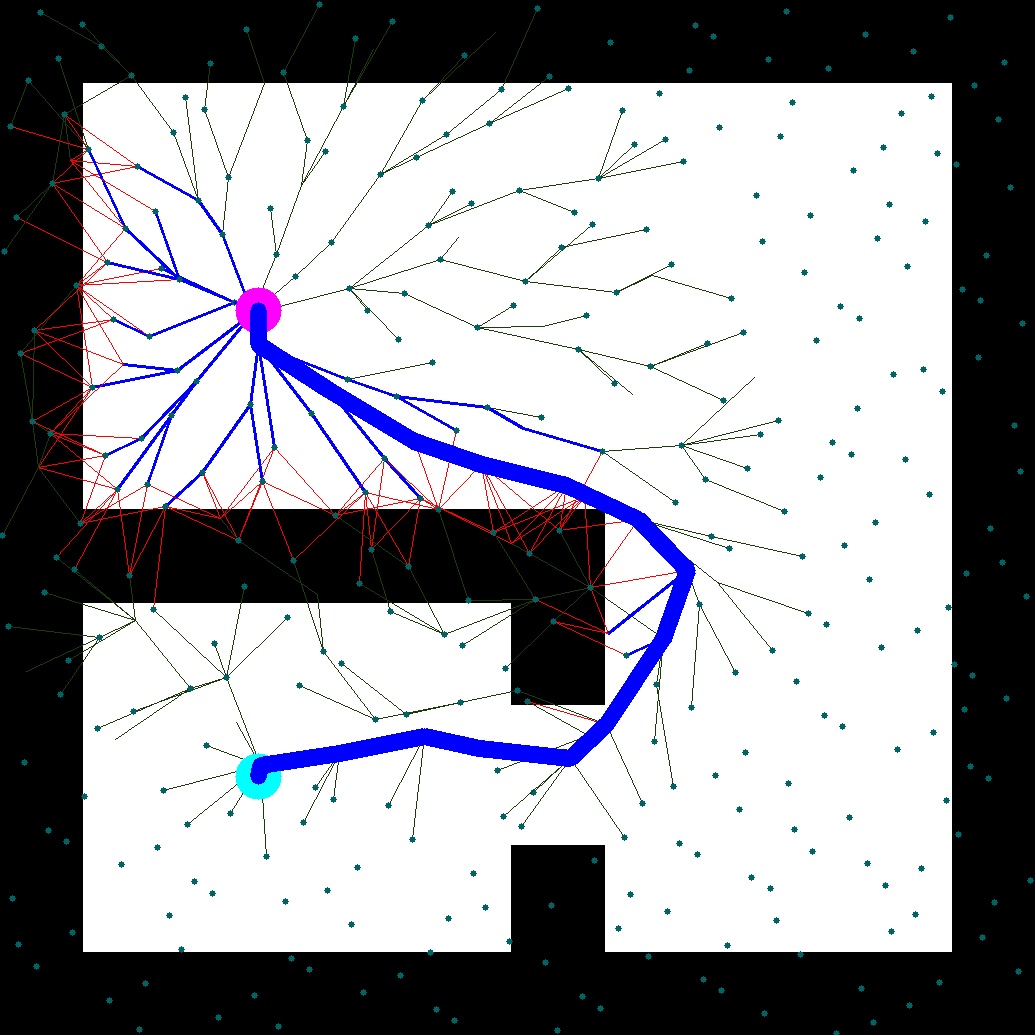}
		\caption{First}
	\end{subfigure}
	\begin{subfigure}{\myMSFigureScalend\textwidth}
		\includegraphics[width=\myLineScale\linewidth]{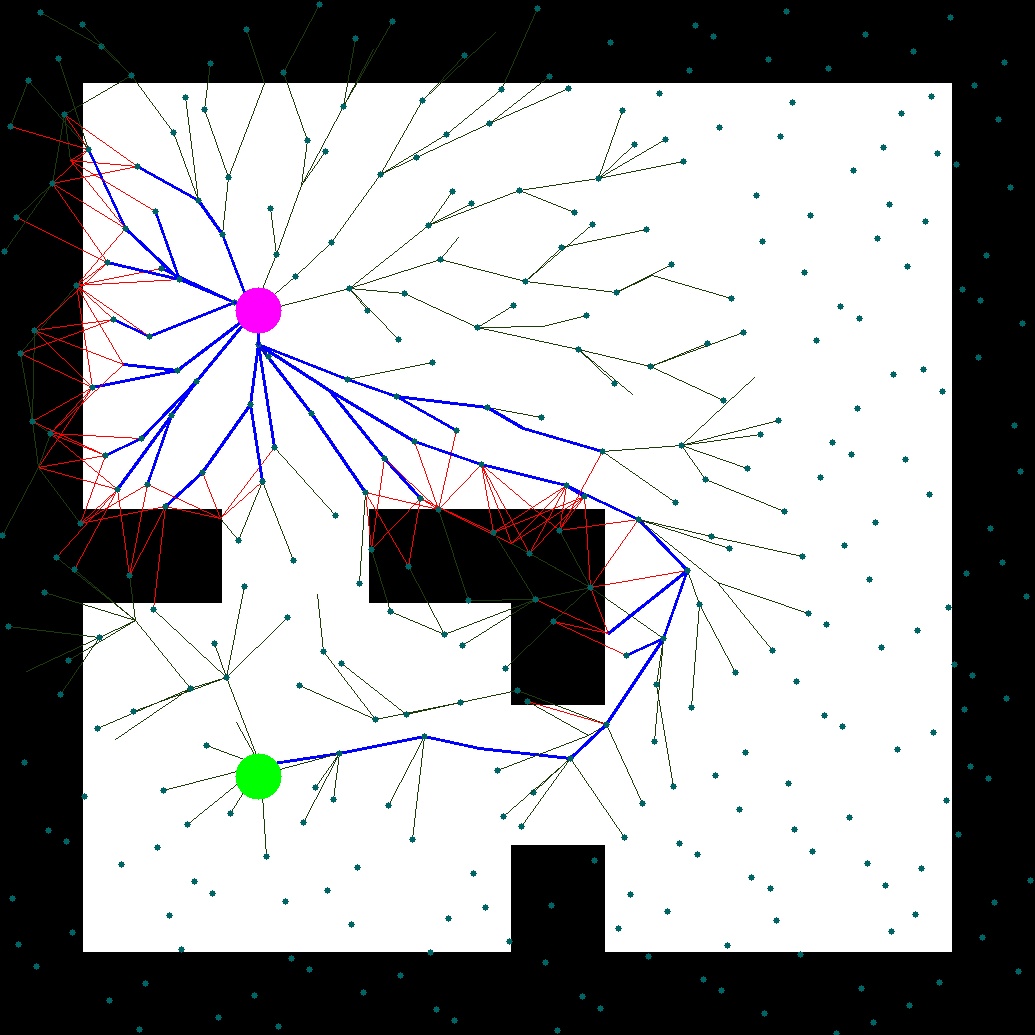}
		\caption{Env. Change}
	\end{subfigure}
	\begin{subfigure}{\myMSFigureScalend\textwidth}
		\includegraphics[width=\myLineScale\linewidth]{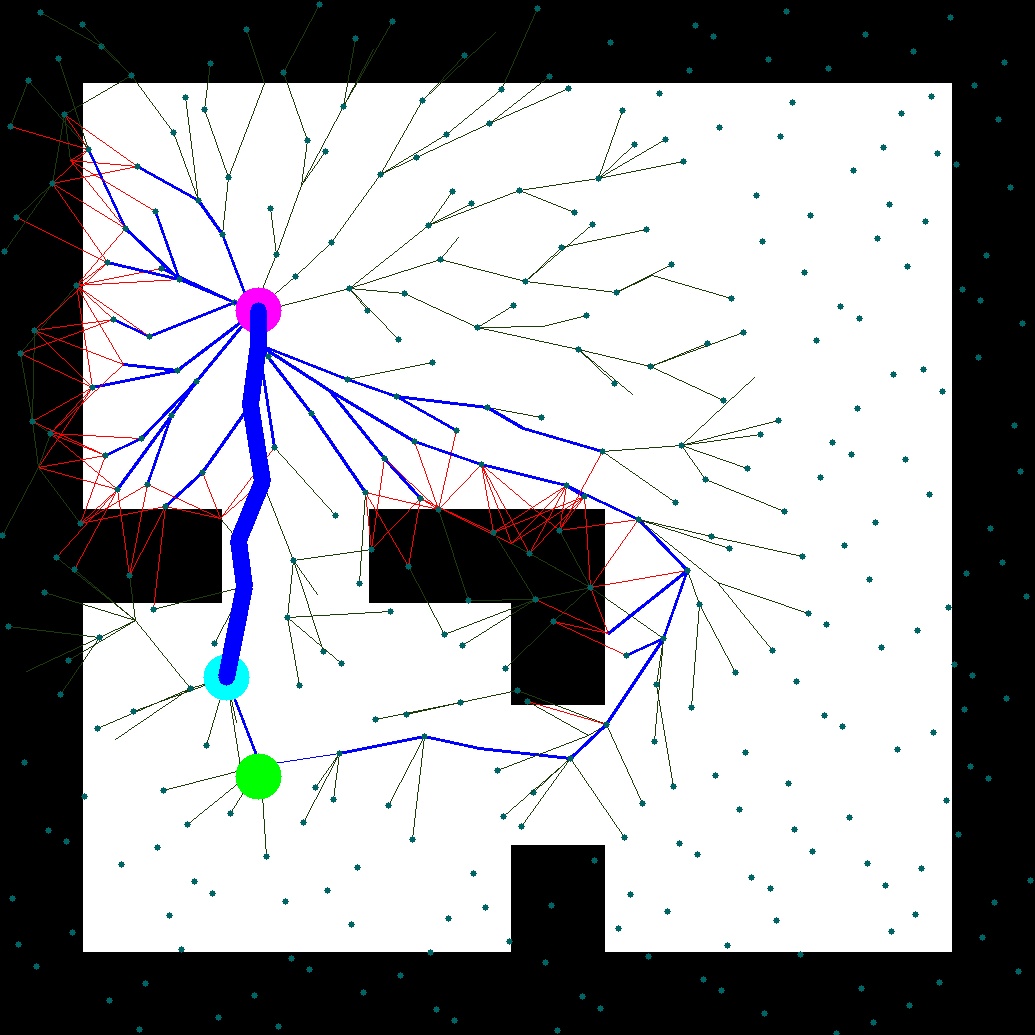}
		\caption{Second}
	\end{subfigure}
	\caption{D*-Lite(top row) and GD* with $\infty$-lookahead (bottom row) search results to find the shortest path from start vertex(\tikzcircle[blue, fill=green]{2.5pt}) to goal vertex(\tikzcircle[blue, fill=magenta]{2.5pt}) per environment change, from left to right: (a) first search, (b) environment change, (c) second search. Current robot position is shown with \tikzcircle[blue, fill=cyan]{2.5pt}.
	Lines(\tikzline[blue,semithick]{}\tikzline[red,semithick]{}) are the evaluated edges during the current search. Bold lines(\tikzline[blue,very thick]{}) are the edges belonging to the current search tree. Blue and red represents free and obstacle, respectively.}
	\label{ds:f:gd-2d}
\end{figure*}

The first search of D*-Lite is equivalent to the A* search with reversed order (swapping start and goal). The D*-Lite algorithm evaluates 1255 edges and expands 113 vertices, whereas GD* only evaluates 154 edges and expands 3305 vertices.
After finding the optimal path, the agent traverses the first segment of the found path, at which point an opening appears, making the path no longer optimal.  The environment change causes both D*-Lite and GD* to find the optimal path via replanning. Using the previous search tree information, the D*-Lite algorithm expands 7 vertices and evaluates 82 edges when replanning, while GD* only evaluates 4 edges and expands 18 vertices. Notice that the GD* algorithm finds the same optimal solution with almost an order of magnitude fewer edge evaluations for this simple 2D example.
These results are summarized in Table~\ref{ds:t:gd-2d}.

\begin{table}[ht]
	\centering
 	\caption{Number of edge evaluations and number of vertex expansions for different planners over two consecutive search queries in a dynamic environment of Figure~\ref{ds:f:gd-2d}. }
	\begin{small}
	\begin{tabular}{lcc}
		\toprule
		\textbf{}      & D*-Lite & GD* \\ 
		\midrule
		\textbf{First Query}    &          &   \\
		\# Edge Evaluation  & 1255 & 154  \\ 
		\# Vertex Expansion & 113 & 3305  \\
		&          &      \\
		\textbf{Second Query}    &          &  \\
		\# Edge Evaluation  & 82 & 4  \\ 
		\# Vertex Expansion & 7 & 18 \\	         
		&          &      \\
  		\textbf{Total}     &          &  \\
		\# Edge Evaluation  & 1347 & 158 \\ 
		\# Vertex Expansion & 120 & 3323 \\	         
		\bottomrule
	\end{tabular}
	\end{small}
	\label{ds:t:gd-2d}
\end{table}

\section{Generalized Bounded Dynamic Search (B-GD*)} \label{sec:generalized_bounded_dstar}

As mentioned earlier, using an inflated heuristic can help find an initial suboptimal solution to the goal faster. 
Note that this solution is suboptimal due to the heuristic function being inadmissible. 
Similarly to the B-LGLS algorithm, 
in this section we combine the idea of an inflated heuristic with truncation rules to introduce a bounded version of the Generalized D* (B-GD*). 
We will first discuss the lazy version of Truncated D*-Lite (TD*)~\cite{Aine2016}, and then include the inflated heuristic to obtain the bounded lazy D*.  

\subsection{Generalized TD*}

The truncation rules described for the Lazy-TLPA* search tree search
can be easily applied to D*-Lite. 
Using truncation, TD*~\cite{Aine2016} terminates the solution early once the solution cost during the replanning is found to be within the chosen bounds of the optimal cost. 
Similarly to TLPA*, TD* having a reversed search direction, also uses an extra cost-to-come value $g^\pi(s)$ along with the current path to the goal vertex for each vertex $\pi(s)$, in order to check for possible truncation. 
The truncation conditions are checked in \textproc{ComputeShortestPath} (Lines \ref{bgd:a:trun1} and \ref{bgd:a:trun2}) for terminating the propagation of inconsistency of a vertex early.  Furthermore, the vertices are only updated if they are not truncated in the \textproc{UpdateVertex}.  The auxiliary routines for TD* are identical to those of the TLPA* (Algorithm \ref{lgls:a:tlpastaraux}) except for the reversed search direction.  The rest of the TD* algorithm procedures remain the same as D*-Lite.  

\subsection{Bounded Generalized D*}

Bounded-GD* (B-GD*) is similar to the GD* except that a lazy TD* tree is used as the underlying search tree instead of lazy D*-Lite and that the heuristic edge values are inflated. 
Algorithm~\ref{alg:B-GD} outlines the B- GD* with the differences with the GD* version (Algorithm~\ref{alg:GD}) outlined in blue. Note the procedure \textproc{EvaluateEdges} for B-GD* algorithm is the same as the one in the B-LGLS (Algorithm \ref{lgls:a:blgls}) and hence skipped.

\begin{algorithm}
\begin{small}
\caption{Bounded-GD* ($G, v_s, v_g$)} \label{alg:B-GD}
\begin{algorithmic}[1]
\Procedure{CalculateKey}{$v$}
    \State  \textbf{return}  $[\min(g(v), rhs(v)) + h(v, v_s) + km,$ \\ $\min(g(v), rhs(v))]$ 
\EndProcedure

\Procedure{UpdateVertex}{$v$}
    \If {$v \neq v_g$}
        \State  $bp(v) \gets \arg\min_{u \in succ(v)} g(u) + \overline{w}(u,v);$
        \State $rhs(v) \gets g(bp(v)) + \overline{w}(bp(v),v);$
    \EndIf
    \If {$v \in Q$}   Q.Remove($v$)
    \EndIf
    
    \If {$g(v) \neq rhs(v)$ and \textcolor{blue}{$v \not\in$ TRUNCATED}}
        \State  Q.Insert(($v$, \textproc{CalculateKey}($v$)));
    \EndIf
\EndProcedure

\Procedure{ComputeShortestPath}{$event$}
    \While {\textproc{Q.TopKey}$<$\textproc{CalculateKey}($v_g$) \textproc{ or } \\ $g(v_g) \neq rhs(v_g)$}
        \State $k_{old}$ $\gets$ \textproc{Q.TopKey};
        \State u $\gets$ \textproc{Q.Pop()};
        \If{$k_{old}<$ \textproc{CalculateKey($u$)}}
            \State \textproc{Q.Insert($u$, CalculateKey}($u$));
        \Else{}
            \State \textcolor{blue}{$g^{\pi}(v_s) \gets \textproc{ComputeGPI($v_s$)}$};
            \If{\textcolor{blue}{$g^{\pi}(v_s) \leq \eps_2(min\{g(u), rhs(u)\} + h(u))$}} \label{bgd:a:trun1}
                \State \Return \textcolor{blue}{\textproc{ObtainPath($v_s$)}}
            \EndIf
            
            \If{$g(u)>rhs(u)$}
                \State $g(u) \gets rhs(u)$;
                \If{\textproc{Event}$(u)$ is triggered}
                     \State \Return path from $v_g$ to $u$
                \EndIf
                \ForAll{$v \in Pred(u)$}
                    \State \textproc{UpdateVertex($v$);}
                \EndFor
            \Else{}
                \State \textcolor{blue}{$g^{\pi}(u) \gets \textproc{ComputeGPI($u$)}$};
                \If{\textcolor{blue}{$g^{\pi}(u) + h(u) \leq$ \\ $\eps_2(min\{g(u), rhs(u)\} + h(u))$}}  \label{bgd:a:trun2}
                    \State \textcolor{blue}{TRUNCATED.\textproc{Insert($u$)}};
                \Else{}
                    \State $g(u) \gets \infty$;
                    \ForAll{$v \in Pred(u) \cup u$}
                        \State \textproc{UpdateVertex($v$);}  \EndFor
                \EndIf
            \EndIf
        \EndIf
    \EndWhile
\EndProcedure

\Procedure{Main}{$ $}
    \ForAll{$e \in E$} \textcolor{blue}{$\overline{w}(e) \gets \eps_1 \widehat{w}(e);$} \EndFor

    \State $km\gets 0$ ; $v_\mathrm{last} \gets v_\mathrm{current}$;
    \State $E_{eval} \gets \emptyset$ ; $rhs(v_g) \gets 0$ ; \textcolor{blue}{$g^{\pi}(v_g)\gets 0$};
    \State \textproc{UpdateVertex($v_g$)};
    
    \While{$v_\mathrm{current} \neq v_g$}
        \Repeat
            \State $\overline{\pi} \gets $ \textproc{ComputeShortestPath($event$)};
            \State $(u,v) \gets $   \textproc{EvaluateEdges($\overline{\pi}$)};
            \State \textproc{UpdateVertex($v$)};
            \ForAll{\textcolor{blue}{$s \in$ TRUNCATED}}
                \State \textcolor{blue}{Remove s from TRUNCATED;
                \State $g^{\pi}(s)\gets \infty$};
            \EndFor
        \Until {$v_s \in \overline{\pi}$ and $\overline{\pi} \subseteq E_{eval}$};
        \State $v_\mathrm{current} \gets bp(v_\mathrm{current})$; Move to $v_\mathrm{current}$;
        \State $km \gets km + h(v_\mathrm{current}, v_\mathrm{last})$;
        \State $v_\mathrm{last} \gets v_\mathrm{current}$;
        \State Perceive the edge changes in $E$;
        \ForAll{$e=(u,v)$ with changed edge cost}
            \State \textcolor{blue}{$\overline{w}(e) \gets \eps_1 \widehat{w}(e)$}; 
            \State $E_{eval} \gets E_{eval} \backslash \{e\}$;
            \State \textproc{UpdateVertex($v$)};
        \EndFor
    \EndWhile

\EndProcedure

\end{algorithmic}
\end{small}
\end{algorithm}

\subsection{Analysis of the Bounded GD* Algorithm}

As TD* inherits the same theoretical properties of TLPA*~\cite{Aine2016},  
B-GD* inherits the same theoretical properties as B-LGLS.
Namely, the final solution is bounded suboptimal, and a vertex is expanded at most twice.  
These properties can be proved similarly to the ones of the TLPA* with minor modifications, thus we omit the details.   

\subsection{Illustrative Example}

Using the same underlying graph, we visualize the results for D*-Lite, TD*, GD*, and B-GD* for two consecutive planning steps in the 2D environment in Figure~\ref{ds:f:bgd-2d}.  
The top row shows the first search.  The bottom row shows the second search after the environment changes, which occurs when the robot traverses one segment of the optimal path from the first plan. Both D*-Lite and GD* find the optimal solution, whereas the TD* and B-GD* find a bounded suboptimal path, saving computational resources. Moreover, the GD* and B-GD* algorithms save computational resources via reducing unnecessary edge evaluations compared to the D*-Lite and TD* algorithms. 
The performance comparison in terms of the number of edge evaluations and vertex expansions for the algorithms are summarized in Table~\ref{gd:t:2d-bd}.

\begin{figure*}[ht]
	\centering
	\begin{subfigure}{\myMSFigureScale\textwidth}
		\includegraphics[width=\myLineScale\linewidth]{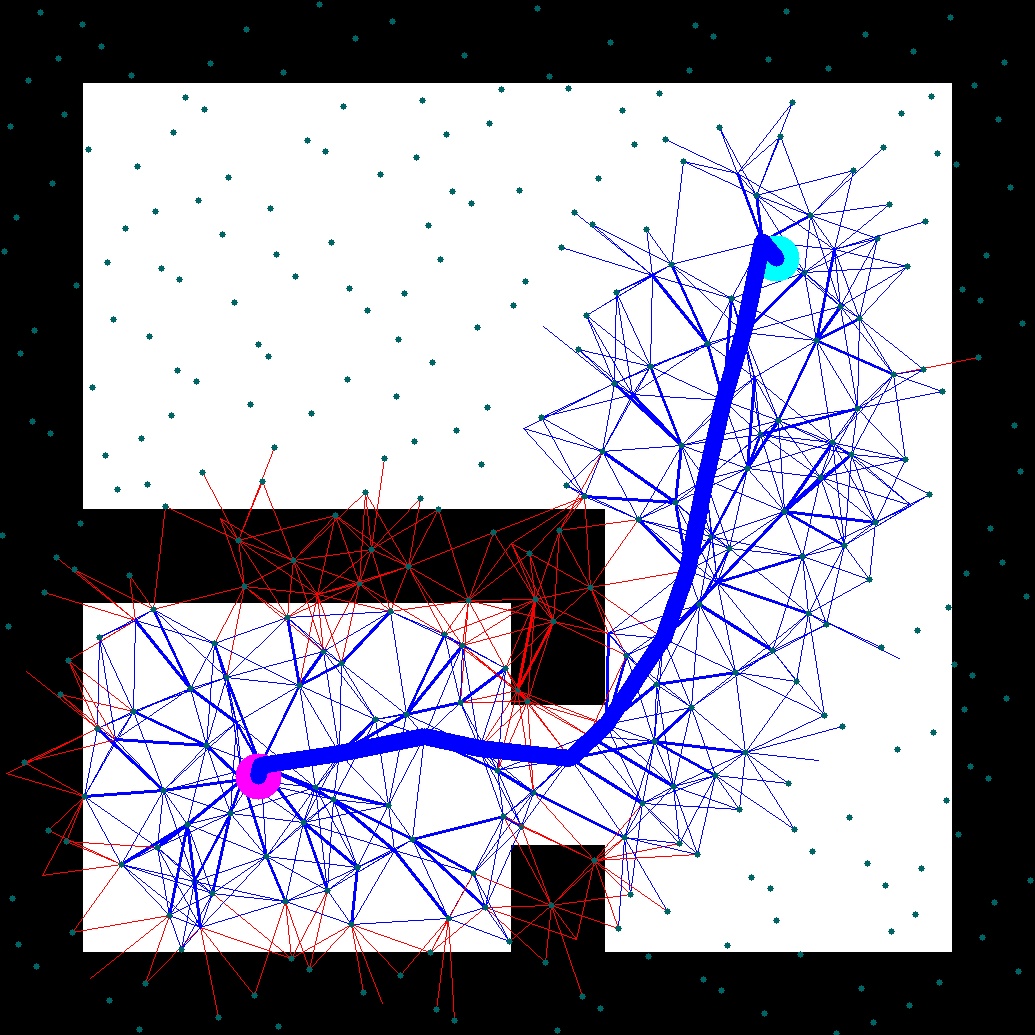}
	\end{subfigure} 
	\begin{subfigure}{\myMSFigureScale\textwidth}
		\includegraphics[width=\myLineScale\linewidth]{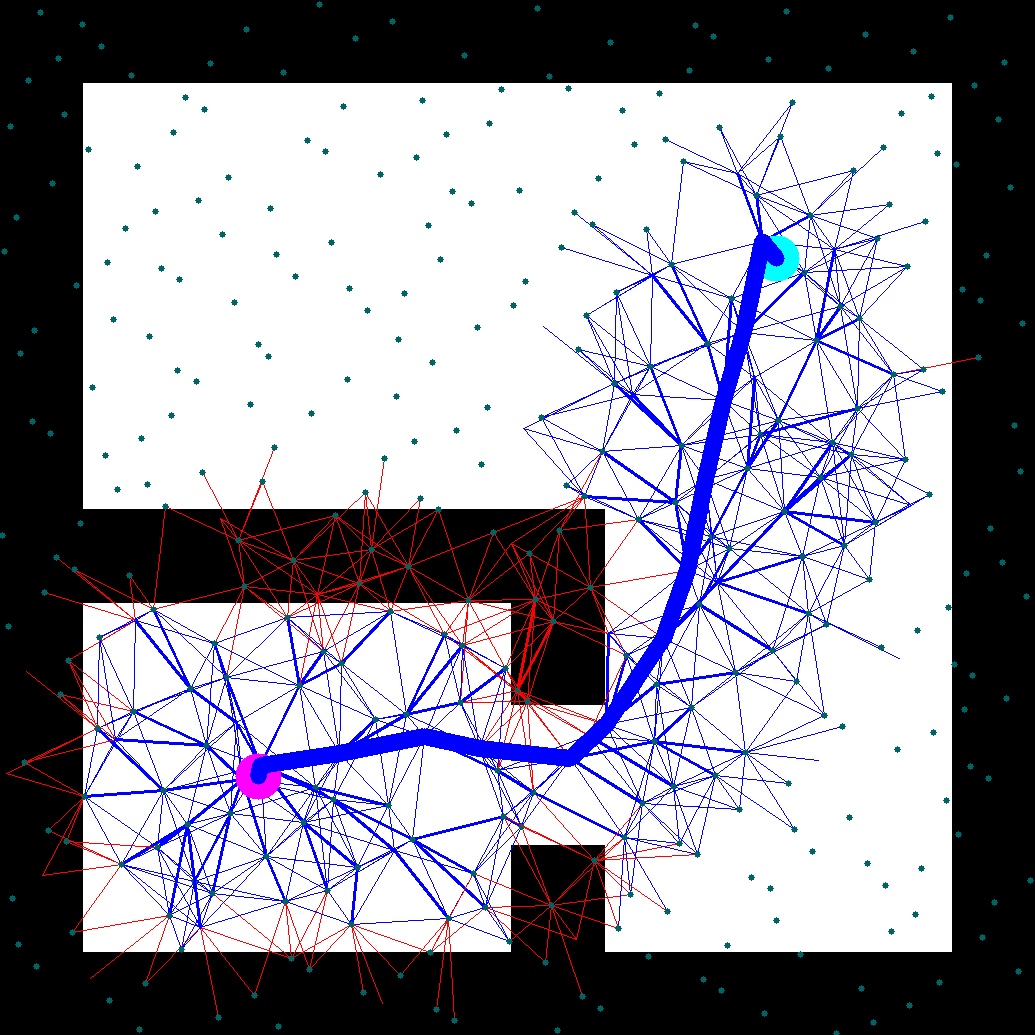}
	\end{subfigure}
	\begin{subfigure}{\myMSFigureScale\textwidth}
		\includegraphics[width=\myLineScale\linewidth]{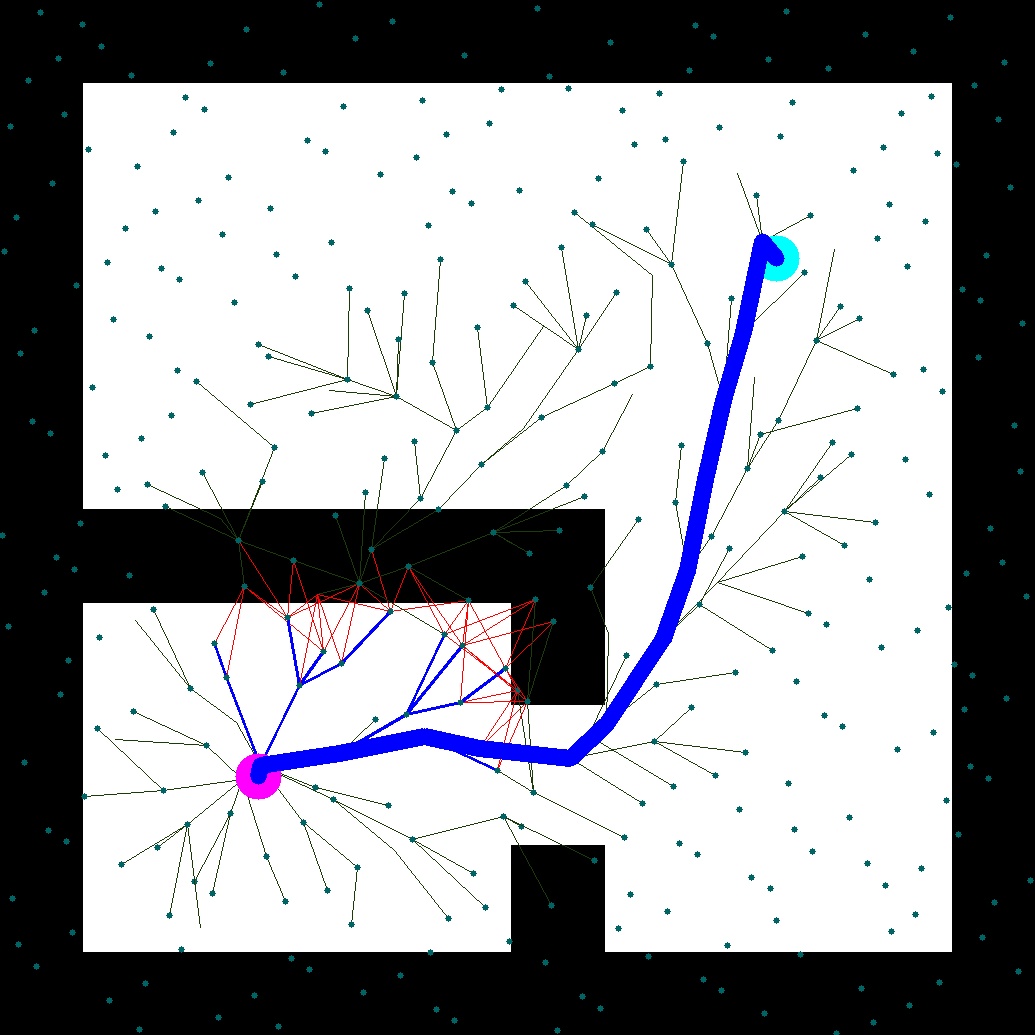}
	\end{subfigure}
	\begin{subfigure}{\myMSFigureScale\textwidth}
		\includegraphics[width=\myLineScale\linewidth]{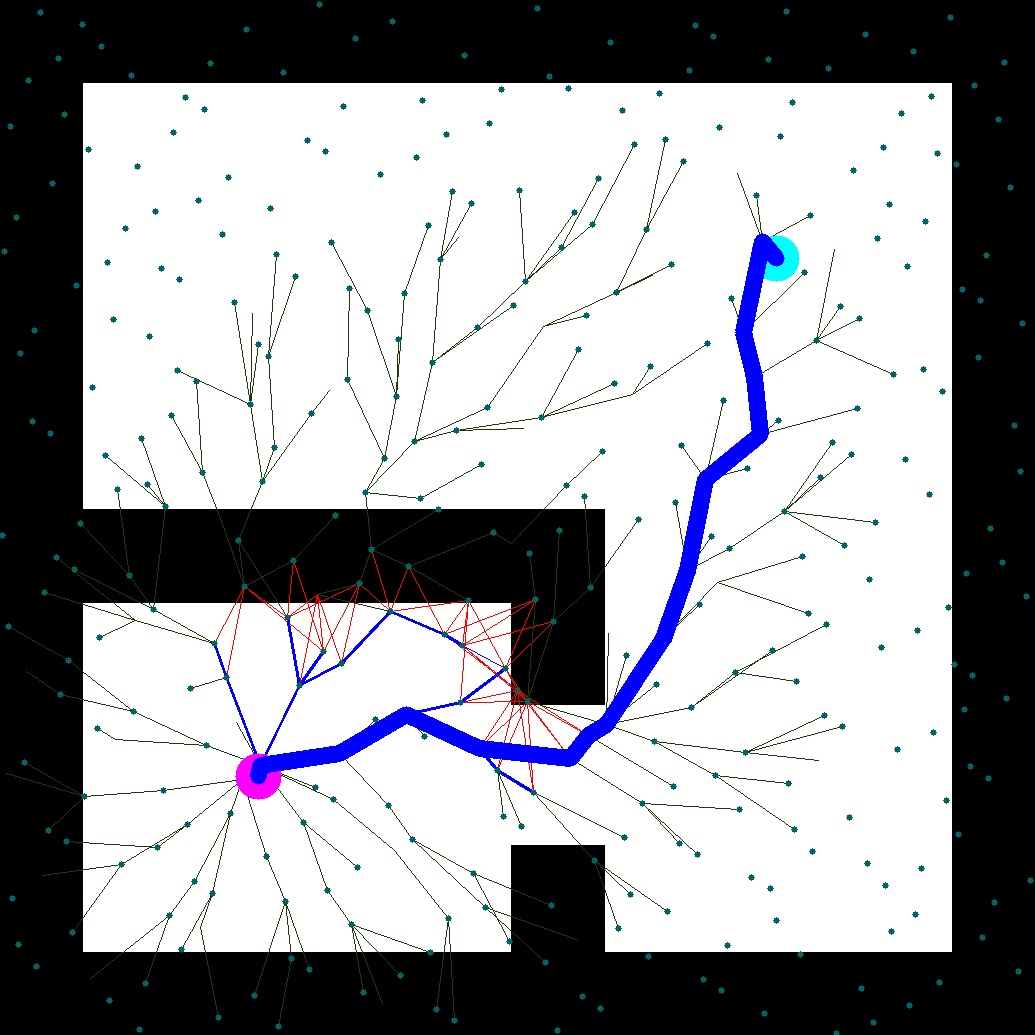}
	\end{subfigure}
	\medskip
	\begin{subfigure}{\myMSFigureScale\textwidth}
		\includegraphics[width=\myLineScale\linewidth]{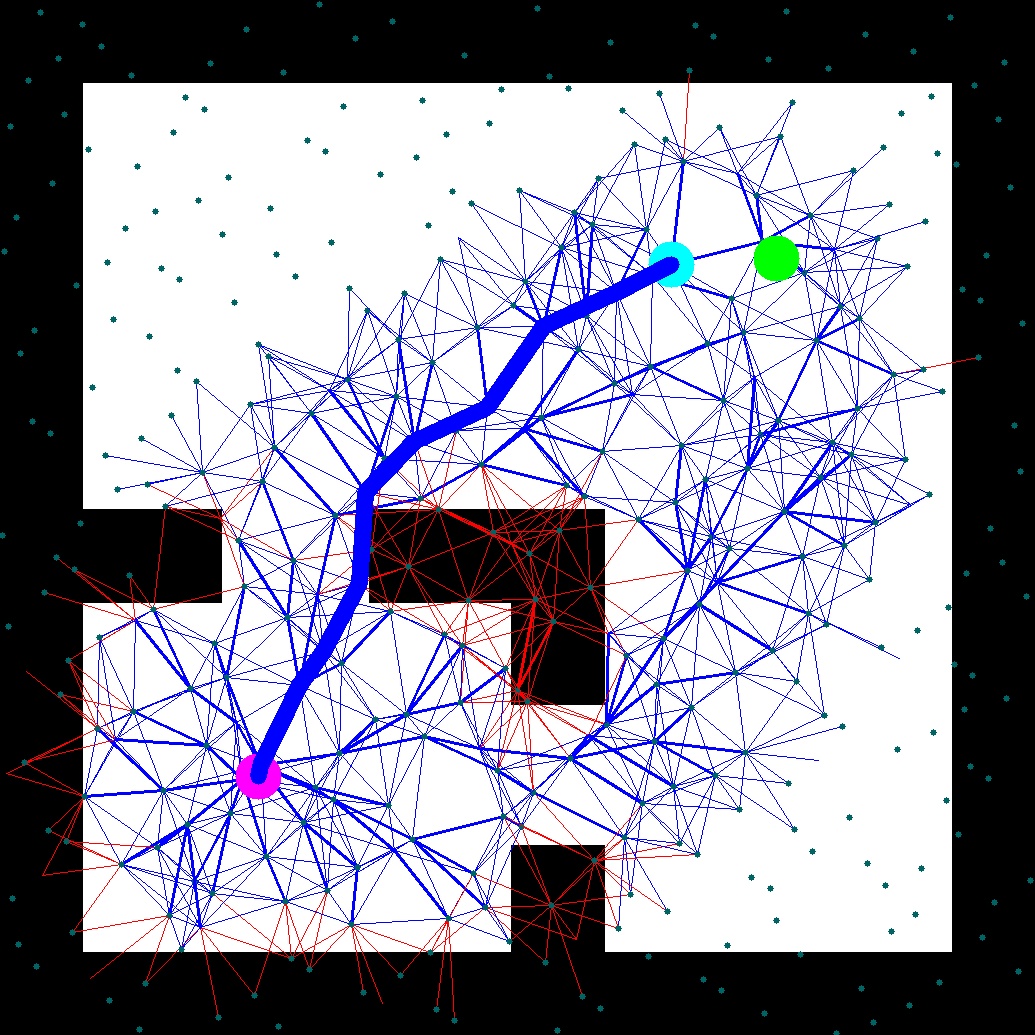}
		\caption{D*-Lite}
	\end{subfigure}
	\begin{subfigure}{\myMSFigureScale\textwidth}
		\includegraphics[width=\myLineScale\linewidth]{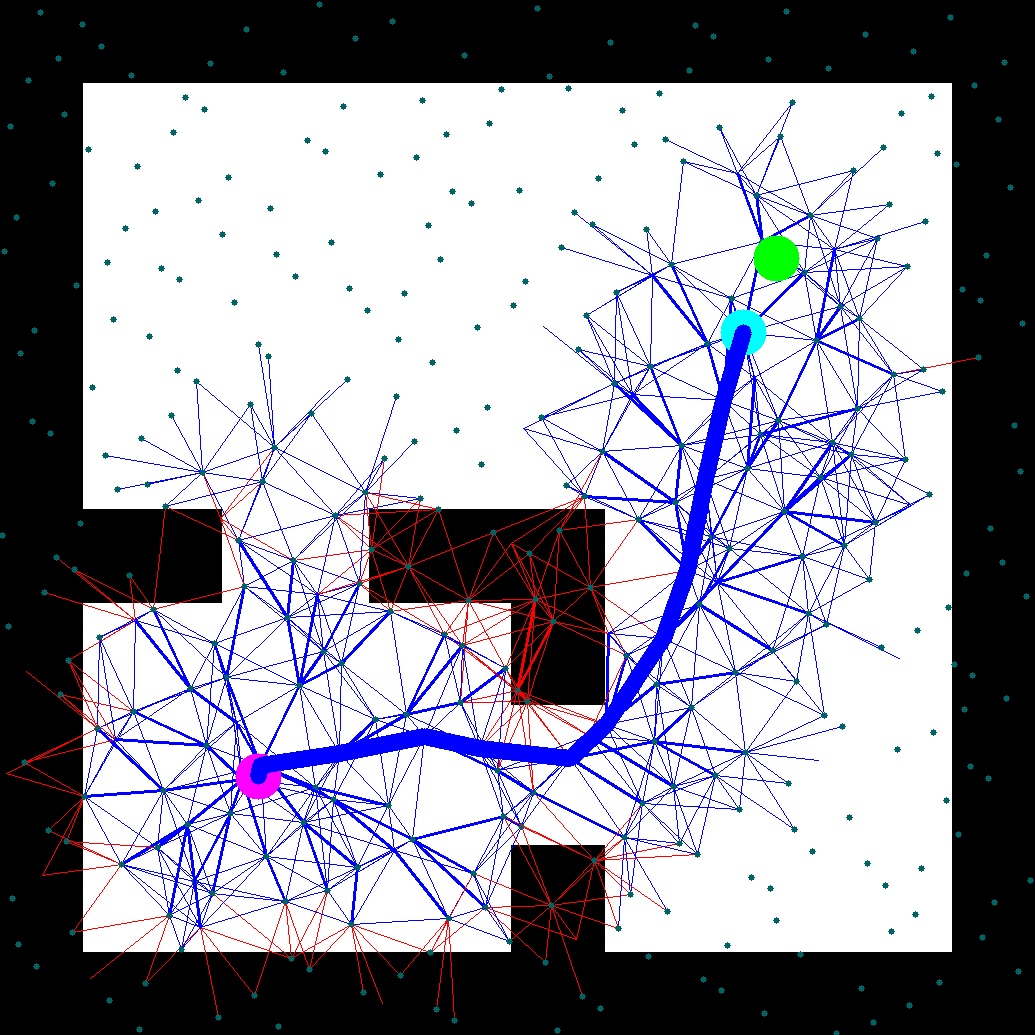}
		\caption{TD* (1.2)}
	\end{subfigure}
	\begin{subfigure}{\myMSFigureScale\textwidth}
		\includegraphics[width=\myLineScale\linewidth]{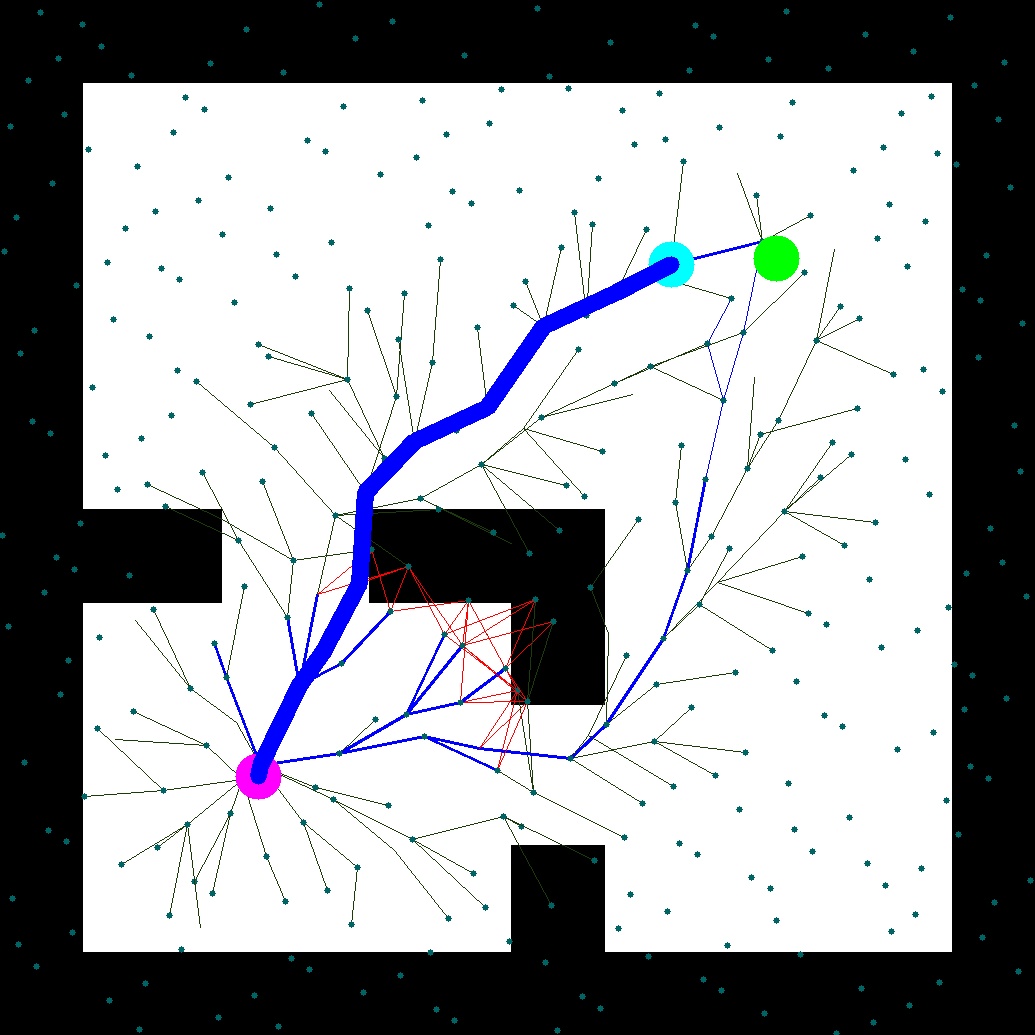}
		\caption{GD*}
	\end{subfigure}
	\begin{subfigure}{\myMSFigureScale\textwidth}
		\includegraphics[width=\myLineScale\linewidth]{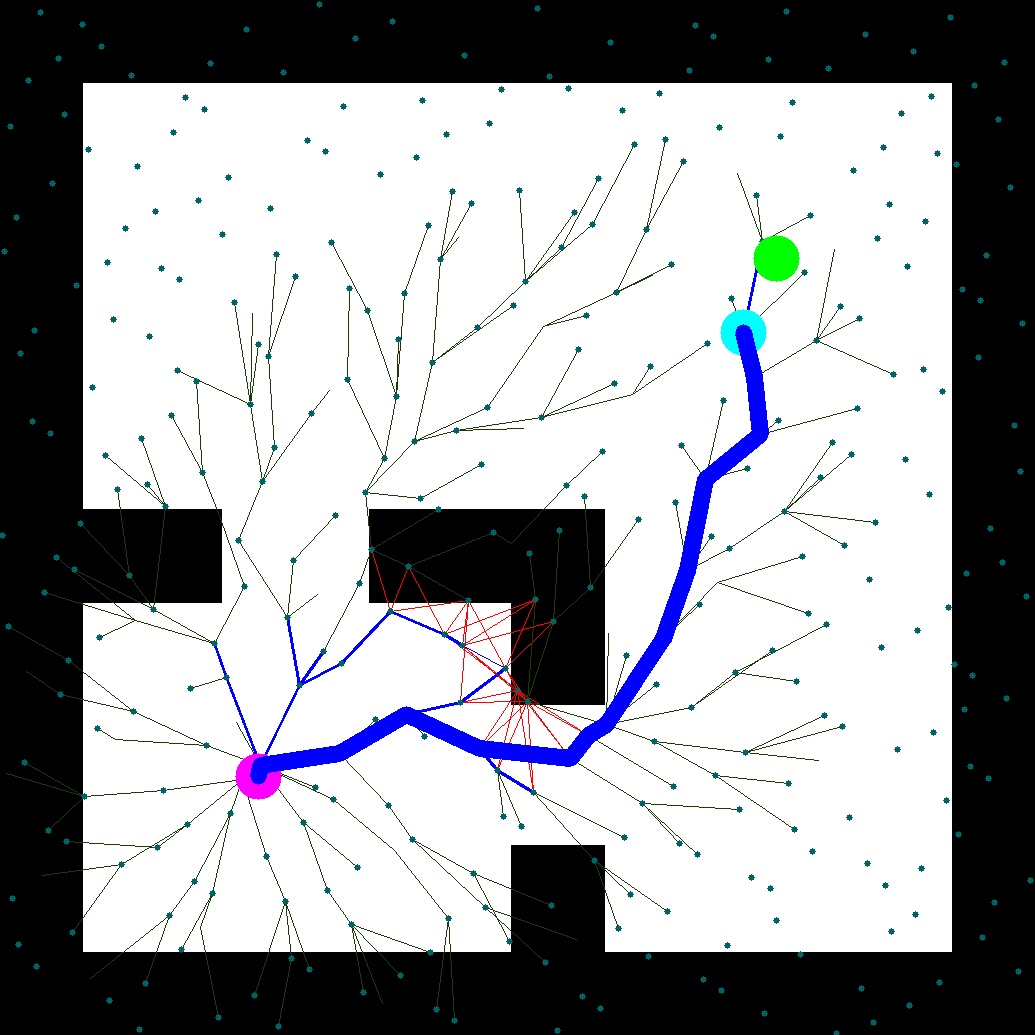}
		\caption{Bounded GD*}
	\end{subfigure}
	\caption{First search (top row) and second search (bottom row) to find a bounded suboptimal path from start vertex(\tikzcircle[blue, fill=green]{2.5pt}) to goal vertex(\tikzsquare[blue, fill=magenta]{4.5pt}) per environment change with $\varepsilon_1=1.2$ and $\varepsilon_2=1.2$, from left to right: (a) D*-Lite, (b) TD*, (c) GD*, and (d) Bounded B-GD*.  The current robot position is shown with \tikzcircle[blue, fill=cyan]{2.5pt}. 
		Lines(\tikzline[blue,semithick]{}\tikzline[red,semithick]{}) are the evaluated edges during the current search. Bold lines(\tikzline[blue,very thick]{}) are the edges belonging to the current search tree. Blue and red represents free and obstacle, respectively.}
	\label{ds:f:bgd-2d}
\end{figure*}

\begin{table*}[ht]
	\centering
 	\caption{Number of edge evaluations, number of vertex expansions, and solution length for different planners over two consecutive search queries in a dynamic environment of Figure \ref{ds:f:bgd-2d}. }
	\begin{small}
	\begin{tabular}{lcccc}
		\toprule
		\textbf{}      & D*-Lite & TD*(1.2) & GD* & B-GD*(1.2,1.2) \\ 
		\midrule
		\textbf{Planning} in scene 1    &          &       &              &       \\
		\# Edge Evaluation  & 832 & 832 & 67 &  78 \\ 
		\# Vertex Expansion & 64 & 64 & 2622 &  449 \\
		Solution Length & 0.878 & 0.878 & 0.878 & 0.907  \\
		&          &       &          &              \\
		\textbf{Replanning} in scene 2    &          &       &          &       \\
		\# Edge Evaluation  & 267 & 117 & 15 &  0 \\ 
		\# Vertex Expansion & 22 & 4 & 226 & 0 \\
		Solution Length & 0.767 & 0.857 & 0.767 & 0.879 \\		         
		&          &       &          &              \\
		\textbf{Total}    &          &       &          &               \\
		\# Edge Evaluation  & 1099 & 949 & 82 & 78 \\ 
		\# Vertex Expansion & 86 & 68 & 2848 & 449 \\
		\bottomrule
	\end{tabular}
	\end{small}
	\label{gd:t:2d-bd}
\end{table*}

\section{Experimental Results}  \label{sec:fixed_query_experiment}

In this section we compare the algorithms developed in the previous sections
(L-GLS, B-LGLS, GD*, B-GD*)
to existing incremental search and lazy search algorithms using various examples. 
This section is divided into two different parts. 
The first part tests the algorithms in a stationary replanning problem instance, while the second part compares the various algorithms in a non-stationary replanning problem instance. 
We also implement the proposed algorithms on a scaled ground vehicle platform to show their performance for robot navigation problems.

\subsection{Stationary Query Experiment (L-GLS, B-LGLS)} 

In this section we present numerical results comparing our lazy incremental search algorithms with LPA*, TLPA*, and GLS for a stationary replanning experiment, in which the start and goal of the planning instances are fixed, but either the environment or the graph approximation of the environment changes over time.
We demonstrate the efficiency of lazy incremental search algorithms in two different scenarios: first, a scenario in which the path planning problem is solved consecutively in a dynamic environment with a fixed-topology graph, and a second scenario in which graph approximations of a static environment get denser over time via incremental sampling.
We consider three different environments for benchmark testing: 
The Piano Movers' problem in SE(2), a manipulation problem in $\R^7$ using a PR2 robot (a mobile robot with 7DoF arm), and a non-holonomic motion planning problem in SE(2) using the RACECAR/J platform robot\footnote{\url{https://racecar.mit.edu/}}. 
All the algorithm implementations were done in C++, and the experiments were run on an 2.20 GHz Intel(R) Core(TM) i7-8750H CPU Ubuntu 18.04 LTS machine with 15.5GB of RAM within a docker container running Ubuntu 18.04 LTS.
The RACECAR/J experiment was run onboard using an NVIDIA Jetson TX2.

\subsubsection{Dynamic Environment and Static Graph.}

We first consider scenarios where the environment is dynamic, and we have available a topologically fixed graph approximating the environment. 
The search is performed on a pre-built graph with uniformly distributed vertices in the configuration space, in which an edge is added between two vertices if they belong to the first $k$-nearest neighbors. 
The graph topology does not change throughout the experiment, and only the edge values change due to underlying environment changes.
We repeat the experiment 100 times with randomly sampled start and goal pairs, comparing 5 different planners: LPA*, TLPA*, GLS, L-GLS, and B-LGLS. 
We report the average number of edge evaluations, the average number of vertex expansions, and the average total runtime of these planners.

\paragraph{Piano Movers' Problem:}
The purpose of the first experiment is to find the shortest path of a piano in a compact and dynamic apartment environment.
For each experiment, we find the shortest paths from the Apartment scenario in OMPL~\cite{Sucan2012} from a start configuration to a goal configuration without colliding with the moving obstacles (see Figure~\ref{lgls:f:piano}).  
There were six consecutive searches in the environment, where the first search was on scene 1 (Figure~\ref{lgls:f:piano}(a)), the second search was on scene 2 (Figure~\ref{lgls:f:piano}(b)), and the third search was on scene 3 (Figure~\ref{lgls:f:piano}(c)), and the fourth, the fifth, and the sixth searches were again repeated on scene 1, scene 2, and scene 3, respectively.
The search was performed on a randomly generated graph with 5,000 randomly sampled vertices and edges connecting two vertices if they belong to the first 20-nearest vertices. The vertices were sampled uniformly in $\R^3$. 
Figure~\ref{lgls:f:piano} show an example of this experiment. 

\begin{figure*}[ht]
	\centering
	\begin{subfigure}{0.30\textwidth}
		\includegraphics[width=\myLineScale\linewidth]{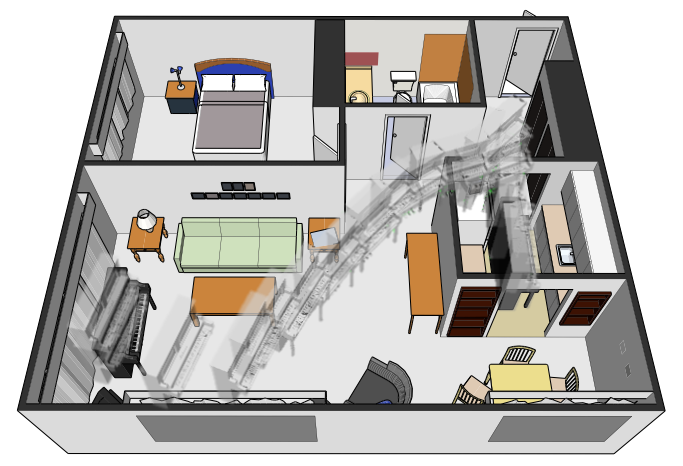}
		\caption{Scene 1}
	\end{subfigure}
	\begin{subfigure}{0.30\textwidth}
		\includegraphics[width=\myLineScale\linewidth]{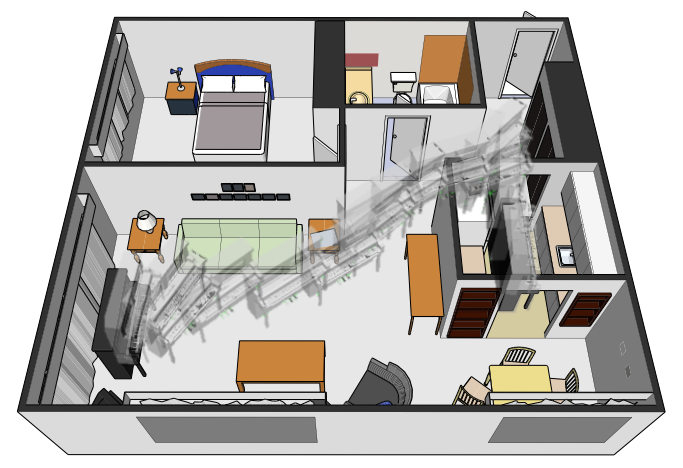}
		\caption{Scene 2}
	\end{subfigure}
	\begin{subfigure}{0.30\textwidth}
		\includegraphics[width=\myLineScale\linewidth]{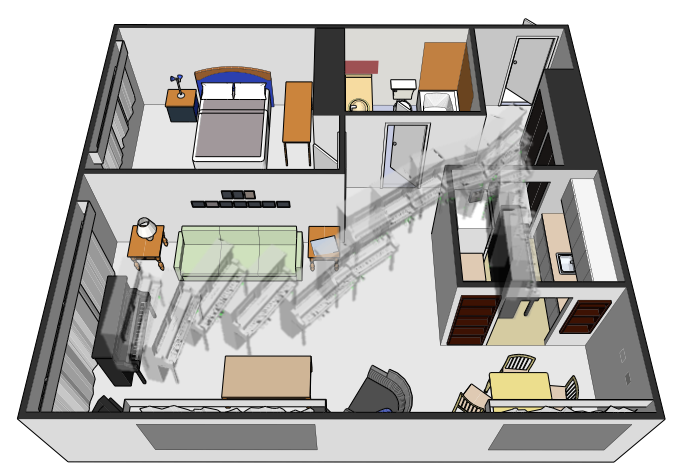}
		\caption{Scene 3}
	\end{subfigure}
	\caption{The shortest paths of the Piano Movers' problems in dynamic environment.}
	\label{lgls:f:piano}
\end{figure*}

We compare the search results, namely the number of edge evaluations, the number of vertex expansions, and the total runtime for 5 different algorithms: LPA*, TLPA* ($\varepsilon_2$=2), GLS (1-step lookahead), L-GLS (1-step lookahead), and B-LGLS(1-step lookahead, $\varepsilon_1=\sqrt{2}$, $\varepsilon_2=\sqrt{2}$). 
We repeated the experiment 100 times with randomly sampled start and goal pairs with 100 different graphs and report the average values.
The results are summarized in Figure~\ref{lgls:f:stat_piano}. 

\begin{figure*}[ht]
	\centering
	\begin{subfigure}{0.32\textwidth}
		\includegraphics[width=\myLineScale\linewidth]{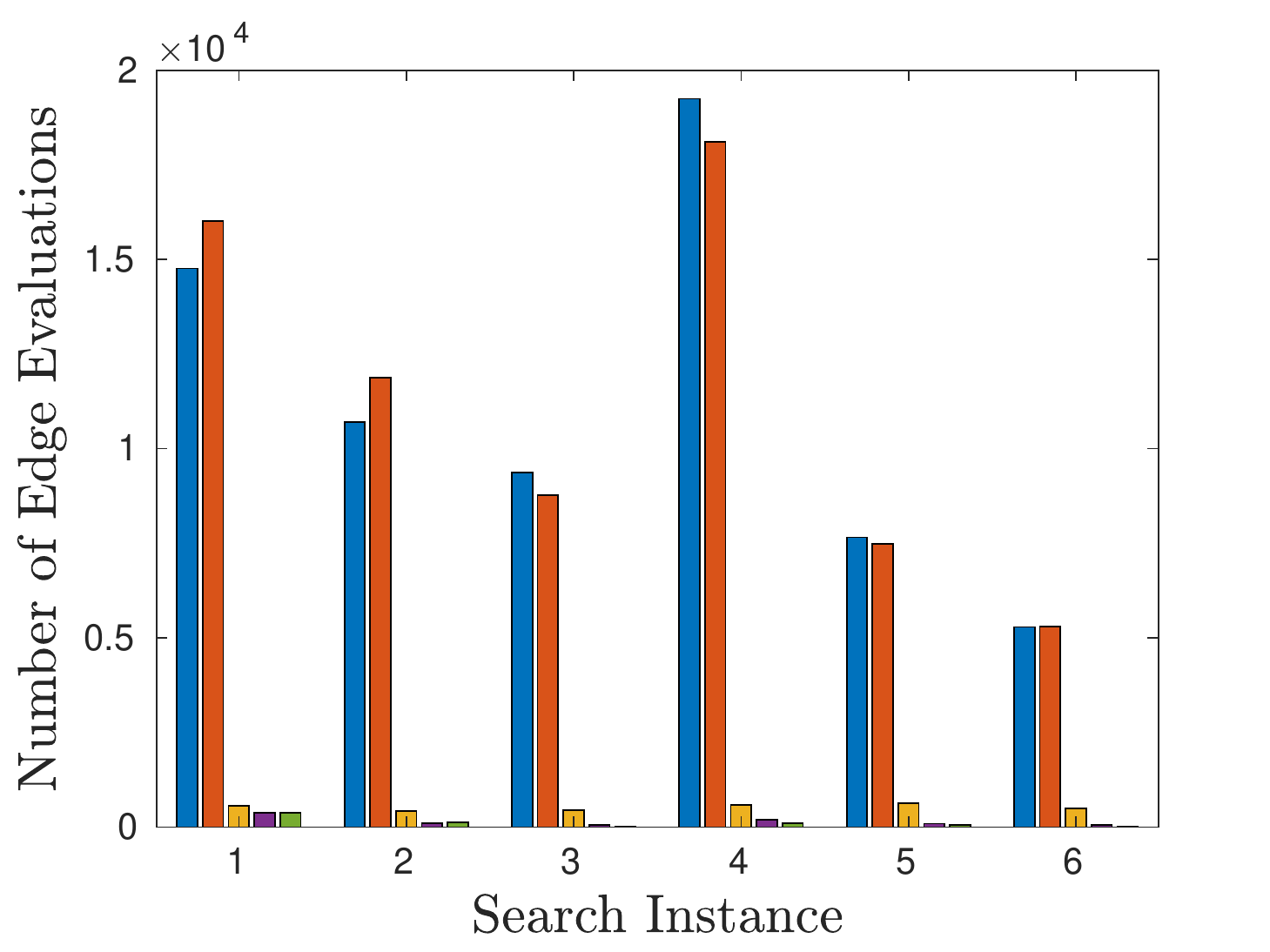}
		\caption{Edge evaluation}
	\end{subfigure}
	\begin{subfigure}{0.32\textwidth}
		\includegraphics[width=\myLineScale\linewidth]{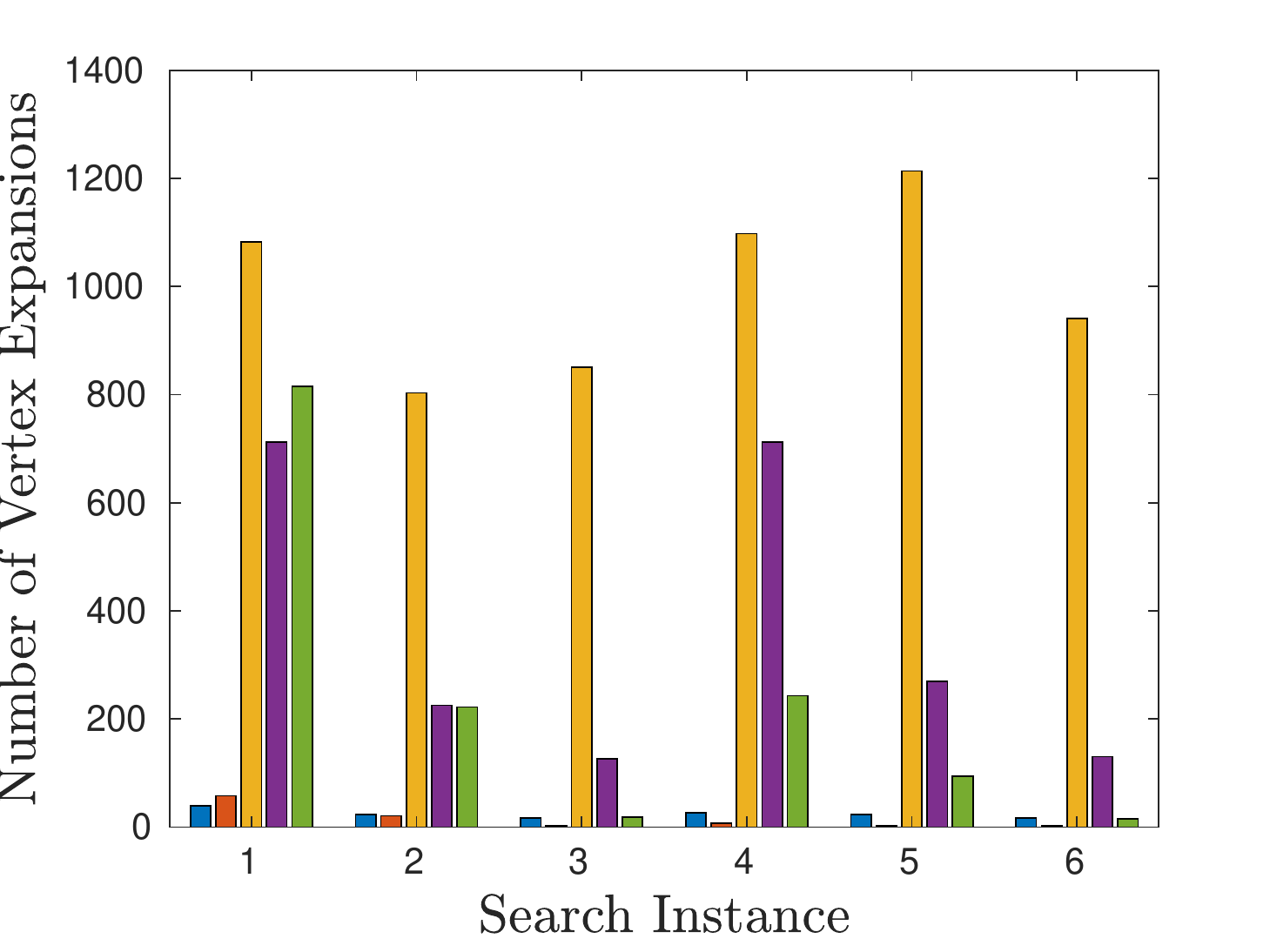}
		\caption{Vertex expansion}
	\end{subfigure}
	\begin{subfigure}{0.32\textwidth}
		\includegraphics[width=\myLineScale\linewidth]{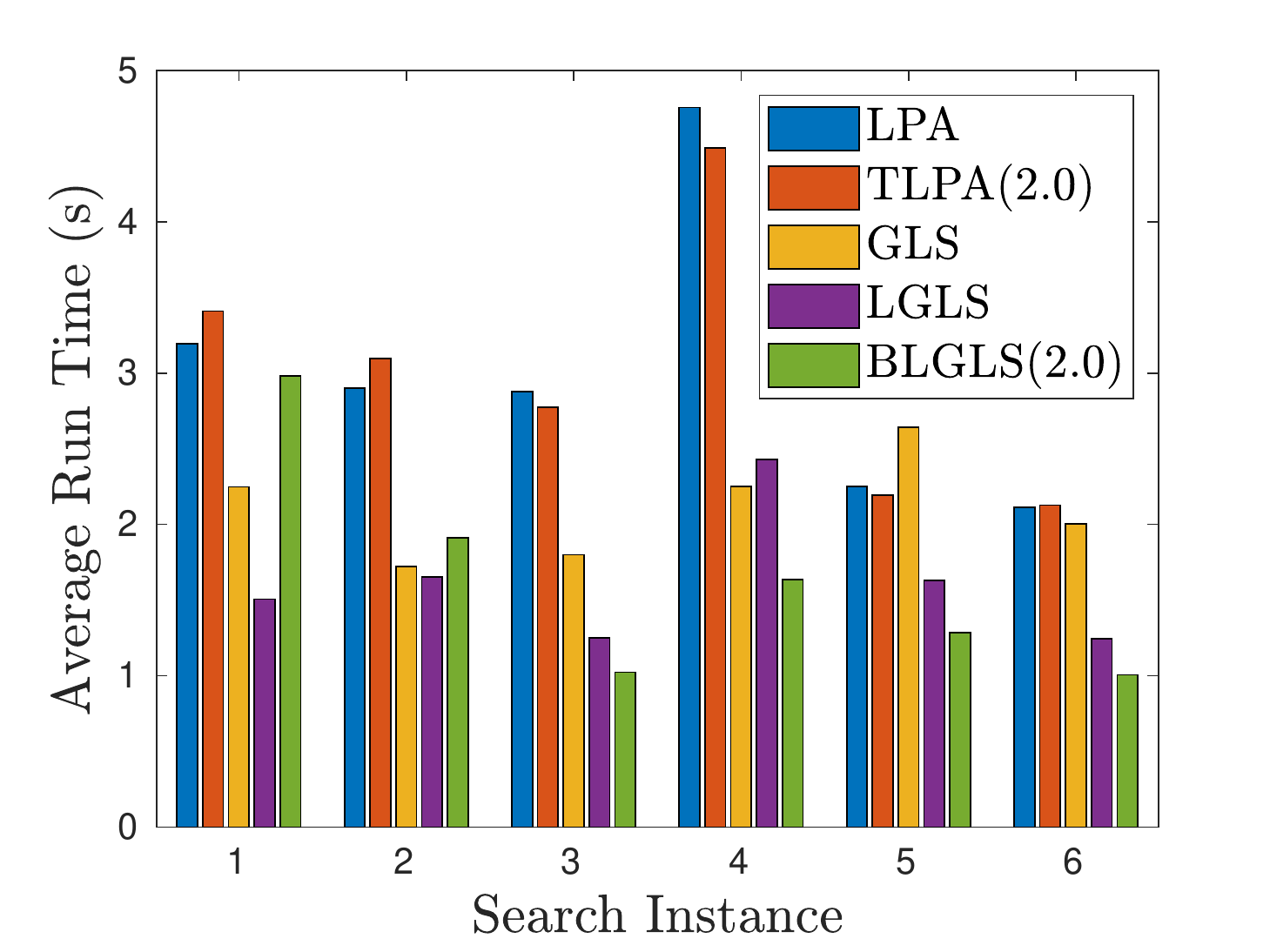}
		\caption{Total time}
	\end{subfigure}
	\caption{The average number of edge evaluations, the average number of vertex expansions, and the average total runtime of 100 random experiments for the Piano Movers' problem.}
	\label{lgls:f:stat_piano}
\end{figure*}

The average number of edge evaluations for the regular incremental search algorithms, namely, LPA* and TLPA*, are similar across the search instances.
TLPA* reduces the number of vertex expansions in consequent planning instances compared to LPA*, as it halts the cost inconsistency propagation as early as the current solution is guaranteed to be bounded suboptimal. 
Lazy search algorithms including GLS, L-GLS, and B-LGLS reduce edge evaluations by
a significant amount compared to the regular incremental search algorithms, and yet these algorithms expand more vertices to repair the cost inconsistencies between the heuristic edge value and the actual value.
L-GLS uses fewer edge evaluations and vertex expansions compared to GLS, as it reuses the previous search results. 
Although Bounded-LGLS uses more vertex expansions in the first search, it uses both fewer edge evaluations and vertex expansions in consequent planning instances compared to L-GLS, as it halts as early as the current solution is guaranteed to be bounded suboptimal.

\paragraph{PR2 Robot:}
The purpose of the second experiment is to find a path for the 7D right arm of PR2 robot from a start configuration to a goal configuration without collision in a dynamic environment where the obstacle moves, as shown in Figure~\ref{lgls:f:pr2}.
\begin{figure*}[ht]
	\centering
	\begin{subfigure}{0.45\textwidth}
		\includegraphics[width=\myLineScale\linewidth]{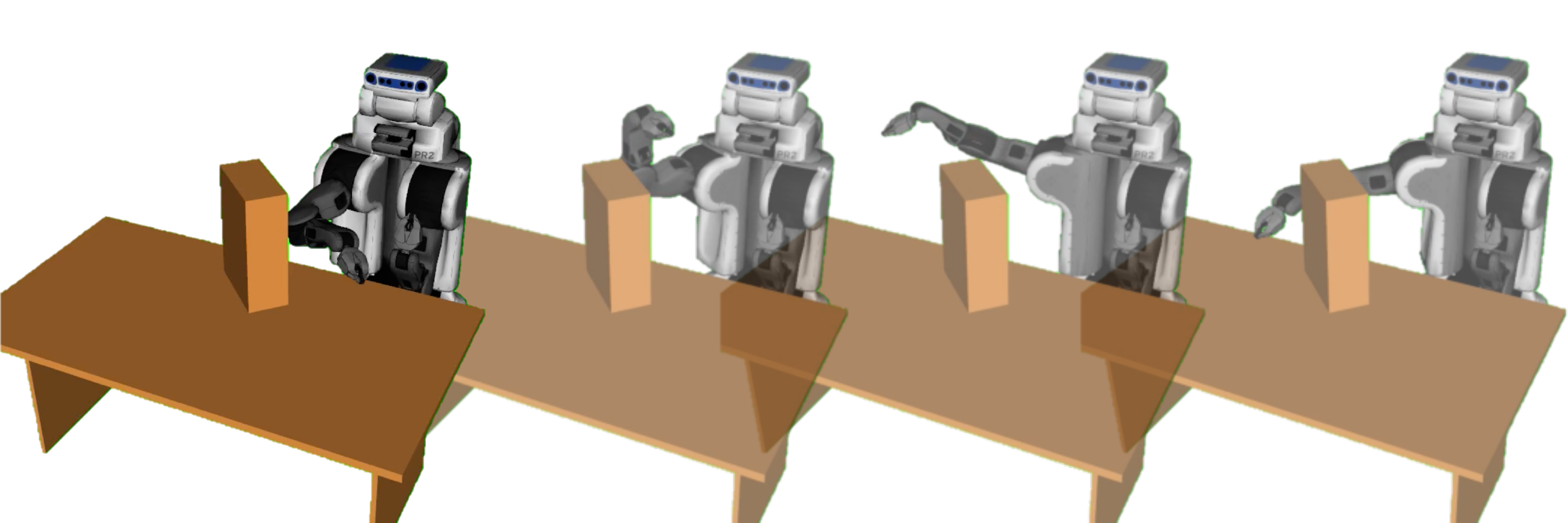}
		\caption{Scene 1}
	\end{subfigure}
	\begin{subfigure}{0.45\textwidth}
		\includegraphics[width=\myLineScale\linewidth]{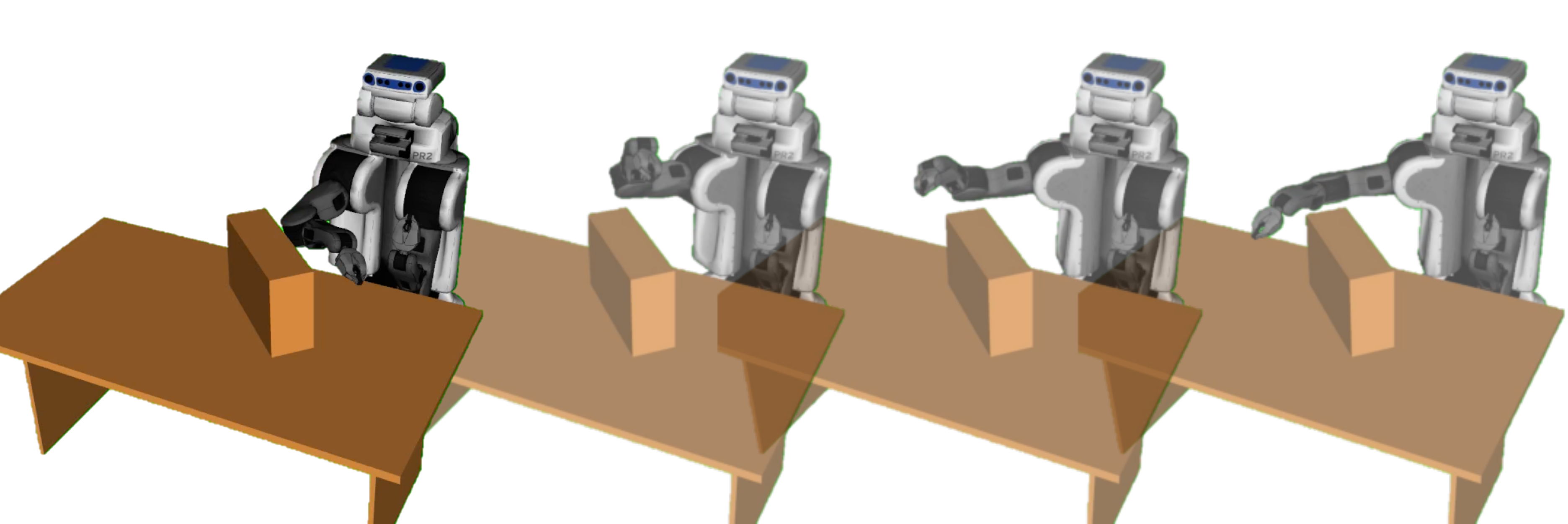}
		\caption{Scene 2}
	\end{subfigure}\hfill
	\caption{The shortest paths of the right arm of PR2 robot for the same query in dynamic environment.}
	\label{lgls:f:pr2}
\end{figure*}

There were six consecutive searches in the environment, where the first search was on scene 1 (Figure~\ref{lgls:f:pr2}(a)), the second search was on scene 2 (Figure~\ref{lgls:f:pr2}(b)), and the remaining searches were repeated alternating between the two scenes. 
We repeated the experiment 100 times with random start and goal pairs and 100 random graphs. 

For each experiment, similarly to the Piano Movers' problem, the search was performed on a randomly generated graph with 5,000 vertices. The vertices were randomly uniformly sampled in $\R^7$, bounded by the PR2 arm's joint-angle bounds. 
Two vertices are adjacent in this graph if they belong to the first 20 nearest vertices mutually.
We implemented LPA*, TLPA*, GLS, L-GLS, and B-LGLS as an OMPL Planner~\cite{Sucan2012} with the MoveIt! interface~\cite{Coleman2014}. 

We compared the search results, namely the number of edge evaluations and the number of vertex expansions for 5 different algorithms: LPA*, TLPA*~($\varepsilon_2$=2), GLS~(1-step lookahead), L-GLS~(1-step lookahead), and B-LGLS~(1-step lookahead, $\varepsilon_1=\sqrt{2}$, $\varepsilon_2=\sqrt{2}$).
The number of edge evaluations, the number of vertex expansions, and the runtime are recorded for each planning instance, and we took the averages of them over 100 random experiments. 
The results are summarized in Figure~\ref{lgls:f:stat_pr2}.

\begin{figure*}[ht]
	\centering
	\begin{subfigure}{0.32\textwidth}
		\includegraphics[width=\myLineScale\linewidth]{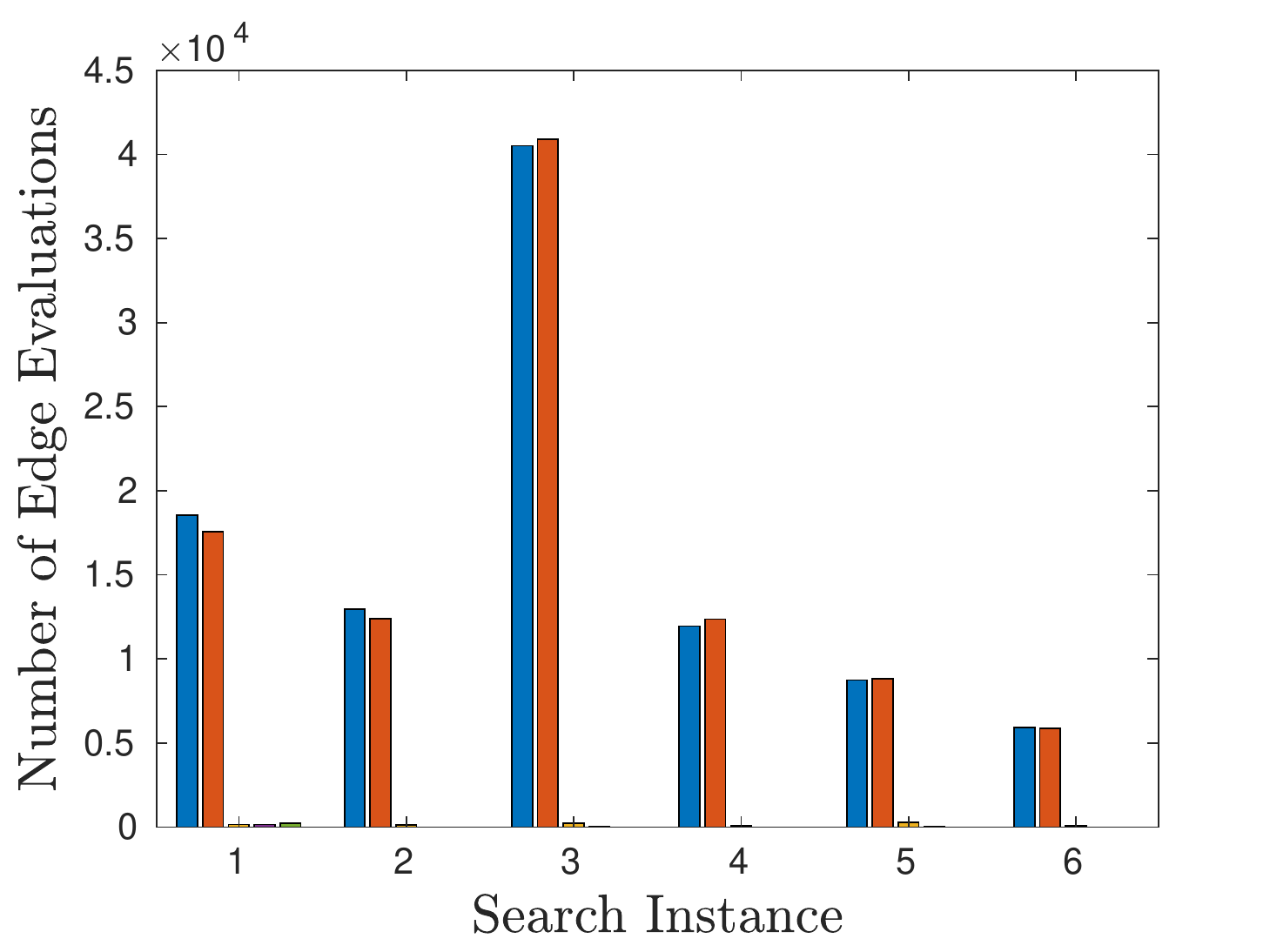}
		\caption{Edge evaluation}
	\end{subfigure}
	\begin{subfigure}{0.32\textwidth}
		\includegraphics[width=\myLineScale\linewidth]{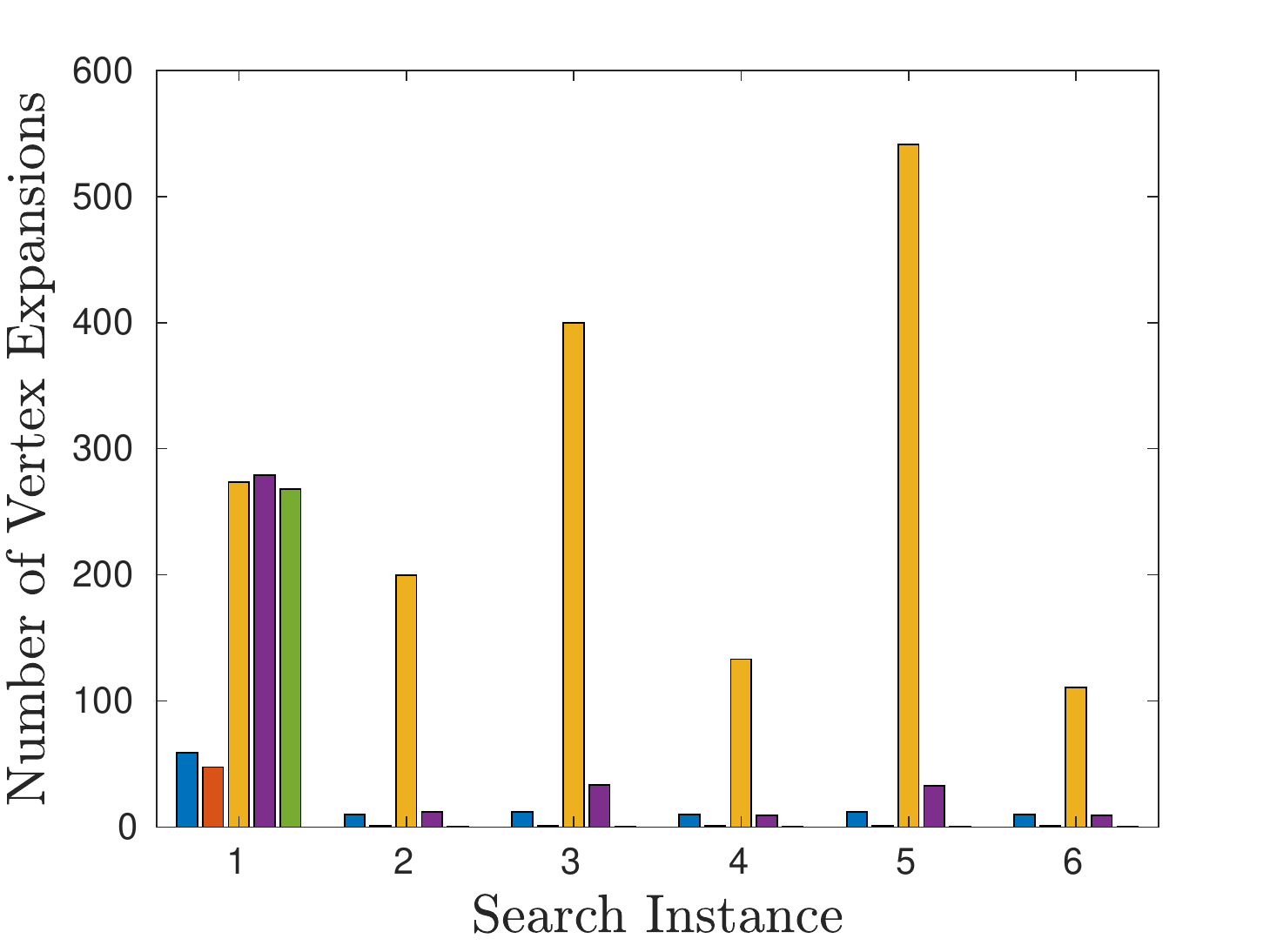}
		\caption{Vertex expansion}
	\end{subfigure}
	\begin{subfigure}{0.32\textwidth}
		\includegraphics[width=\myLineScale\linewidth]{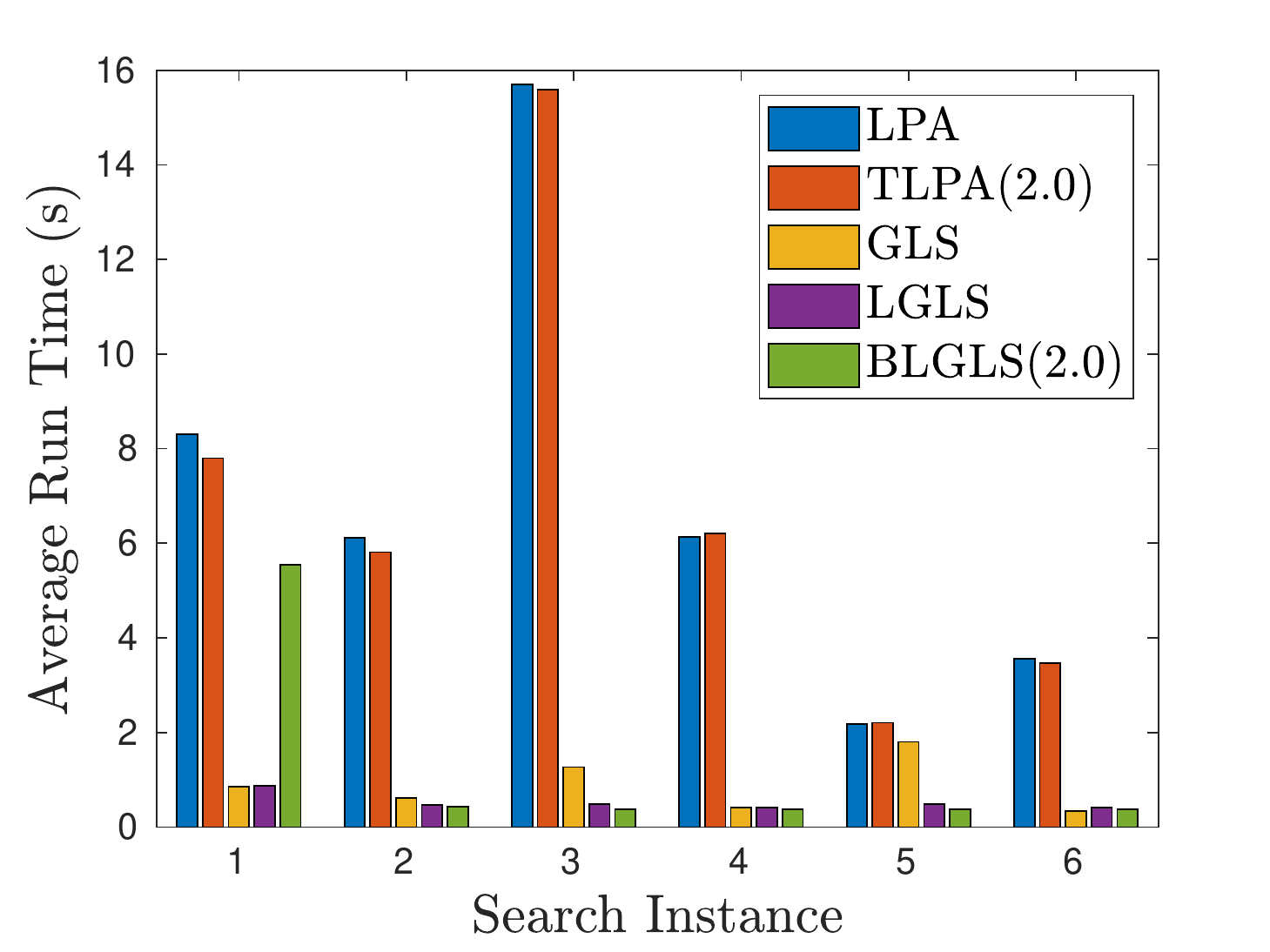}
		\caption{Total time}
	\end{subfigure}
	\caption{The average number of edge evaluations, the average number of vertex expansions, and the average total runtime taken over 100 random experiments for the PR2 robotic arm problem.}
	\label{lgls:f:stat_pr2}
\end{figure*}

A similar trend to the Piano Movers' problem
is observed in this PR2 robotic arm experiment as well. 
The regular incremental search algorithms, namely, LPA* and TLPA*, incur unnecessary edge evaluations compared to the lazy search algorithms. LPA* and TLPA* resulted in the same amount of edge evaluations, but TLPA* saves a few more vertex expansions compared to LPA* in consecutive searches.
The  Lazy search algorithms, namely, GLS, L-GLS, and B-LGLS save a significant amount of edge evaluations at the expense of more search.
The lazy incremental search algorithms, L-GLS and B-LGLS save both edge evaluations and vertex expansions in consecutive searches compared to GLS, as they use the previous search results.
B-LGLS further saves a number of edge evaluations and vertex expansions compared to L-GLS in the consecutive searches as it halts the search and evaluation as soon as the current solution is guaranteed to be bounded suboptimal. 

\paragraph{RacecarJ Experiment:}
The purpose of the third experiment is to find a dynamically feasible path for a non-holonomic vehicle from a start configuration to a goal configuration without collision in a dynamic environment where the obstacle moves, as shown in Figure~\ref{lgls:f:racecar}.
\begin{figure*}[ht]
	\centering
	\begin{subfigure}{0.32\textwidth}
		\includegraphics[width=\myLineScale\linewidth]{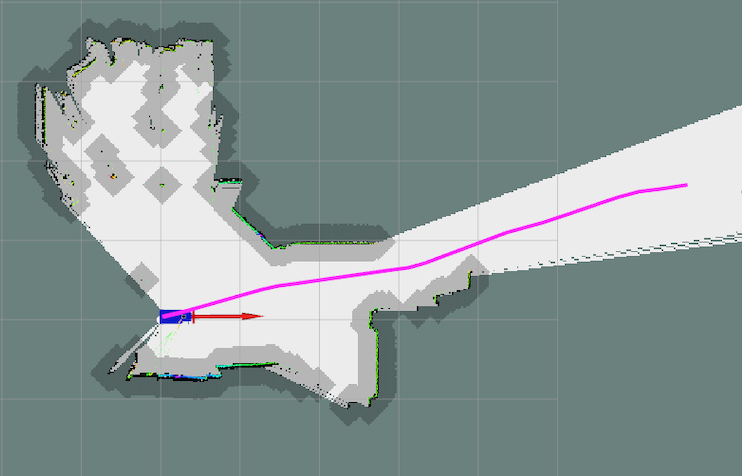}
		\caption{Scene 1}
	\end{subfigure}
	\begin{subfigure}{0.32\textwidth}
		\includegraphics[width=\myLineScale\linewidth]{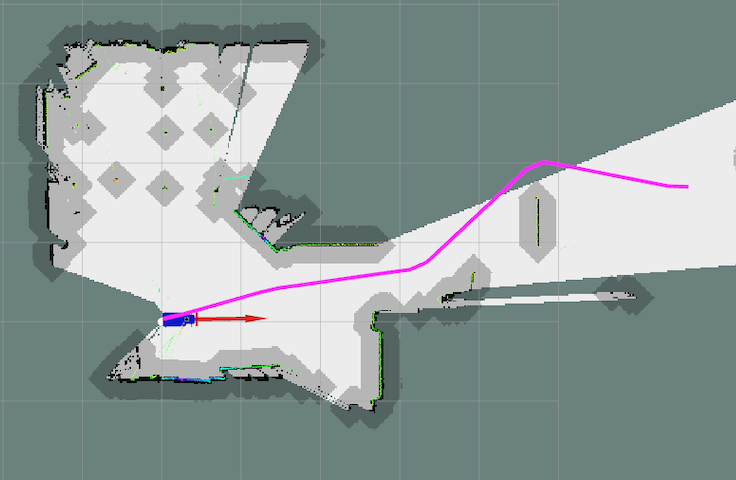}
		\caption{Scene 2}
	\end{subfigure}\hfill
	\begin{subfigure}{0.32\textwidth}
		\includegraphics[width=\myLineScale\linewidth]{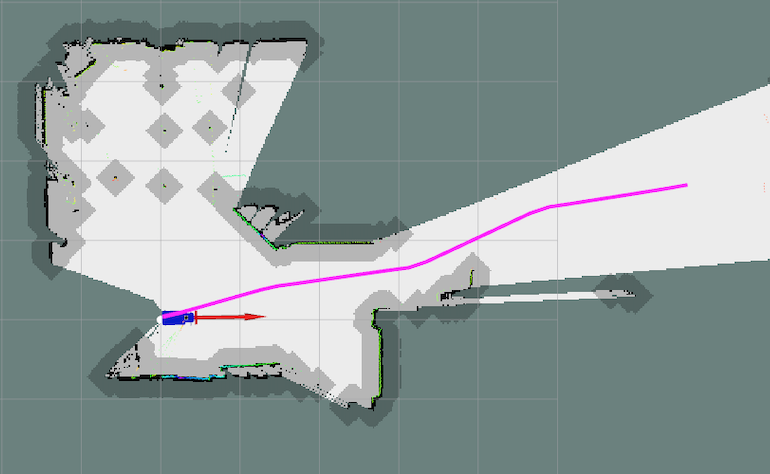}
		\caption{Scene 3}
	\end{subfigure}\hfill
	\caption{The shortest paths of the nonholonomic vehicle for the same query in dynamic environment.}
	\label{lgls:f:racecar}
\end{figure*}

There were six consecutive searches in the environment, where after each search, a randomly chosen 10 percent of the edges changed.
Figure~\ref{lgls:f:racecar} shows an example of the consecutive searches. 

The search was performed on a randomly sampled graph with 500 vertices where the edges were defined for two vertices if they belong to the first 10-nearest vertices. 
The vertices were sampled using a uniform distribution in $\R^3$ bounded by the map size, and the edges were connected using Reeds-Shepp curves~\cite{Reeds1990}.
The length of Reeds-Shepp curve connecting the two end vertices is the edge weight.
We implemented LPA*, TLPA*, GLS, L-GLS, and B-LGLS as an OMPL Planner~\cite{Sucan2012} on board a RacecarJ 1/10 scale robotic vehicle equipped with an NVIDIA Jetson TX2 board (Figure~\ref{fig:jetson}). 

\begin{figure}[ht]
	\centering
	\begin{subfigure}{0.4\textwidth}
    	\includegraphics[width=\linewidth]{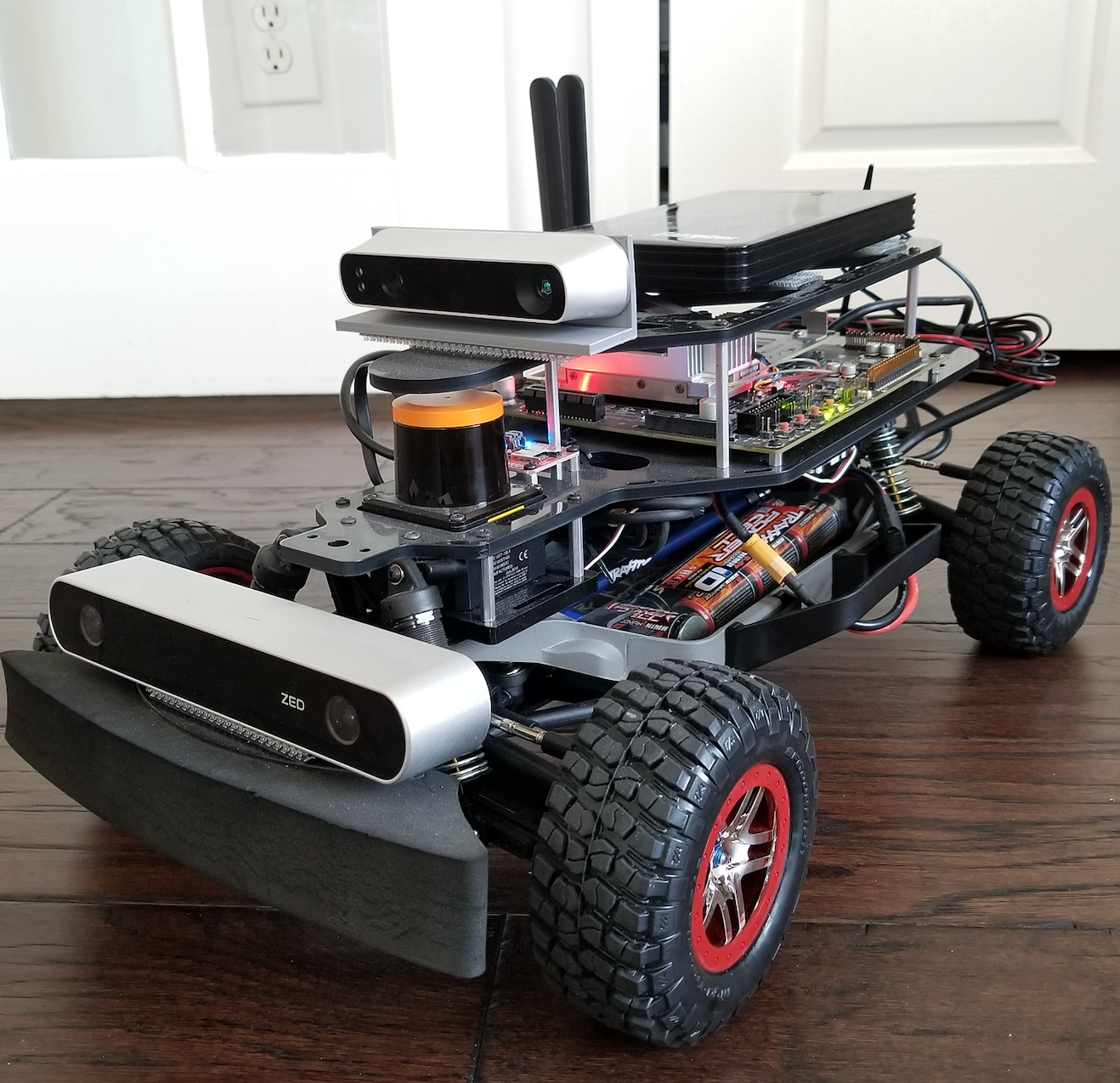}
	\end{subfigure}
	\caption{The RacecarJ platform.}
	\label{fig:jetson}
\end{figure}

We compare the search results, namely the number of edge evaluations and the number of vertex expansions for 5 different algorithms: LPA*, TLPA*~($\varepsilon_2$=2), GLS~(infinite-step lookahead), L-GLS~(infinite-step lookahead), and B-LGLS~(infinite-step lookahead, $\varepsilon_1=\sqrt{2}$, $\varepsilon_2=\sqrt{2}$).
The number of edge evaluations, the number of vertex expansions, and the runtime are recorded for each planning instance, and the average values over 100 random experiments are plotted in Figure~\ref{lgls:f:stat_racecar}.
\begin{figure*}[ht]
	\centering
	\begin{subfigure}{0.32\textwidth}
	\includegraphics[width=\myLineScale\linewidth]{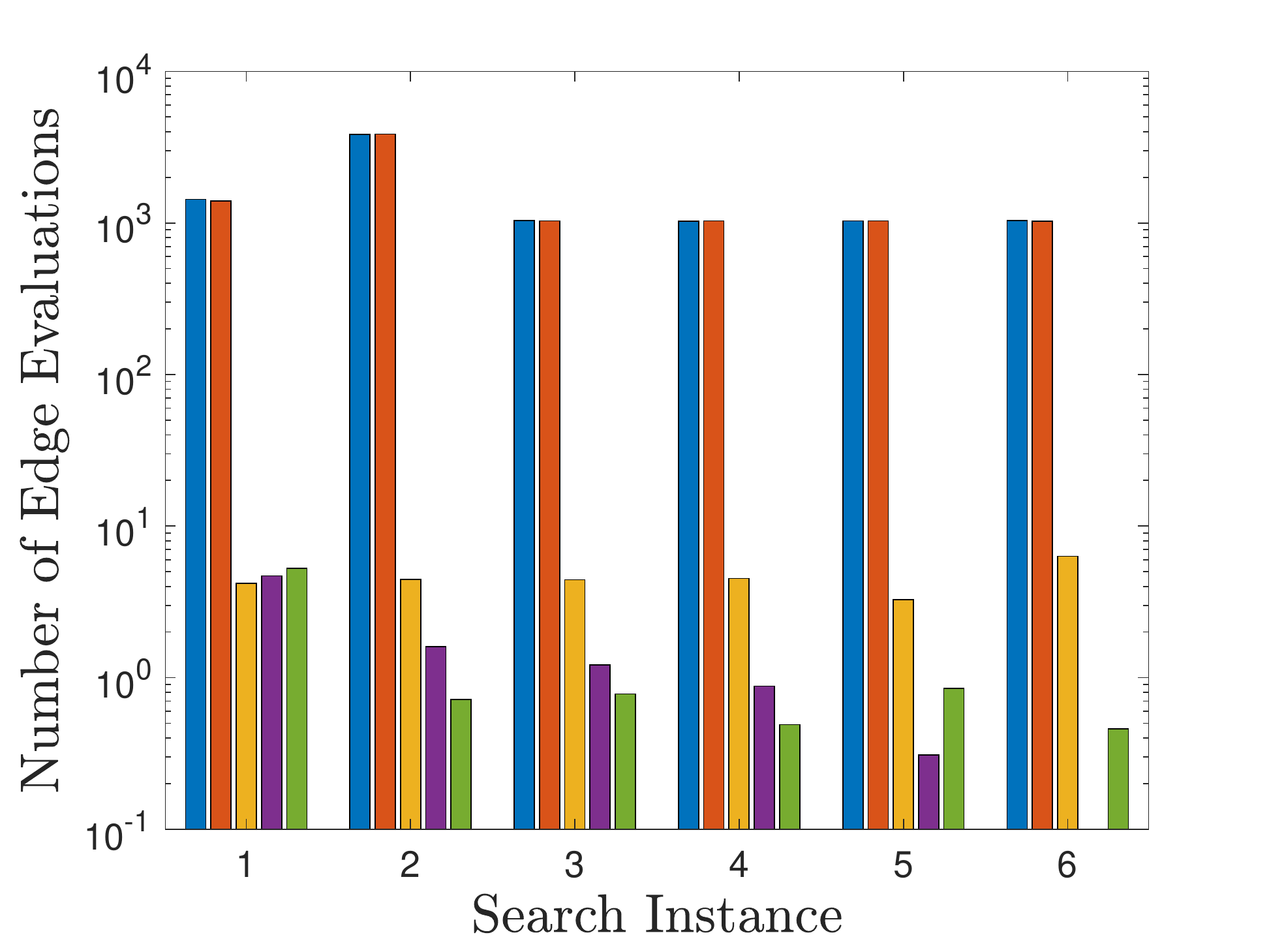}
		\caption{Edge evaluation}
	\end{subfigure}
	\begin{subfigure}{0.32\textwidth}
		\includegraphics[width=\myLineScale\linewidth]{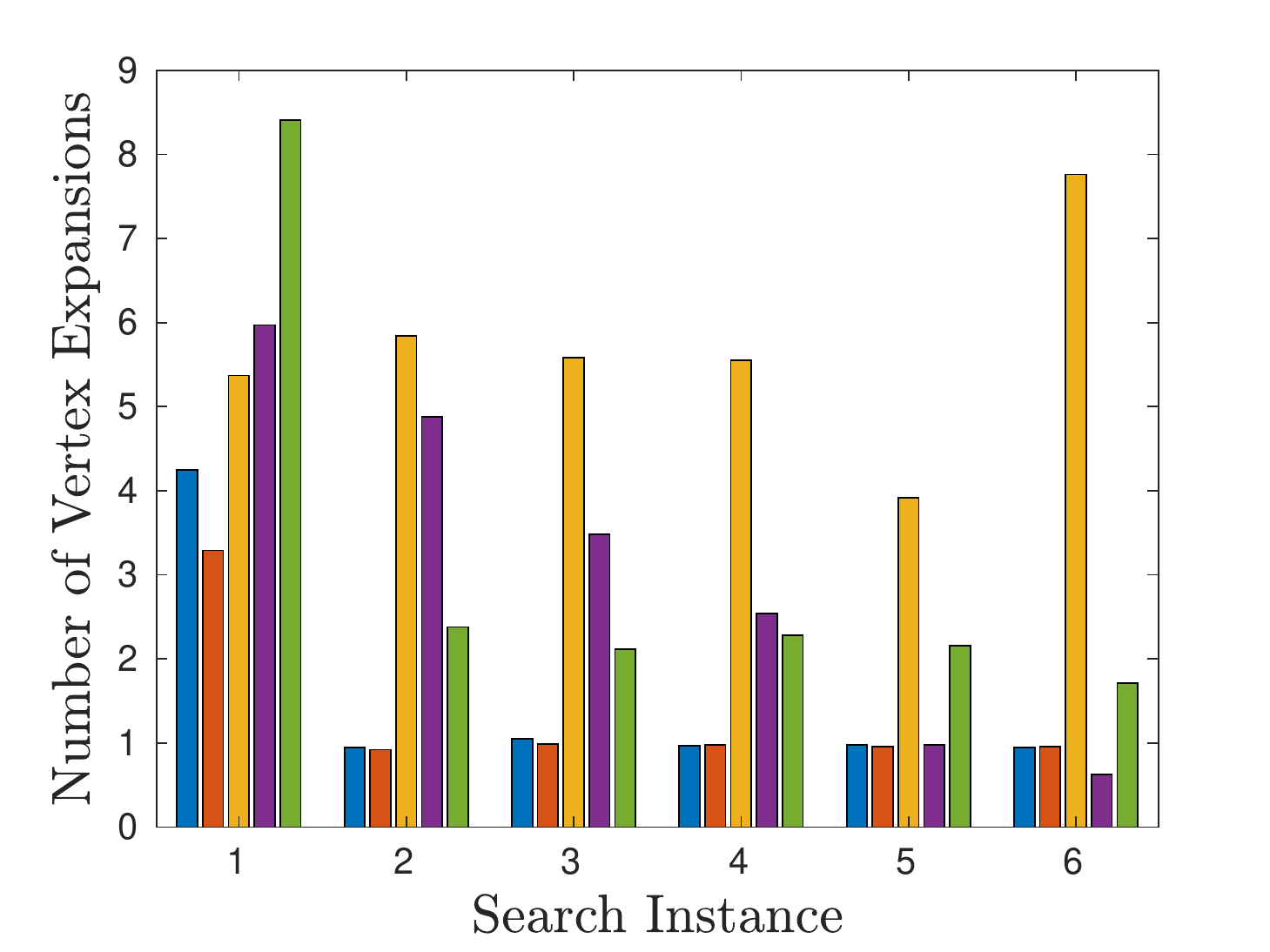}
		\caption{Vertex expansion}
	\end{subfigure}
	\begin{subfigure}{0.32\textwidth}
		\includegraphics[width=\myLineScale\linewidth]{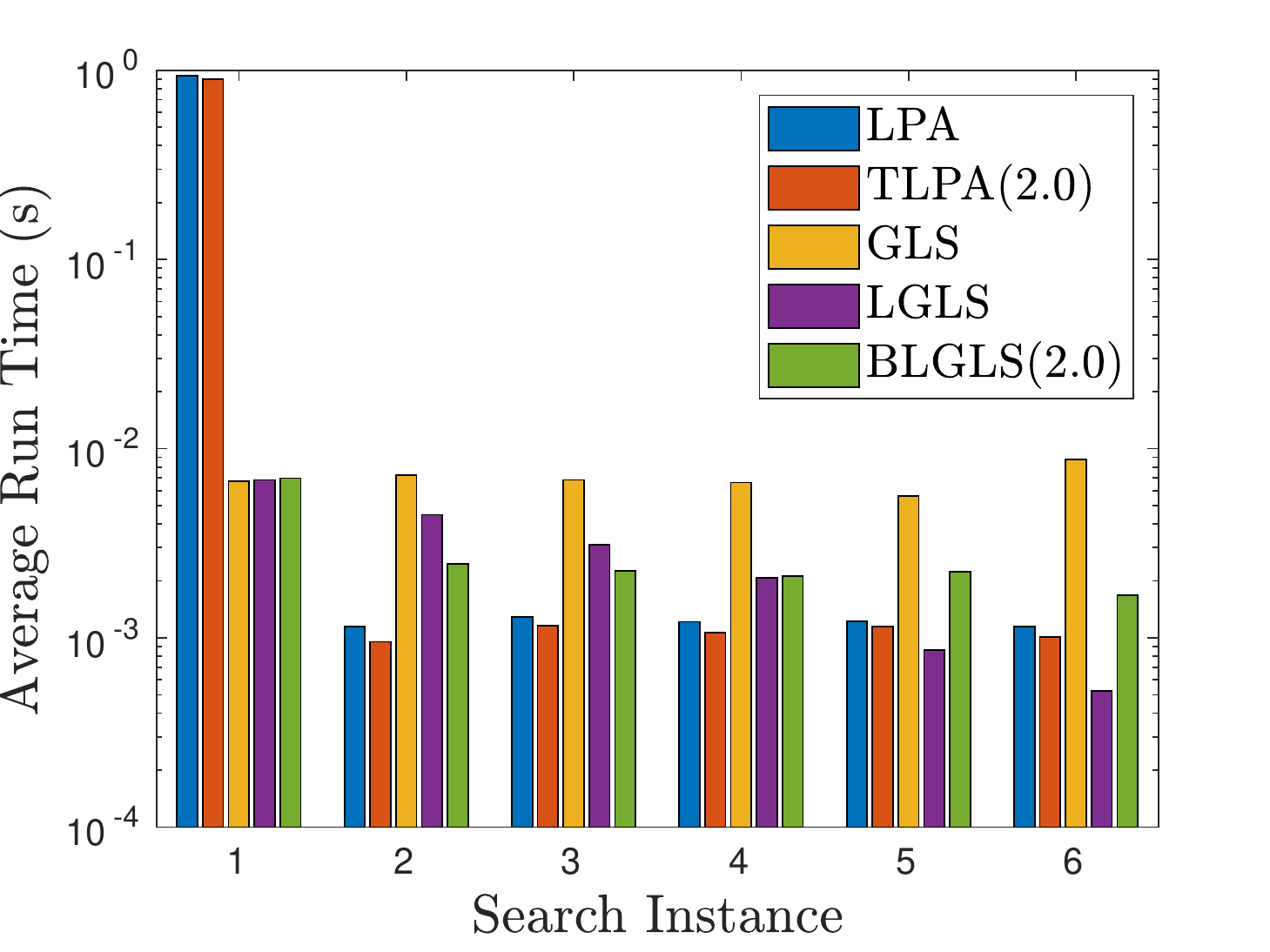}
		\caption{Total time}
	\end{subfigure}
	\caption{The average number of edge evaluations, the average number of vertex expansions, and the average total runtime for six consecutive searches for the non-holonomic racecar problem.}
	\label{lgls:f:stat_racecar}
\end{figure*}

A similar trend as before is observed in this experiment as well. 
The regular incremental search algorithms, namely, LPA* and TLPA*, incur unnecessary edge evaluations compared to the lazy search algorithms. 
The  lazy search versions of the algorithms, namely, GLS, L-GLS, and B-LGLS save a significant amount of edge evaluations at the expense of more search.
The lazy incremental search algorithms, L-GLS and B-LGLS save the computation time in consecutive searches compared to GLS, as they make use of previous search results.
B-LGLS further reduces the number of edge evaluations and vertex expansions compared to L-GLS in  consecutive searches, as it halts the search and evaluation as soon as the current solution is guaranteed to be bounded suboptimal.

\subsubsection{Static Environment and Dynamic Graph.}

In this section, we consider scenarios where the graph representation of a static environment becomes denser over time via incremental sampling. A batch of 100 random uniform samples were added per instance with rejection of infeasible samples due to obstacles, and an improved solution is sought incrementally for each instance based on the previous search result. For each search instance, the time to evaluate the edges, the time to expand the vertices, and the total time to find a new solution were recorded. Each experiment was repeated 50 times, and the average and the standard deviation of the accumulated time of edge evaluations, the accumulated time of vertex expansions, and the solution length as a function of time are plotted for the success ratio to find a solution is above 80\%.

First, we study the effect of the three parameters of B-LGLS, namely the lookahead value, the inflation factor, and the truncation factor. 

\paragraph{Lookahead Variation:} \label{lgls:e:lookahead}

The lookahead value determines how far the heuristic search tree is repaired before evaluating the edges along the path. 
Hence, a higher lookahead value delays edge evaluation, but the search tree may become inconsistent often when there is a discrepancy between the heuristic value and the actual value, requiring more vertex expansions. On the other hand, a lower lookahead value reduces the vertex expansions at the expanse of more edge evaluations. 
In all of the Piano Movers' problem instances, the PR2 robotic arm, and the non-holonomic racecar problems, this trend is clear. 
Figures~\ref{lgls:f:piano_inc_lookahead}, \ref{lgls:f:pr2_inc_lookahead}, and
\ref{lgls:f:racecar_inc_lookahead} show the total edge evaluation time and the total vertex expansion time accumulated as the number of vertices in the current graph for the Piano Movers' problem, the PR2 robotic arm problem, and the non-holonomic racecar problem, respectively.
The higher the lookahead value, the lower the accumulated time to evaluate edges. The lower the lookahead value, the lower the accumulated time to expand vertices. The trade-off between the edge evaluation and the vertex expansion determines the solution convergence rate.

\begin{figure*}[ht]
	\centering
	\begin{subfigure}{0.32\textwidth}
		\includegraphics[width=\myLineScale\linewidth]{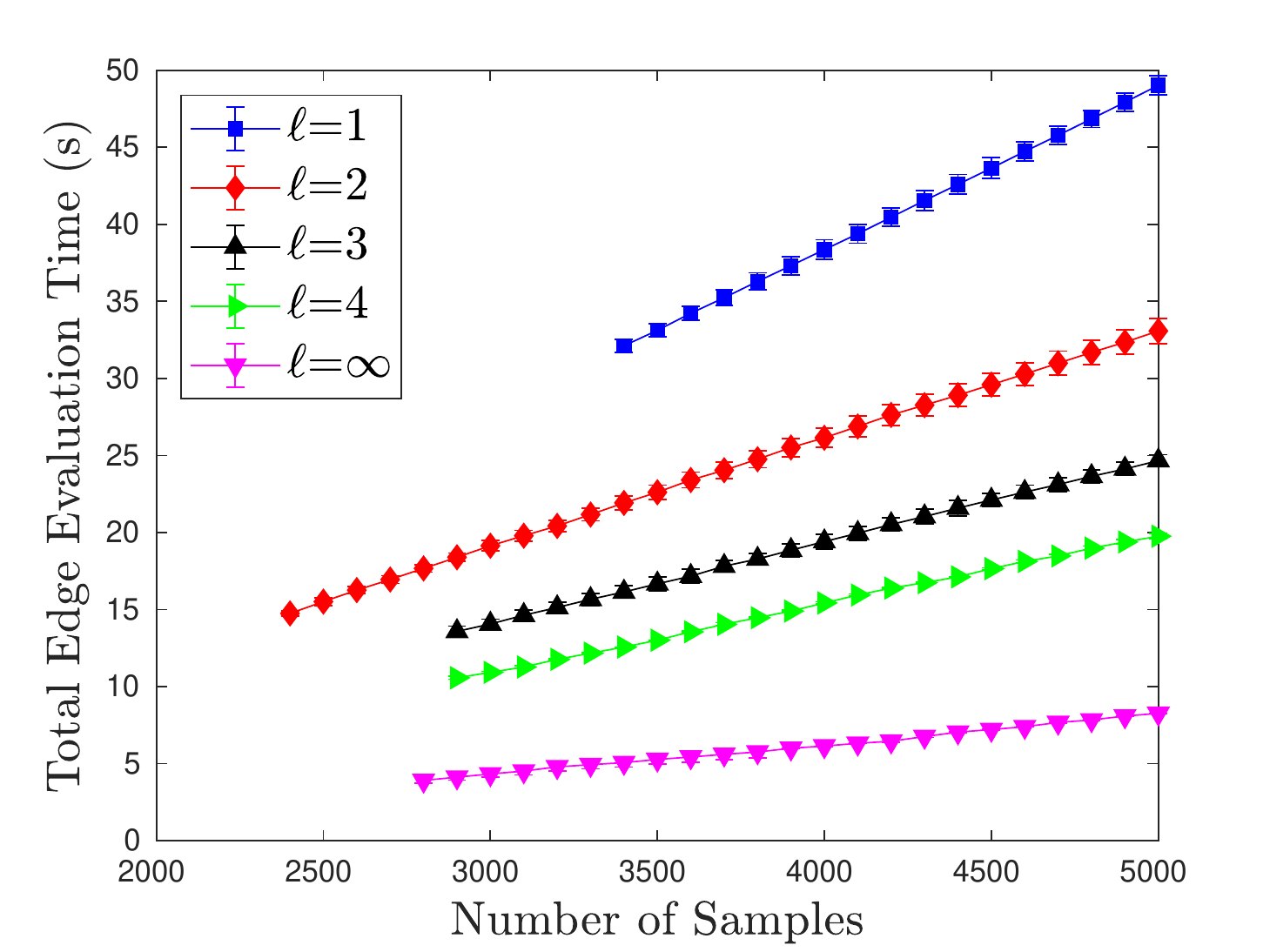}
		\caption{Edge Evaluation}
	\end{subfigure}
	\begin{subfigure}{0.32\textwidth}
		\includegraphics[width=\myLineScale\linewidth]{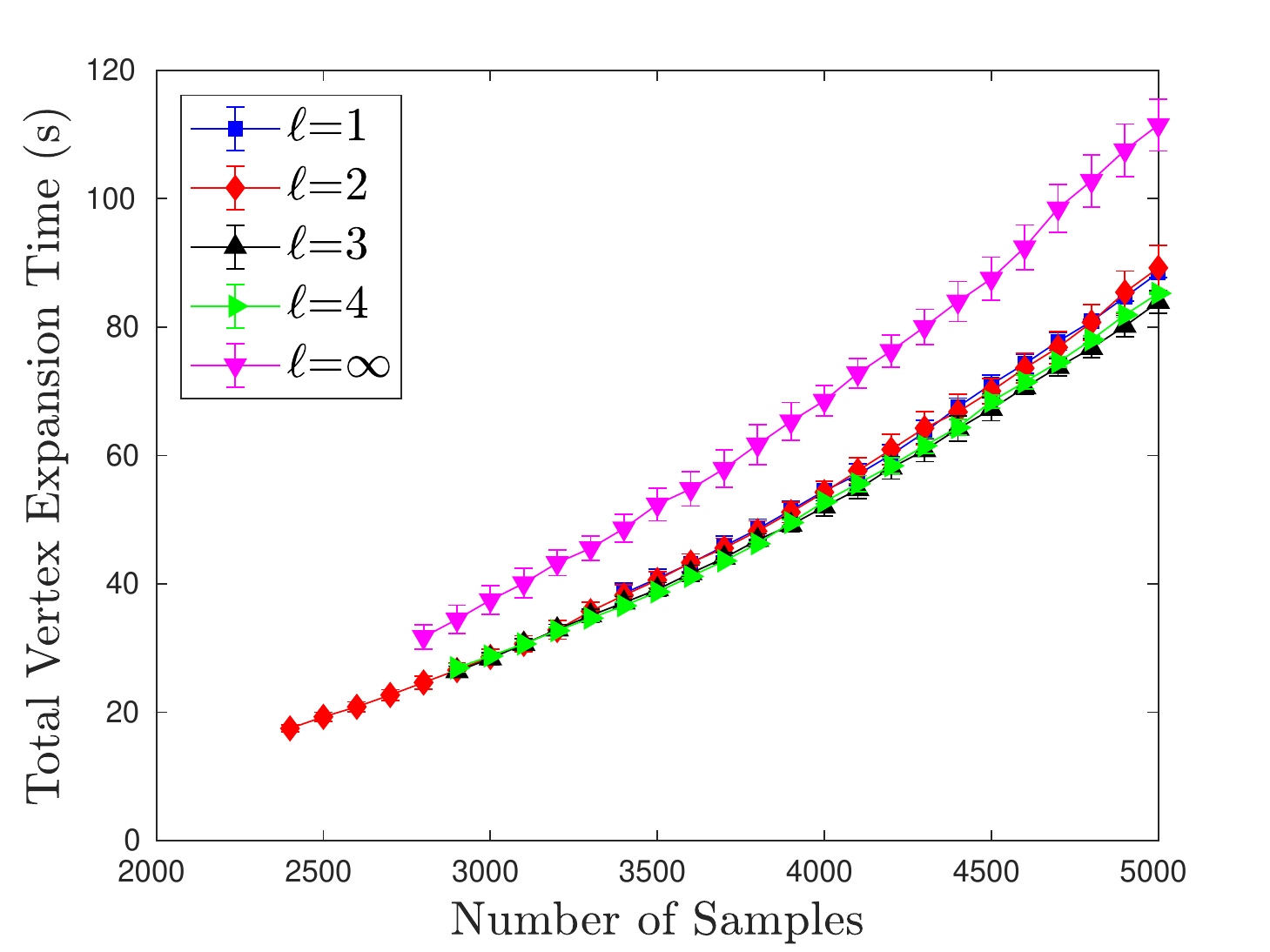}
		\caption{Vertex Expansion}
	\end{subfigure}
	\begin{subfigure}{0.32\textwidth}
		\includegraphics[width=\myLineScale\linewidth]{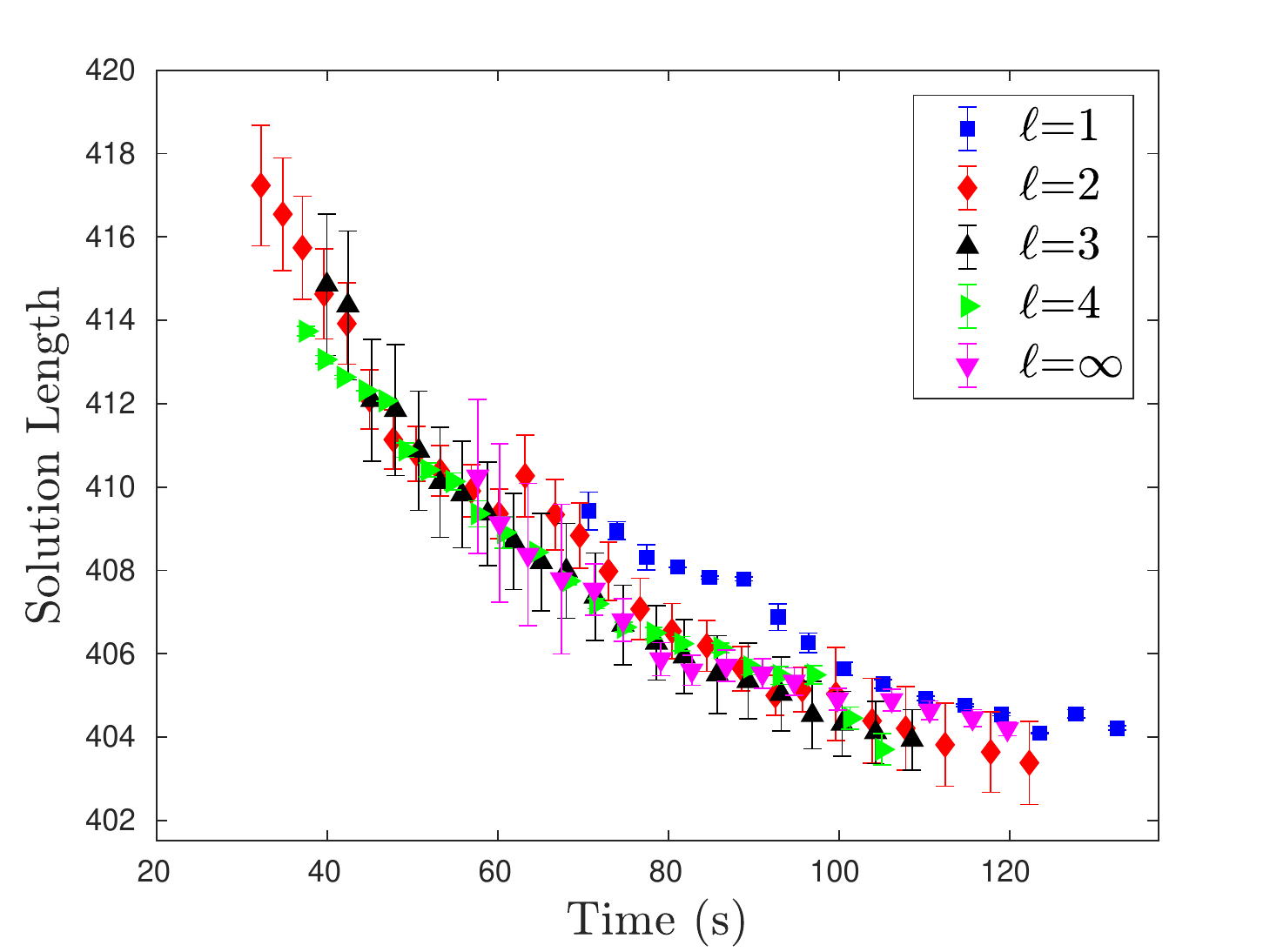}
		\caption{Solution Length}
	\end{subfigure}
	\caption{Lookahead variation of the Piano Movers' problem with $\varepsilon_1=1$ and $\varepsilon_2=1$.}
	\label{lgls:f:piano_inc_lookahead}
\end{figure*}

\begin{figure*}[ht]
	\centering
	\begin{subfigure}{0.32\textwidth}
		\includegraphics[width=\myLineScale\linewidth]{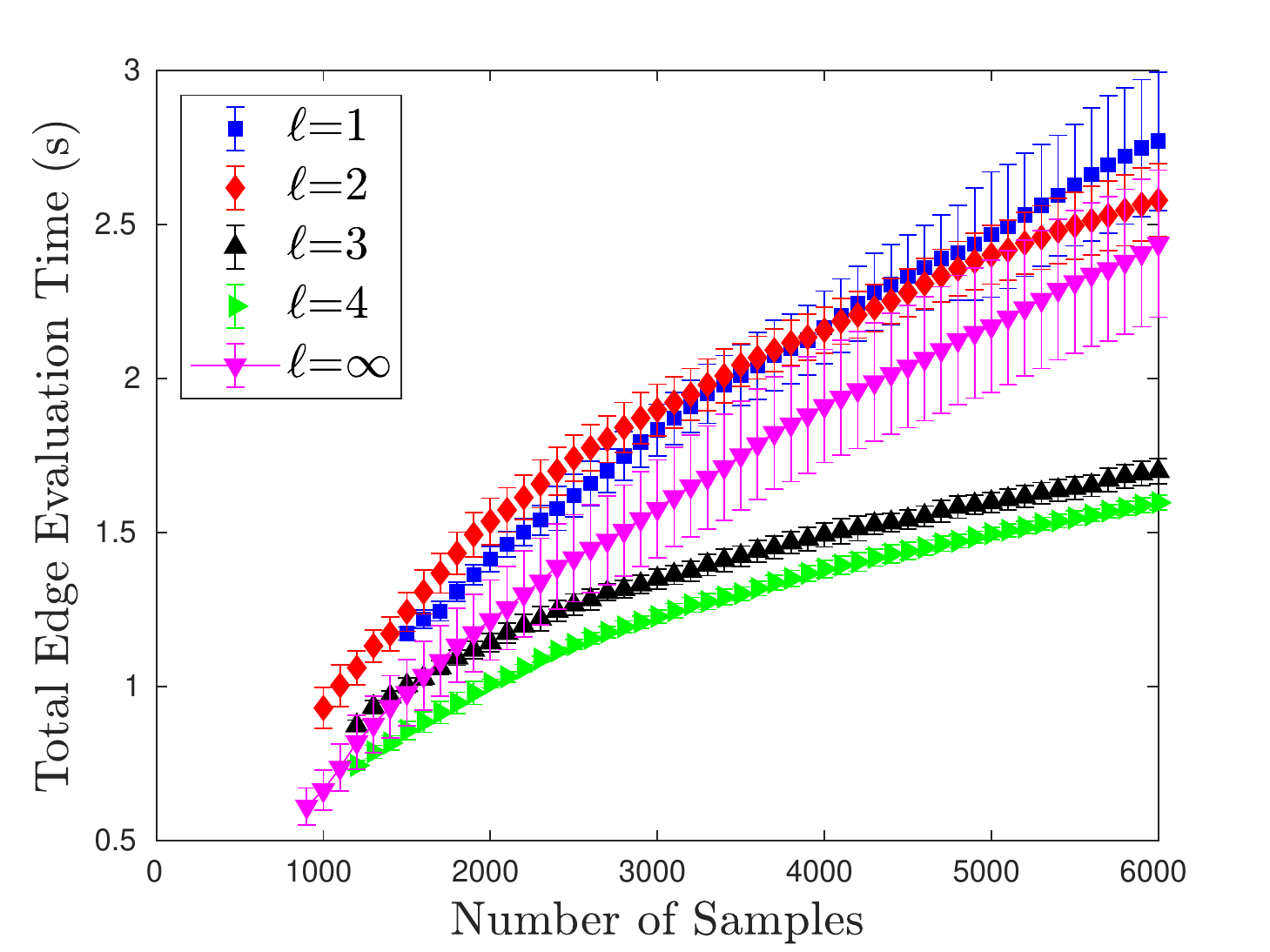}
		\caption{Edge Evaluation}
	\end{subfigure}
	\begin{subfigure}{0.32\textwidth}
		\includegraphics[width=\myLineScale\linewidth]{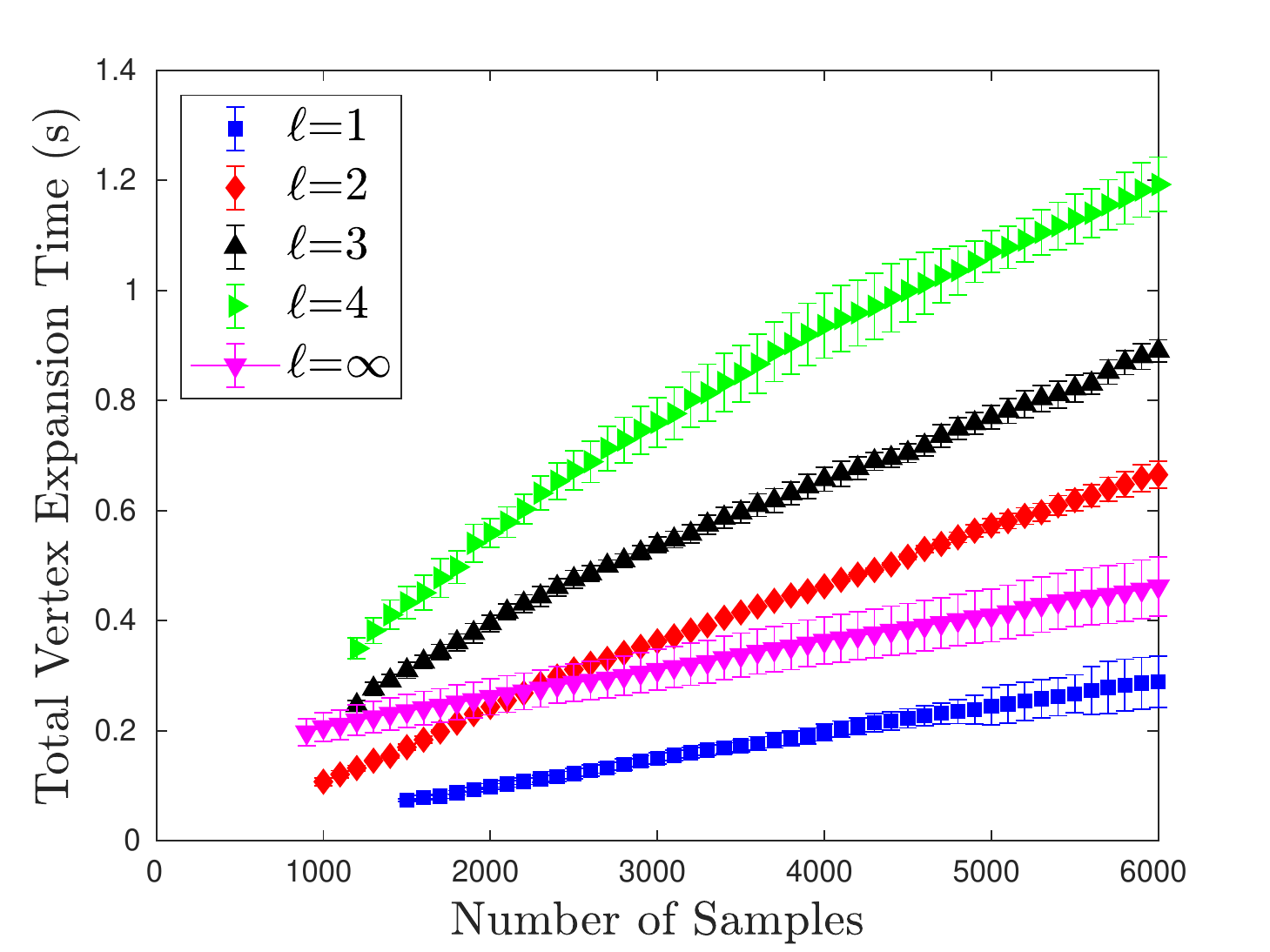}
		\caption{Vertex Expansion}
	\end{subfigure}
	\begin{subfigure}{0.32\textwidth}
		\includegraphics[width=\myLineScale\linewidth]{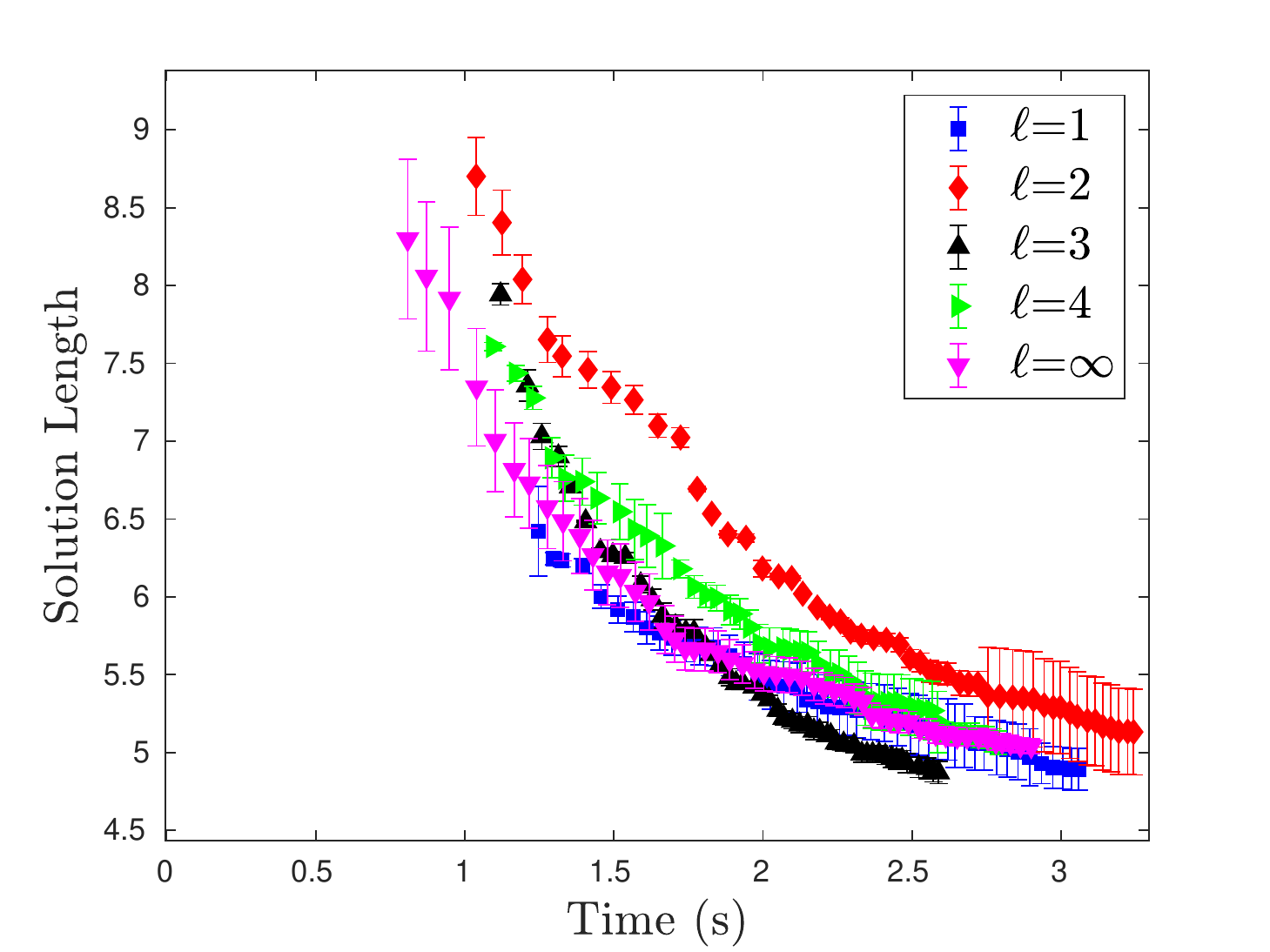}
		\caption{Solution Length}
	\end{subfigure}
	\caption{Lookahead variation of the PR2 robotic arm problem with $\varepsilon_1=1$ and $\varepsilon_2=1$.}
	\label{lgls:f:pr2_inc_lookahead}
\end{figure*}

\begin{figure*}[ht]
	\centering
	\begin{subfigure}{0.32\textwidth}
		\includegraphics[width=\myLineScale\linewidth]{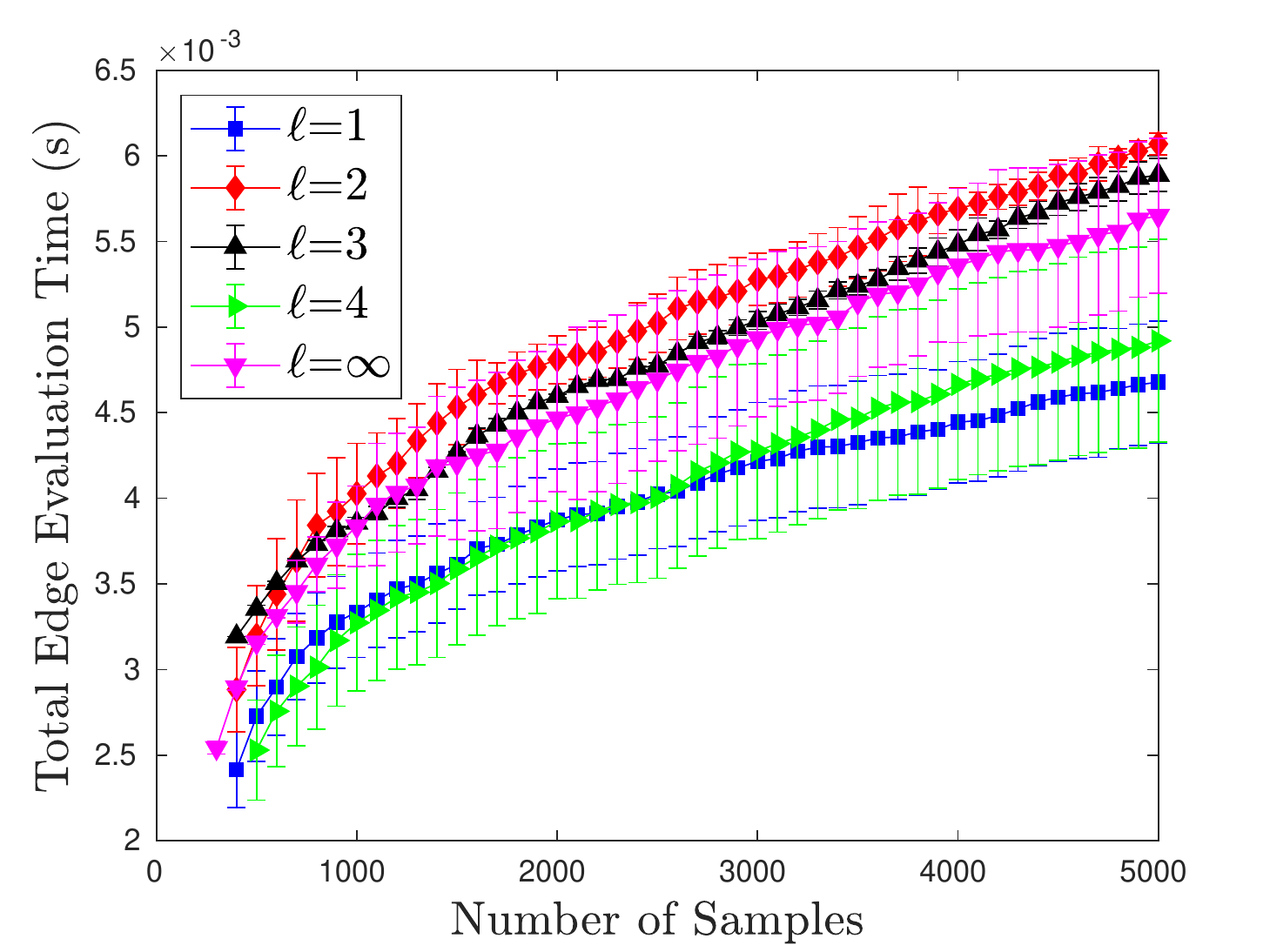}
		\caption{Edge Evaluation}
	\end{subfigure}
	\begin{subfigure}{0.32\textwidth}
		\includegraphics[width=\myLineScale\linewidth]{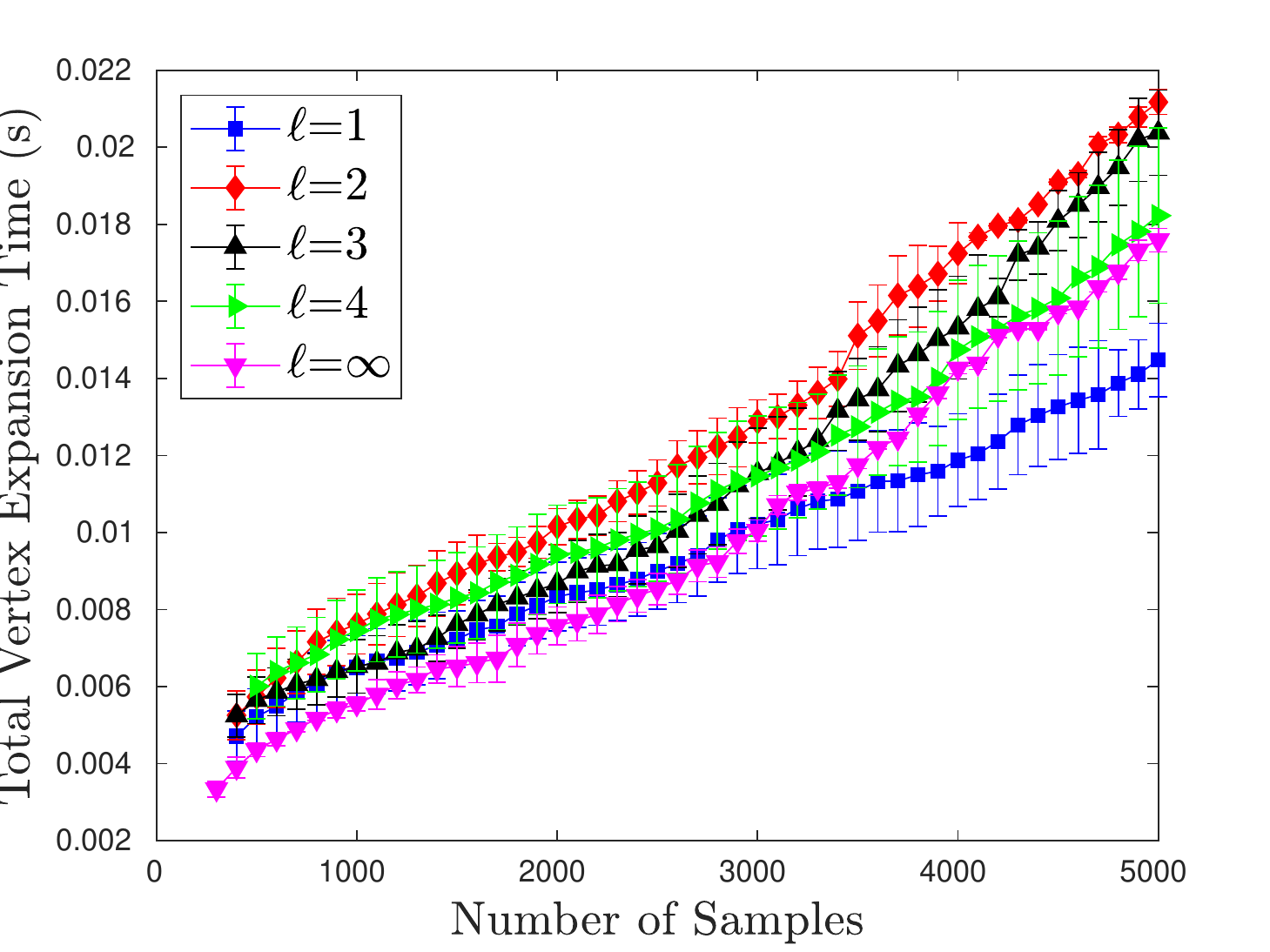}
		\caption{Vertex Expansion}
	\end{subfigure}
	\begin{subfigure}{0.32\textwidth}
		\includegraphics[width=\myLineScale\linewidth]{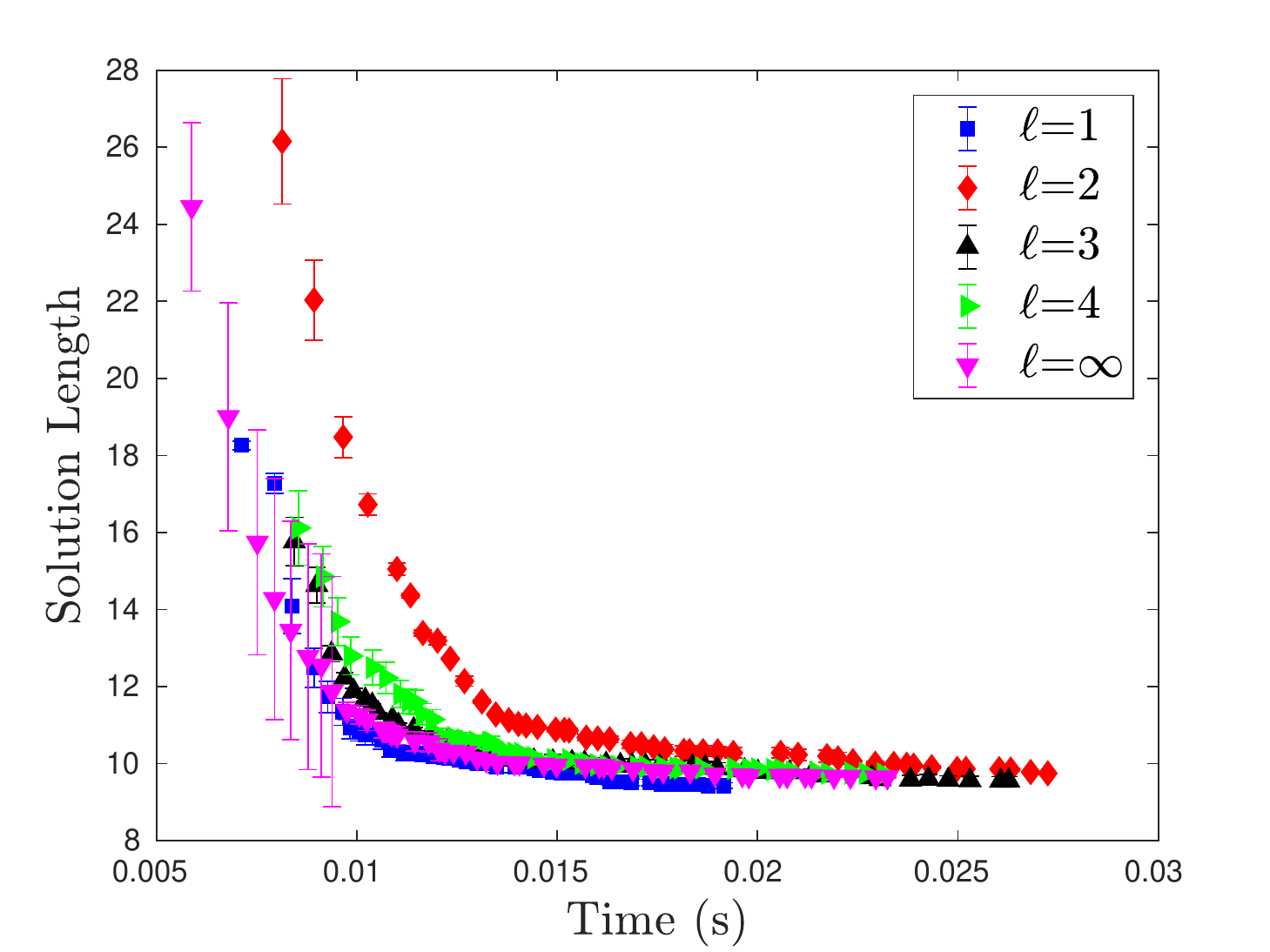}
		\caption{Solution Length}
	\end{subfigure}
	\caption{Lookahead variation of the racecar problem with $\varepsilon_1=1$ 
	and $\varepsilon_2=1$.}
	\label{lgls:f:racecar_inc_lookahead}
\end{figure*}

\paragraph{Effect of Inflation Factor:}

We also investigated the effect of the inflation factor variation of the solution quality and speed.
We inflated the heuristic edge value of the unevaluated edges by a multiplicative factor.
The inflation factor has two effects.
First, it makes the search greedy, making to find the first solution faster. 
Second, it defers the evaluation of newly added edges unless a path including such edges are promising to improve the current solution significantly.
Although a high inflation factor makes it faster to find a first solution, it may not converge asymptotically to the optimal solution if the current solution is already within the given bound.
To address this issue, 
we dynamically varied the inflation factor as a function of the number of vertices in the graph. 
We chose $\varepsilon_1 = 1+5/q$, where $q$ is the number of vertices~\cite{Strub2020b}. 
Having a dynamic inflation factor finds a first solution fast, and then it asymptotically converges to the optimal solution as $\varepsilon_1$ converges to 1.
Figures~\ref{lgls:f:piano_inc_inflation}, \ref{lgls:f:pr2_inc_inflation}, and
\ref{lgls:f:racecar_inc_inflation} show the total edge evaluation time and the total vertex expansion time accumulated as the number of vertices in the current graph for the Piano Movers' problem, the PR2 robotic arm problem, and the non-holonomic racecar problem, respectively. 
Having an inflation factor greater than 1 reduced the edge evaluation time, but it also incurred more vertex expansion time, generally resulting in a slower convergence. 

\begin{figure*}[ht]
	\centering
	\begin{subfigure}{0.32\textwidth}
		\includegraphics[width=\myLineScale\linewidth]{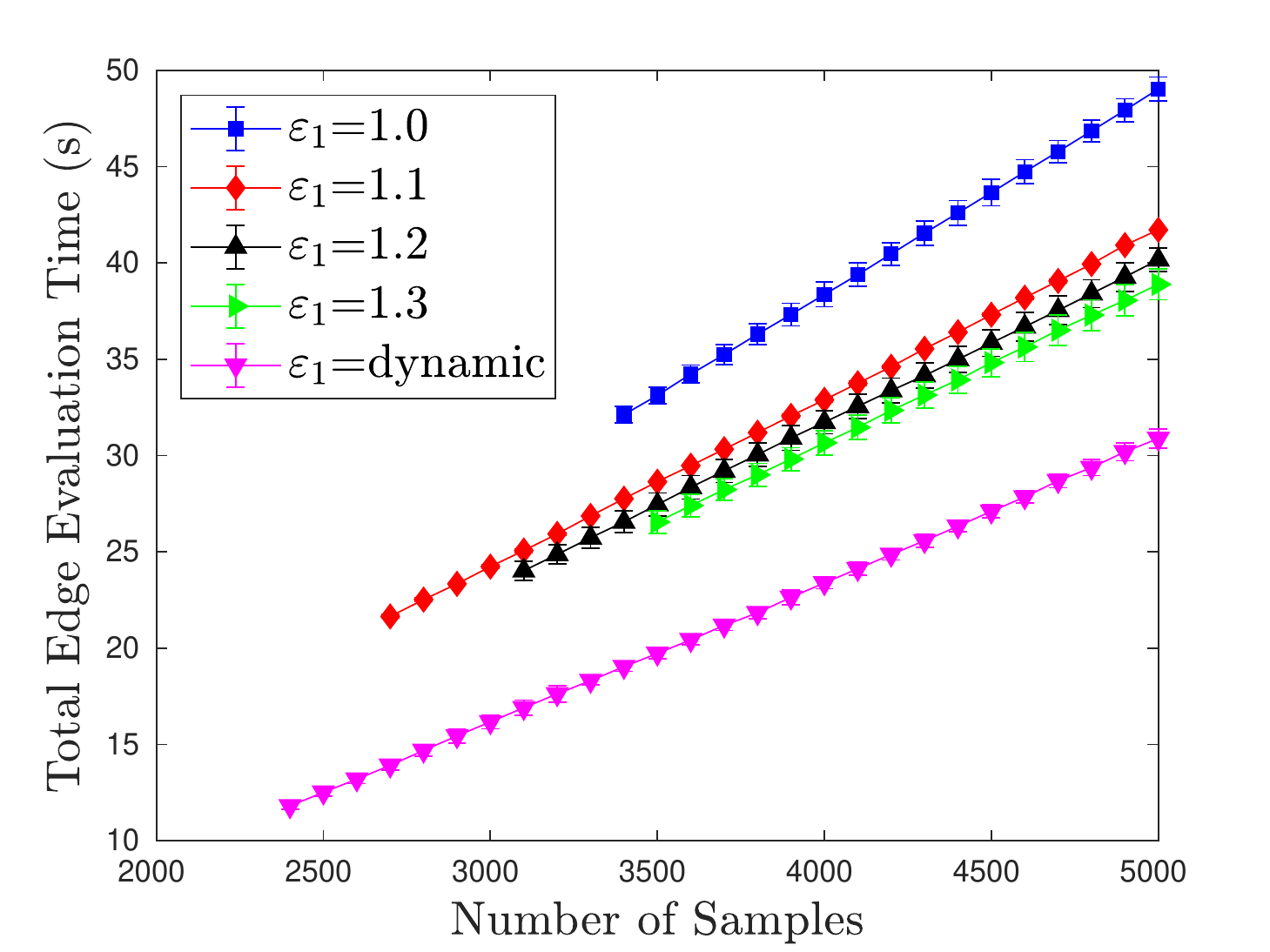}
		\caption{Edge Evaluation}
	\end{subfigure}
	\begin{subfigure}{0.32\textwidth}
		\includegraphics[width=\myLineScale\linewidth]{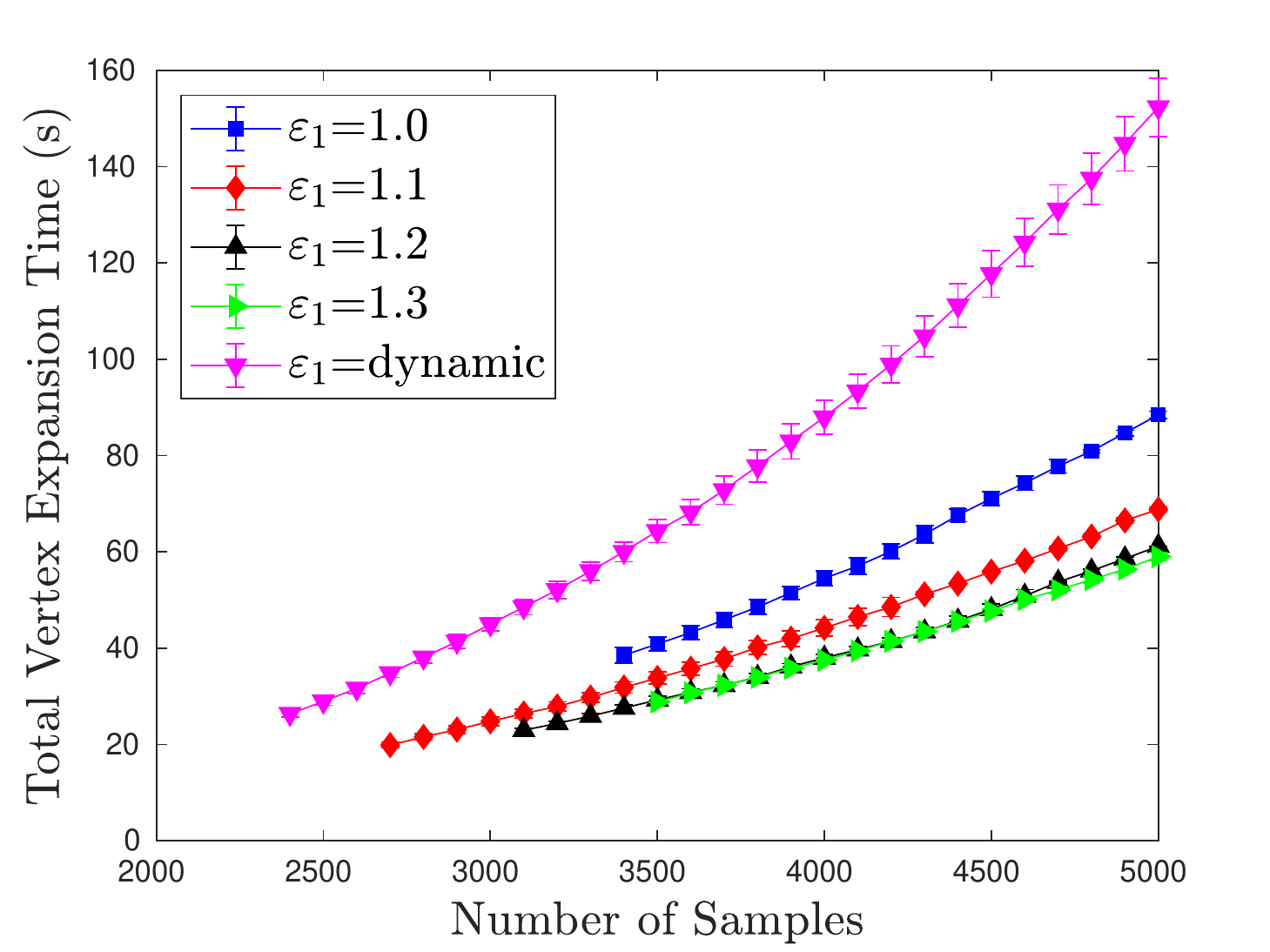}
		\caption{Vertex Expansion}
	\end{subfigure}
	\begin{subfigure}{0.32\textwidth}
		\includegraphics[width=\myLineScale\linewidth]{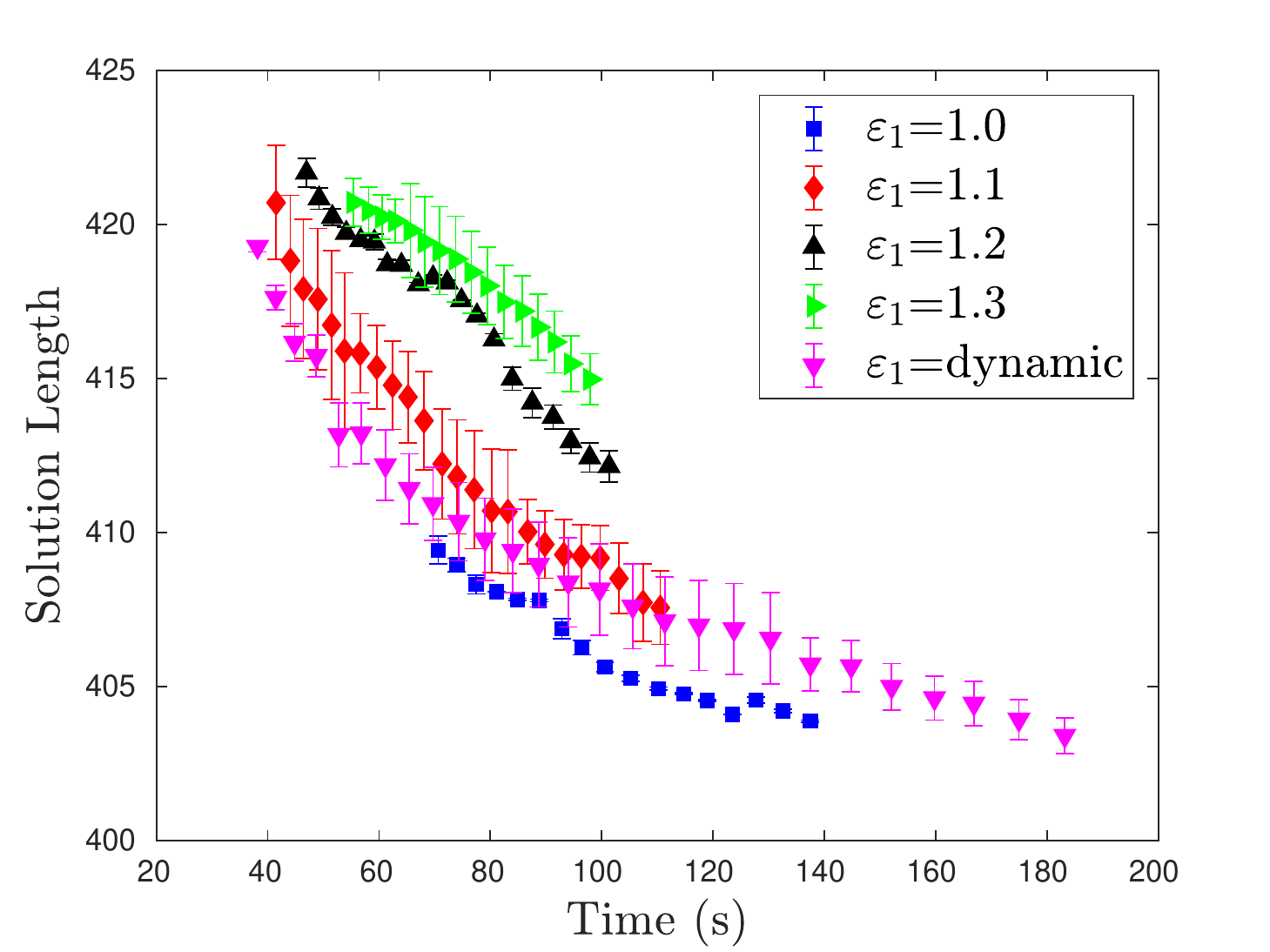}
		\caption{Solution Length}
	\end{subfigure}
	\caption{Inflation factor variation of the Piano Movers' problem.}
	\label{lgls:f:piano_inc_inflation}
\end{figure*}

\begin{figure*}[ht]
	\centering
	\begin{subfigure}{0.32\textwidth}
		\includegraphics[width=\myLineScale\linewidth]{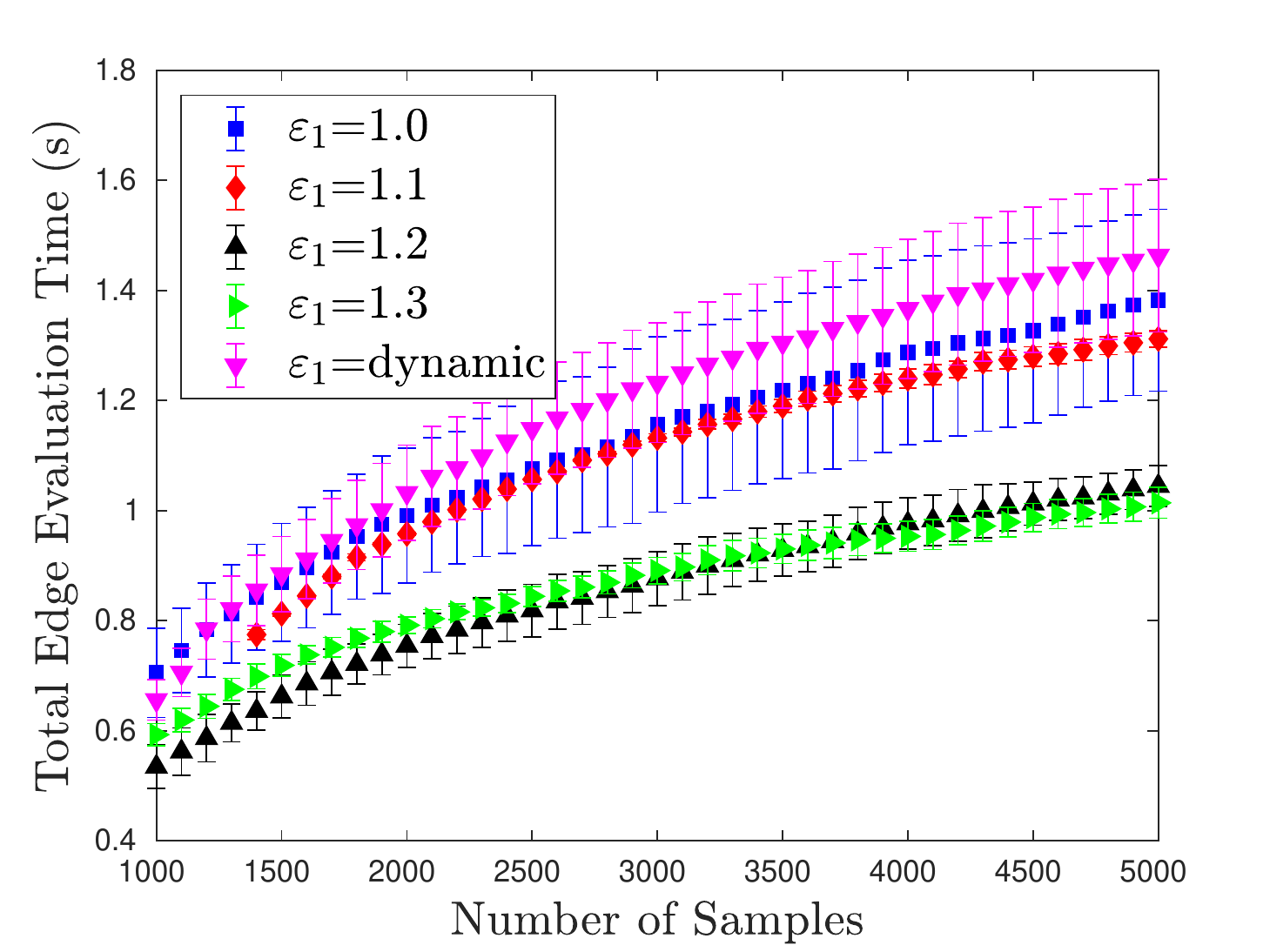}
		\caption{Edge Evaluation}
	\end{subfigure}
	\begin{subfigure}{0.32\textwidth}
		\includegraphics[width=\myLineScale\linewidth]{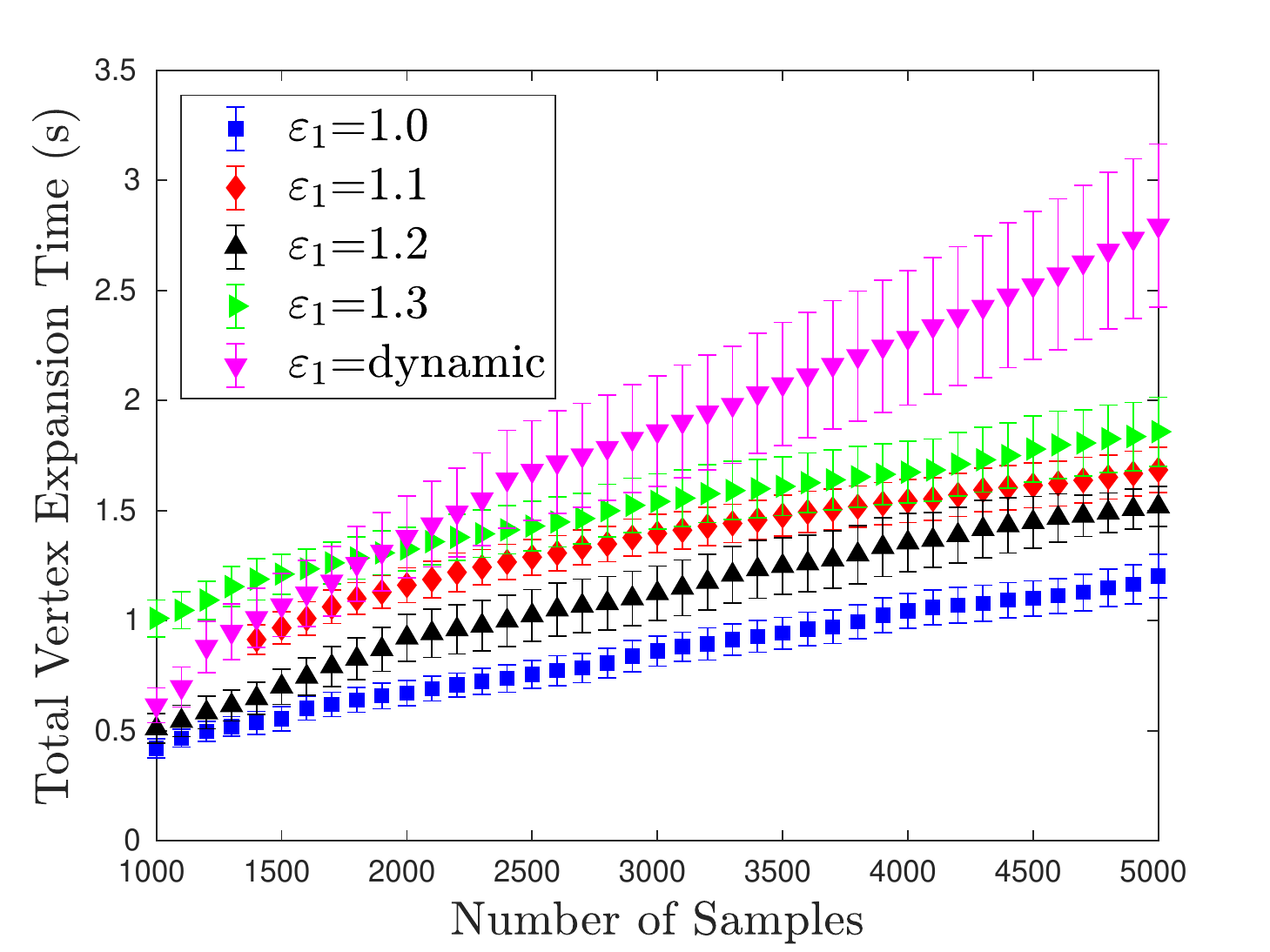}
		\caption{Vertex Expansion}
	\end{subfigure}
	\begin{subfigure}{0.32\textwidth}
		\includegraphics[width=\myLineScale\linewidth]{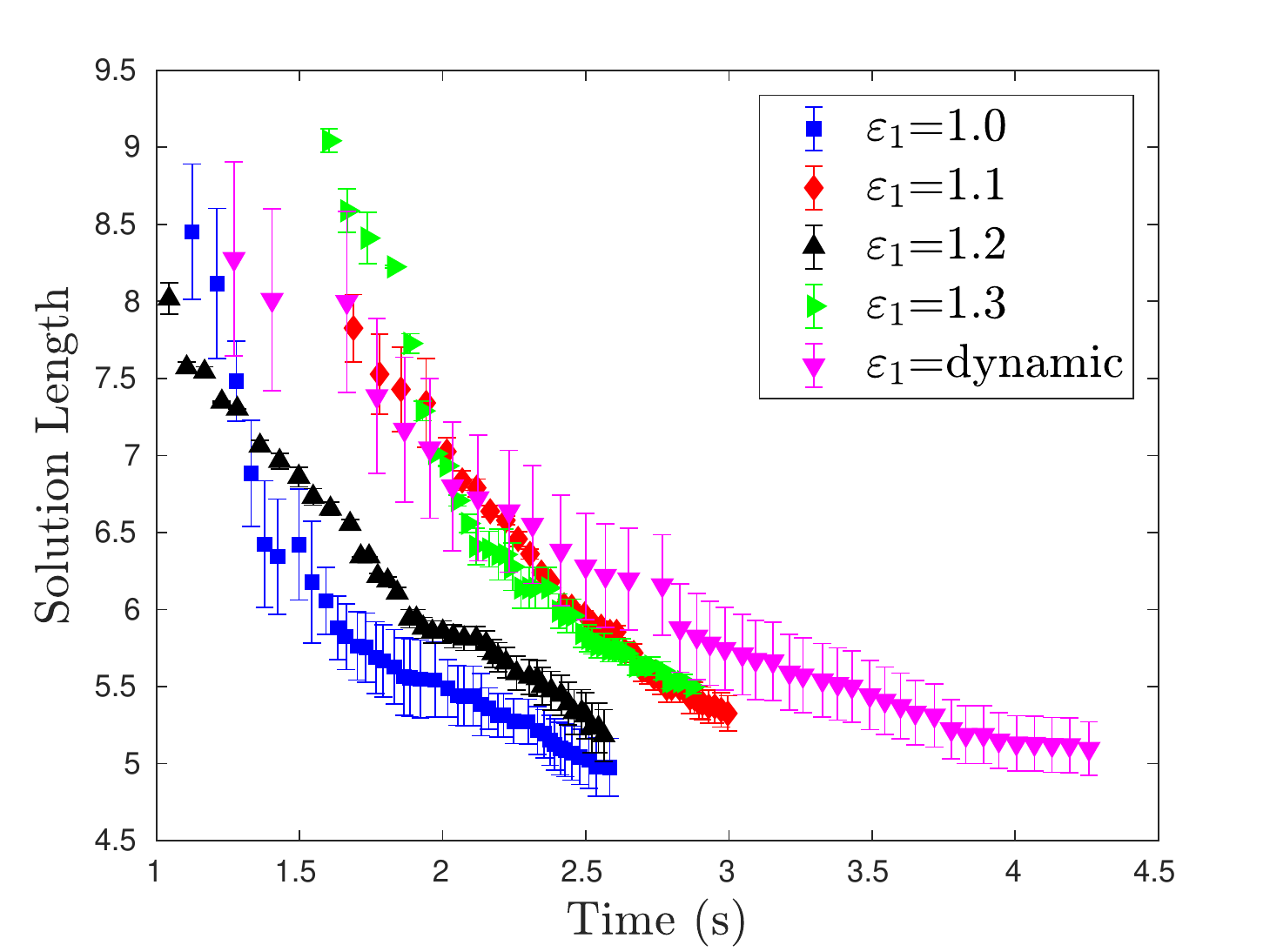}
		\caption{Solution Length}
	\end{subfigure}
	\caption{Inflation factor variation of the PR2 robotic arm problem.}
	\label{lgls:f:pr2_inc_inflation}
\end{figure*}

\begin{figure*}[ht]
	\centering
	\begin{subfigure}{0.32\textwidth}
		\includegraphics[width=\myLineScale\linewidth]{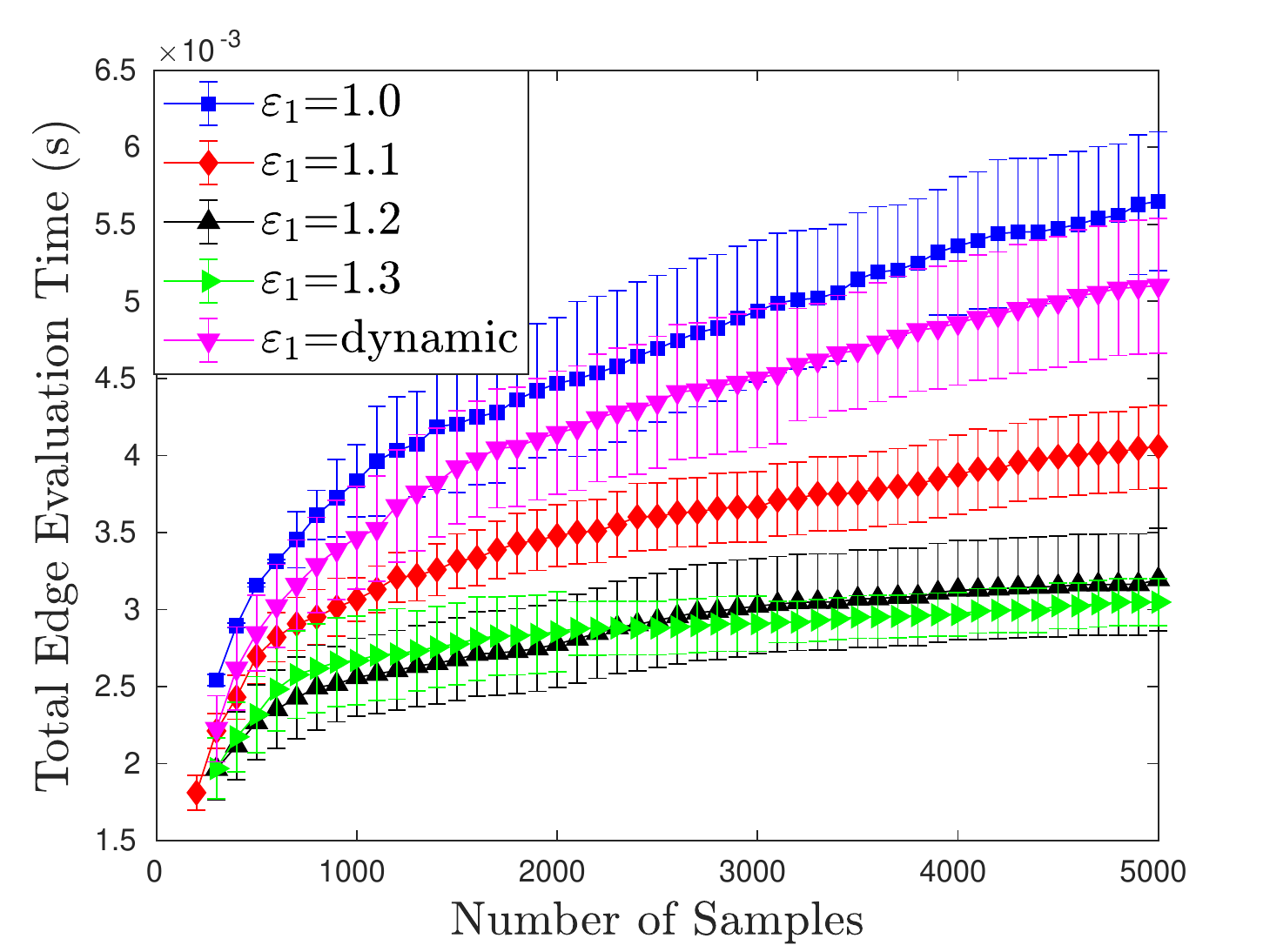}
		\caption{Edge Evaluation}
	\end{subfigure}
	\begin{subfigure}{0.32\textwidth}
		\includegraphics[width=\myLineScale\linewidth]{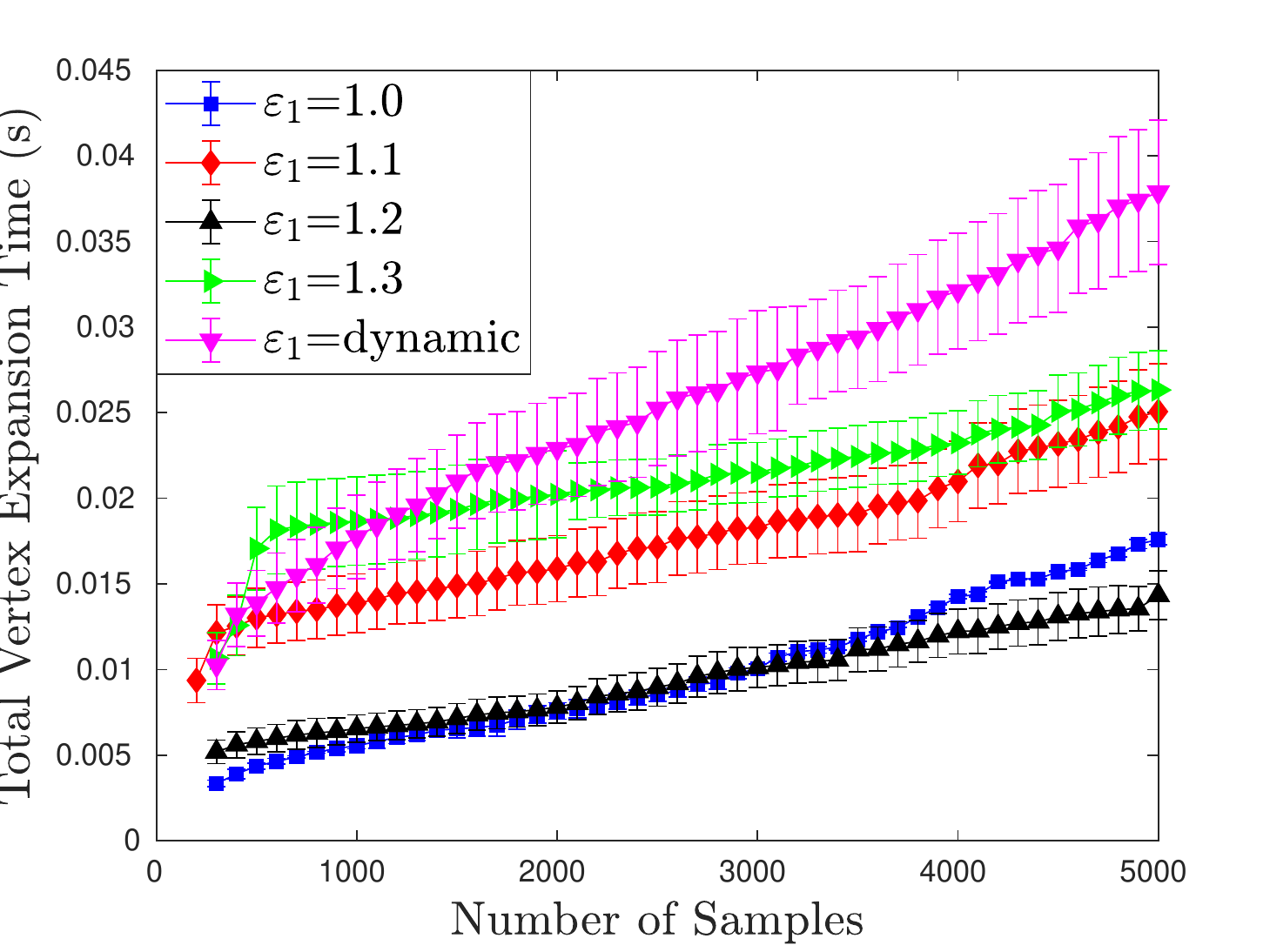}
		\caption{Vertex Expansion}
	\end{subfigure}
	\begin{subfigure}{0.32\textwidth}
		\includegraphics[width=\myLineScale\linewidth]{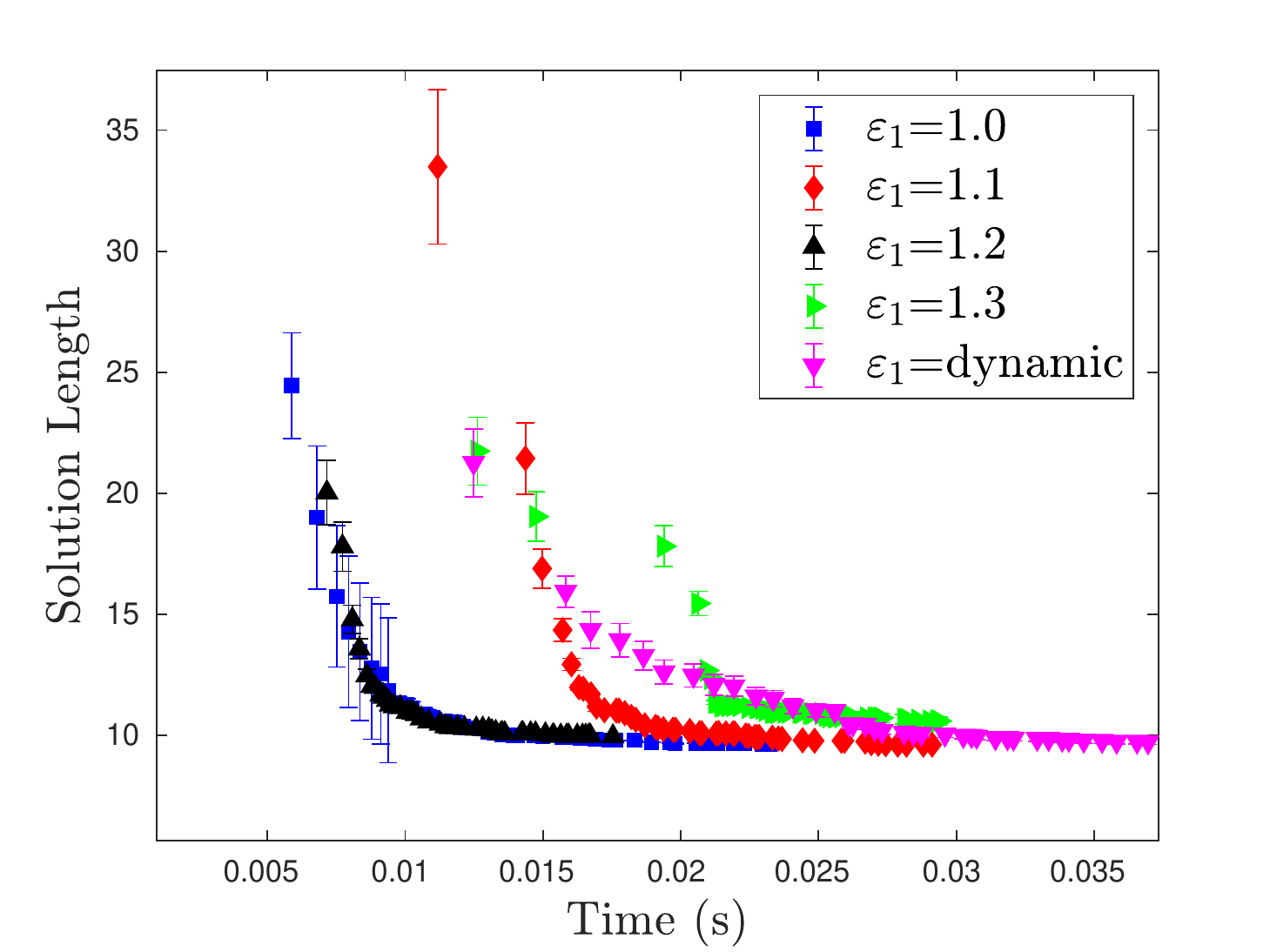}
		\caption{Solution Length}
	\end{subfigure}
	\caption{Inflation factor variation of the RacecarJ problem.}
	\label{lgls:f:racecar_inc_inflation}
\end{figure*}

\paragraph{Truncation Factor Variation:}

The truncation factor determines when to truncate the inconsistency propagation given the changes in the graph. It truncates the inconsistency propagation as soon as the current solution is guaranteed to be bounded suboptimal, saving computational resources until the samples are drawn, which could potentially improve the current solution significantly. 
It saves both the edge evaluation and the vertex expansion times; 
however, as a result of ignoring samples that do not improve the solution by a certain margin, the solution may not converge to the optimal solution. 
To address this issue, we dynamically reduce the truncation factor as a function of the number of vertices, i.e., $\varepsilon_2 = 1 + 10/q$, where $q$ is the number of vertices. 
As $\varepsilon_2$ approaches 1, the solution asymptotically converges to the optimal solution.
Figures~\ref{lgls:f:piano_inc_truncation}, \ref{lgls:f:pr2_inc_truncation}, and
\ref{lgls:f:racecar_inc_truncation} show the total edge evaluation time and the total vertex expansion time accumulated as the number of vertices in the current graph for the piano movers' problem, the PR2 robotic arm problem, and the non-holonomic racecar problem, respectively. 
In all three experiments, having a truncation factor greater than 1 helped find the first solution faster. 
Having a dynamic truncation factor improved the convergence in the Piano Movers' problem and the PR2 robotic arm problem.

\begin{figure*}[ht]
	\centering
	\begin{subfigure}{0.32\textwidth}
		\includegraphics[width=\myLineScale\linewidth]{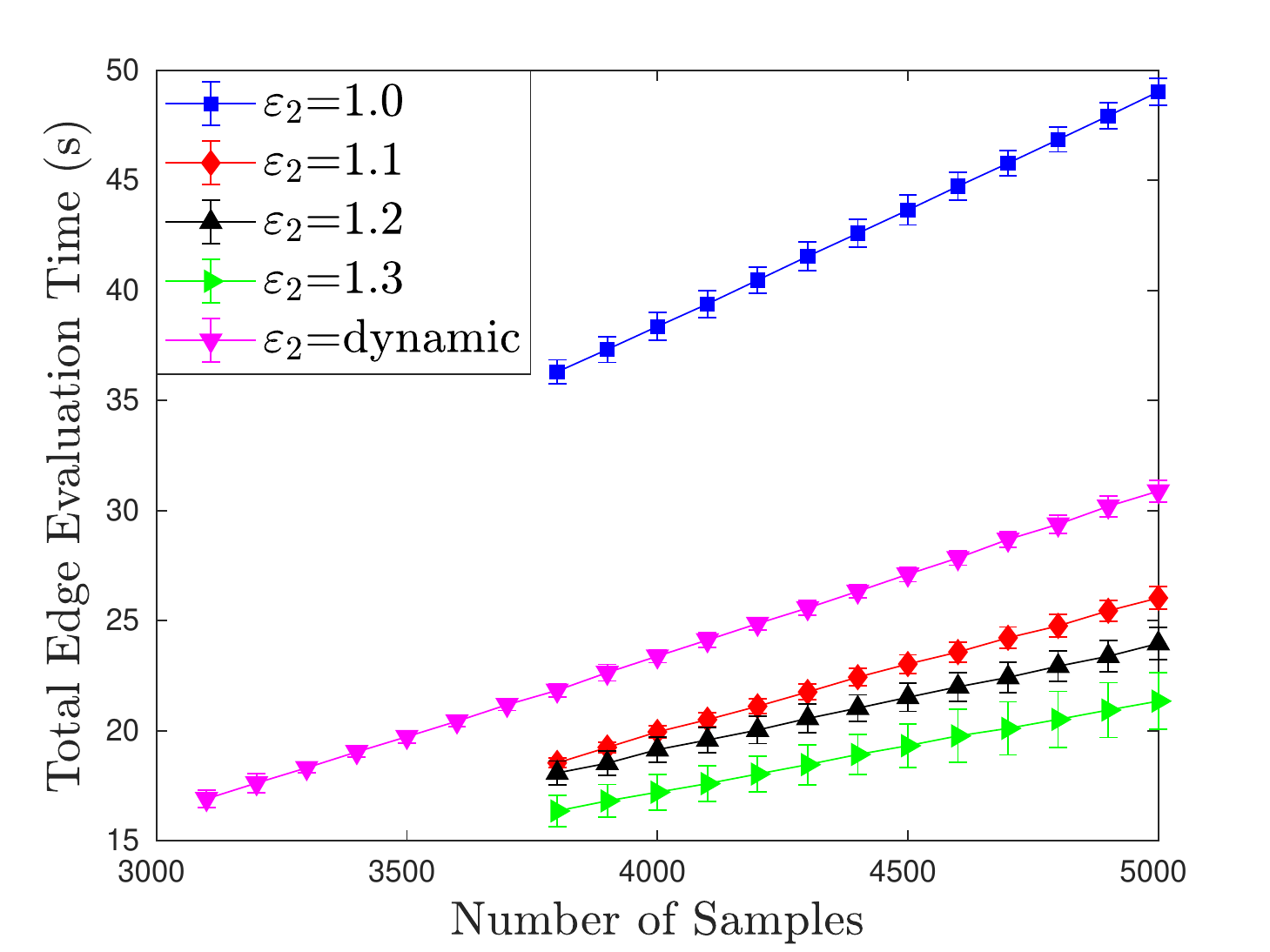}
		\caption{Edge Evaluation}
	\end{subfigure}
	\begin{subfigure}{0.32\textwidth}
		\includegraphics[width=\myLineScale\linewidth]{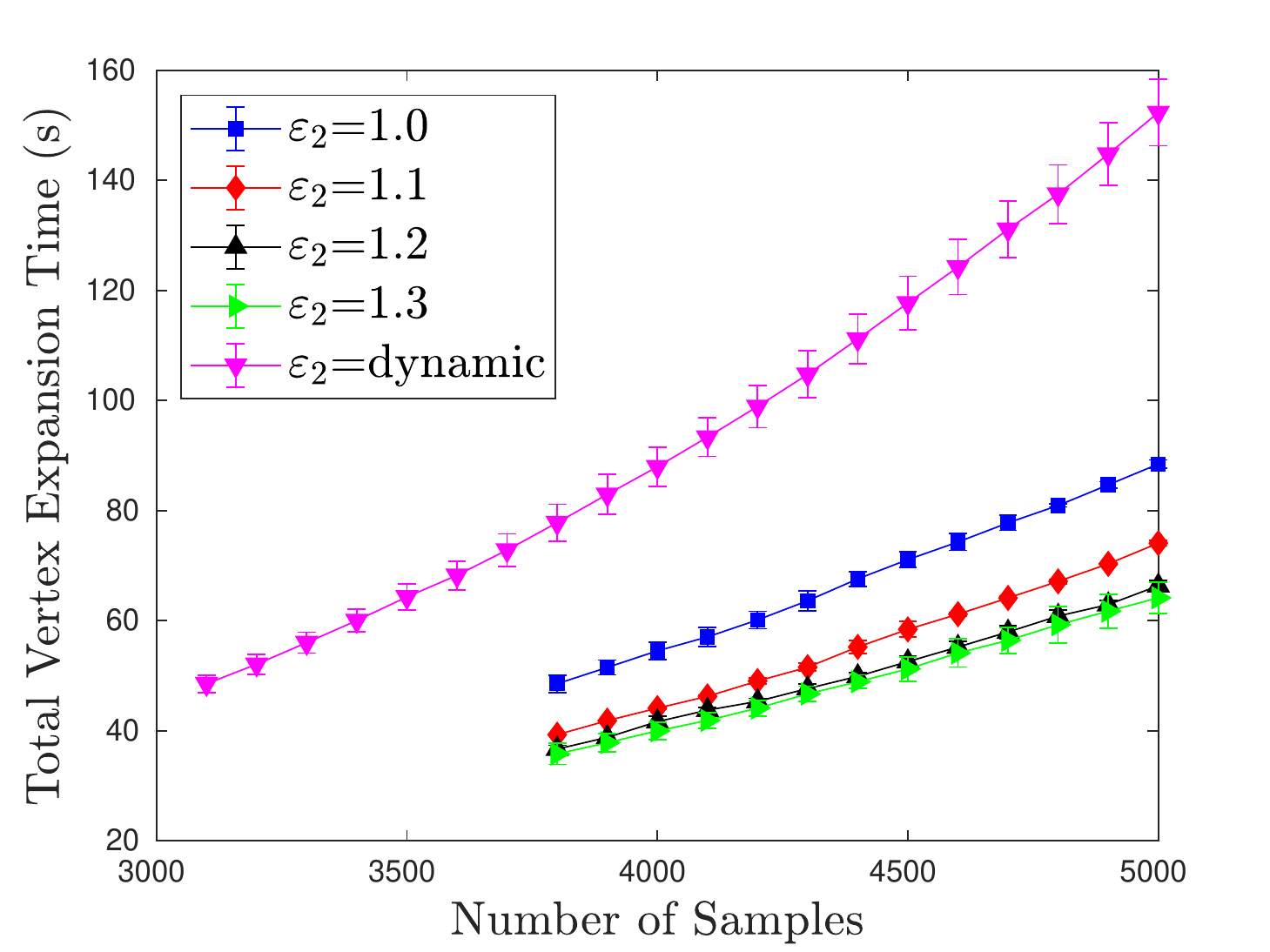}
		\caption{Vertex Expansion}
	\end{subfigure}
	\begin{subfigure}{0.32\textwidth}
		\includegraphics[width=\myLineScale\linewidth]{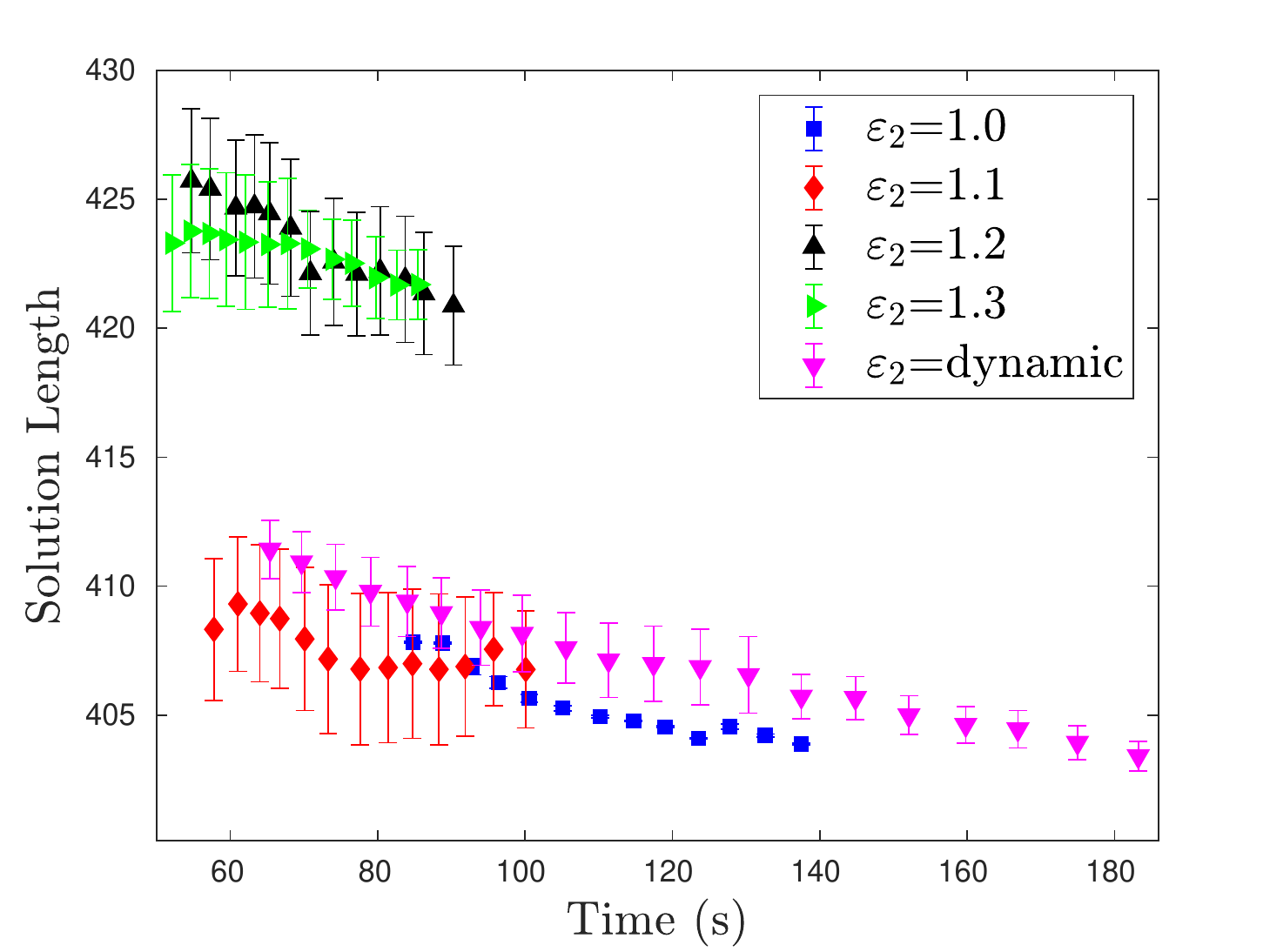}
		\caption{Solution Length}
	\end{subfigure}
	\caption{Truncation factor variation of the Piano Movers' problem.}
	\label{lgls:f:piano_inc_truncation}
\end{figure*}

\begin{figure*}[ht]
	\centering
	\begin{subfigure}{0.32\textwidth}
		\includegraphics[width=\myLineScale\linewidth]{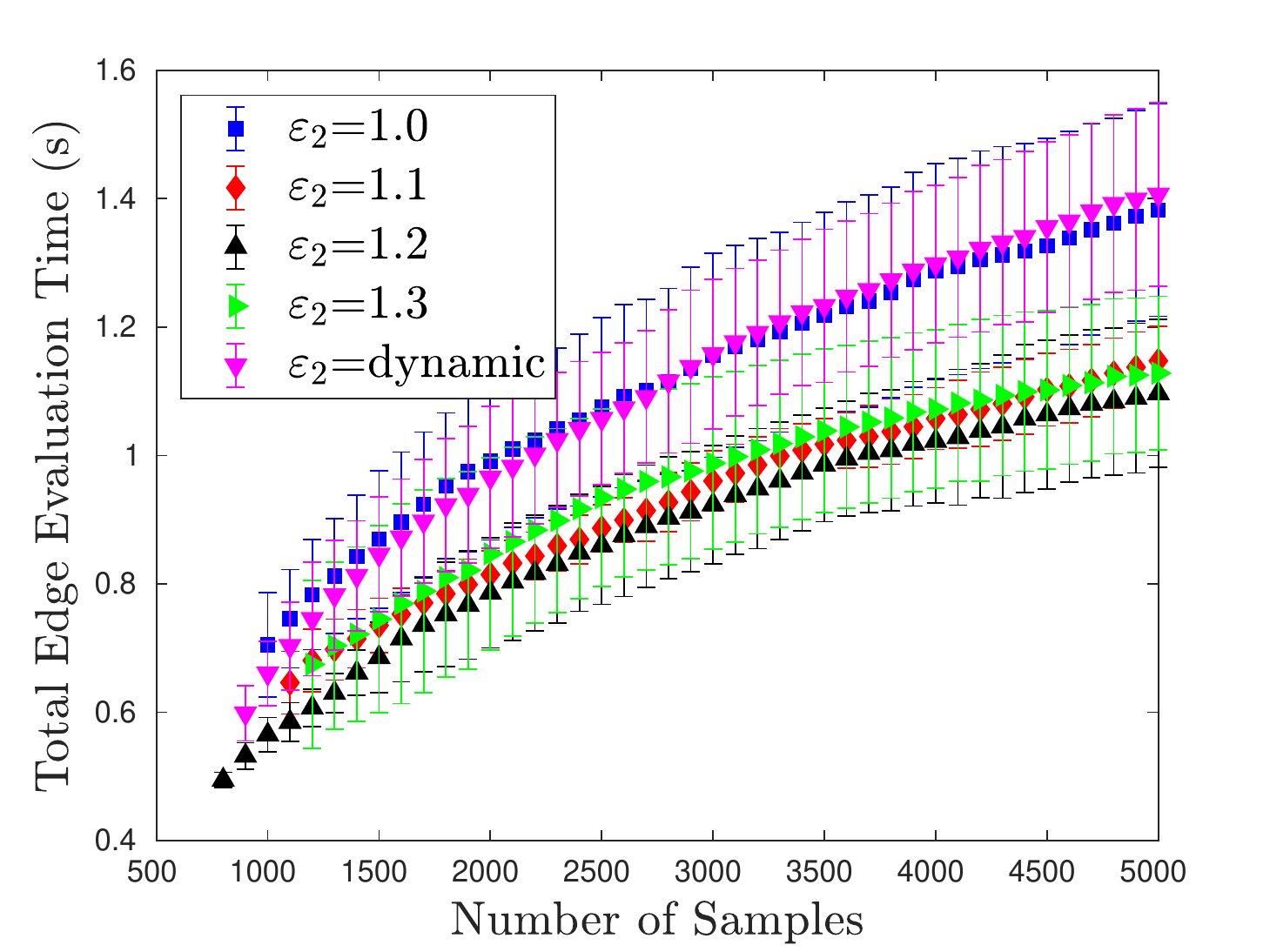}
		\caption{Edge Evaluation}
	\end{subfigure}
	\begin{subfigure}{0.32\textwidth}
		\includegraphics[width=\myLineScale\linewidth]{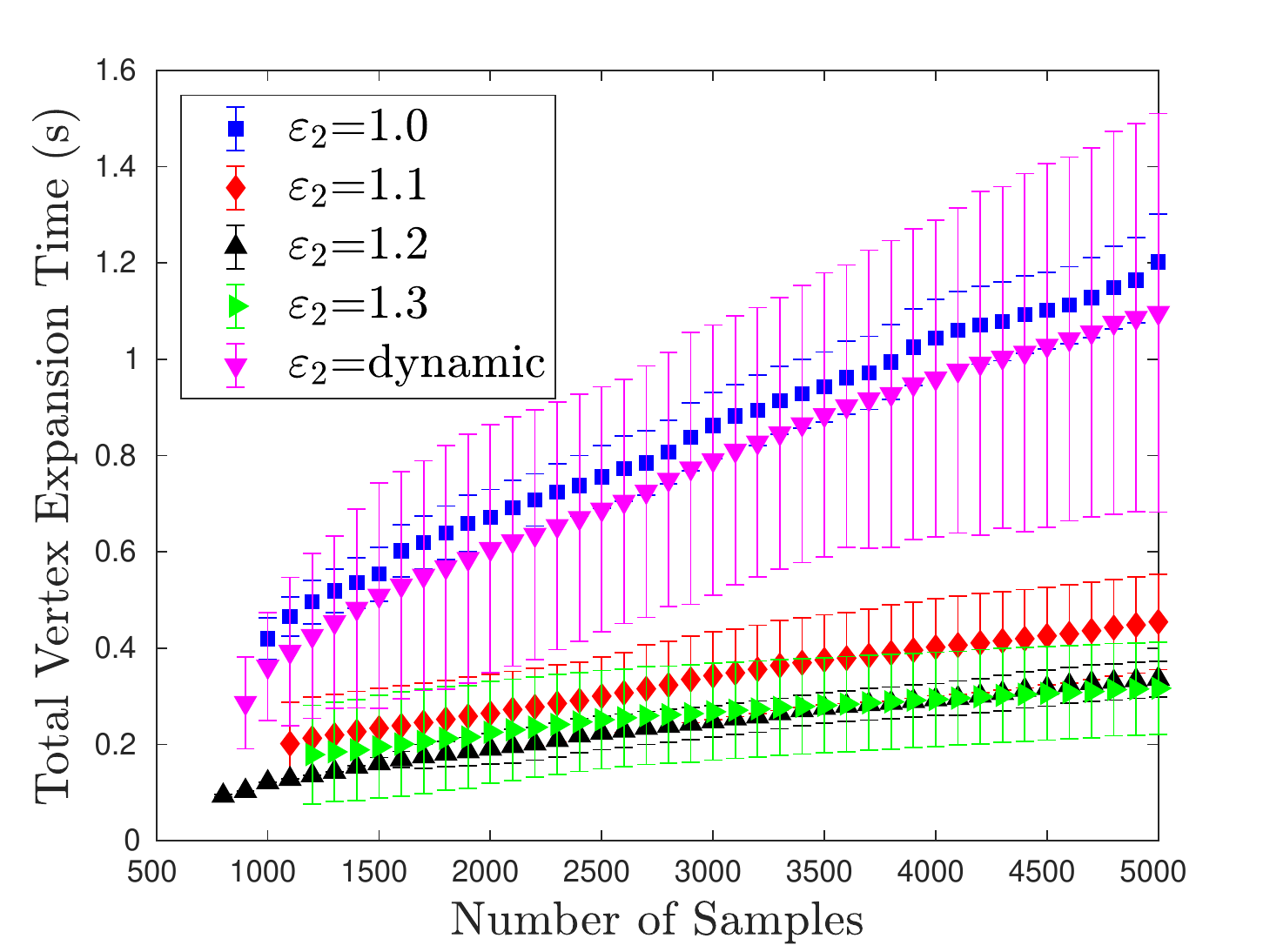}
		\caption{Vertex Expansion}
	\end{subfigure}
	\begin{subfigure}{0.32\textwidth}
		\includegraphics[width=\myLineScale\linewidth]{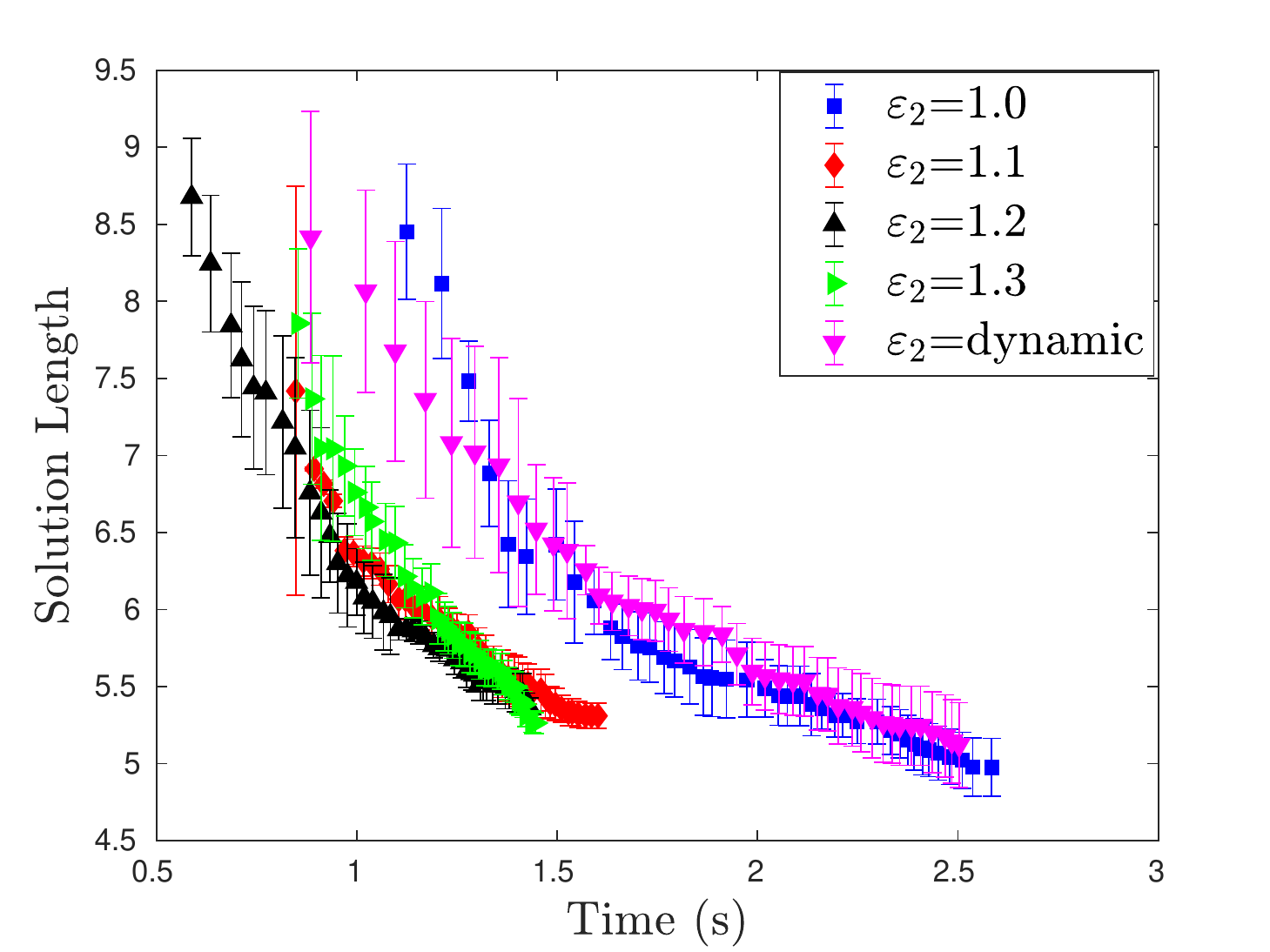}
		\caption{Solution Length}
	\end{subfigure}
	\caption{Truncation factor variation of the PR2 robotic arm problem.}
	\label{lgls:f:pr2_inc_truncation}
\end{figure*}

\begin{figure*}[ht]
	\centering
	\begin{subfigure}{0.32\textwidth}
		\includegraphics[width=\myLineScale\linewidth]{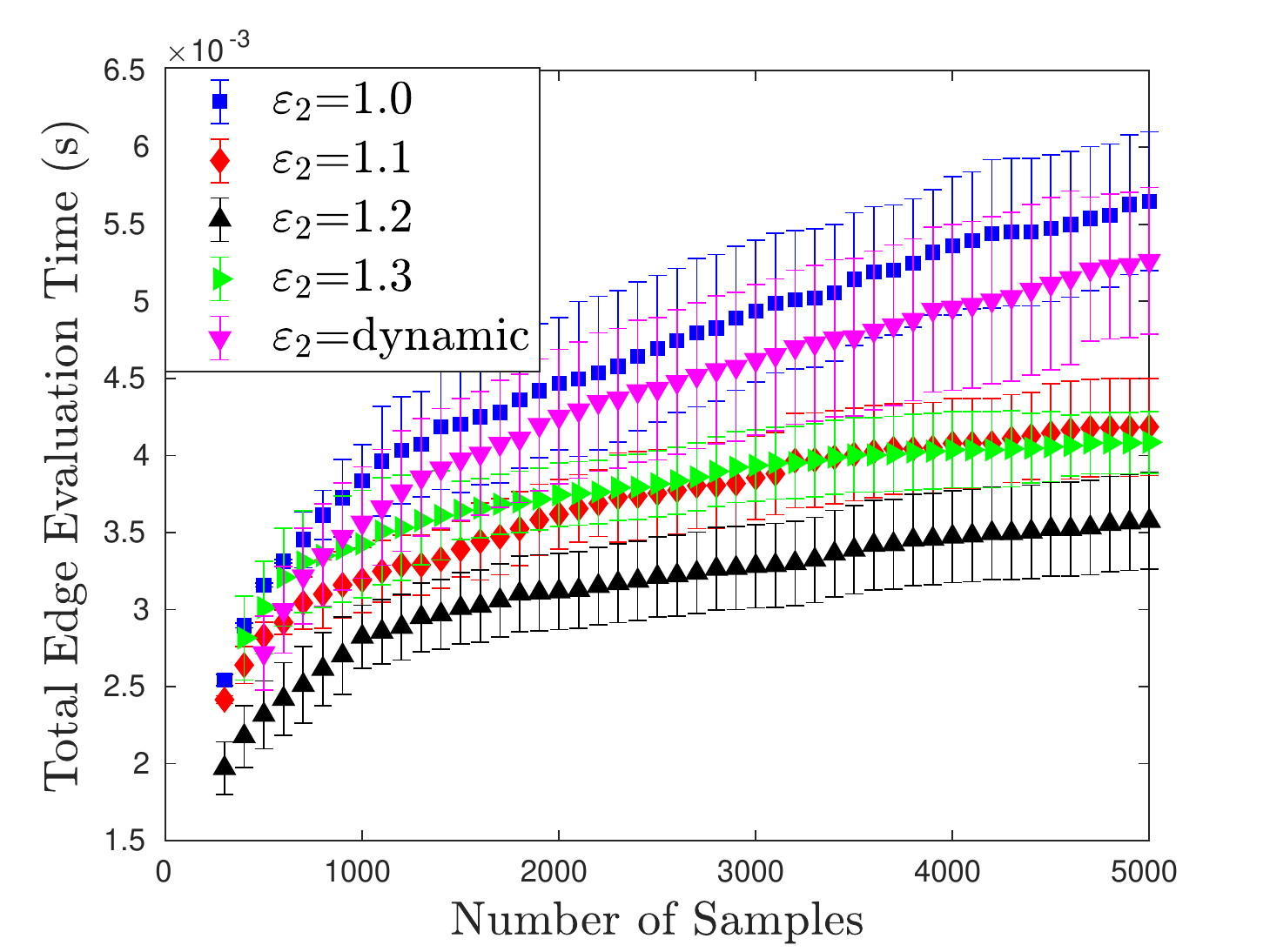}
		\caption{Edge Evaluation}
	\end{subfigure}
	\begin{subfigure}{0.32\textwidth}
		\includegraphics[width=\myLineScale\linewidth]{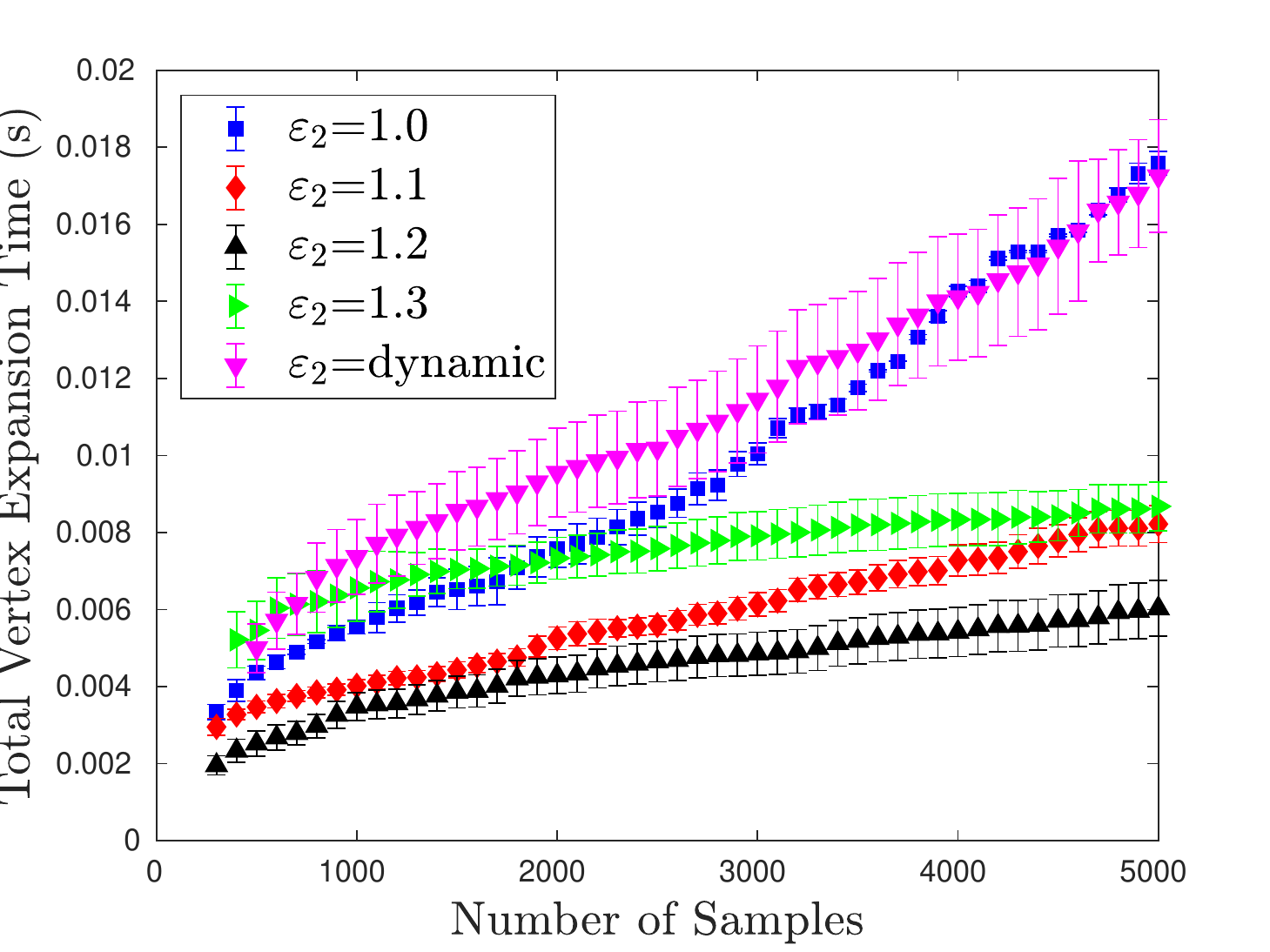}
		\caption{Vertex Expansion}
	\end{subfigure}
	\begin{subfigure}{0.32\textwidth}
		\includegraphics[width=\myLineScale\linewidth]{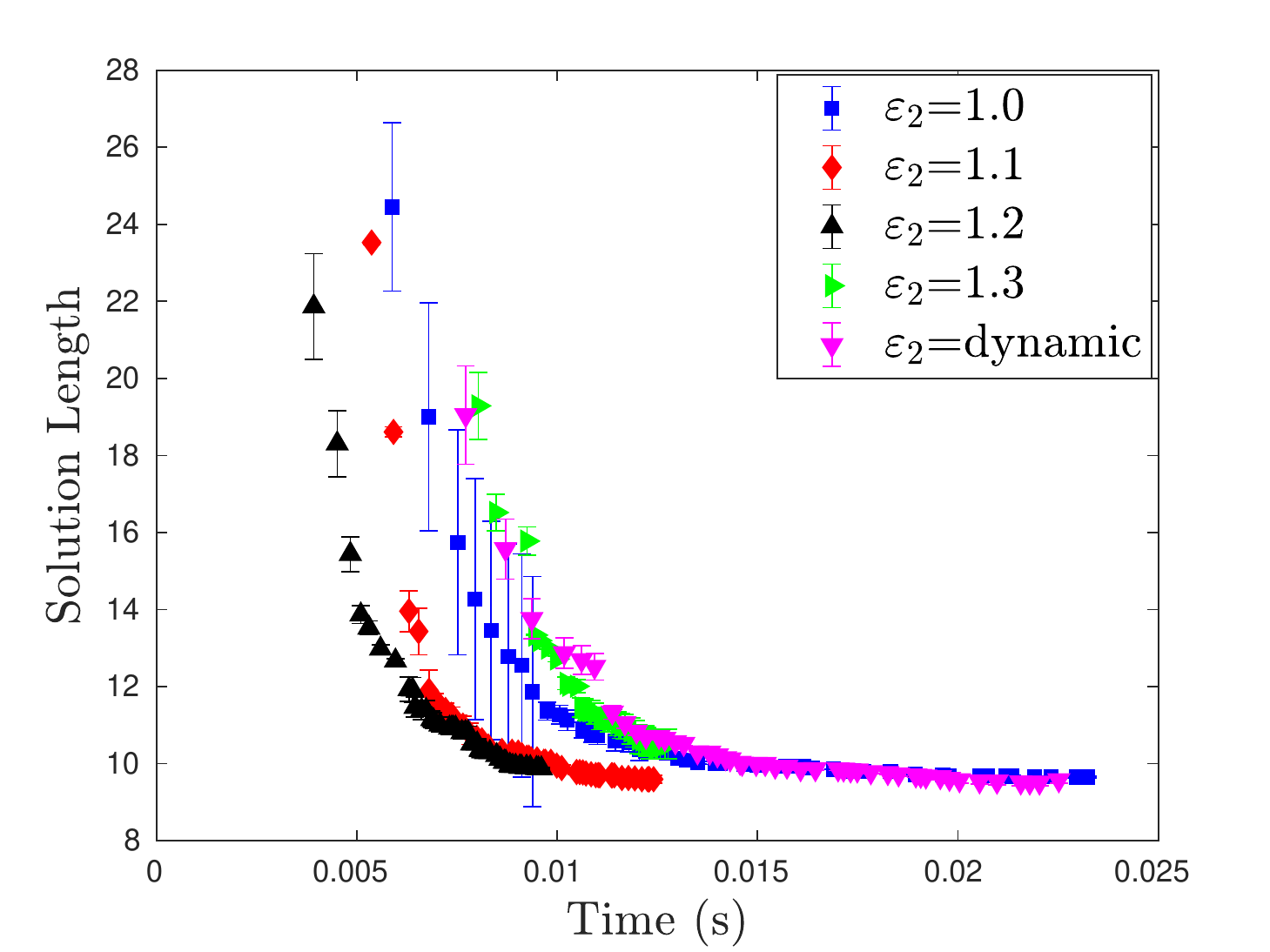}
		\caption{Solution Length}
	\end{subfigure}
	\caption{Truncation factor variation of the RacecarJ problem.}
	\label{lgls:f:racecar_inc_truncation}
\end{figure*}

\paragraph{Comparison between LPA*, TLPA*, L-GLS, and B-LGLS:}

Having examined the effects of the lookahead, the inflation factor, and the truncation factor, in this section we compare the performance of our proposed lazy incremental search algorithms to their non-lazy incremental search counterparts. 
We compared LPA*, TLPA*, L-GLS, and B-LGLS in the same experiment sets.
All of these incremental search algorithms improve the solution as the graph becomes denser as we add 100 samples per iteration.
Figures~\ref{lgls:f:piano_inc_comp}, \ref{lgls:f:pr2_inc_comp} and \ref{lgls:f:racecar_inc_comp} show the total time to evaluate edges, the total time to expand vertices, and the solution length as a function of time, respectively, for the Piano Movers' problem, the PR2 robotic arm problem, and the non-holonomic racecar problem. 
\begin{figure*}[ht]
	\centering
	\begin{subfigure}{0.32\textwidth}
		\includegraphics[width=\myLineScale\linewidth]{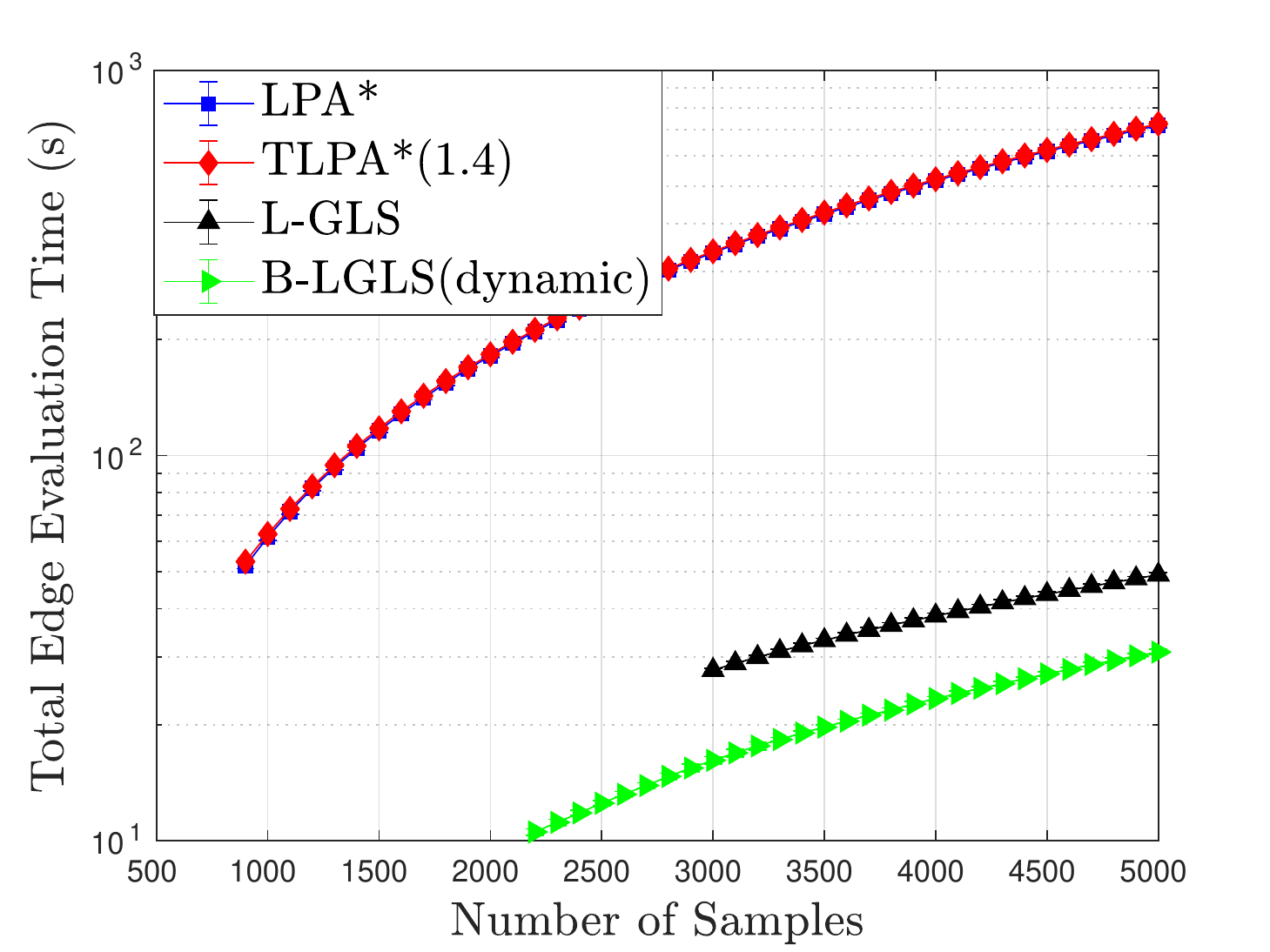}
		\caption{Edge Evaluation}
	\end{subfigure}
	\begin{subfigure}{0.32\textwidth}
		\includegraphics[width=\myLineScale\linewidth]{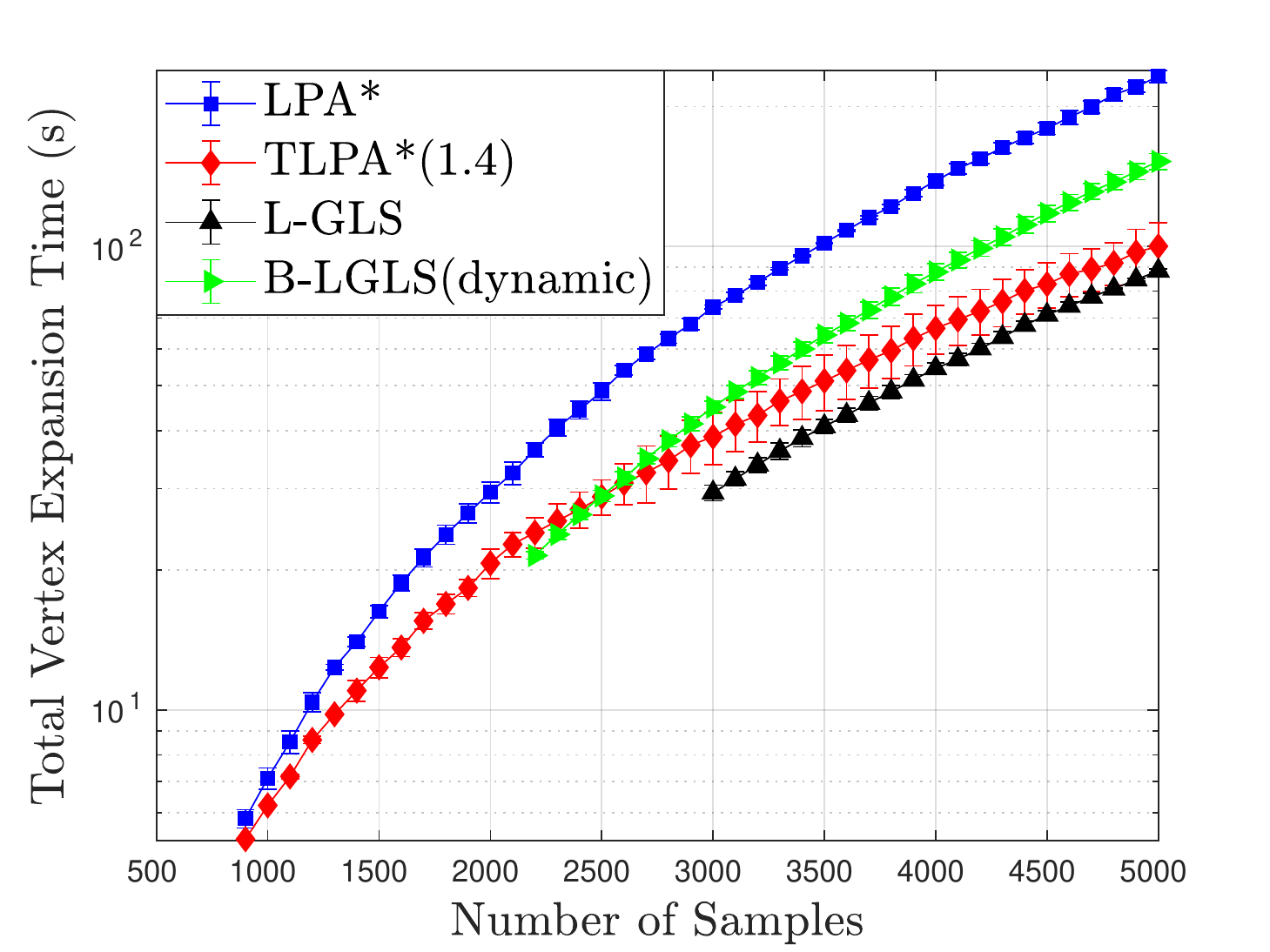}
		\caption{Vertex Expansion}
	\end{subfigure}
	\begin{subfigure}{0.32\textwidth}
		\includegraphics[width=\myLineScale\linewidth]{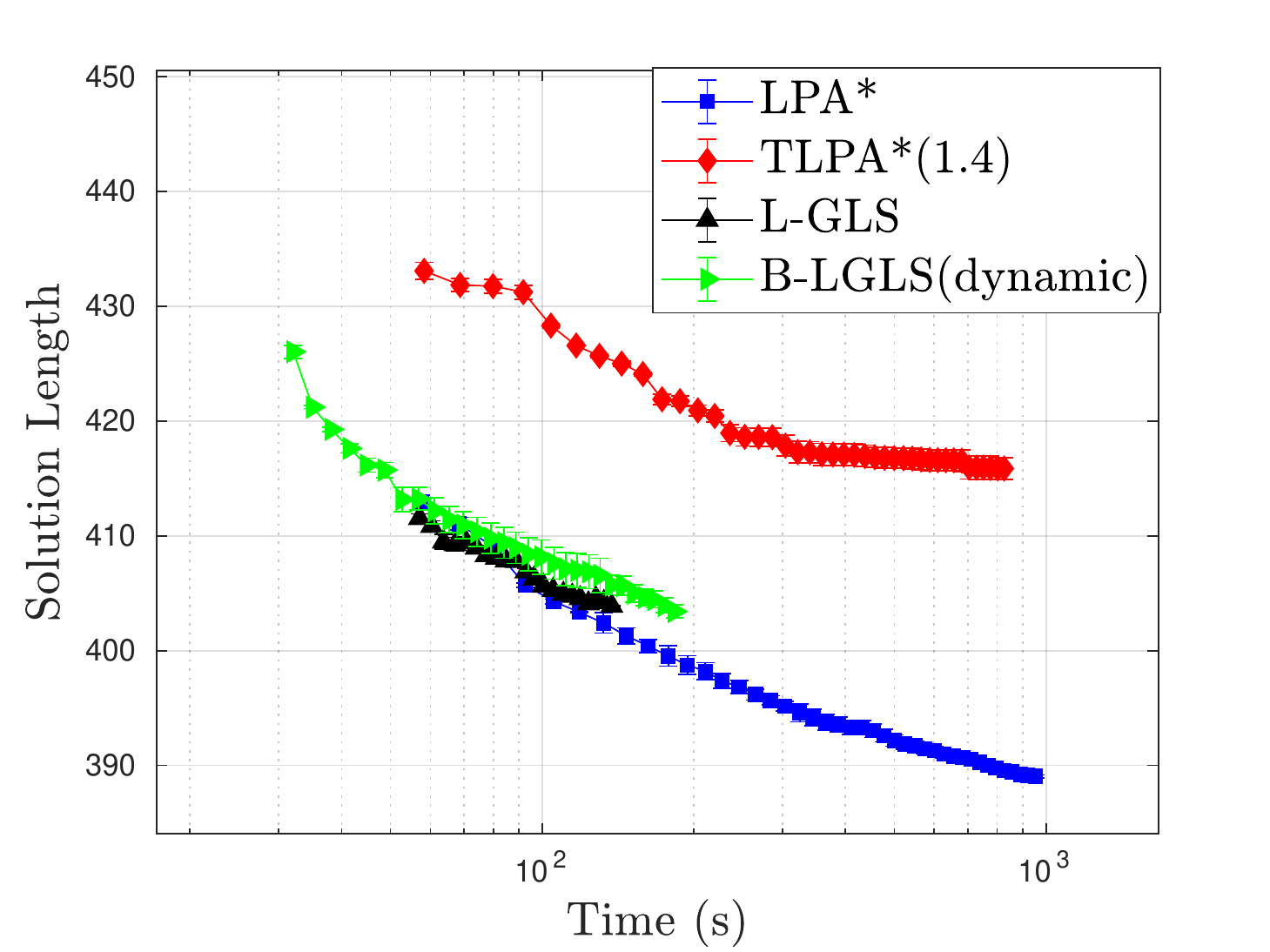}
		\caption{Solution Length}
	\end{subfigure}
	\caption{Comparison between LPA*, TLPA*, B-LGLS for the Piano Movers' problem.}
	\label{lgls:f:piano_inc_comp}
\end{figure*}

\begin{figure*}[ht]
	\centering
	\begin{subfigure}{0.32\textwidth}
		\includegraphics[width=\myLineScale\linewidth]{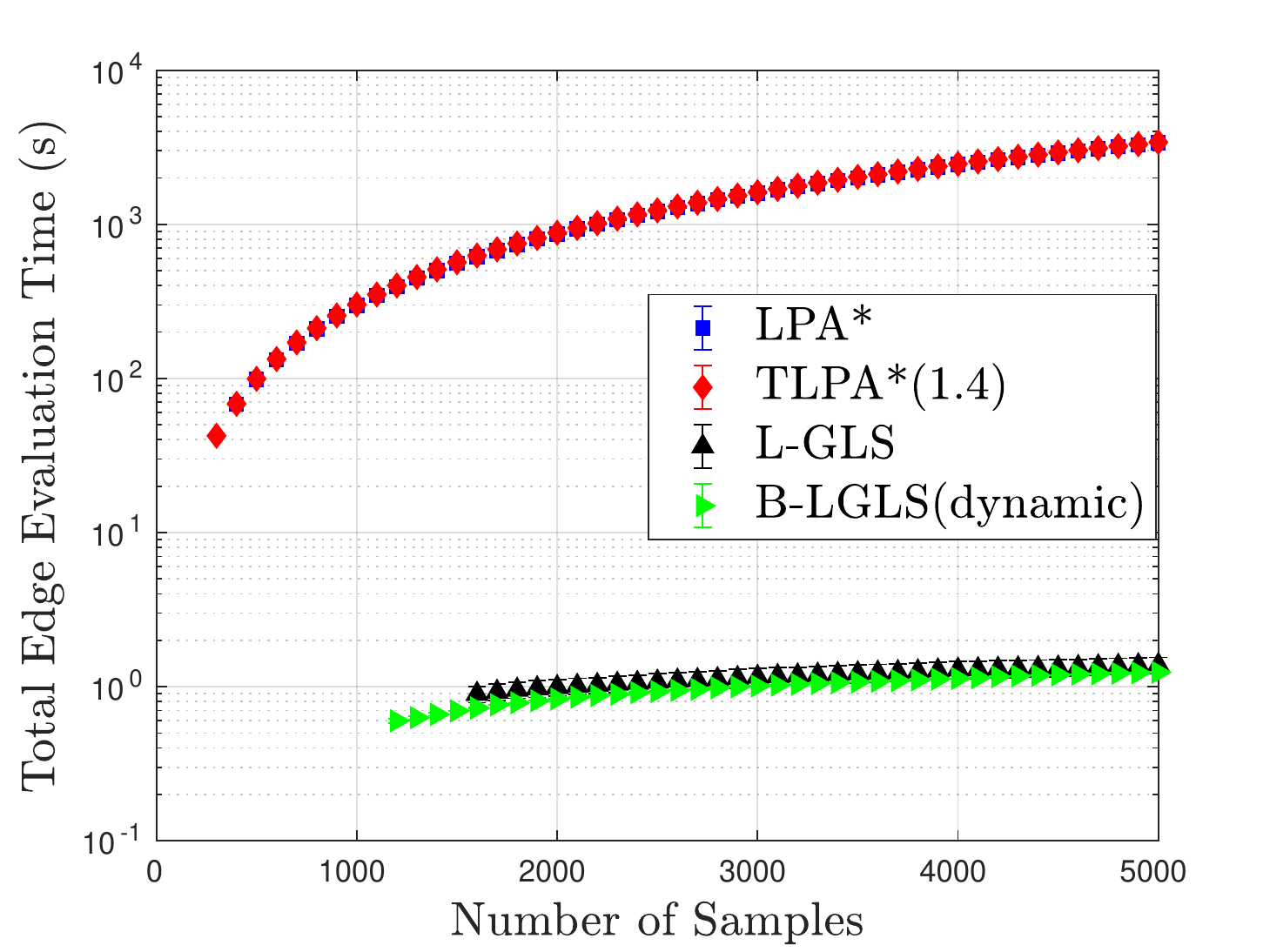}
		\caption{Edge Evaluation}
	\end{subfigure}
	\begin{subfigure}{0.32\textwidth}
		\includegraphics[width=\myLineScale\linewidth]{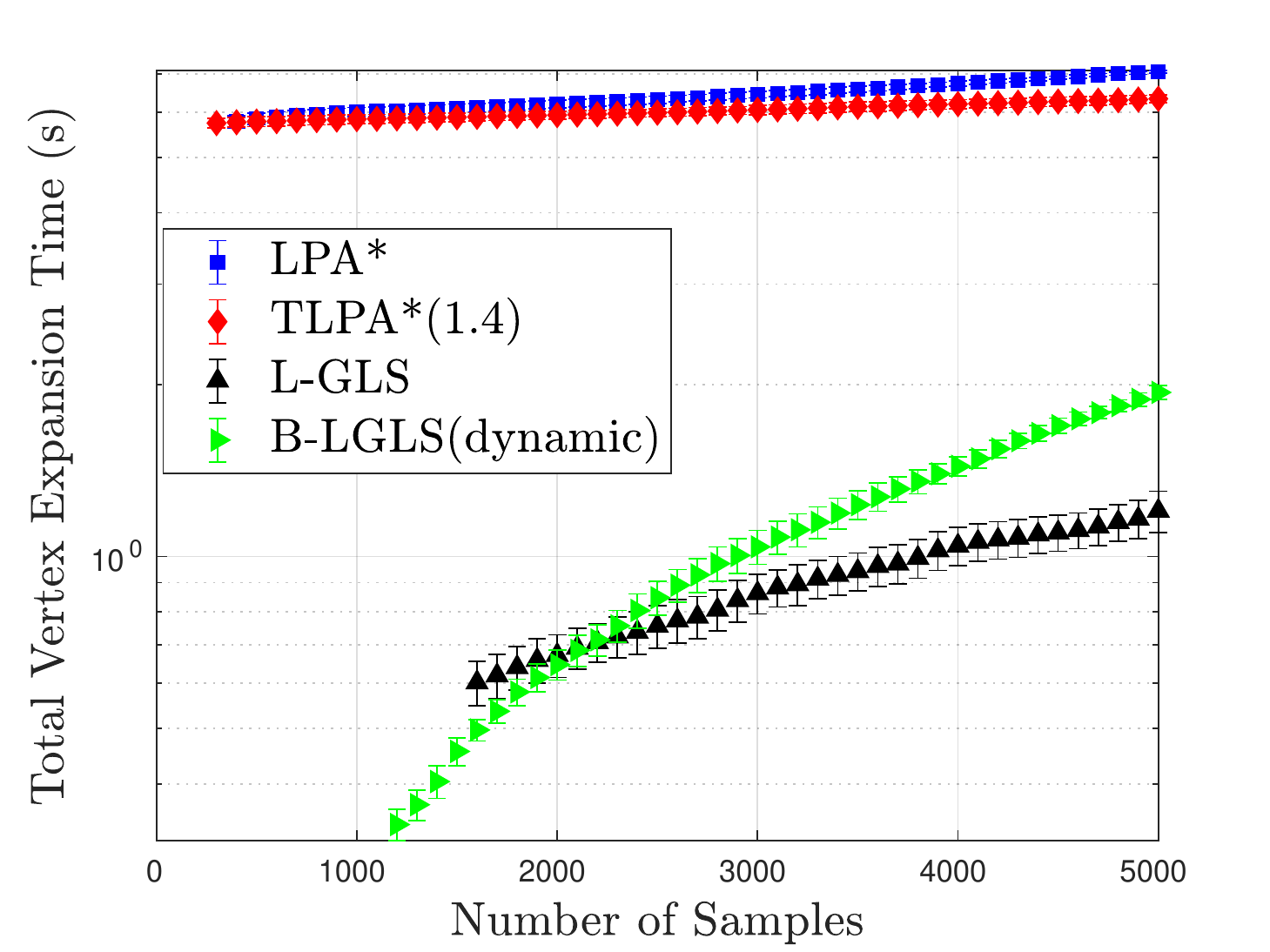}
		\caption{Vertex Expansion}
	\end{subfigure}
	\begin{subfigure}{0.32\textwidth}
		\includegraphics[width=\myLineScale\linewidth]{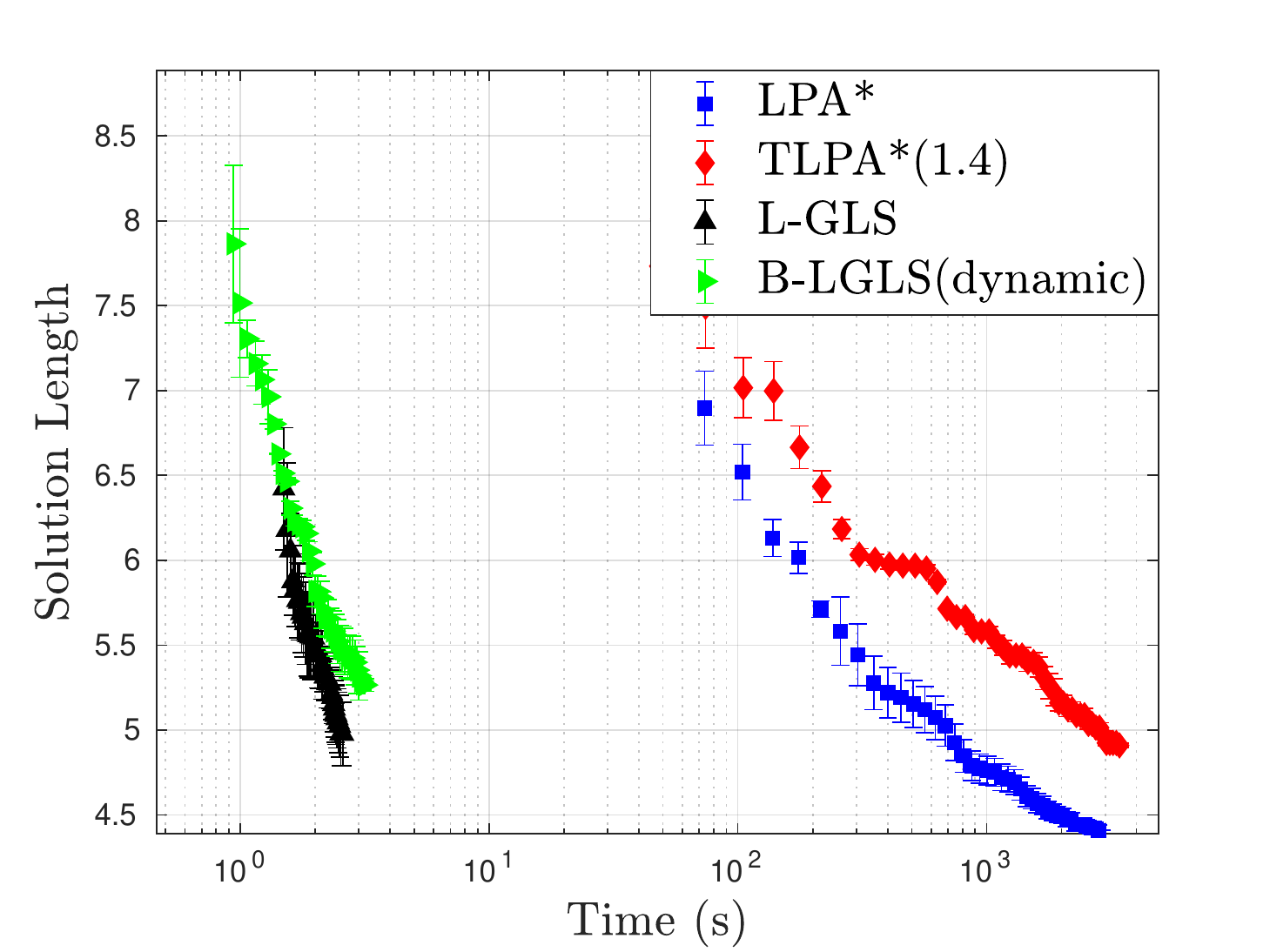}
		\caption{Solution Length}
	\end{subfigure}
	\caption{Comparison between LPA*, TLPA*, B-LGLS for the PR2 robotic arm problem.}
	\label{lgls:f:pr2_inc_comp}
\end{figure*}

\begin{figure*}[ht]
	\centering
	\begin{subfigure}{0.32\textwidth}
		\includegraphics[width=\myLineScale\linewidth]{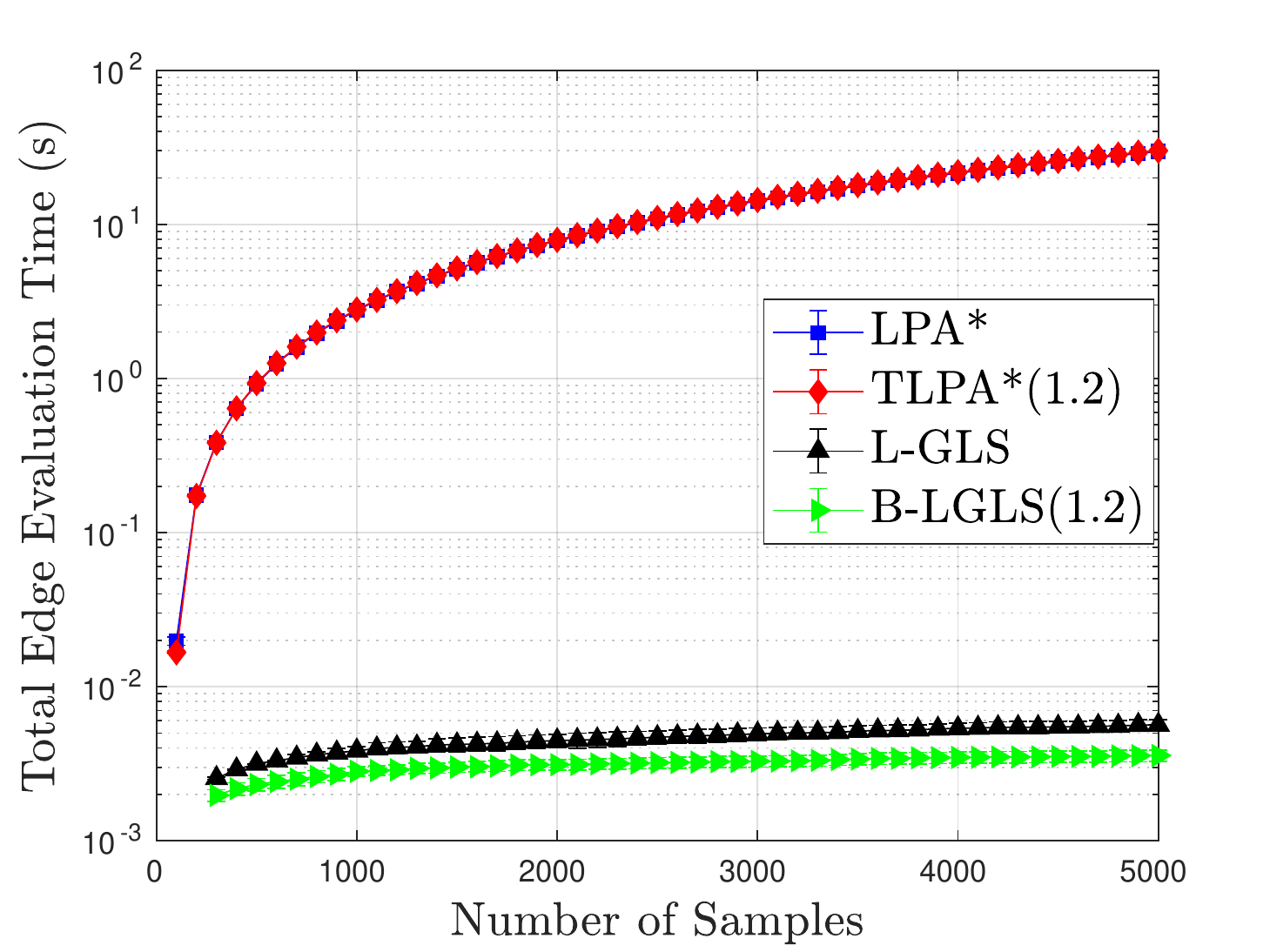}
		\caption{Edge Evaluation}
	\end{subfigure}
	\begin{subfigure}{0.32\textwidth}
		\includegraphics[width=\myLineScale\linewidth]{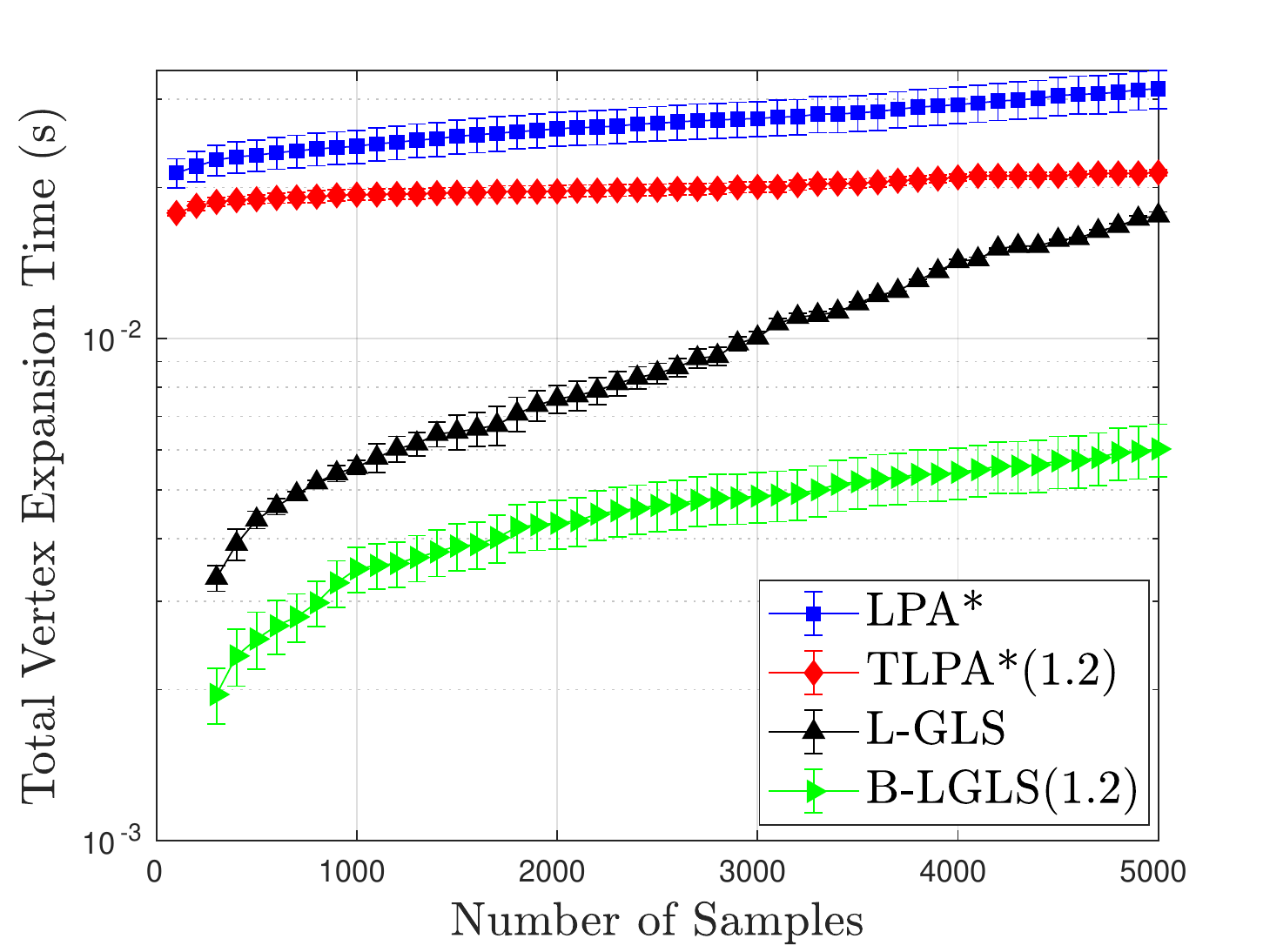}
		\caption{Vertex Expansion}
	\end{subfigure}
	\begin{subfigure}{0.32\textwidth}
		\includegraphics[width=\myLineScale\linewidth]{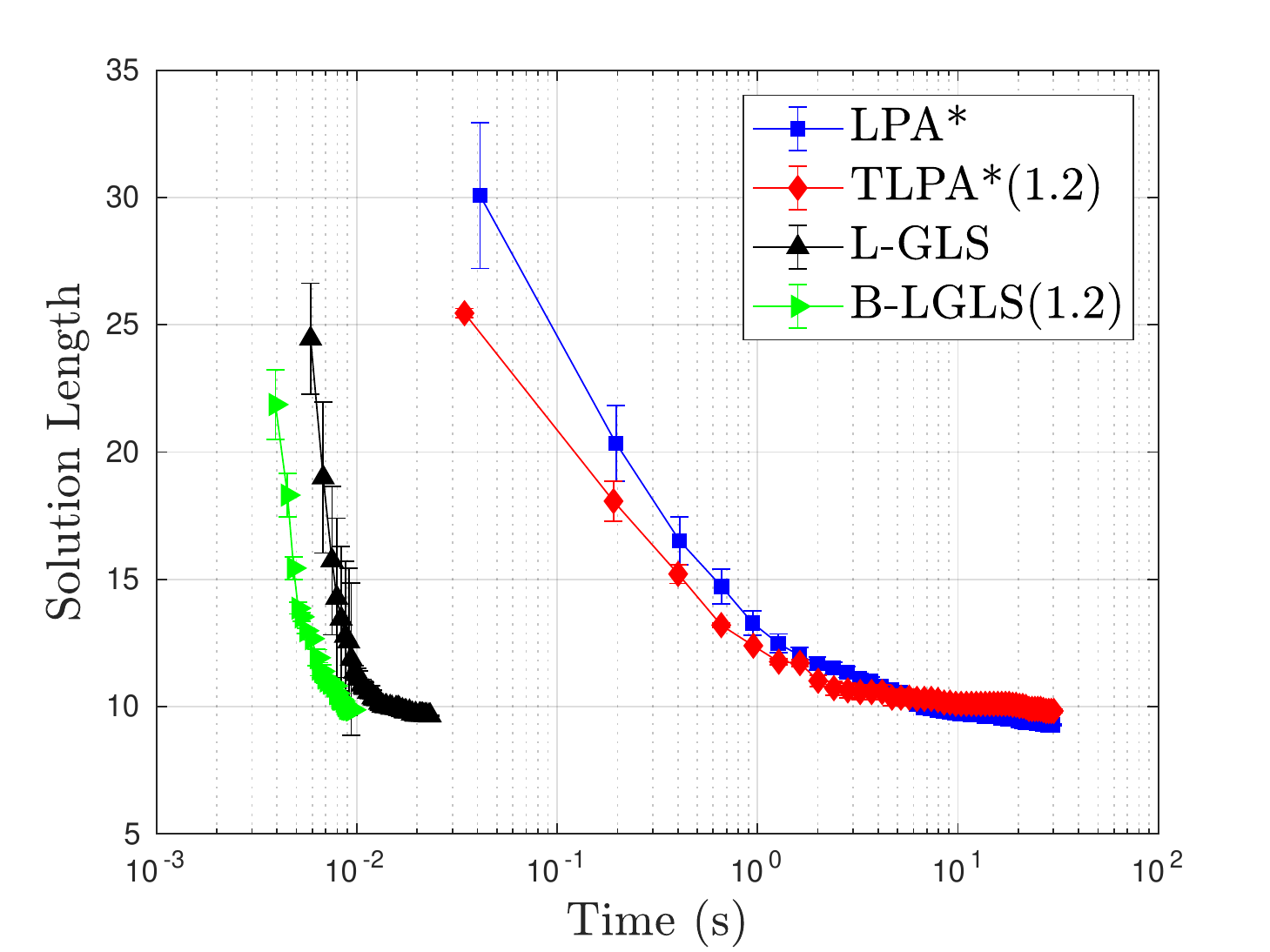}
		\caption{Solution Length}
	\end{subfigure}
	\caption{Comparison between LPA*, TLPA*, B-LGLS for the RacecarJ problem.}
	\label{lgls:f:racecar_inc_comp}
\end{figure*}

LPA* and TLPA* do not use heuristic edge values when they propagate the cost inconsistencies, incurring a significant number of edge evaluations. The proposed lazy incremental search algorithms save time when evaluating edges per search iteration more than one order of magnitude for the Piano Movers' problem and more than two orders of magnitude for the PR2 robotic arm problem and the non-holonomic racecar problem.
These savings resulted in a faster solution convergence rate.  
B-LGLS finds the first solution faster than L-GLS, guided by the inflated heuristic edge weight. 
As the inflation factor and the truncation factor asymptotically approaches 1, B-LGLS also asymptotically converges to the optimal solution. 

For the remaining part of the paper, we describe non-stationary versions of the proposed lazy incremental search algorithms for replanning when the start vertices change. 
We use D*-Lite and Truncated D*-Lite algorithms to find a solution quickly for a mobile robot with expensive edge evaluations and compare the results.

\subsection{Non-Stationary Query Experiment} \label{sec:dynamic_query_experiment}

In this section, we present numerical simulations of our lazy algorithms and compare them 
to D*-Lite and TD*. 
Similar to the stationary replanning problem instances,
we conduct the experiments in the PR2 simulation environment, the Piano Movers problem, and the RacecarJ environment. 
%
For the first two experiments, we generated a random graph with 5,000 vertices each connected to 20 nearest neighbors, 
and solved the motion planning problem for 100 randomly sampled start and goal configurations similar to the previous stationary case.
However, we now change the environment while the robot is moving towards the goal on the found solution path, and thus force the robot to replan.

\subsubsection{Comparison of the Algorithms.}

For the first experiment, we find the shortest path for a random start and goal configuration in scene 1 (Figure~\ref{lgls:f:piano}(a)) of the Piano Movers' problem.
Once an optimal path is computed, the piano is moved towards the goal for one segments of the optimal path when the environment is changed to the one in scene 2 (Figure~\ref{lgls:f:piano}(b)). 
The piano again traverses one segment of the newly found path and the environment is changed again into the one in the scene 3 (Figure~\ref{lgls:f:piano}(c)), forcing the planner to find a new solution. 
We cycle through the environments once more to obtain a total of six consecutive searches.

We compared the results for four different algorithms, namely, D*-Lite, GD* (1-step lookahead), TD* ($\eps_2=2$), and B-GD* (1-step lookahead, $\eps_1=\sqrt{2}, \eps_2=\sqrt{2}$) for 100 randomly sampled start and goal configurations. 
The average results from the performance comparison of the different algorithms in terms of the number of edge evaluations, vertex expansions, and runtime for the piano movers problem is shown in Figure~\ref{dstar:f:piano}.
As shown in Figure~\ref{dstar:f:piano}, the proposed lazy algorithms significantly reduce the number of edge evaluations. 
Furthermore, the bounded algorithms generally reduce the number of vertex expansions via early termination in case a bounded suboptimal solution is found. 
Thus, the average runtime is also reduced for the generalized algorithms.

\begin{figure*}[ht]
	\centering
	\begin{subfigure}{0.32\textwidth}
		\includegraphics[width=\myLineScale\linewidth]{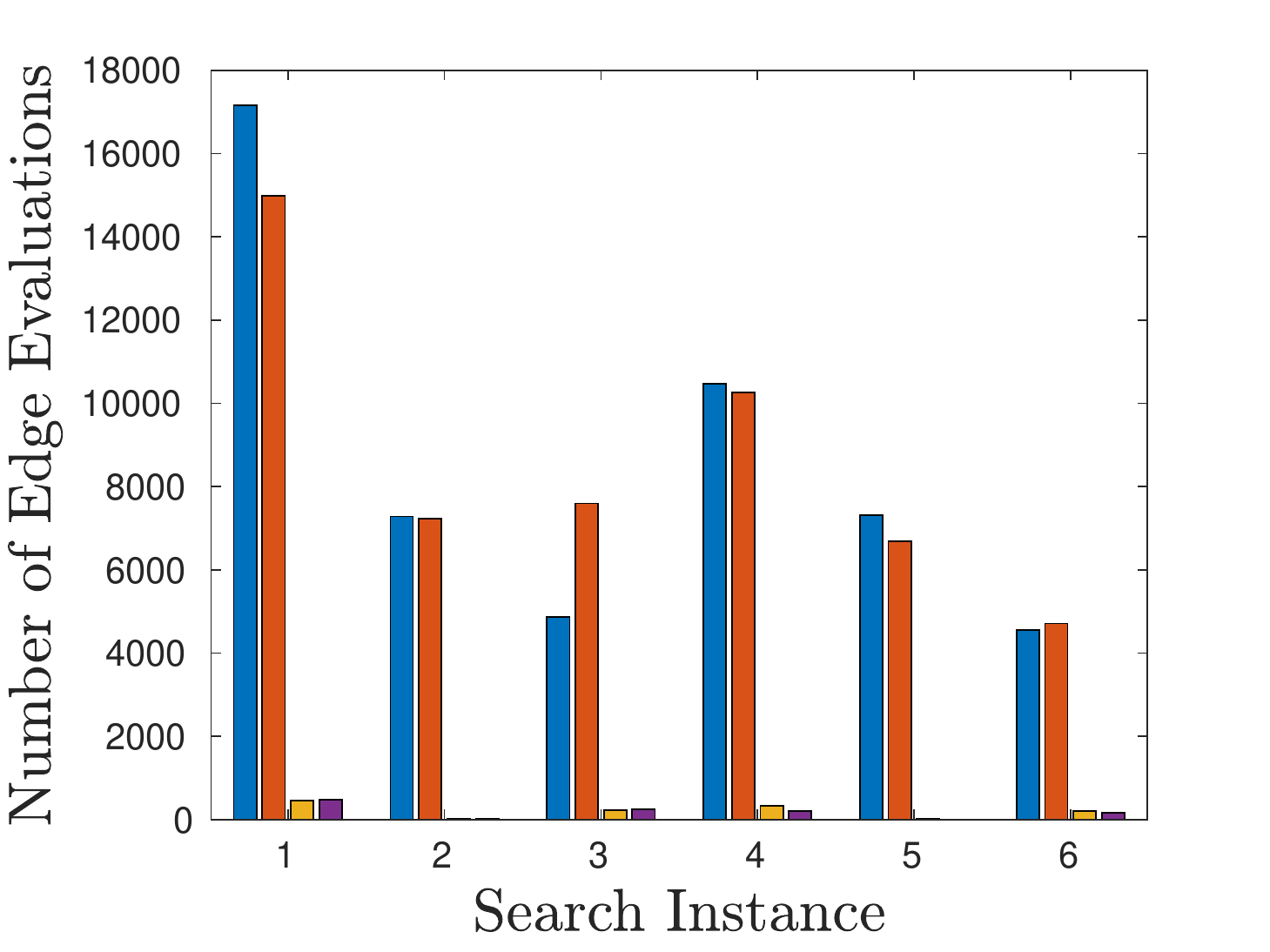}
		\caption{Edge evaluation}
	\end{subfigure}
	\begin{subfigure}{0.32\textwidth}
		\includegraphics[width=\myLineScale\linewidth]{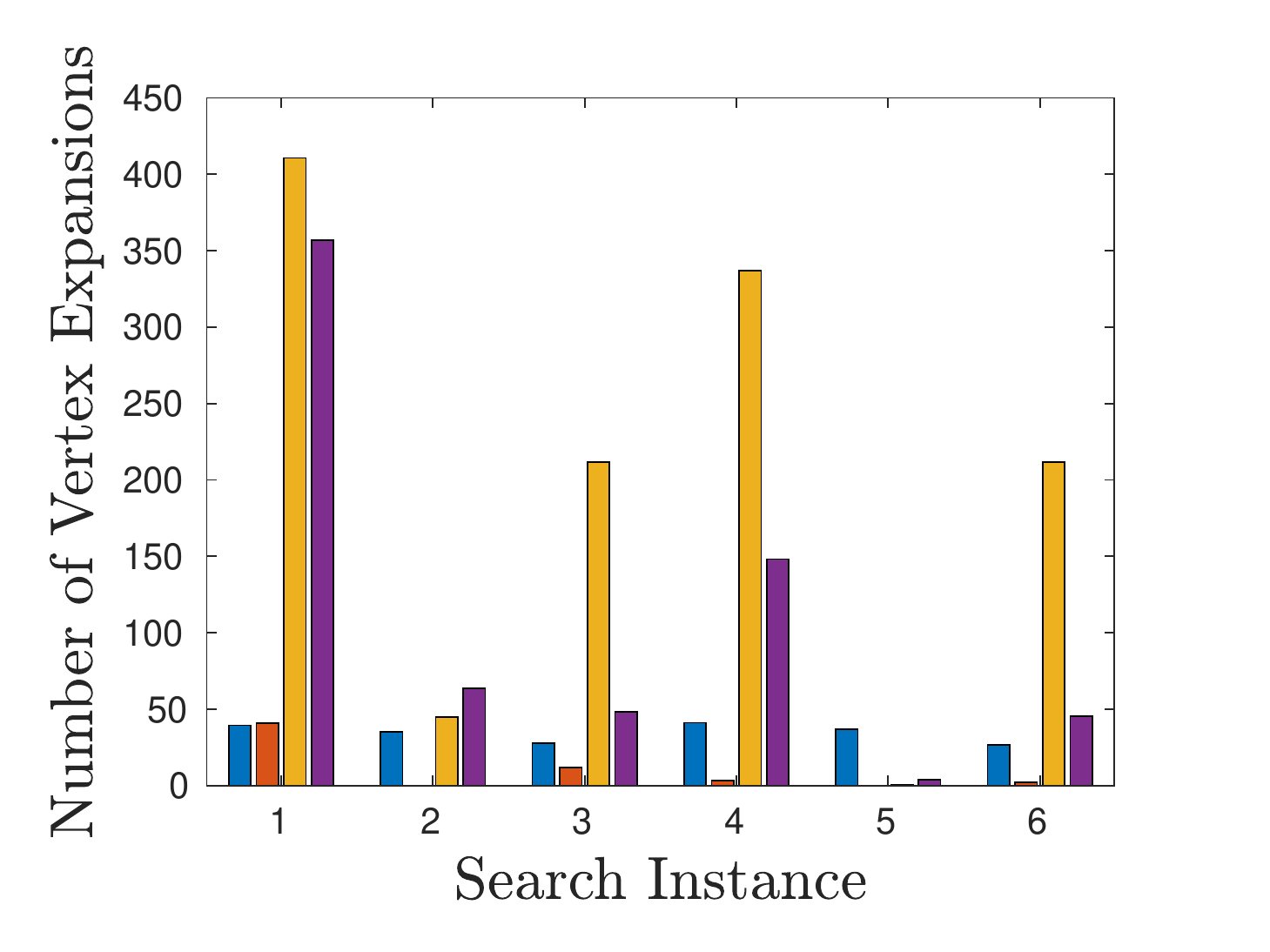}
		\caption{Vertex expansion}
	\end{subfigure}
	\begin{subfigure}{0.32\textwidth}
		\includegraphics[width=\myLineScale\linewidth]{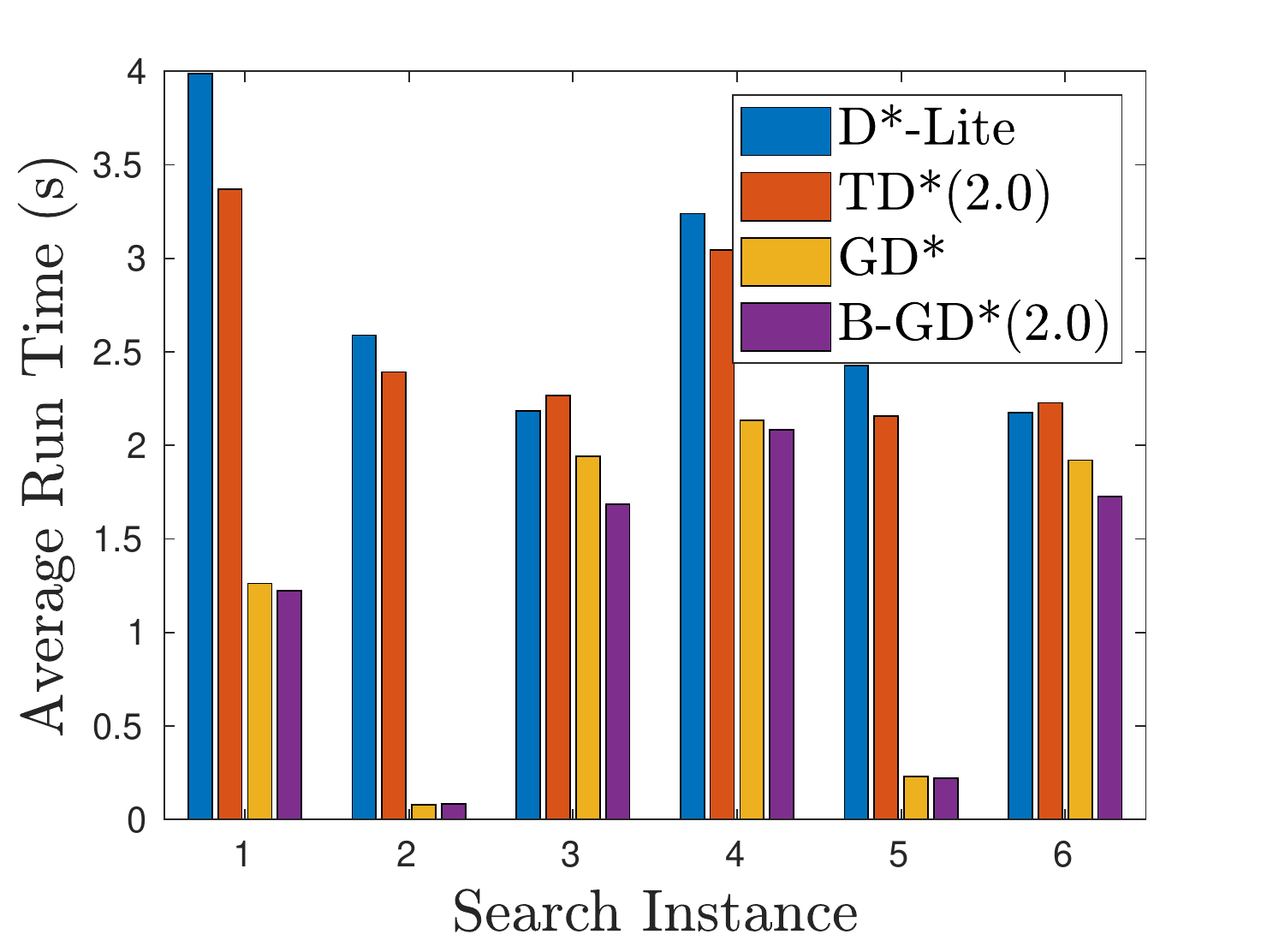}
		\caption{Total time}
	\end{subfigure}
	\caption{The average number of edge evaluations, the average number of vertex expansions, and the average total runtime taken over 100 random experiments in the Piano Movers' environment.}
	\label{dstar:f:piano}
\end{figure*}


A similar experiment for the PR2 environment was conducted. 
Once an optimal path is computed for the environment in scene 1 (Figure~\ref{lgls:f:pr2}(a))
the robot moves towards the goal traversing one segment of the optimal path. 
The environment is then changed to the one shown in scene 2 (Figure~\ref{lgls:f:pr2}(b)), and thus the algorithms replan to find a new optimal path to the goal.
We cycle through the environments twice more to obtain the total of six consecutive searches. 
We conducted the experiment for 100 different randomly generated start and goal configurations for the robot.
The comparison of the average performances results solved by the four different algorithms, namely, D*, GD* (1-step lookahead), TD* ($\eps_2=2$), and B-GD* (1-step lookahead, $\eps_1=\sqrt{2}, \eps_2=\sqrt{2}$) is shown 
in Figure~\ref{dstar:f:stat_pr2}. 

\begin{figure*}[ht]
	\centering
	\begin{subfigure}{0.32\textwidth}
		\includegraphics[width=\myLineScale\linewidth]{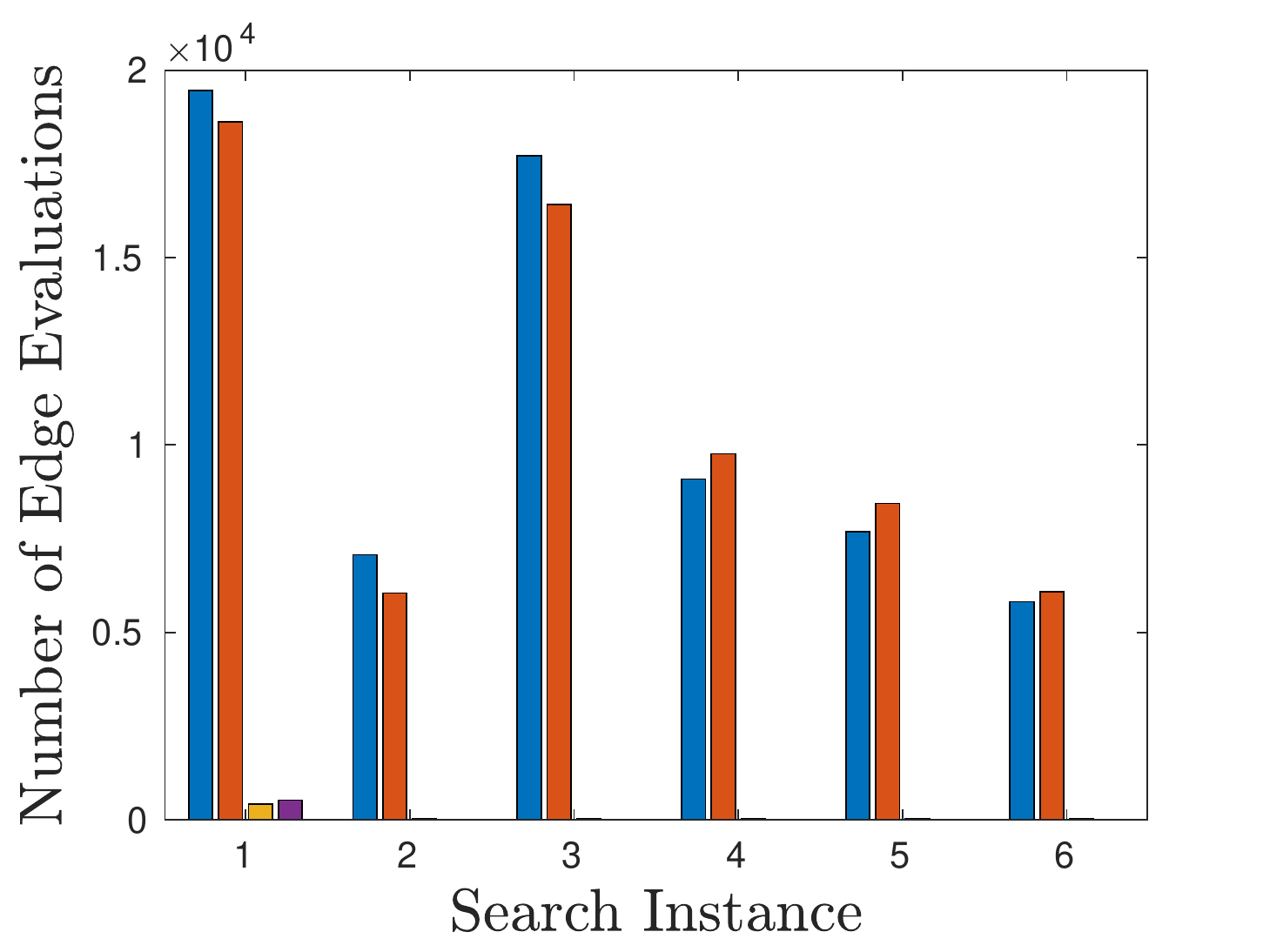}
		\caption{Edge evaluation}
	\end{subfigure}
	\begin{subfigure}{0.32\textwidth}
		\includegraphics[width=\myLineScale\linewidth]{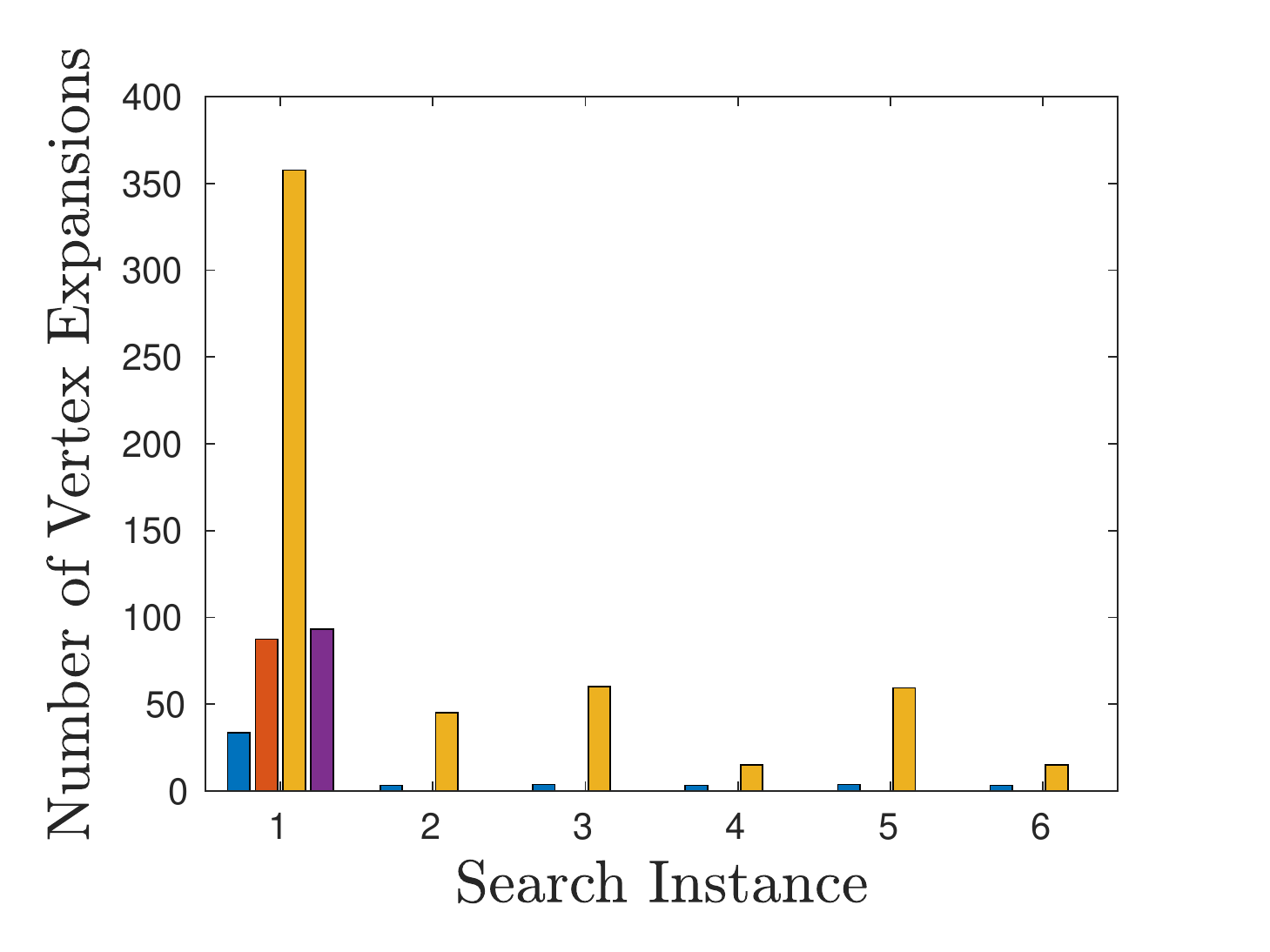}
		\caption{Vertex expansion}
	\end{subfigure}
	\begin{subfigure}{0.32\textwidth}
		\includegraphics[width=\myLineScale\linewidth]{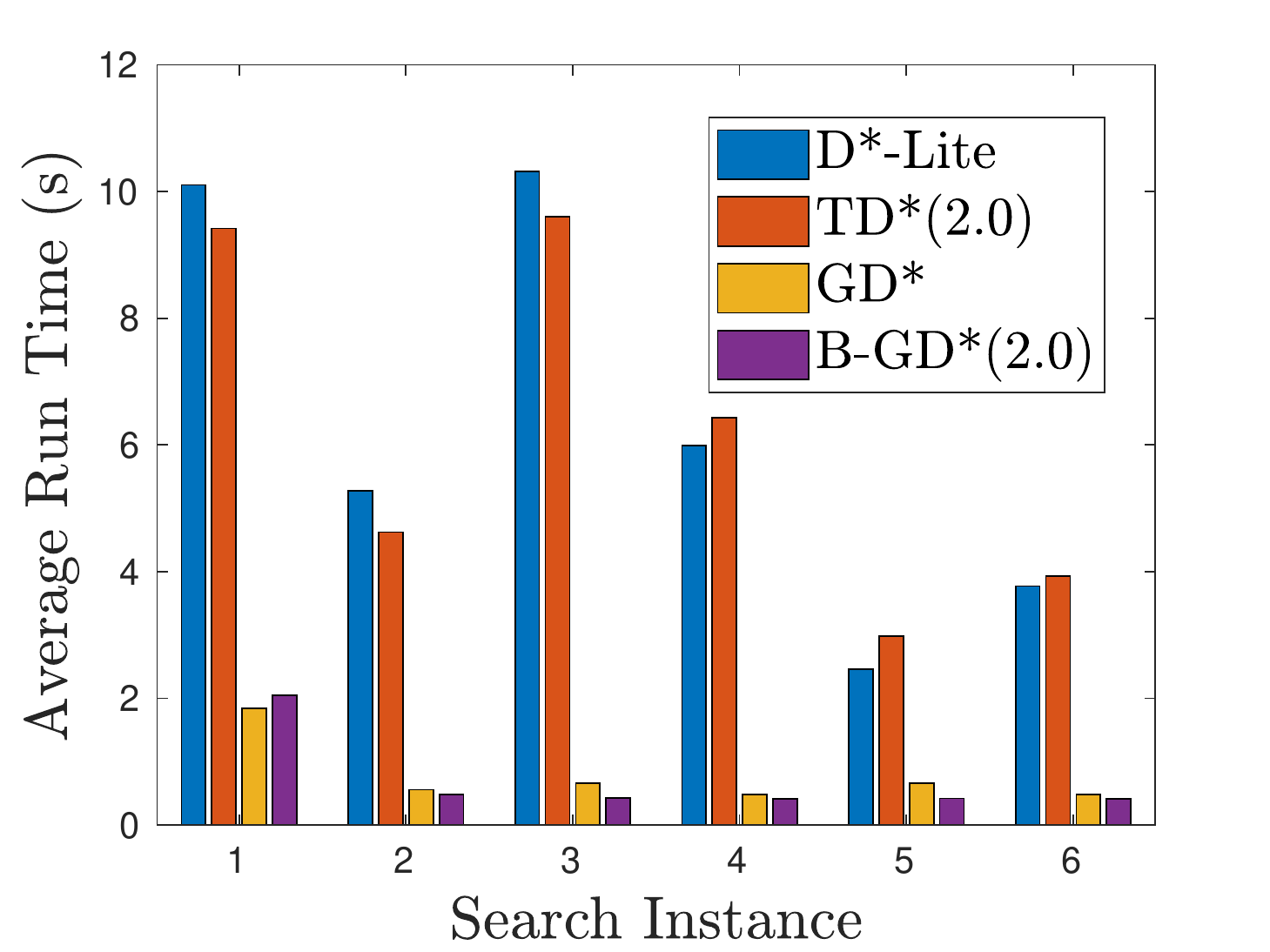}
		\caption{Total time}
	\end{subfigure}
	\caption{The average number of edge evaluations, the average number of vertex expansions, and the average total runtime taken over 100 random experiments in the PR2 environment.}
	\label{dstar:f:stat_pr2}
\end{figure*}

It can be seen that the lazy algorithms (GD* and B-GD*) significantly reduce the number of edge evaluations, and hence the total runtime. 
The bounded version of the algorithm (B-GD*) further reduces the total runtime via terminating upon finding a bounded suboptimal solution. Note that for the PR2 experiment, the GD* algorithm results in an approximately 7X speed-up in finding the initial solution.

Lastly, a similar experiment was conducted in the RacecarJ environment (Figure~\ref{lgls:f:racecar}). 
We constructed the graph via 500 randomly sampled vertices, each having edges connected to the 10 nearest neighbors via Reeds-Shepp curves similar to the stationary case.
Once an optimal path is found, the Racecar traverses one segment of the optimal path and we randomly change 10 percent of the edges in the graph simulating a dynamic environment.
We repeat this to obtain a total of six consecutive searches.
We gathered the experimental data using 100 randomly generated start and goal configurations.

\begin{figure*}[ht]
	\centering
	\begin{subfigure}{0.32\textwidth}
		\includegraphics[width=\myLineScale\linewidth]{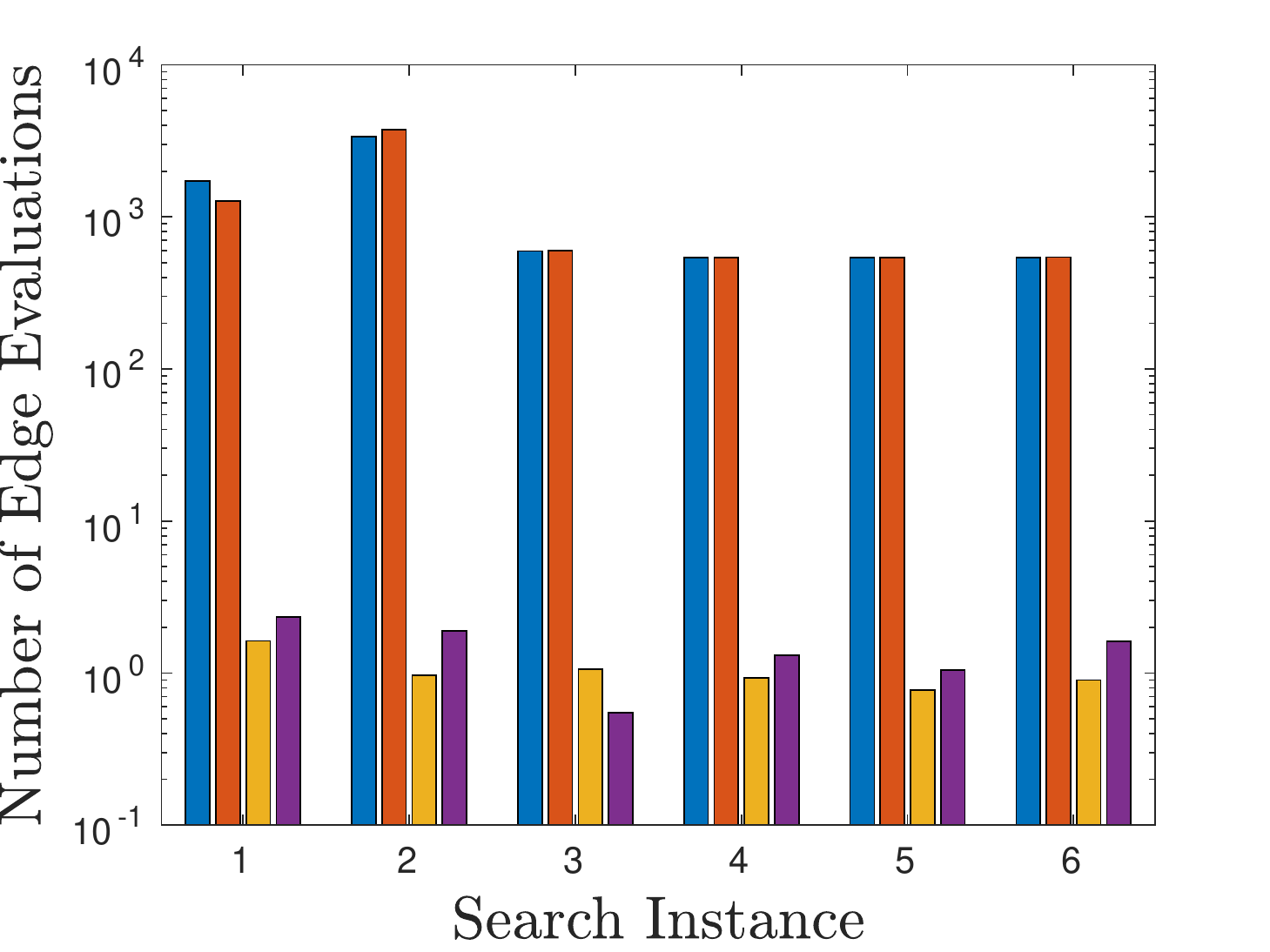}
		\caption{Edge evaluation}
	\end{subfigure}
	\begin{subfigure}{0.32\textwidth}
		\includegraphics[width=\myLineScale\linewidth]{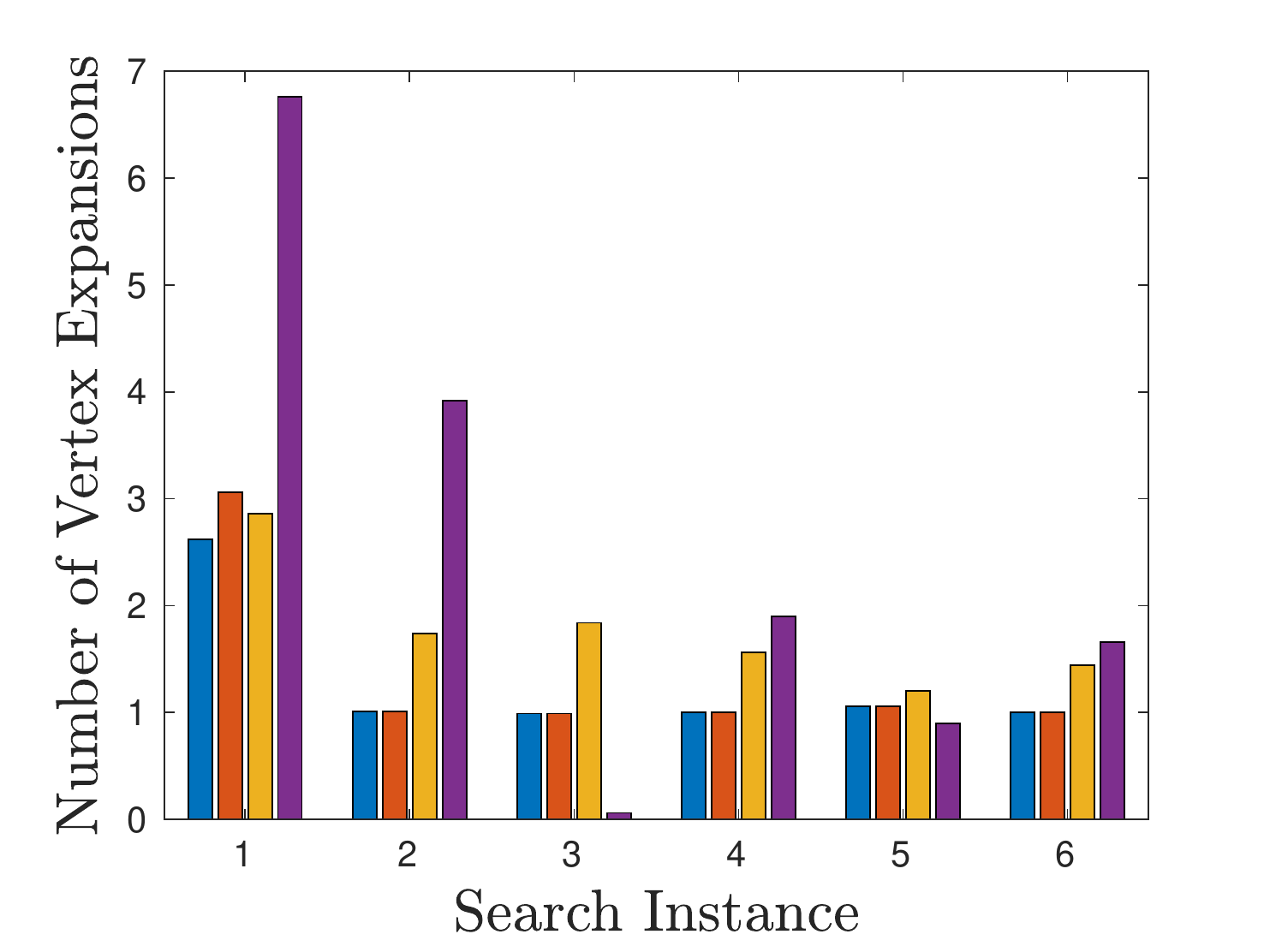}
		\caption{Vertex expansion}
	\end{subfigure}
	\begin{subfigure}{0.32\textwidth}
		\includegraphics[width=\myLineScale\linewidth]{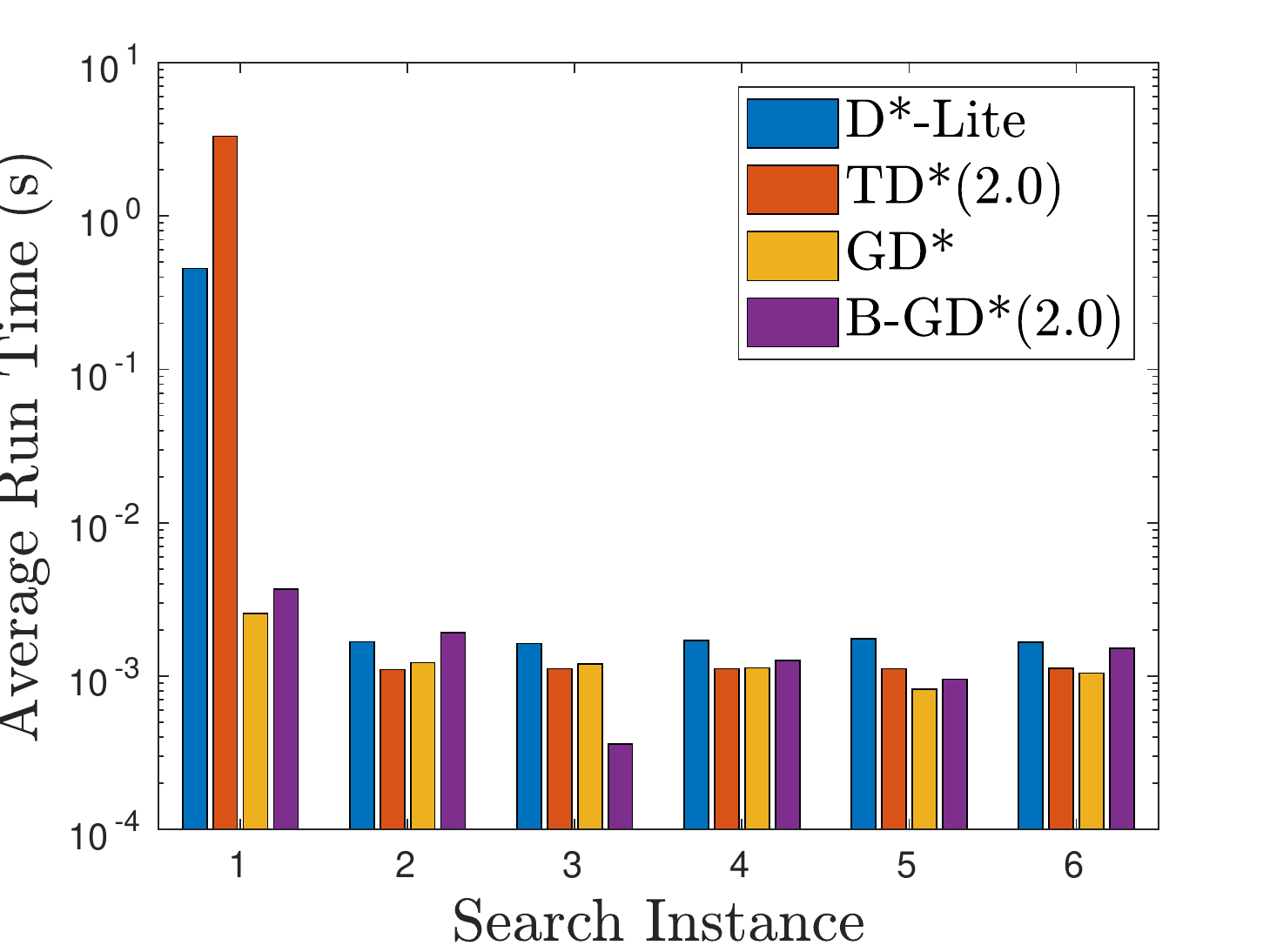}
		\caption{Total time}
	\end{subfigure}
	\caption{The average number of edge evaluations, the average number of vertex expansions, and the average total runtime taken over 100 random experiments for the RacecarJ environment.}
	\label{dstar:f:stat_racecar}
\end{figure*}

The comparison of the average performances results solved by the four different algorithms, namely, D*, GD* (infinite-step lookahead), TD* ($\eps_2=2$), and B-GD* (infinite-step lookahead, $\eps_1=\sqrt{2}, \eps_2=\sqrt{2}$) is shown in Figure ~\ref{dstar:f:stat_racecar}. As expected, the generalized algorithms, namely, GD* and B-GD* significantly reduce the number of edge evaluations in the graph in all the six consecutive searches with the trade-off of increasing the number of vertex expansions (conducting more search). Note that the runtime is also significantly reduced for the first search in this environment.

\subsubsection{Parametric Study.}

The choice of the lookahead value can have a significant effect on the performance of the algorithms. We demonstrate this for the dynamic algorithms for some sample experimental cases.
The comparison between 1-step lookahead and infinite-step lookahead for the PR2 experiments are shown in Figure~\ref{dstar:f:event_pr2}.  
As shown in  Figure~\ref{dstar:f:event_pr2}, an infinite-step lookahead reduces 
edge evaluations, while a 1-step lookahead minimizes vertex expansions within the lazy framework. 
For this particular problem, the infinite-step lookahead is about 5X faster than the $1$-step lookahead.
However, the optimal lookahead value is dependent on the problem.

\begin{figure*}[ht]
	\centering
	\begin{subfigure}{0.45\textwidth}
		\includegraphics[width=\myLineScale\linewidth]{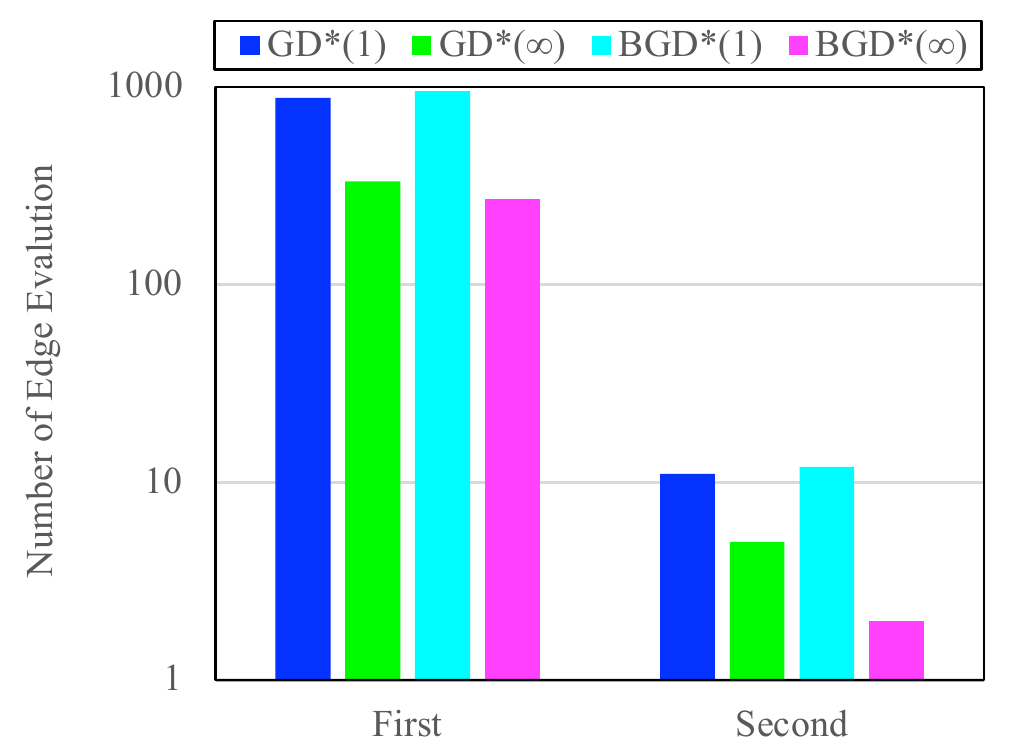}
		\caption{Edge evaluation}
	\end{subfigure}
	\begin{subfigure}{0.45\textwidth}
		\includegraphics[width=\myLineScale\linewidth]{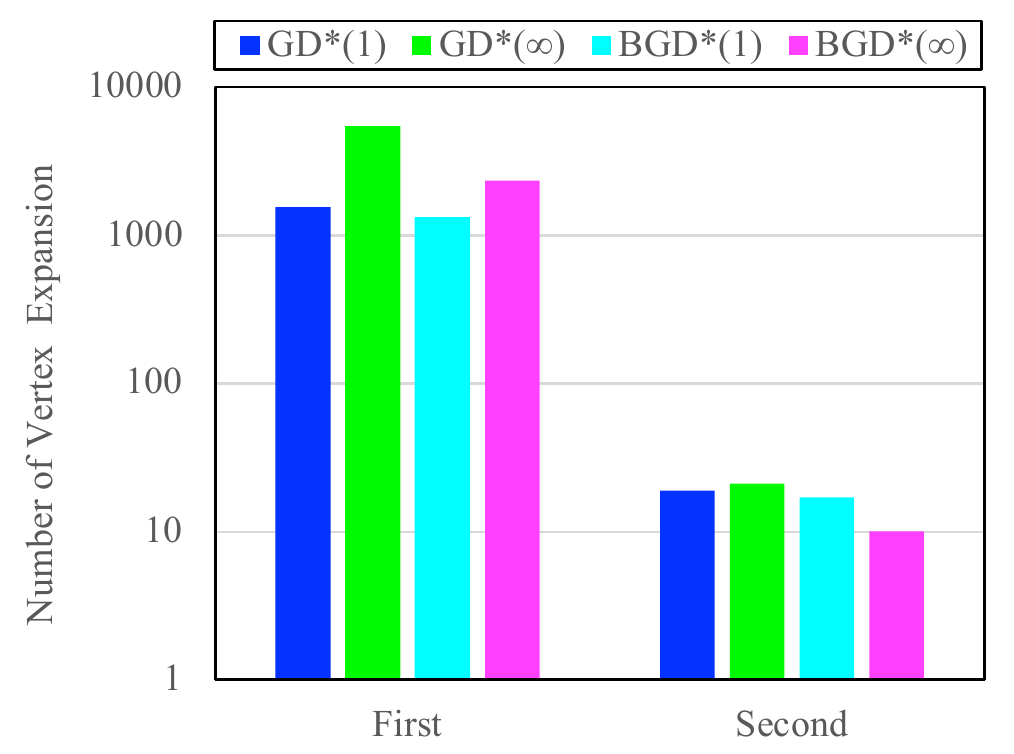}
		\caption{Vertex expansion}
	\end{subfigure}
	\caption{Performance comparison for the extreme lookahead values ($1-$step and $\infty$).}
	\label{dstar:f:event_pr2}
\end{figure*}

Figure~\ref{dstar:f:factor_pr2} demonstrates the effect of different values for the truncation and inflation factors in the PR2 simulation environment. 
As expected, increasing the bounds on the suboptimal solution generally leads to the reduction of both the number of vertex expansions and edge evaluations, resulting in reduced total solution time. 
\begin{figure*}[ht]
	\centering
	\begin{subfigure}{0.45\textwidth}
		\includegraphics[width=\myLineScale\linewidth]{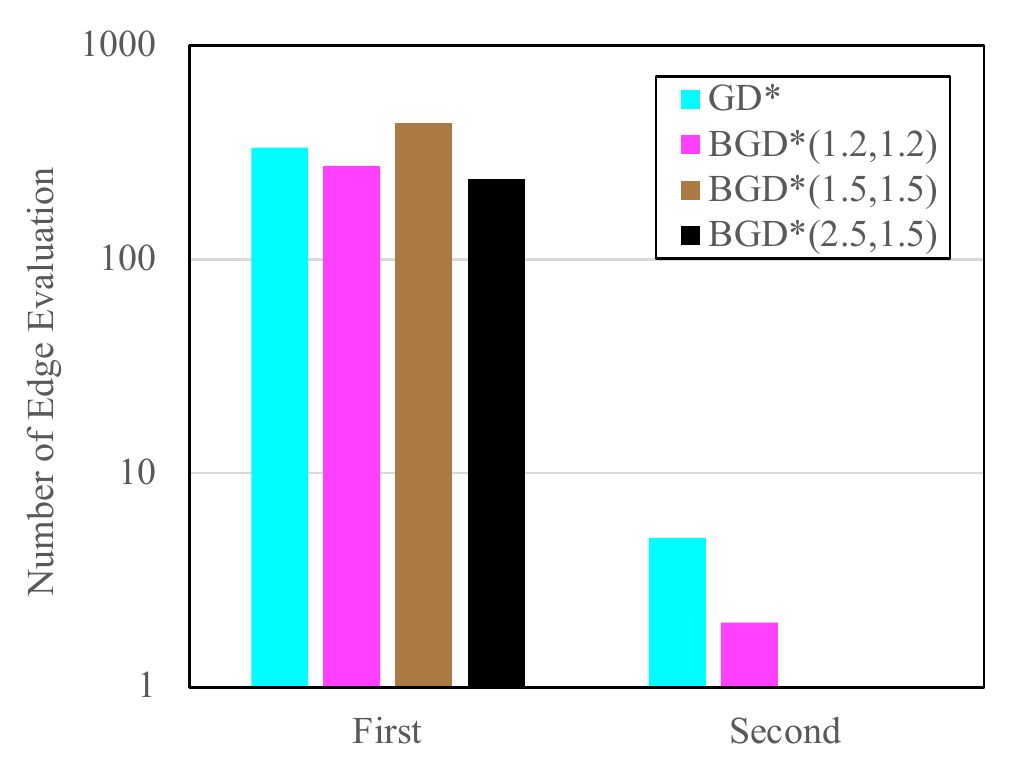}
		\caption{Edge evaluation}
	\end{subfigure}
	\begin{subfigure}{0.45\textwidth}
		\includegraphics[width=\myLineScale\linewidth]{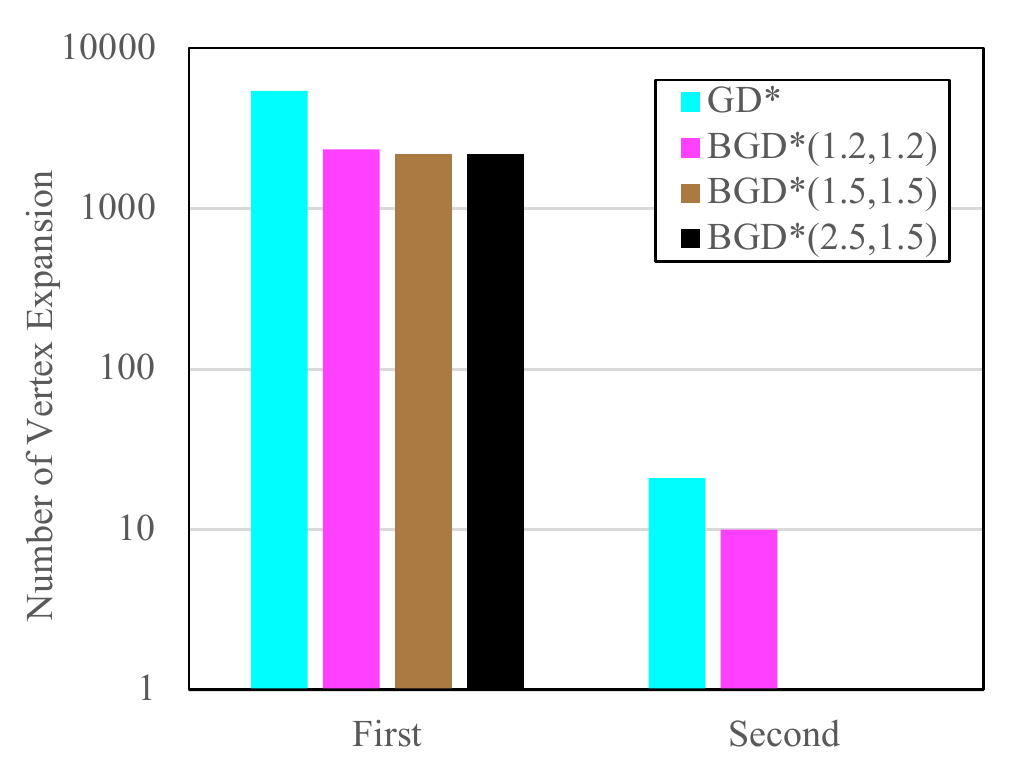}
		\caption{Vertex expansion}
	\end{subfigure}
	\caption{Performance comparison for the different truncation $(\eps_2)$ and inflation factors $(\eps_1)$.}
	\label{dstar:f:factor_pr2}
\end{figure*}

\section{Conclusion} \label{sec:conclusion}

We have presented a novel replanning framework that combines the vertex efficiency of incremental search methods with the edge efficiency of lazy search methods.
Within the proposed lazy incremental search framework, we have presented four different algorithms: L-GLS, B-LGLS, GD*, and B-GD*.
L-GLS and B-LGLS solve a sequence of planning problems with fixed start and goal vertices using the previous search results to efficiently restrict unnecessary edge evaluations.
GD* and B-GD* solve a sequence of planning problems with moving start vertices, generalizing D*-Lite and TD* within the lazy search framework.
We have proven that these algorithms are complete and correct in returning a bounded suboptimal solution given a graph change. 
In our numerical experiments, it was shown that our generalization significantly improves the performance of the classical incremental search algorithms via reducing unnecessary edge evaluations, speeding up replanning.
The proposed improvements enable classical algorithms to be applied in a broader range of applications where edge evaluation is expensive.



\begin{acks}
  This work has been supported by ARL under DCIST CRA W911NF-17-2-0181 and SARA CRA W911NF-20-2-0095, and by NSF under award IIS-2008686.
\end{acks}

\bibliographystyle{plainnat}
\bibliography{root}




\end{document}